\documentclass[preprint, 11pt, a4paper]{article} 

\usepackage[truedimen,margin=25truemm]{geometry}

\usepackage{times}
\usepackage{graphicx} 
\usepackage{epsfig} 
\usepackage{subfigure} 

\usepackage{natbib}

\usepackage{algorithm}
\usepackage{algorithmic}

\usepackage{url}

\usepackage{graphicx} 
\usepackage{epsfig}
\usepackage{subfigure}

\usepackage{amssymb}
\usepackage{amsmath}
\usepackage{amsthm}
\usepackage{mathtools}
\usepackage[raggedright]{titlesec} 
\newtheorem{prop}{Proposition}\newtheorem{lemma}{Lemma}

\usepackage{arydshln}
\def\mat#1{\mbox{\bf #1}}

\DeclareMathAlphabet\mathbfcal{OMS}{cmsy}{b}{d}

\newcommand{\changeBM}[1]{#1}
\newcommand{\changeHK}[1]{#1}
\newcommand{\changeHKK}[1]{#1}
\newcommand{\changeHKKK}[1]{#1}

\newcommand{\argmin}{\operatornamewithlimits{arg\,min}}
\makeatletter
\def\hlinewd#1{%
  \noalign{\ifnum0=`}\fi\hrule \@height #1 \futurelet
   \reserved@a\@xhline}
\makeatother

\title{Low-rank tensor completion:\\ a Riemannian manifold preconditioning approach}

\author{Hiroyuki Kasai\\\  
\small{The University of Electro-Communications} \\
\small{Tokyo, Japan} \\
\small{\tt kasai@is.uec.ac.jp} 
\and Bamdev Mishra\\
\small{Amazon Development Centre India} \\ 
\small{Bangalore, India} \\
\small{\tt bamdevm@amazon.com}}

\begin{document}

\maketitle

\begin{abstract}
We propose a novel \changeHK{Riemannian} \changeBM{manifold} preconditioning approach for the tensor completion problem with rank constraint. A novel Riemannian metric or inner product is proposed that exploits the least-squares structure of the cost function and takes into account the structured symmetry that exists in Tucker decomposition. The specific metric allows to use the versatile framework of Riemannian optimization on quotient manifolds to develop preconditioned nonlinear conjugate gradient \changeBM{and stochastic gradient descent algorithms for batch and online setups, respectively}. Concrete matrix representations of various optimization-related ingredients are listed. Numerical comparisons suggest that our proposed \changeBM{algorithms robustly outperform} state-of-the-art algorithms across different synthetic and real-world datasets\footnote{This paper extends the earlier work \cite{kasai_arXiv_2015} to include a stochastic gradient descent algorithm for low-rank tensor completion.}.
\end{abstract}

\section{Introduction}
\label{sec:Introduction}
This paper addresses the problem of low-rank tensor completion when the rank is a priori known or estimated. \changeBM{We focus on 3-order tensors in the paper, but the developments can be generalized to higher order tensors in a straightforward way.} Given a tensor $\mathbfcal{X}^{n_1 \times n_2 \times n_3}$, whose entries $\mathbfcal{X}_{i_1, i_2, i_3}^{\star}$ are only known for some indices $(i_1, i_2, i_3) \in \Omega$, where $\Omega$ is a subset of the complete set of indices 
$\{(i_1, i_2, i_3): i_d \in \{1, \ldots, n_d \}, d \in \{1,2,3\}\}$,
the \emph{fixed-rank tensor completion problem} is formulated as 
\begin{equation}
\begin{array}{lll}
\label{Eq:CostFunction}
\displaystyle{\min_{\mathbfcal{X} \in
\mathbb{R}^{n_1 \times n_2 \times n_3}} }&   
\displaystyle{\frac{1}{|\Omega |}
\| \mathbfcal{P}_{\Omega}(\mathbfcal{X}) - 
\mathbfcal{P}_{\Omega}(\mathbfcal{X}^{\star}) \|^2_F} \\
{\rm subject\ to}& {\rm rank}(\mathbfcal{X}) = {\bf r},
\end{array}
\end{equation}
where the operator $\mathbfcal{P}_{\Omega}(\mathbfcal{X})_{i_1\changeHK{,}i_2\changeHK{,}i_3} = \mathbfcal{X}_{i_1\changeHK{,}i_2\changeHK{,}i_3}$ if $(i_1, i_2, i_3) \in \Omega$ and $\mathbfcal{P}_{\Omega}(\mathbfcal{X})_{i_1\changeHK{,}i_2\changeHK{,}i_3}  = 0$ otherwise and (with a slight abuse of notation) $\|\cdot \|_F$ is the Frobenius norm. \changeHKK{$|\Omega|$ is the number of known entries.} ${\rm rank}(\mathbfcal{X})$ $ (={\bf r}=(r_1, r_2, r_3))$, called the \emph{multilinear rank} of $\mathbfcal{X}$, is the set of the ranks of for each of mode-$d$ unfolding matrices. $r_d \ll n_d$ enforces a low-rank structure. The {\it mode} is a matrix obtained by concatenating the mode-$d$ fibers along columns, and mode-$d$ {\it unfolding} of \changeBM{a $D$-order tensor} $\mathbfcal{X}$ is $\mat{X}_{d} \in \mathbb{R}^{n_d \times n_{d+1}\cdots n_D n_1 \cdots n_{d-1}}$ for $d=\{1,\ldots,D\}$.

Problem (\ref{Eq:CostFunction}) has many variants, and one of those is extending the nuclear norm regularization approach from the matrix case \cite{Candes_FoundCompuMath_2009_s} to the tensor case. This results in a summation of nuclear norm regularization terms, each one corresponds to each of the unfolding matrices of $\mathbfcal{X}$. While this generalization leads to good results \cite{Liu_IEEETransPAMI_2013_s, Tomioka_Latent_2011_s, Signoretto_MachineLearning_2014_s}, its applicability to large-scale instances is not trivial, especially due to the necessity of high-dimensional singular value decomposition computations. A different approach exploits \emph{Tucker decomposition} \citep[Section~4]{Kolda_SIAMReview_2009_s} of a low-rank tensor $\mathbfcal{X}$ to develop large-scale algorithms for (\ref{Eq:CostFunction}), e.g., in \cite{Filipovi_MultiSysSigPro_2013_s,Kressner_BIT_2014_s}.

The present paper exploits both the \emph{symmetry} present in Tucker decomposition and the \emph{least-squares} structure of the cost function of (\ref{Eq:CostFunction}) to develop competitive \changeBM{algorithms}. \changeBM{The multilinear rank constraint forms a smooth manifold \citep{Kressner_BIT_2014_s}.} To this end, we use the concept of \emph{manifold preconditioning}. While preconditioning in unconstrained optimization is well studied \citep[Chapter~5]{Nocedal_NumericalOpt_2006_s}, preconditioning on constraints with \emph{symmetries}, owing to non-uniqueness of Tucker decomposition \cite{Kolda_SIAMReview_2009_s}, is not straightforward. We build upon the recent work \cite{Bamdev_arXiv_2014_s} that suggests to use \emph{preconditioning} with a \emph{tailored metric} (inner product) in the Riemannian optimization framework on quotient manifolds \cite{Absil_OptAlgMatManifold_2008, Edelman98a, Bamdev_arXiv_2014_s}. \changeBM{The differences with respect to the work of \citet{Kressner_BIT_2014_s}, which also exploits the manifold structure, are twofold. (i) \citet{Kressner_BIT_2014_s} exploit the search space as an \emph{embedded submanifold} of the Euclidean space, whereas we view it as a product of simpler search spaces with symmetries. Consequently, certain computations have straightforward interpretation. (ii) \citet{Kressner_BIT_2014_s} work with the standard Euclidean metric, whereas we use a metric that is tuned to the least-squares cost function, thereby inducing a preconditioning effect. This novel idea of using a tuned metric leads to a superior performance of our algorithms. \changeBM{They} also connect to state-of-the-art algorithms proposed in \cite{Ngo_NIPS_2012_s, Wen_MPC_2012_s, Mishra_ICDC_2014_s, boumal15a}.}

The paper is organized as follows. Section \ref{sec:Fixed-rankTuckerFactorization} discusses the two fundamental structures of symmetry and least-squares associated with (\ref{Eq:CostFunction}) and proposes a novel metric that captures the relevant second order information of the problem. The optimization-related ingredients on the Tucker manifold are developed in Section \ref{sec:OptimizationRelatedIngredients}. The cost function specific ingredients are developed in Section \ref{sec:AlgorithmDetails}. The final formulas are listed in Table \ref{tab:FinalFormulas}, \changeBM{which allow to develop preconditioned conjugate gradient descent algorithm in the batch setup and stochastic gradient descent algorithm in the online setup}. In Section \ref{sec:NumericalComparisons}, numerical comparisons with state-of-the-art algorithms on various synthetic and real-world benchmarks suggest a superior performance of our proposed algorithms. Our proposed algorithms are implemented in the Matlab toolbox Manopt \cite{Boumal_Manopt_2014_s}.
%
%
%
The concrete \changeBM{proofs  of propositions, development of} optimization-related ingredients\changeHK{,} and additional numerical experiments are shown in Sections {\bf A} and {\bf B}, respectively, of the supplementary material file. 
The Matlab codes for first and second order implementations, e.g., gradient descent and trust-region methods, are available at \url{https://bamdevmishra.com/codes/tensorcompletion/}.



\section{Exploiting the problem structure}
\label{sec:Fixed-rankTuckerFactorization}
Construction of efficient algorithms depends on properly exploiting the problem structure. To this end, we focus on two fundamental structures in (\ref{Eq:CostFunction}): \emph{symmetry} in the constraints and the \emph{least-squares structure} of the cost function. Finally, a novel metric is proposed.

{\bf The symmetry structure in Tucker decomposition.} The Tucker decomposition of a tensor $\mathbfcal{X} \in \mathbb{R}^{n_1 \times n_2 \times n_3}$ of rank {\bf r} (=$(r_1, r_2, r_3)$) is 
\begin{equation}
\begin{array}{lll}
\label{Eq:TuckerFactorization}
\mathbfcal{X} &  = & 
\mathbfcal{G} {\times_1} \mat{U}_{1} {\times_2} \mat{U}_{2} {\times_3} \mat{U}_{3}, 
\end{array}
\end{equation}
where $\mat{U}_{d} \in {\rm St}(r_d, n_d)$ for $d \in \{1,2,3\}$ belongs to the \emph{Stiefel manifold} of matrices of size $n_d \times r_d$ with orthogonal columns and $\mathbfcal{G} \in \mathbb{R}^{r_1 \times r_2 \times r_3}$ \cite{Kolda_SIAMReview_2009_s}. Here, $\mathbfcal{W} \times_d \mat{V} \in \mathbb{R}^{n_1 \times \cdots n_{d-1} \times m \times n_{d+1}  \times \cdots n_D}$ computes the {\it d-mode product} of a tensor $\mathbfcal{W} \in \mathbb{R}^{n_1 \times \cdots \times n_D}$ and a matrix $\mat{V} \in \mathbb{R}^{m \times n_d}$. 
Tucker decomposition (\ref{Eq:TuckerFactorization}) is \emph{not
unique} as $\mathbfcal{X}$ remains unchanged under the transformation 
\begin{equation}
\begin{array}{lll}
\label{Eq:EquivalenceClass_3}
(\mat{U}_{1}, \mat{U}_{2}, \mat{U}_{3}, \mathbfcal{G}) &\mapsto &
(  \mat{U}_{1}\mat{O}_{1}, \mat{U}_{2}\mat{O}_{2}, \mat{U}_{3}\mat{O}_{3}, \mathbfcal{G} {\times_1} \mat{O}^T_{1} {\times_2} \mat{O}^T_{2} {\times_3} \mat{O}^T_{3}) 
\end{array}
\end{equation}
for all $\mat{O}_{d} \in \mathcal{O}(r_d)$, \changeBM{which is} the set of orthogonal matrices of size of $r_d \times r_d$. The classical remedy to remove this indeterminacy is to have additional structures on $\mathbfcal{G}$ like sparsity or restricted orthogonal rotations \citep[Section~4.3]{Kolda_SIAMReview_2009_s}. In contrast, we encode the transformation (\ref{Eq:EquivalenceClass_3}) in an abstract search space of \emph{equivalence classes}, defined as,
\begin{equation}
\label{Eq:EquivalenceClass}
\begin{array}{lll}
[(\mat{U}_{1}, \mat{U}_{2}, \mat{U}_{3}, \mathbfcal{G}) ] &: = &
 \{(  \mat{U}_{1}\mat{O}_{1}, \mat{U}_{2}\mat{O}_{2}, \mat{U}_{3}\mat{O}_{3}, 
 \mathbfcal{G} {\times_1} \mat{O}^T_{1} {\times_2} \mat{O}^T_{2} {\times_3} \mat{O}^T_{3}) : \mat{O}_{d} \in \mathcal{O}(r_d)\}.
\end{array}
\end{equation}

The set of equivalence classes is the \emph{quotient manifold} 
\cite{Lee03a}
\begin{equation}
\begin{array}{lll}
\label{Eq:QuotientSpace}
\mathcal{M}/\!\sim
&:=& \mathcal{M}/ 
(\mathcal{O}{(r_1)} \times \mathcal{O}{(r_2)} \times \mathcal{O}{(r_3)}),
\end{array}
\end{equation}
where $\mathcal{M}$ is called the \emph{total space} (computational space) that is the product space
\begin{equation}
\begin{array}{lll}
\label{Eq:TotalSpace}
\mathcal{M}
&:=& {\rm St}(r_1, n_1) \times {\rm St}(r_2, n_2) \times {\rm St}(r_3, n_3) \times \mathbb{R}^{r_1 \times r_2 \times r_3}.
\end{array}
\end{equation}

Due to the invariance (\ref{Eq:EquivalenceClass_3}), the local minima of  (\ref{Eq:CostFunction}) in $\mathcal{M}$ are not isolated, but they become isolated on $\mathcal{M}/ \!\sim$. Consequently, the problem (\ref{Eq:CostFunction}) is an optimization problem on a quotient manifold for which systematic procedures are proposed in \cite{Absil_OptAlgMatManifold_2008, Edelman98a}. \changeBM{A requirement is to endow} endow $\mathcal{M}/ \!\sim$ with a Riemannian structure, \changeBM{which conceptually translates (\ref{Eq:CostFunction}) into an unconstrained optimization problem over the search space $\mathcal{M}/ \!\sim$}. We call $\mathcal{M}/ \!\sim$, defined in (\ref{Eq:QuotientSpace}), the \emph{Tucker manifold} as it results from Tucker decomposition.

{\bf The least-squares structure of the cost function.}
In unconstrained optimization, the Newton method is interpreted as a \emph{scaled} steepest descent method, where the search space is endowed with a metric (inner product) induced by the Hessian of the cost function \cite{Nocedal_NumericalOpt_2006_s}. This induced metric (or its approximation) resolves convergence issues of first order optimization algorithms. Analogously, finding a good inner product for (\ref{Eq:CostFunction}) is of profound consequence. Specifically for the case of quadratic optimization with rank constraint (matrix case), Mishra and Sepulchre 
\cite{Bamdev_arXiv_2014_s} 
propose a family of Riemannian metrics from the Hessian of the cost function. Applying this approach directly for the particular cost function of (\ref{Eq:CostFunction}) is computationally costly. To circumvent the issue, we consider a simplified cost function by assuming that $\Omega$ contains the full set of indices, i.e., we focus on $\| \mathbfcal{X} - \mathbfcal{X}^{\star}\|_F^2$ to propose a metric candidate. Applying the metric tuning approach of 
\cite{Bamdev_arXiv_2014_s} 
to the simplified cost function leads to a family of Riemannian metrics. A good trade-off between computational cost and simplicity is by considering only the \emph{block diagonal} elements of the Hessian of $\| \mathbfcal{X} - \mathbfcal{X}^{\star}\|_F^2$. It should be noted that the cost function $\| \mathbfcal{X} - \mathbfcal{X}^{\star}\|_F^2$  is {\it convex and quadratic} in $\mathbfcal{X}$. Consequently, it is also convex and quadratic in the arguments 
$(\mat{U}_{1}, \mat{U}_{2}, \mat{U}_{3}, \mathbfcal{G})$ individually. Equivalently, the block diagonal approximation of the Hessian of $\| \mathbfcal{X} - \mathbfcal{X}^{\star}\|_F^2$ in $(\mat{U}_{1}, \mat{U}_{2}, \mat{U}_{3}, \mathbfcal{G})$ is  
\begin{equation}
\begin{array}{lll}
\label{Eq:BlockApproximation}
((\mat{G}_{1} \mat{G}_{1}^T) \otimes \mat{I}_{n_1}, (\mat{G}_{2} \mat{G}_{2}^T) \otimes \mat{I}_{n_2}, (\mat{G}_{3} \mat{G}_{3}^T) \otimes \mat{I}_{n_3}, 
\mat{I}_{r_1 r_2  r_3}),
\end{array}
\end{equation}
where $\mat{G}_{d}$ is the mode-$d$ unfolding of $\mathbfcal{G}$ and is assumed to be full rank. \changeHK{$\otimes$ is the Kronecker product.} The terms $\mat{G}_{d} \mat{G}_{d}^T$ for $d  \in \{1, 2,3\}$ are \emph{positive definite} when $r_1 \leq r_2 r_3$, $r_2 \leq r_1 r_3$, and $r_3 \leq r_1 r_2$, which is a reasonable assumption.

{\bf A novel Riemannian metric.} An element $x$ in the total space $\mathcal{M}$ has the matrix representation $(\mat{U}_{1}, \mat{U}_{2}, \mat{U}_{3}, \mathbfcal{G})$. Consequently, the tangent space $T_{{x}} \mathcal{M}$ is the Cartesian product of the tangent spaces of the individual manifolds of (\ref{Eq:TotalSpace}), i.e., $T_{{x}} \mathcal{M}$ has the matrix characterization \cite{Edelman98a}
\begin{equation}
\begin{array}{lll}
\label{Eq:tangent_space}
T_{{x}} {\mathcal{M}} 
& = &  \{ (\mat{Z}_{{\bf U}_{1}}, \mat{Z}_{{\bf U}_{2}}, \mat{Z}_{{\bf U}_{3}}, \mat{Z}_{\mathbfcal{G}})
\in 
\mathbb{R}^{n_1 \times r_1} \times  
\mathbb{R}^{n_2 \times r_2} \times 
\mathbb{R}^{n_3 \times r_3} \times 
\mathbb{R}^{r_1 \times r_2 \times r_3} : \\
&  & \mat{U}_{d}^T \mat{Z}_{{\bf U}_{d}} +  \mat{Z}_{{\bf U}_{d}}^T \mat{U}_{d} = 0,\ {\rm for\ } d \in \{1,2,3 \} \}.
\end{array}
\end{equation}

From the earlier discussion on symmetry and least-squares structure, we propose the novel metric \changeBM{or inner product} ${g}_{x}:T_x \mathcal{M} \times T_x \mathcal{M} \rightarrow \mathbb{R}$
\begin{equation}
\begin{array}{lll}
\label{Eq:metric}
{g}_{x}(\xi_{x}, \eta_{x}) 
 &=&  
\langle \xi_{\scriptsize \mat{U}_{1}},
{\eta}_{\scriptsize\mat{U}_{1}} (\mat{G}_{1} \mat{G}_{1}^T) \rangle +
\langle \xi_{\scriptsize \mat{U}_{2}},
{\eta}_{\scriptsize\mat{U}_{2}} (\mat{G}_{2} \mat{G}_{2}^T) \rangle +
\langle \xi_{\scriptsize \mat{U}_{3}},
{\eta}_{\scriptsize\mat{U}_{3}} (\mat{G}_{3} \mat{G}_{3}^T) \rangle 
+ \langle {\xi}_{\scriptsize \mathbfcal{G}}, {\eta}_{\scriptsize \mathbfcal{G}}\rangle ,
\end{array}
\end{equation}
where ${\xi}_{{x}}, {\eta}_{{x}} \in T_{{x}} {\mathcal{M}}$ are 
tangent vectors with matrix characterizations, shown in (\ref{Eq:tangent_space}), $({\xi}_{\scriptsize \mat{U}_{1}}, {\xi}_{\scriptsize \mat{U}_{2}}, {\xi}_{\scriptsize \mat{U}_{3}}, {\xi}_{\scriptsize \mathbfcal{G}})$ and $({\eta}_{\scriptsize \mat{U}_{1}}, {\eta}_{\scriptsize \mat{U}_{2}}, {\eta}_{\scriptsize \mat{U}_{3}}, {\eta}_{\scriptsize \mathbfcal{G}})$, respectively and $\langle \cdot, \cdot \rangle$ is the Euclidean inner product. It should be emphasized that the proposed metric (\ref{Eq:metric}) is induced from (\ref{Eq:BlockApproximation}). 
\changeHK{
\begin{prop}\label{prop:invariance_metric}
Let $(\xi_{\scriptsize \mat{U}_{1}}, \xi_{\scriptsize \mat{U}_{2}}, \xi_{\scriptsize \mat{U}_{3}}, \xi_{\scriptsize \mathbfcal{G}})$ and $(\eta_{\scriptsize \mat{U}_{1}}, \eta_{\scriptsize \mat{U}_{2}}, \eta_{\scriptsize \mat{U}_{3}}, \eta_{\scriptsize \mathbfcal{G}})$ be tangent vectors to the quotient manifold (\ref{Eq:QuotientSpace}) 
at $(\mat{U}_{1}, \mat{U}_{2}, \mat{U}_{3}, \mathbfcal{G})$, and $(\xi_{\scriptsize \mat{U}_{1}\mat{O}_{1}}, \xi_{\scriptsize \mat{U}_{2}\mat{O}_{2}}, \xi_{\scriptsize \mat{U}_{3}\mat{O}_{3}}, 
\xi_{\scriptsize \mathbfcal{G} {\times_1 \mat{O}^T_{1} {\times_2} \mat{O}^T_{2} {\times_3} \mat{O}^T_{3}}})$ 
and $(\eta_{\scriptsize \mat{U}_{1}\mat{O}_{1}}, \eta_{\scriptsize \mat{U}_{2}\mat{O}_{2}}, \eta_{\scriptsize \mat{U}_{3}\mat{O}_{3}}, \eta_{\scriptsize \mathbfcal{G} {\times_1} \mat{O}^T_{1} {\times_2} \mat{O}^T_{2} {\times_3} \mat{O}^T_{3}})$ be tangent vectors to the quotient manifold (\ref{Eq:QuotientSpace}) at $(\mat{U}_{1}\mat{O}_{1}, \mat{U}_{2}\mat{O}_{2}, \mat{U}_{3}\mat{O}_{3}, \mathbfcal{G} {\times_1} \mat{O}^T_{1} {\times_2} \mat{O}^T_{2} {\times_3} \mat{O}^T_{3})$. The metric (\ref{Eq:metric}) is invariant along the equivalence class (\ref{Eq:EquivalenceClass}), i.e.,
\begin{equation*}
\label{Eq:InvariantMetric}
\begin{array}{l}
\hspace*{-0.5cm}g_{\scriptsize (\mat{U}_{1}, \mat{U}_{2}, \mat{U}_{3}, \mathbfcal{G})}((\xi_{\scriptsize \mat{U}_{1}}, \xi_{\scriptsize \mat{U}_{2}}, \xi_{\scriptsize \mat{U}_{3}}, \xi_{\scriptsize \mathbfcal{G}}), (\eta_{\scriptsize \mat{U}_{1}}, \eta_{\scriptsize \mat{U}_{2}}, \eta_{\scriptsize \mat{U}_{3}}, \eta_{\scriptsize \mathbfcal{G}})) \\
= g_{\scriptsize (\mat{U}_{1}\mat{O}_{1}, \mat{U}_{2}\mat{O}_{2}, \mat{U}_{3}\mat{O}_{3}, \mathbfcal{G} {\times_1} \mat{O}^T_{1} {\times_2} \mat{O}^T_{2} {\times_3} \mat{O}^T_{3})}
((\xi_{\scriptsize \mat{U}_{1}\mat{O}_{1}}, \xi_{\scriptsize \mat{U}_{2}\mat{O}_{2}}, \xi_{\scriptsize \mat{U}_{3}\mat{O}_{3}}, 
\xi_{\scriptsize \mathbfcal{G} {\times_1} \mat{O}^T_{1} {\times_2} \mat{O}^T_{2} {\times_3} \mat{O}^T_{3}}), \\
\hspace*{6cm}(\eta_{\scriptsize \mat{U}_{1}\mat{O}_{1}}, \eta_{\scriptsize \mat{U}_{2}\mat{O}_{2}}, \eta_{\scriptsize \mat{U}_{3}\mat{O}_{3}}, \eta_{\scriptsize \mathbfcal{G} {\times_1} \mat{O}^T_{1} {\times_2} \mat{O}^T_{2} {\times_3} \mat{O}^T_{3}})).
\end{array}
\end{equation*} 
\end{prop}
}

%
%
%
\section{Notions of manifold optimization}
\label{sec:OptimizationRelatedIngredients}
%
%
%



\begin{figure}[h]
\center
\includegraphics[scale = 0.5]{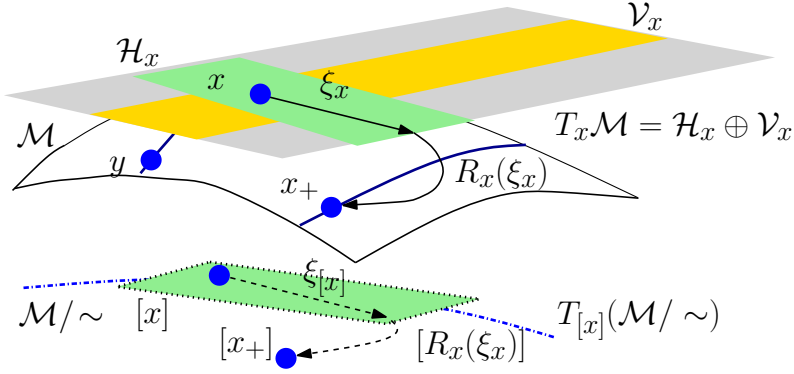}
\caption{Riemannian optimization framework: geometric objects, shown in dotted lines, on quotient manifold $\mathcal{M}/\!\sim$ call for matrix representatives, shown in solid lines, in the total space $\mathcal{M}$.}
\label{fig:OptimizationOnSymmetries}
\end{figure}

Each point on a quotient manifold represents an entire equivalence class of matrices in the total space. Abstract geometric objects on the quotient manifold $\mathcal{M}/ \!\sim$ call for matrix representatives in the total space $\mathcal{M}$. Similarly, algorithms are run in the total space $\mathcal{M}$, but under appropriate compatibility between the Riemannian structure of $\mathcal{M}$ and the Riemannian structure of the quotient manifold $\mathcal{M}/ \!\sim$, they define algorithms on the quotient manifold. The key is endowing $\mathcal{M}/ \!\sim$ with a Riemannian structure. Once this is the case, a constraint optimization problem, for example (\ref{Eq:CostFunction}), is conceptually transformed into an unconstrained optimization over the Riemannian quotient manifold (\ref{Eq:QuotientSpace}). Below we briefly show the development of various geometric objects that are required to optimize a smooth cost function on the quotient manifold (\ref{Eq:QuotientSpace}) with first order methods, e.g., conjugate gradients.





{\bf Quotient manifold representation and horizontal lifts.} Figure \ref{fig:OptimizationOnSymmetries} illustrates a schematic view of optimization with equivalence classes, where the points $x$ and $y$ in $\mathcal{M}$ belong to the same equivalence class (shown in solid blue color) and they represent a single point $[x]:=\{ y \in \mathcal{M} : y \sim x\}$ on the quotient manifold $\mathcal{M}/\!\sim$. The abstract tangent space $T_{[x]}(\mathcal{M}/\!\sim)$ at $[x] \in \mathcal{M}/\!\sim$ has the matrix representation in $T_x \mathcal{M}$, but restricted to the directions that do not induce a displacement along the equivalence class $[x]$. This is realized by decomposing $T_x \mathcal{M}$ into two complementary subspaces, the vertical and horizontal subspaces. The vertical space $\mathcal{V}_x$ is the tangent space of the equivalence class $[x]$. On the other hand, the horizontal space $\mathcal{H}_x$ is the \emph{orthogonal subspace} to $\mathcal{V}_x$ in the sense of the metric (\ref{Eq:metric}). Equivalently, $T_x \mathcal{M} =\mathcal{V}_x \oplus \mathcal{H}_x $. The horizontal subspace $\mathcal{H}_x $ provides a valid matrix representation to the abstract tangent space $T_{[x]}(\mathcal{M}/\!\sim)$. An abstract tangent vector $\xi_{[x]} \in T_{[x]}(\mathcal{M}/\!\sim)$ at $[x]$ has a unique element $\xi_x \in \mathcal{H}_x$ that is called its \emph{horizontal lift}.

A Riemannian metric $g_x:T_x \mathcal{M} \times  T_x \mathcal{M} \rightarrow \mathbb{R}$ at $x \in \mathcal{M}$ defines a Riemannian metric $g_{[x]}:T_{[x]}(\mathcal{M}/\! \sim) \times T_{[x]}(\mathcal{M}/\!\sim) \rightarrow \mathbb{R}$, i.e., $g_{[x]}(\xi_{[x]},\eta_{[x]}) := g_{x}(\xi_{x},\eta_{x})$ on the quotient manifold $\mathcal{M}/ \!\sim$, if $g_{x}(\xi_{x},\eta_{x})$ does not depend on a specific representation along the equivalence class $[x]$. Here, $\xi_{[x]}$ and $\eta_{[x]}$ are tangent vectors in $T_{[x]}(\mathcal{M}/\!\sim)$, and $\xi_{x}$ and $\eta_{x}$ are their horizontal lifts in $\mathcal{H}_x$ at $x$, respectively. Equivalently, the definition of the Riemannian metric is well posed when $g_x(\xi_x, \zeta_x)=g_x(\xi_y, \zeta_y)$ for all $x, y \in [x]$, where $\xi_x, \zeta_x \in \mathcal{H}_x$ and $\xi_y, \zeta_y \in \mathcal{H}_y$ are the horizontal lifts of $\xi_{[x]}, \zeta_{[x]} \in T_{[x]}(\mathcal{M}/\!\sim)$ along the same equivalence class $[x]$. \changeBM{This holds true for the proposed metric (\ref{Eq:metric}) as shown in Proposition \ref{prop:invariance_metric}.} From \cite{Absil_OptAlgMatManifold_2008}, endowed with the Riemannian metric (\ref{Eq:metric}), the quotient manifold $\mathcal{M}/ \!\sim$ is a {\it Riemannian submersion} of $\mathcal{M}$. The submersion principle allows to work out concrete matrix representations of abstract object on $\mathcal{M}/ \!\sim$, e.g., the gradient of a smooth cost function 
\cite{Absil_OptAlgMatManifold_2008}.

Starting from an arbitrary matrix (with appropriate dimensions), two linear projections are needed: the first projection $\Psi_{x}$ is onto the tangent space $T_x\mathcal{M}$, while the second projection $\Pi_{{x}}$ is onto the horizontal subspace $\mathcal{H}_x$. The computation cost of these is $O(n_1 r_1^2 + n_2 r_2^2 + n_3r_3  ^2)$.

The tangent space $T_{x} \mathcal{M}$ projection is obtained by extracting the component normal to $T_{x} \mathcal{M}$ in the ambient space. The normal space $N_{x} \mathcal{M}$ has the matrix characterization $\{(\mat{U}_{1}\mat{S}_{\scriptsize \mat{U}_{1}}(\mat{G}_{1}  \mat{G}_{1}^T)^{-1},
\mat{U}_{2}\mat{S}_{\scriptsize \mat{U}_{2}}(\mat{G}_{2}  \mat{G}_{2}^T)^{-1},\\
\mat{U}_{3}\mat{S}_{\scriptsize \mat{U}_{3}}(\mat{G}_{3}  \mat{G}_{3}^T)^{-1}, 0) :\mat{S}_{\scriptsize \mat{U}_{d}}  \in \mathbb{R}^{r_d \times r_d},  \mat{S}_{\scriptsize \mat{U}_{d}}^T  = \mat{S}_{\scriptsize \mat{U}_{d}}, \text{ for\ } d \in \{1, 2, 3\} \}$. Symmetric matrices ${\mat S}_{\scriptsize \mat{U}_d}$ for all $d \in \{1,2,3 \}$ parameterize the normal space. Finally, the operator $\Psi_{{x}}: 
\mathbb{R}^{n_1 \times r_1} \times  
\mathbb{R}^{n_2 \times r_2} \times 
\mathbb{R}^{n_3 \times r_3} \times 
\mathbb{R}^{r_1 \times r_2 \times r_3} \rightarrow T_{{x}} {\mathcal{M}} :(\mat{Y}_{\scriptsize \mat{U}_{1}}, \mat{Y}_{\scriptsize \mat{U}_{2}}, \mat{Y}_{\scriptsize \mat{U}_{3}}, \mat{Y}_{\scriptsize \mathbfcal{G}} )$
$\mapsto \Psi_{{x}}(\mat{Y}_{\scriptsize \mat{U}_{1}}, \mat{Y}_{\scriptsize \mat{U}_{2}}, \mat{Y}_{\scriptsize \mat{U}_{3}}, \mat{Y}_{\scriptsize \mathbfcal{G}} )$ \changeHK{is given as follows.}
\begin{prop}\label{prop:tangent_space_projector}
\changeHK{The quotient manifold (\ref{Eq:QuotientSpace}) endowed with the metric (\ref{Eq:metric}) admits \changeBM{the} tangent space projector defined as}
\begin{equation}
\label{Eq:B_Requirements}
\begin{array}{lll}
\Psi_{{x}}(\mat{Y}_{\scriptsize \mat{U}_{1}}, \mat{Y}_{\scriptsize \mat{U}_{2}}, \mat{Y}_{\scriptsize \mat{U}_{3}}, \mat{Y}_{\scriptsize \mathbfcal{G}} )
&=& (\mat{Y}_{\scriptsize \mat{U}_{1}} - \mat{U}_{1} \mat{S}_{\scriptsize \mat{U}_{1}} (\mat{G}_{1}  \mat{G}_{1}^T)^{-1},
\mat{Y}_{\scriptsize \mat{U}_{2}} - \mat{U}_{2} \mat{S}_{\scriptsize \mat{U}_{2}} (\mat{G}_{2}  \mat{G}_{2}^T)^{-1}, \\
&&  \mat{Y}_{\scriptsize \mat{U}_{3}} - \mat{U}_{3}\mat{S}_{\scriptsize \mat{U}_{3}} (\mat{G}_{3}  \mat{G}_{3}^T)^{-1},
\mat{Y}_{\scriptsize \mathbfcal{G}}),
\end{array}
\end{equation}
where $\mat{S}_{\scriptsize \mat{U}_{d}}$ is the solution to the Lyapunov equation $\mat{S}_{\scriptsize \mat{U}_{d}} \mat{G}_{d}  \mat{G}_{d}^T + \mat{G}_{d}  \mat{G}_{d}^T \mat{S}_{\scriptsize \mat{U}_{d}}   
 = \mat{G}_{d}  \mat{G}_{d}^T (\mat{Y}_{\scriptsize \mat{U}_{d}}^T \mat{U}_{d} + \mat{U}_{d}^T \mat{Y}_{\scriptsize \mat{U}_{d}}) \mat{G}_{d}  \mat{G}_{d}^T$ for $ d\in \{ 1,2,3\}$.
\end{prop}
 
\changeBM{The Lyapunov equations in Proposition \ref{prop:tangent_space_projector} are solved efficiently with} the Matlab's \verb+lyap+ routine.

The horizontal space projection of a tangent vector is obtained by removing the component along the vertical space. The vertical space $\mathcal{V}_{x} $ has the matrix characterization
$\{ (\mat{U}_1{\bf \Omega}_1, \mat{U}_2 {\bf \Omega}_2, \mat{U}_3 {\bf \Omega}_3, 
 - (\mathbfcal{G}{\times_1} {\bf \Omega}_1  + 
\mathbfcal{G}{{\times_2}} {\bf \Omega}_2 +
\mathbfcal{G}{{\times_3}} {\bf \Omega}_3)): {\bf \Omega}_d \in \mathbb{R}^{r_d \times r_d}, {\bf \Omega}_d ^T = -{\bf \Omega}_d \text{ for\ } d \in \{1, 2, 3\} \}$. Skew symmetric matrices ${\bf \Omega}_{d}$ for all $d \in \{1,2,3 \}$ parameterize the vertical space. Finally, the horizontal projection operator $\Pi_{{x}}: T_{{x}} {\mathcal{M}} : \rightarrow \mathcal{H}_{{x}} : $ $\eta_{{x}} \mapsto \Pi_{{x}}(\eta_{{x}})$ 
\changeHK{ is given as follows.}
\begin{prop}\label{prop:horizontal_space_projector}
\changeHK{The quotient manifold (\ref{Eq:QuotientSpace}) endowed with the metric (\ref{Eq:metric}) admits \changeBM{the} horizontal projector defined as}
\begin{equation*}
\begin{array}{lll}
\Pi_{{x}}(\eta_{{x}} )
&=& (
\eta_{\scriptsize \mat{U}_{1}} - \mat{U}_{1} {\bf \Omega}_{1},
\eta_{\scriptsize \mat{U}_{2}} - \mat{U}_{2} {\bf \Omega}_{2}, 
\eta_{\scriptsize \mat{U}_{3}} - \mat{U}_{3} {\bf \Omega}_{3}, 
\eta_{\scriptsize \mathbfcal{G}} - ( - (\mathbfcal{G}{\times_1} {\bf \Omega}_{1}  + 
\mathbfcal{G}{{\times_2}} {\bf \Omega}_{2} +
\mathbfcal{G}{{\times_3}} {\bf \Omega}_{3})) 
),
\end{array}
\end{equation*}
where $\eta_x = (\eta_{\scriptsize \mat{U}_{1}}, \eta_{\scriptsize \mat{U}_{2}}, \eta_{\scriptsize \mat{U}_{3}}, \eta_{\scriptsize \mathbfcal{G}}) \in T_x \mathcal{M}$ and  ${\bf \Omega}_{d}$ is a skew-symmetric matrix of size $r_d \times r_d$ that is the solution to the coupled Lyapunov equations
\begin{equation}
\begin{array}{lll}
\label{Eq:OmegaRequirements}
\left\{
\begin{array}{l}
\mat{G}_{1}  \mat{G}_{1}^T {\bf \Omega}_{1} + {\bf \Omega}_{1} \mat{G}_{1}  \mat{G}_{1}^T -\mat{G}_{1}(\mat{I}_{r_3} \otimes {\bf \Omega}_{2}) \mat{G}_{1}^T 
- \mat{G}_{1}( {\bf \Omega}_{3} \otimes \mat{I}_{r_2} )\mat{G}_{1}^T  \\
\hspace{5.8cm} = {\rm Skew}(\mat{U}_1^T\eta_{\scriptsize \mat{U}_1}\mat{G}_{1}  \mat{G}_{1}^T) + {\rm Skew}(\mat{G}_{1}\eta_{\scriptsize \mat{G}_{1}}^T), \\
\mat{G}_{2}  \mat{G}_{2}^T {\bf \Omega}_{2} + {\bf \Omega}_{2} \mat{G}_{2}  \mat{G}_{2}^T -\mat{G}_{2}(\mat{I}_{r_3} \otimes {\bf \Omega}_{1}) \mat{G}_{2}^T 
- \mat{G}_{2}( {\bf \Omega}_{3} \otimes \mat{I}_{r_1} )\mat{G}_{2}^T  \\
\hspace{5.8cm} = {\rm Skew}(\mat{U}_2^T\eta_{\scriptsize \mat{U}_2}\mat{G}_{2}  \mat{G}_{2}^T) + {\rm Skew}(\mat{G}_{2}\eta_{\scriptsize \mat{G}_{2}}^T), \\
\mat{G}_{3}  \mat{G}_{3}^T {\bf \Omega}_{3} + {\bf \Omega}_{3} \mat{G}_{3}  \mat{G}_{3}^T -\mat{G}_{3}(\mat{I}_{r_2} \otimes {\bf \Omega}_{1}) \mat{G}_{3}^T 
- \mat{G}_{3}( {\bf \Omega}_{2} \otimes \mat{I}_{r_1} )\mat{G}_{3}^T  \\
\hspace{5.8cm} = {\rm Skew}(\mat{U}_3^T\eta_{\scriptsize \mat{U}_3}\mat{G}_{3}  \mat{G}_{3}^T) + {\rm Skew}(\mat{G}_{3}\eta_{\scriptsize \mat{G}_{3}}^T),
\end{array}
\right.
\end{array}
\end{equation}
where ${\rm Skew}(\cdot)$ extracts the skew-symmetric part of a square matrix, i.e., ${\rm Skew}(\mat{D})=(\mat{D}-\mat{D}^T)/2$. 
\end{prop}

The coupled Lyapunov equations (\ref{Eq:OmegaRequirements}) are solved efficiently with the Matlab's \verb+pcg+ routine that is combined with a specific \changeBM{symmetric} preconditioner resulting from the Gauss-Seidel approximation of (\ref{Eq:OmegaRequirements}). \changeBM{For the variable $ {\bf \Omega}_1$, the preconditioner is of the form $\mat{G}_1\mat{G}_1^T {\bf \Omega}_1 + {\bf \Omega}_1 \mat{G}_1\mat{G}_1^T$. Similarly, for the variables ${\bf \Omega}_2$ and ${\bf \Omega}_3$.}

{\bf Retraction.}
A retraction is a mapping that maps vectors in the horizontal space to points on the search space $\mathcal{M}$ and 
satisfies the local rigidity condition 
\cite{Absil_OptAlgMatManifold_2008}. 
It provides a natural way to move on the manifold along a search direction. Because the total space  $\mathcal{M}$ has the product nature, we can choose a retraction by combining retractions on the individual manifolds, i.e.,
$
R_{x} (\xi_x)  =   ({\rm uf}(\mat{U}_{1}+\xi_{\scriptsize \mat{U}_{1}}), {\rm uf}(\mat{U}_{2}+\xi_{\scriptsize \mat{U}_{2}}), {\rm uf}(\mat{U}_{3}+\xi_{\scriptsize \mat{U}_{3}}), \mathbfcal{G}+\xi_{\scriptsize \mathbfcal{G}}),
$
where $\xi_x \in \mathcal{H}_x$ and ${\rm uf}(\cdot)$ extracts the orthogonal factor of a full column rank matrix, i.e., 
${\rm uf}(\mat{A})=\mat{A}(\mat{A}^T\mat{A})^{-1/2}$. The retraction $ {R}_{  x}$ defines a retraction ${R}_{[x]}( {\xi}_{[x]}) : =[R_x (\xi_x)] $ on the quotient manifold ${\mathcal{M}}/\sim$, as the equivalence class $[R_x (\xi_x)] $ does not depend on specific matrix representations of $[x]$ and ${\xi}_{[x]}$, where $ {\xi}_{  x}$ is the horizontal lift of the abstract tangent vector $\xi_{[x]} \in T_{[x]} (\mathcal{M} /\sim)$.

{\bf Vector transport.}
A vector transport on a manifold $\mathcal{M}$ is a smooth mapping that transports a tangent vector $\xi_x \in T_x \mathcal{M}$ at $x \in \mathcal{M}$ to a vector in the tangent space at \changeBM{a point} $R_x(\eta_x)$. \changeBM{It is defined by the symbol $\mathcal{T}_{\eta_x} \xi_x$.} It generalizes the classical concept of translation of vectors in the Euclidean space to manifolds \citep[Section~8.1.4]{Absil_OptAlgMatManifold_2008}. The horizontal lift of the abstract vector transport $\mathcal{T}_{\eta_{[x]}} \xi_{[x]}$ on $\mathcal{M}/\!\sim$ has the matrix characterization $\Pi_{R_x(\eta_x)}(\mathcal{T}_{\eta_x} \xi_x) = \Pi_{R_x(\eta_x)}(\Psi_{R_x(\eta_x)}(\xi_x))$, where $\xi_x$ and $\eta_x$ are the horizontal lifts in $\mathcal{H}_x$ of $\xi_{[x]}$ and $\eta_{[x]}$ that belong to $T_{[x]}(\mathcal{M}/ \!\sim)$. \changeBM{$\Psi_x(\cdot)$ and $\Pi_x(\cdot)$ are projectors defined in Propositions \ref{prop:tangent_space_projector} and \ref{prop:horizontal_space_projector}}. The computational cost of transporting a vector solely depends on the projection and retraction operations.

\begin{table*}[t]
\begin{center}  \small 
\caption{\changeBM{Tucker manifold related optimization ingredients} for (\ref{Eq:CostFunction})}
\label{tab:FinalFormulas}
\begin{tabular}{l|l}
\hline
Matrix representation  & $x =  
(\mat{U}_{1}, \mat{U}_{2}, \mat{U}_{3}, \mathbfcal{G})$
\\ 
\hdashline
Computational space $\mathcal{M}$ &
${\rm St}(r_1, n_1) \times {\rm St}(r_2, n_2) \times {\rm St}(r_3, n_3) \times \mathbb{R}^{r_1 \times r_2 \times r_3}$
\\ \hdashline
Group action & 
$\{ (\mat{U}_{1}\mat{O}_{1}, \mat{U}_{2}\mat{O}_{2}, \mat{U}_{3}\mat{O}_{3}, \mathbfcal{G}{\times_1} \mat{O}^T_{1}{{\times_2}} \mat{O}^T_{2}{{\times_3}} \mat{O}^T_{3}): \mat{O}_{d} \in \mathcal{O}{(r_d)}, \text{for\ }d \in \{1,2,3\} \}$
\\ \hdashline
Quotient space $\mathcal{M}/\!\sim$ & 
 ${\rm St}(r_1, n_1) \times {\rm St}(r_2, n_2) \times {\rm St}(r_3, n_3) \times \mathbb{R}^{r_1 \times r_2 \times r_3}$
  $/ (\mathcal{O}{(r_1)} \times \mathcal{O}{(r_2)} \times \mathcal{O}{(r_3)})$ 
\\ 
\hdashline
Ambient space &
$\mathbb{R}^{ n_1 \times r_1} \times \mathbb{R}^{ n_2 \times r_2} \times \mathbb{R}^{ n_3 \times r_3} \times \mathbb{R}^{r_1 \times r_2 \times r_3}$
\\
\hdashline
Tangent vectors in $T_x \mathcal{M}$& 
$
\{ (\mat{Z}_{{\bf U}_{1}}, \mat{Z}_{{\bf U}_{2}}, \mat{Z}_{{\bf U}_{3}}, \mat{Z}_{\mathbfcal{G}}) \in 
\mathbb{R}^{n_1 \times r_1} \times  
\mathbb{R}^{n_2 \times r_2} \times 
\mathbb{R}^{n_3 \times r_3} \times 
\mathbb{R}^{r_1 \times r_2 \times r_3} $
\\
 & : $\mat{U}_{d}^T \mat{Z}_{{\bf U}_{d}} +  \mat{Z}_{{\bf U}_{d}}^T \mat{U}_{d} = 0, \text{ for }d \in \{1,2,3\} \}
$
\\ \hdashline
Metric ${g}_{x}(\xi_{x}, \eta_{x})$ for & 
$\langle \xi_{\scriptsize \mat{U}_{1}},
{\eta}_{\scriptsize\mat{U}_{1}} (\mat{G}_{1} \mat{G}_{1}^T) \rangle  + 
\langle \xi_{\scriptsize \mat{U}_{2}},
{\eta}_{\scriptsize\mat{U}_{2}} (\mat{G}_{2} \mat{G}_{2}^T) \rangle  + 
\langle \xi_{\scriptsize \mat{U}_{3}},
{\eta}_{\scriptsize\mat{U}_{3}} (\mat{G}_{3} \mat{G}_{3}^T) \rangle 
 +  \langle {\xi}_{\scriptsize \mathbfcal{G}}, {\eta}_{\scriptsize \mathbfcal{G}}\rangle $
\\
any $\xi_x, \eta_x \in T_x \mathcal{M}$ & 
\\ \hdashline
Vertical tangent vectors & 
$\{ (\mat{U}_{1} {\bf \Omega}_{1}, \mat{U}_{2}  {\bf \Omega}_{2}, \mat{U}_{3}  {\bf \Omega}_{3}, 
 - (\mathbfcal{G}{\times_1} {\bf \Omega}_{1}  + 
\mathbfcal{G}{{\times_2}} {\bf \Omega}_{2} +
\mathbfcal{G}{{\times_3}} {\bf \Omega}_{3})):$
\\
 in $\mathcal{V}_x$ & $ {\bf \Omega}_{d} \in \mathbb{R}^{r_d \times r_d}, {\bf \Omega}_{d}^T = -{\bf \Omega}_{d}, \text{for }d \in \{1,2,3\} \}$
\\ \hdashline
Horizontal tangent vectors& 
$\{(\zeta_{\mat{U}_{1}}, \zeta_{\mat{U}_{2}}, \zeta_{\mat{U}_{3}}, \zeta_{\mathbfcal{G}}) \in T_x \mathcal{M} : (\mat{G}_{d}  \mat{G}_{d}^T) \zeta_{\scriptsize \mat{U}_{d}}^T \mat{U}_{d}  
+ \zeta_{\scriptsize \mat{G}_{d}}  \mat{G}_{d}^T \text{ is symmetric},$ \\
 in $\mathcal{H}_x$& $\text{for }d \in \{1,2,3\} \}$
\\ \hdashline
$\Psi(\cdot)$ projects an ambient & 
$(
\mat{Y}_{\scriptsize \mat{U}_{1}} - \mat{U}_{1} \mat{S}_{\scriptsize \mat{U}_{1}} (\mat{G}_{1}  \mat{G}_{1}^T)^{-1},
\mat{Y}_{\scriptsize \mat{U}_{2}} - \mat{U}_{2} \mat{S}_{\scriptsize \mat{U}_{2}} (\mat{G}_{2}  \mat{G}_{2}^T)^{-1}, 
$
\\
vector  $(\mat{Y}_{\scriptsize \mat{U}_{1}},\!\mat{Y}_{\scriptsize \mat{U}_{2}},\!\mat{Y}_{\scriptsize \mat{U}_{3}},\!\mat{Y}_{\scriptsize \mathbfcal{G}})$& $\mat{Y}_{\scriptsize \mat{U}_{3}} - \mat{U}_{3} \mat{S}_{\scriptsize \mat{U}_{3}} (\mat{G}_{3}  \mat{G}_{3}^T)^{-1},
\mat{Y}_{\scriptsize \mathbfcal{G}}  
)$, where $\mat{S}_{\scriptsize \mat{U}_{d}}$ for $d \in \{1, 2,3 \}$ are computed
\\
onto $T_x \mathcal{M}$&  by solving Lyapunov equations as in (\ref{Eq:B_Requirements}).
%
\\ \hdashline
$\Pi(\cdot)$ projects a tangent   &  $(
\xi_{\scriptsize \mat{U}_{1}} - \mat{U}_{1} {\bf \Omega}_{1},
\xi_{\scriptsize \mat{U}_{2}} - \mat{U}_{2} {\bf \Omega}_{2}, 
\xi_{\scriptsize \mat{U}_{3}} - \mat{U}_{3} {\bf \Omega}_{3}, $ 
\\
vector $\xi$ onto $\mathcal{H}_x$ & 
$\xi_{\scriptsize \mathbfcal{G}} - ( - (\mathbfcal{G}{\times_1} {\bf \Omega}_{1}  + 
\mathbfcal{G}{{\times_2}} {\bf \Omega}_{2} +
\mathbfcal{G}{{\times_3}} {\bf \Omega}_{3})) 
)$, ${\bf \Omega}_{d}$ is computed in (\ref{Eq:OmegaRequirements}).
\\ \hdashline
First order derivative of & 
$(\mat{S}_{1} (\mat{U}_{3} \otimes \mat{U}_{2}) \mat{G}_{1}^T, \mat{S}_{2} (\mat{U}_{3} \otimes \mat{U}_{1}) \mat{G}_{2}^T, \mat{S}_{3} (\mat{U}_{2} \otimes \mat{U}_{1}) \mat{G}_{3}^T), \mathbfcal{S} \times_1 \mat{U}_{1}^T \times_2 \mat{U}_{2}^T \times_3 \mat{U}_{3}^T),$\\
$f(x)$ & where $\mathbfcal{S} = \frac{2}{|{\Omega} |} (\mathbfcal{P}_{\Omega}(\mathbfcal{G}{\times_1} {\mat{U}_{1}}{\times_2} {\mat{U}_{2}}{\times_3} {\mat{U}_{3}}) - 
\mathbfcal{P}_{\Omega}(\mathbfcal{X}^{\star}))$.
\\ \hdashline
Retraction $R_x(\xi_x)$ & 
$({\rm uf}(\mat{U}_{1}+\xi_{\scriptsize \mat{U}_{1}}), {\rm uf}(\mat{U}_{2}+\xi_{\scriptsize \mat{U}_{2}}), {\rm uf}(\mat{U}_{3}+\xi_{\scriptsize \mat{U}_{3}}), \mathbfcal{G}+\xi_{\scriptsize \mathbfcal{G}})$
\\ \hdashline
Horizontal lift of the & 
$\Pi_{R_x(\eta_x)}(\Psi_{R_x(\eta_x)}(\xi_x))$
\\
vector transport $\mathcal{T}_{\eta_{[x]}} \xi_{[x]}$  & \\
\hline
\end{tabular}
\end{center}
\end{table*}

\section{Riemannian algorithms for (\ref{Eq:CostFunction})}
\label{sec:AlgorithmDetails}
\changeBM{We propose two Riemannian preconditioned algorithms for the tensor completion problem (\ref{Eq:CostFunction}) that are based on the developments in Section \ref{sec:OptimizationRelatedIngredients}.} The preconditioning effect follows from the specific choice of the metric (\ref{Eq:metric}). \changeBM{In the batch setting,} we use the off-the-shelf conjugate gradient implementation of Manopt for any smooth cost function \cite{Boumal_Manopt_2014_s}. A complete description of the Riemannian nonlinear conjugate gradient method is in 
\citep[Chapter~8]{Absil_OptAlgMatManifold_2008}. \changeBM{In the online setting, we use the stochastic gradient descent implementation \citep{bonnabel13a}.} \changeHK{For fixed rank, theoretical convergence of the Riemannian algorithms are to a stationary point, and} the convergence analysis follows from \cite{Sato15a, Ring_SIAMJOptim_2012_s,bonnabel13a}. \changeHK{However, as simulations show, \changeBM{convergence to global minima} is observed in many challenging instances.}

\changeBM{In addition to the manifold-related ingredients in Section \ref{sec:OptimizationRelatedIngredients},} the ingredients needed are the cost function specific ones. To this end, we show the computation of the Riemannian gradient as well as a way to compute an initial guess for the step-size, which is used in the conjugate gradient method. The concrete formulas are shown in Table \ref{tab:FinalFormulas}.


{\bf Riemannian gradient computation.} Let $f(\mathbfcal{X})=\| \mathbfcal{P}_{\Omega}(\mathbfcal{X}) - \mathbfcal{P}_{\Omega}(\mathbfcal{X}^{\star}) \|^2_F/|\Omega |$ be the mean square error function of (\ref{Eq:CostFunction}), and 
$\mathbfcal{S} = 2 (\mathbfcal{P}_{\Omega}(\mathbfcal{G}{\times_1} \mat{U}_{1} {\times_2} \mat{U}_{2}{\times_3} \mat{U}_{3}) - 
\mathbfcal{P}_{\Omega}(\mathbfcal{X}^{\star}))/|{\Omega}|$ 
be an auxiliary sparse tensor variable that is interpreted as the Euclidean gradient of $f$ in $\mathbb{R}^{n_1 \times n_2 \times n_3}$. The partial derivatives of $f$ with respect to $(\mat{U}_{1}, \mat{U}_{2}, \mat{U}_{3}, \mathbfcal{G})$ are computed in terms of the unfolding matrices $\mat{S}_{d}$. Due to the specific scaled metric (\ref{Eq:metric}), the partial derivatives are further scaled by $((\mat{G}_{1}\mat{G}_{1}^T)^{-1}, (\mat{G}_{2}\mat{G}_{2}^T)^{-1}, (\mat{G}_{3}\mat{G}_{3}^T)^{-1}, \mathbfcal{I})$, denoted as ${\rm egrad}_{x} f$ (after scaling). Finally, from the Riemannian submersion theory \citep[Section ~3.6.2]{Absil_OptAlgMatManifold_2008}, 
the horizontal lift of ${\rm grad}_{[x]}f $ is equal to ${\rm grad}_{x}f \ =\ \Psi({\rm egrad}_{x} f )$. The total numerical cost of computing the Riemannian gradient depends on computing the partial derivatives, which is $O(|\Omega| r_1 r_2 r_3)$.

\begin{prop}
\changeHK{The cost function (\ref{Eq:CostFunction}) at $(\mat{U}_{1}, \mat{U}_{2}, \mat{U}_{3}, \mathbfcal{G})$ under the quotient manifold (\ref{Eq:QuotientSpace}) endowed with the Riemannian metric (\ref{Eq:metric}) admits the horizontal lift of the Riemannian gradient}
\begin{equation}
\label{Eq:RiemannianGradient}
\begin{array}{lll}
(\mat{S}_{1} (\mat{U}_{3} \otimes \mat{U}_{2}) \mat{G}_{1}^T (\mat{G}_{1}\mat{G}_{1}^T)^{-1} 
- \mat{U}_{1} \mat{B}_{\scriptsize \mat{U}_{1}}(\mat{G}_{1}\mat{G}_{1}^T)^{-1}, \\
 \hspace{0cm} \mat{S}_{2} (\mat{U}_{3} \otimes \mat{U}_{1}) \mat{G}_{2}^T (\mat{G}_{2}\mat{G}_{2}^T)^{-1}
- \mat{U}_{2} \mat{B}_{\scriptsize \mat{U}_{2}}(\mat{G}_{2}\mat{G}_{2}^T)^{-1}, \\
 \hspace{0cm} \mat{S}_{3} (\mat{U}_{2} \otimes \mat{U}_{1}) \mat{G}_{3}^T(\mat{G}_{3}\mat{G}_{3}^T)^{-1}
- \mat{U}_{3} \mat{B}_{\scriptsize \mat{U}_{3}}(\mat{G}_{3}\mat{G}_{3}^T)^{-1}, 
 \mathbfcal{S} \times_1 \mat{U}_{1}^T \times_2 \mat{U}_{2}^T \times_3 \mat{U}_{3}^T), 
\end{array}
\end{equation}
where $\mat{B}_{\scriptsize \mat{U}_{d}}$ for $d \in \{1, 2,3\} $ are the solutions to the Lyapunov equations
\begin{equation*}
\left\{
\begin{array}{lll}
\mat{B}_{\scriptsize \mat{U}_{1}} \mat{G}_{1}  \mat{G}_{1}^T + \mat{G}_{1}  \mat{G}_{1}^T \mat{B}_{\scriptsize \mat{U}_{1}} &=&  2 {\rm Sym} (\mat{G}_{1}  \mat{G}_{1}^T\mat{U}_{1}^T (\mat{S}_{1} (\mat{U}_{3} \otimes \mat{U}_{2}) \mat{G}_{2}^T), \\ 
\mat{B}_{\scriptsize \mat{U}_{2}} \mat{G}_{2}  \mat{G}_{2}^T + \mat{G}_{2}  \mat{G}_{2}^T \mat{B}_{\scriptsize \mat{U}_{2}} &=&   2 {\rm Sym} (\mat{G}_{2}  \mat{G}_{2}^T\mat{U}_{2}^T (\mat{S}_{2} (\mat{U}_{3} \otimes \mat{U}_{1}) \mat{G}_{2}^T), \\ 
\mat{B}_{\scriptsize \mat{U}_{3}} \mat{G}_{3}  \mat{G}_{3}^T + \mat{G}_{3}  \mat{G}_{3}^T \mat{B}_{\scriptsize \mat{U}_{3}} &=&   2 {\rm Sym} (\mat{G}_{3}  \mat{G}_{3}^T\mat{U}_{3}^T (\mat{S}_{3} (\mat{U}_{2} \otimes \mat{U}_{1}) \mat{G}_{3}^T),
\end{array}
\right.
\end{equation*}
where ${\rm Sym}(\cdot)$ extracts the symmetric part of a matrix.
\end{prop}


{\bf Initial guess for the step-size.} Following \cite{Mishra_ICDC_2014_s, Vandereycken_SIAMOpt_2013_s, Kressner_BIT_2014_s}, the least-squares structure of the cost function in (\ref{Eq:CostFunction}) is exploited to compute a \emph{linearized} step-size guess efficiently along a search direction by considering a polynomial approximation of degree $2$ over the manifold. Given a search direction $\xi_x \in \mathcal{H}_x$, the step-size guess is 
$\argmin_{s \in \mathbb{R}_{+}}
\| \mathbfcal{P}_{\Omega}(
\mathbfcal{G}{\times_1} \mat{U}_{1}{\times_2} \mat{U}_{2}{\times_3} \mat{U}_{3} 
+ s \mathbfcal{G}{\times_1} {\xi_{\scriptsize \mat{U}_{1}}}{\times_2} \mat{U}_{2}{\times_3} \mat{U}_{3} 
+ s \mathbfcal{G}{\times_1} \mat{U}_{1}{\times_2} {\xi_{\scriptsize \mat{U}_{2}}}{\times_3} \mat{U}_{3} 
+ s \mathbfcal{G}{\times_1} \mat{U}_{1}{\times_2}  \mat{U}_{2}{\times_3} {\xi_{\scriptsize \mat{U}_{3} }}
+ s {\xi_{\scriptsize \mathbfcal{G}}}{\times_1} \mat{U}_{1}{\times_2} \mat{U}_{2}{\times_3} \mat{U}_{3} 
) 
- \mathbfcal{P}_{\Omega}(\mathbfcal{X}^{\star}) \|^2_F$, which has a closed-form expression and the numerical cost of computing it is $O(|\Omega| r_1 r_2 r_3)$.

\changeHK{
{\bf Stochastic gradient descent in online setting.} \changeBM{In the online setting, we update} $(\mat{U}_{1}, \mat{U}_{2}, \mat{U}_{3}, \mathbfcal{G})$ every time a frontal slice, i.e., a matrix $\in \mathbb{R}^{n_1 \times n_2}$, is randomly sampled from $\mathbfcal{X}_{i_1, i_2, i_3}^{\star}$. \changeBM{Equivalently, we assume that the tensor grows along the third dimension.} More concretely, we calculate the \changeBM{\emph{rank-one}} Riemannian gradient (\ref{Eq:RiemannianGradient}) for the input slice. $(\mat{U}_{1}, \mat{U}_{2}, \mat{U}_{3}, \mathbfcal{G})$ are updated by taking a step along the negative Riemannian gradient direction. Subsequently, we retract using $R_x$. A popular formula for the step-size $\gamma_k$ at $k$-th update is $\gamma_k  =  \gamma_0/(1 + \gamma_0 \lambda k)$, where $\gamma_0$ is the \changeBM{initial} step-size and $\lambda$ is a fixed reduction factor. Following \cite{Bottou_SGDTricks_2012_s}, we select $\gamma_0$ in the {\it pre-training phase} using a {\it small sample size} of a training set. \changeBM{$\lambda$ is fixed to $10^{-7}$.}
}

\changeBM{{\bf Computational cost.} The total computational cost per iteration of our proposed conjugate gradient implementation is $O(|\Omega| r_1 r_2 r_3)$, where $|\Omega|$ is the number of known entries. It should be stressed that the computational cost of our conjugate gradient implementation is equal to that of \citep{Kressner_BIT_2014_s}.} In the online setting, each stochastic gradient descent update costs $O(|\Omega_{\rm slice}| r_1 r_2 + n_1 r_1^2 + n_2 r_2^2 + \changeHKKK{T} r_3^2 + r_1 r_2 r_3)$, where $|\Omega_{\rm slice}|$ is the number of known entries of the current frontal slice of the incomplete tensor $\mathbfcal{X}_{i_1, i_2, i_3}^{\star}$, \changeHKKK{and $T$ is the number of slices that we have seen along $n_3$ direction.} 


\section{Numerical comparisons}\label{sec:NumericalComparisons}
In the batch setting, we show a number of numerical comparisons of our proposed \changeBM{conjugate gradient algorithm} with state-of-the-art algorithms that include TOpt \cite{Filipovi_MultiSysSigPro_2013_s} and geomCG \cite{Kressner_BIT_2014_s}, for comparisons with Tucker decomposition based algorithms, and  HaLRTC \cite{Liu_IEEETransPAMI_2013_s}, Latent \cite{Tomioka_Latent_2011_s}, and Hard \cite{Signoretto_MachineLearning_2014_s} as nuclear norm minimization algorithms. \changeBM{In the online setting, we compare our proposed stochastic gradient descent algorithm with CANDECOMP/PARAFAC based TeCPSGD \cite{Mardani_IEEETransSP_2015} and OLSTEC \cite{Kasai_IEEEICASSP_2016_s}.} All simulations are performed in Matlab on a 2.6 GHz Intel Core i7 machine with 16 GB RAM. For specific operations with unfoldings of $\mathbfcal{S}$, we use the \verb+mex+ interfaces for Matlab that are provided by the authors of geomCG. For large-scale instances, our algorithm is only compared with geomCG as others cannot handle them. \changeBM{Cases S and R are for batch instances, whereas Case O is for online instances.}

Since the dimension of the space of a tensor $\in \mathbb{R}^{n_1 \times n_2 \times n_3}$ of rank ${\bf r} = (r_1, r_2, r_3)$ is $\text{dim}(\mathcal{M}/\!\sim) = \sum_{d=1}^3 (n_d r_d - r_d^2) +  r_1 r_2 r_3$, we randomly and uniformly select known entries based on a multiple of the dimension, called the \emph{over-sampling} (OS) ratio, to create the \changeHK{train} set $\Omega$. Algorithms are initialized randomly, as suggested in \cite{Kressner_BIT_2014_s}, and are stopped when either the mean square error (MSE) on the \changeHK{train} set $\Omega$ is below $10^{-12}$ or the number of iterations exceeds $250$. We also evaluate the mean square error on a test set $\Gamma$, which is different from $\Omega$. Five runs are performed in each scenario and the plots show all of them. The time plots are shown with standard deviations. \changeBM{It should be noted that we show most numerical comparisons on the \emph{test set} $\Gamma$ as it allows to compare with nuclear norm minimization algorithms, which optimize a different (training) cost function. Additional plots are provided as supplementary material.}



\begin{figure*}[t]
\begin{center}
\begin{tabular}{cc}
\begin{minipage}{0.32\hsize}
\begin{center}
\includegraphics[width=\hsize]{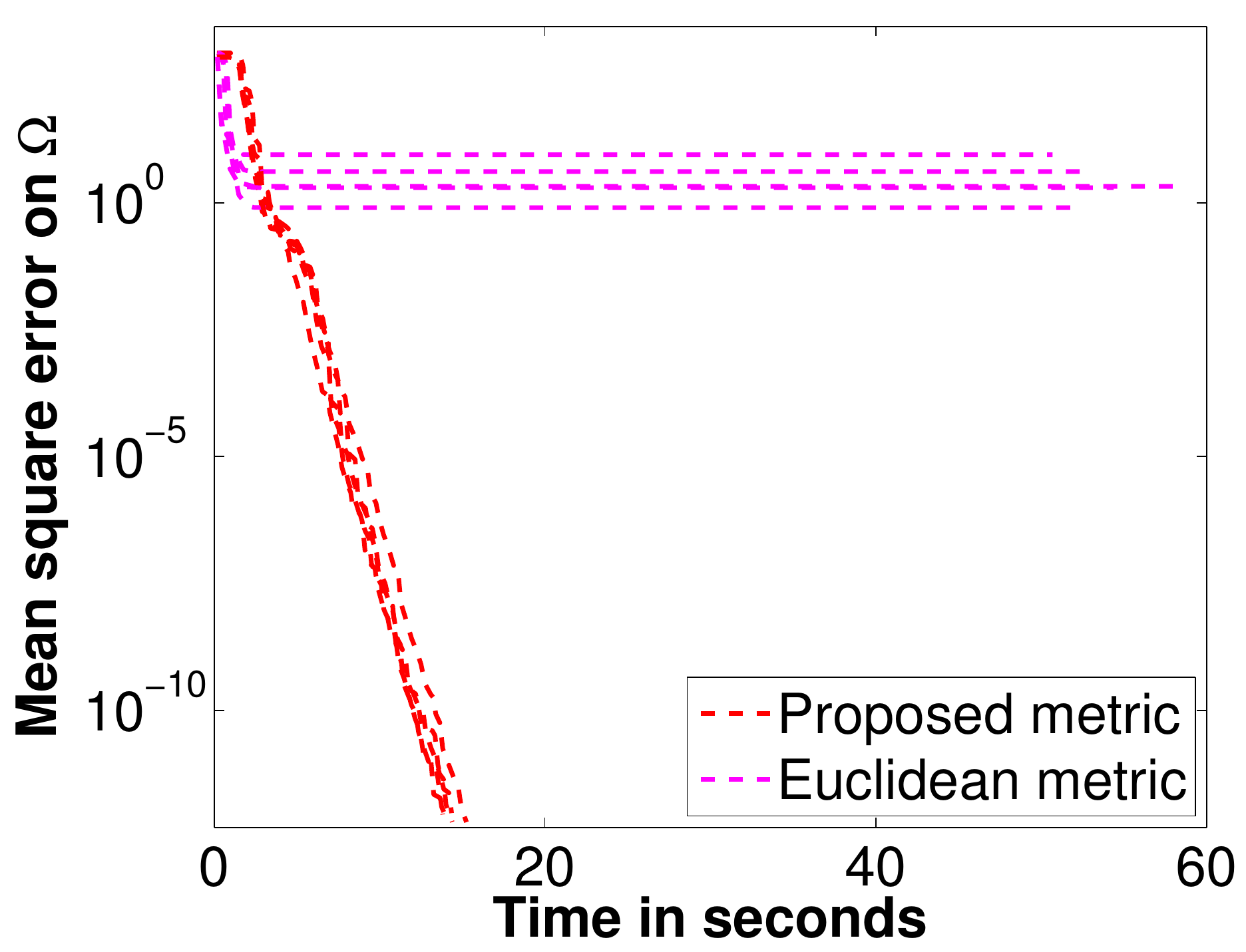}\\
{\changeHK{\scriptsize(a) {\bf Case S1:} comparison between metrics (\changeBM{train} error).}}
\end{center}
\end{minipage}
\begin{minipage}{0.32\hsize}
\begin{center}
\includegraphics[width=\hsize]{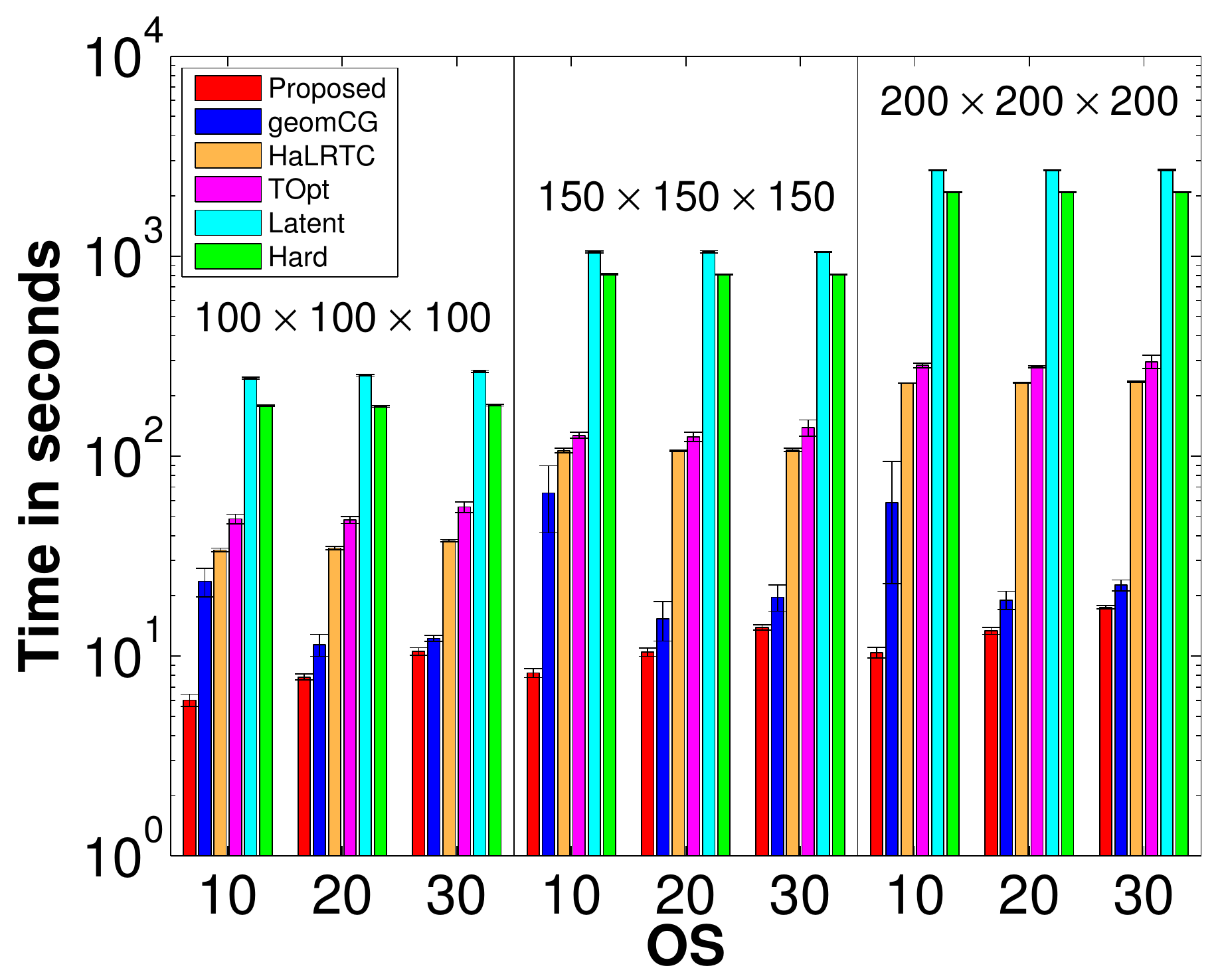}\\
{\scriptsize(b) {\bf Case S2:} {\bf r} = $(10,10,10)$.}
\end{center}
\end{minipage}
\begin{minipage}{0.32\hsize}
\begin{center}
\includegraphics[width=\hsize]{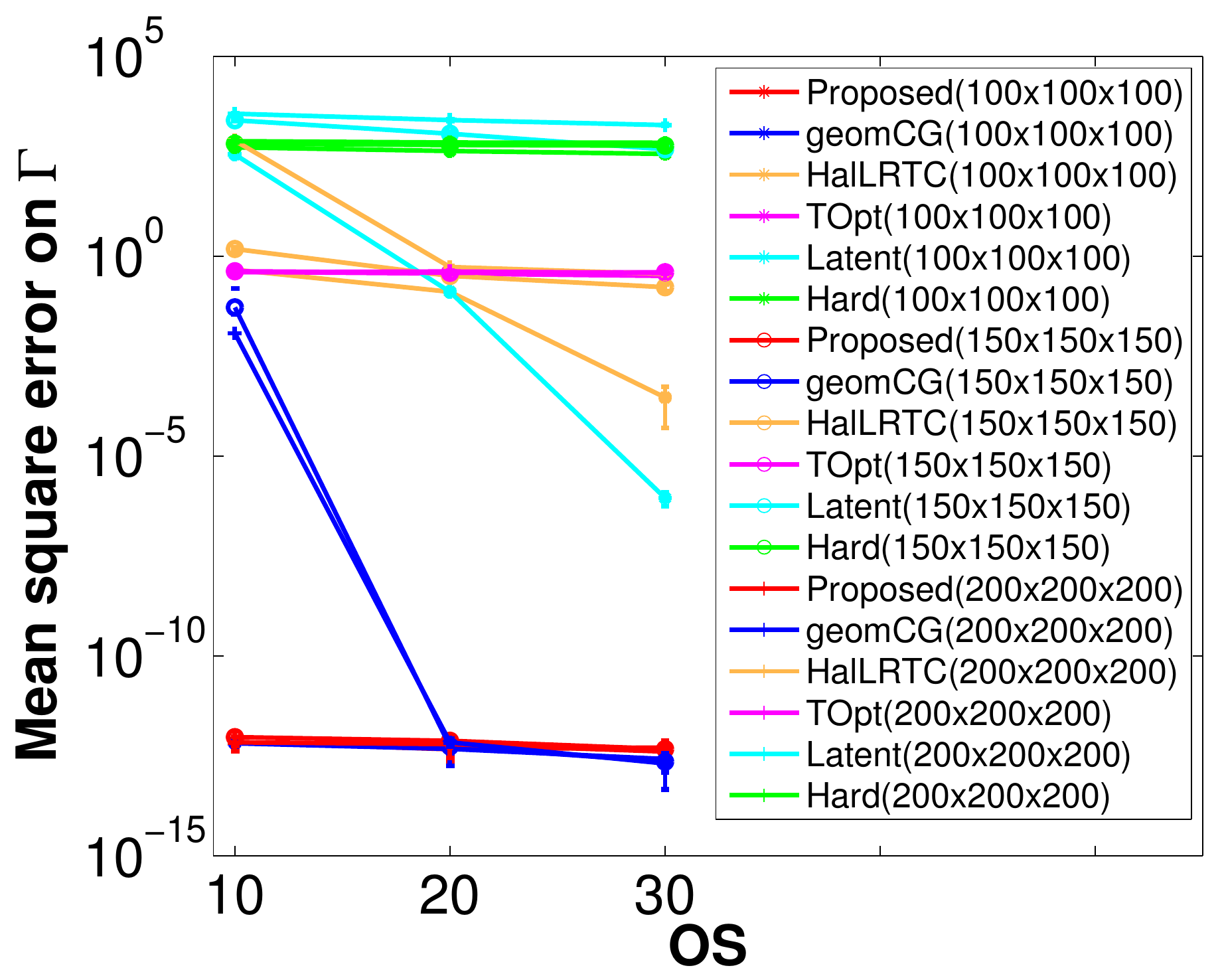}\\
{\scriptsize(c) {\bf Case S2:} {\bf r} = $(10,10,10)$.}
\end{center}
\end{minipage}\\
\begin{minipage}{0.32\hsize}
\begin{center}
\includegraphics[width=\hsize]{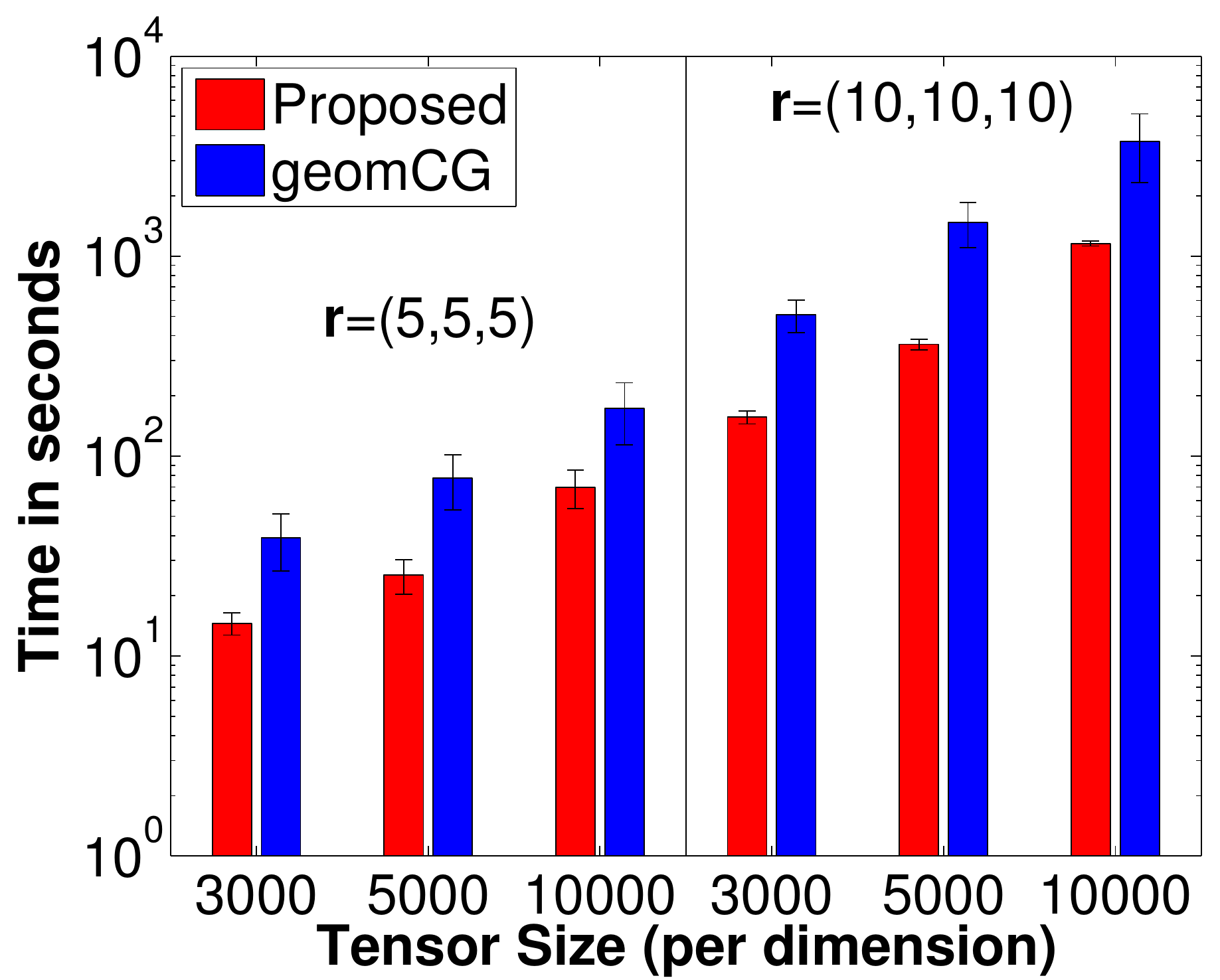}\\
{\scriptsize(d) {\bf Case S3}.}
\end{center}
\end{minipage}
\begin{minipage}{0.32\hsize}
\begin{center}
\includegraphics[width=\hsize]{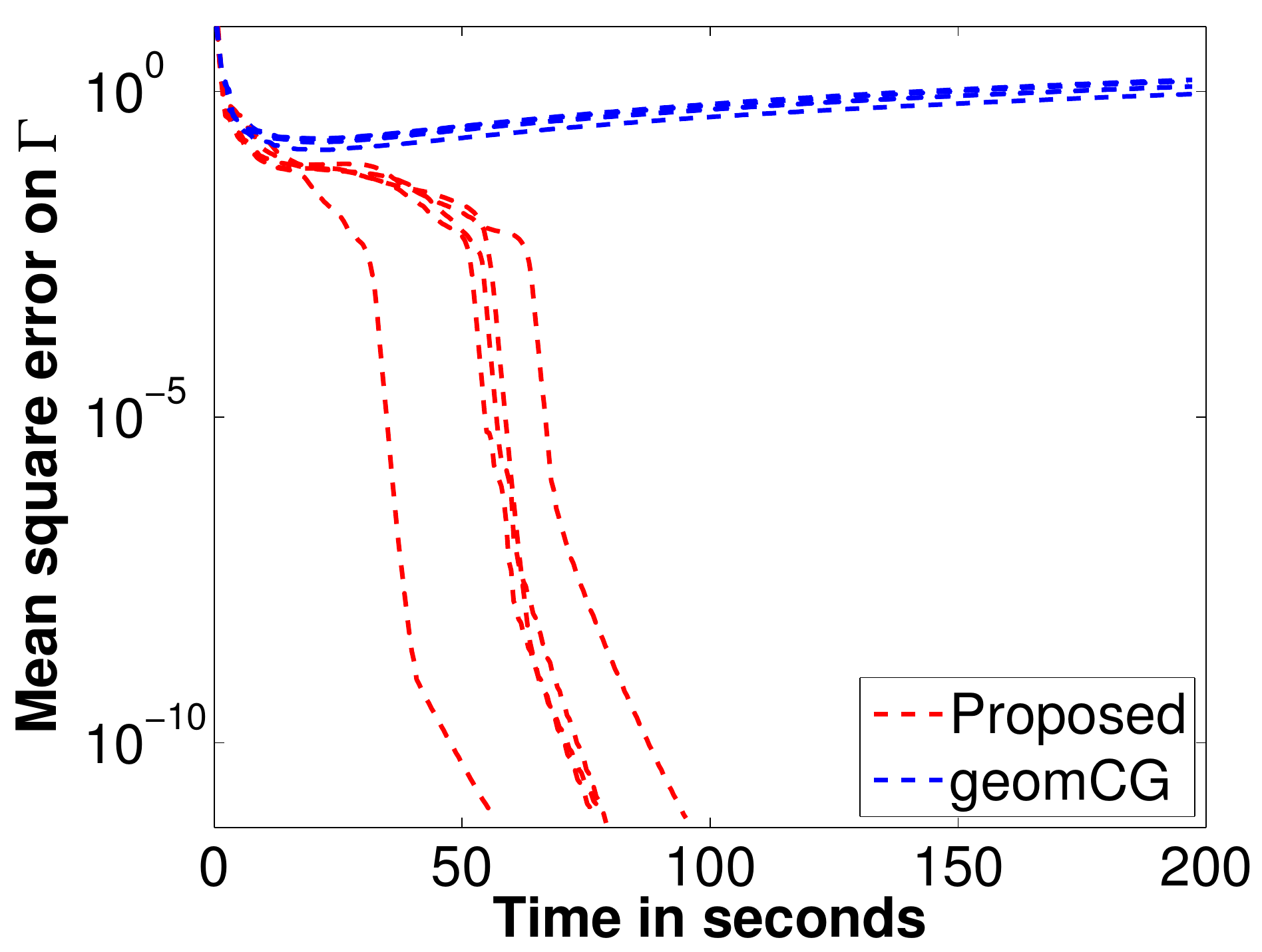}\\
{\scriptsize(e) {\bf Case S4:} OS = $4$.}
\end{center}
\end{minipage}
\begin{minipage}{0.32\hsize}
\begin{center}
\includegraphics[width=\hsize]{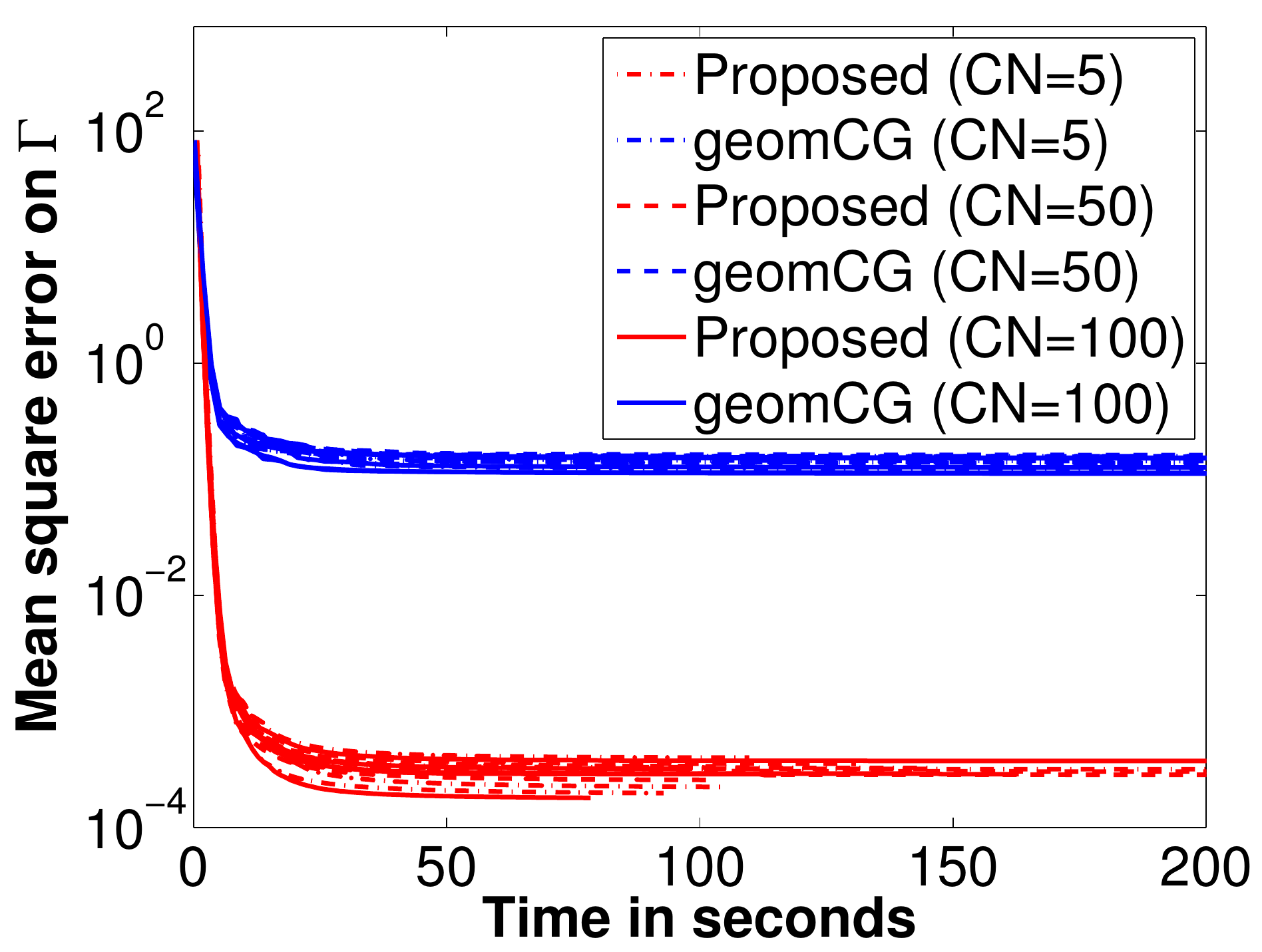}\\
{\scriptsize (f) {\bf Case S5:} CN = $\{5,50,100\}$.}
\end{center}
\end{minipage}\\
\begin{minipage}{0.32\hsize}
\begin{center}
\includegraphics[width=\hsize]{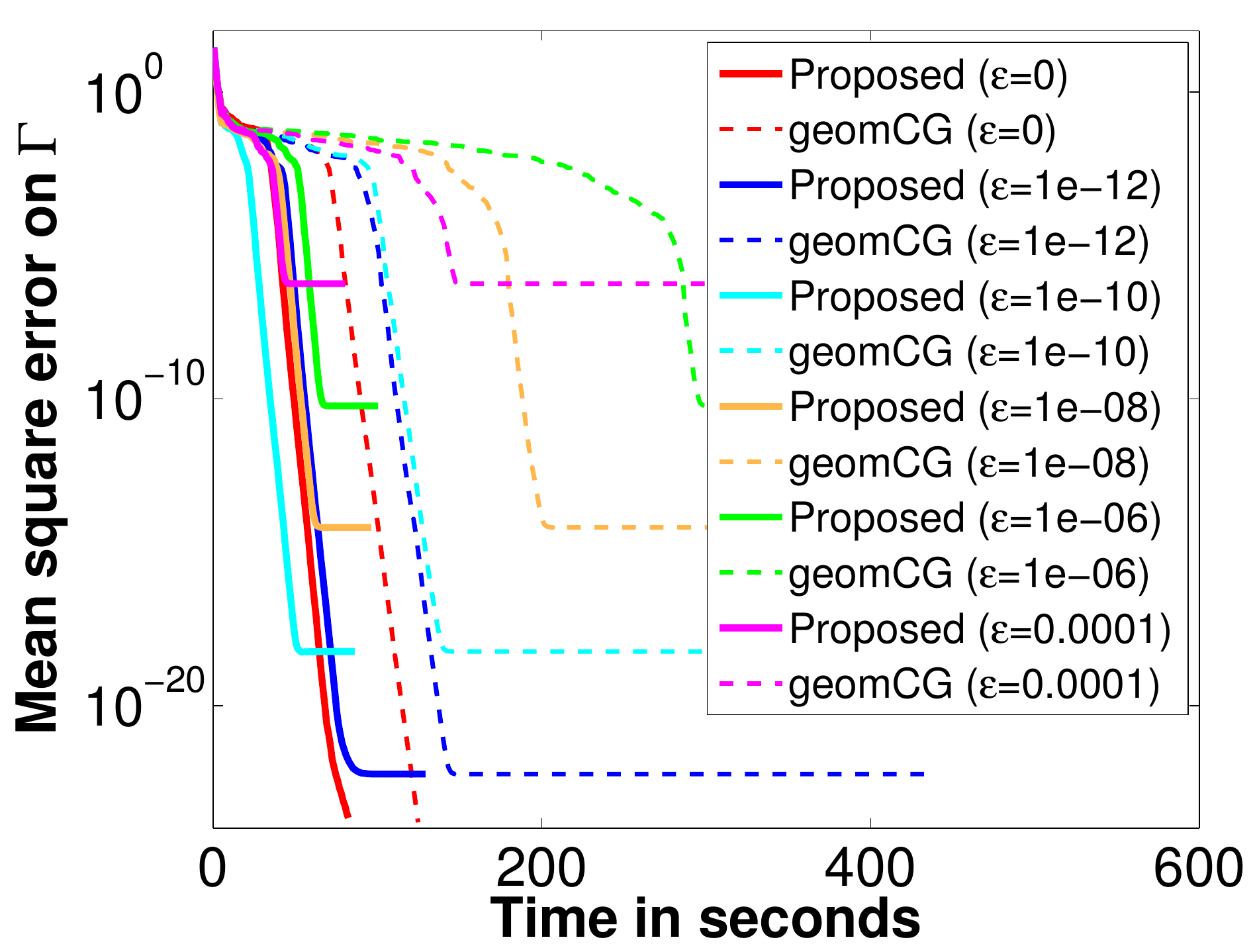}\\
{\scriptsize (g) {\bf Case S6:} noisy data.}
\end{center}
\end{minipage}
\begin{minipage}{0.32\hsize}
\begin{center}
\includegraphics[width=\hsize]{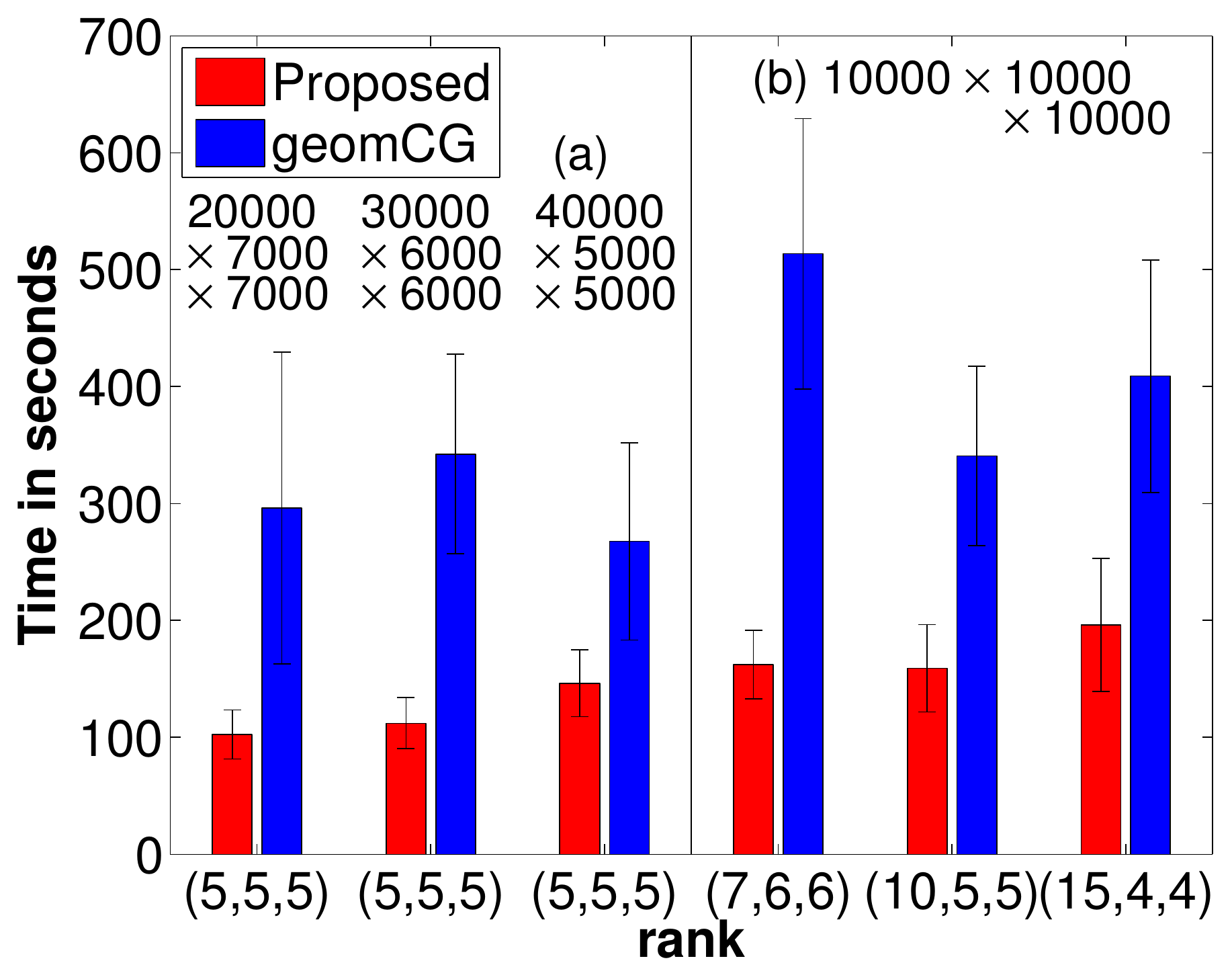}\\
{\scriptsize (h) {\bf Case S7:} \changeHK{rectangular} tensors.}
\end{center}
\end{minipage}
\begin{minipage}{0.32\hsize}
\begin{center}
\includegraphics[width=\hsize]{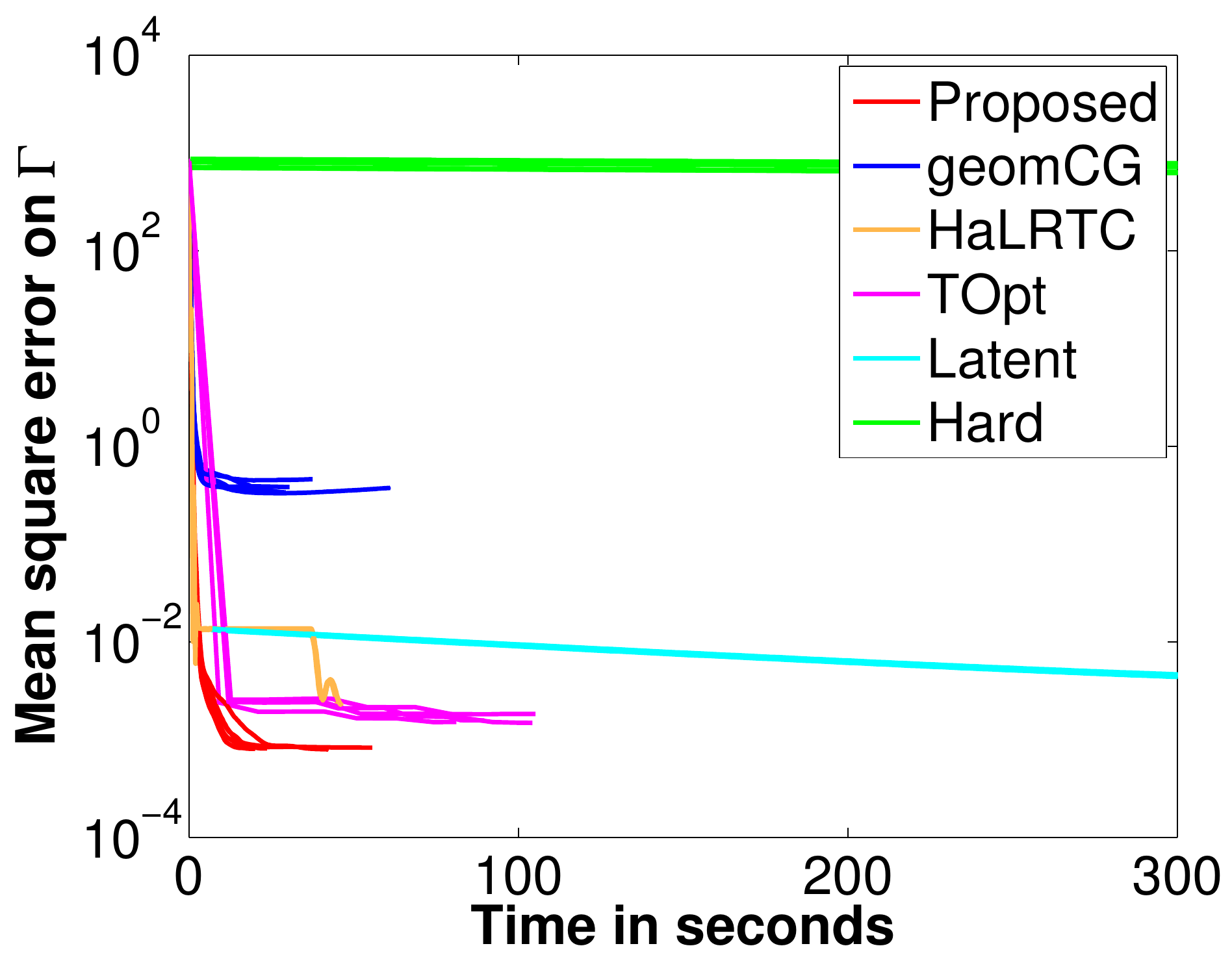}\\
{\scriptsize(i) {\bf Case R1:} Ribeira, OS = 11.}
\end{center}
\end{minipage}
\end{tabular}
\caption{Experiments on synthetic and real datasets.}
\label{fig:syntheticReal}
\end{center}
\end{figure*}

\begin{table*}[t]
\caption{{\bf Cases R1 and R2:} test MSE on $\Gamma$ and time in seconds}
\label{tab:R1R2}
\begin{center}
{\scriptsize
\begin{tabular}{l|l|l|l|l}
\hlinewd{1.0pt}
\textbf{Ribeira}& \multicolumn{2}{c|}{OS  = $11$}  & \multicolumn{2}{c}{OS = $22$}  \\
\hdashline
\quad Algorithm & \qquad Time & \qquad \qquad  MSE on $\Gamma$ & \qquad Time & \qquad \qquad  MSE on $\Gamma$  \\
\hline 
Proposed &${\bf 33 \pm 13}$ & ${\bf 8.2095\cdot 10^{-4} \pm 1.7\cdot 10^{-5} }$ & $67 \pm 43$ & ${\bf 6.9516 \cdot 10^{-4} \pm 1.1 \cdot  10^{-5} }$\\
\hdashline
geomCG & $36 \pm 14$ & $3.8342 \cdot  10^{-1} \pm 4.2 \cdot  10^{-2} $ & $150 \pm 48$ & $6.2590 \cdot 10^{-3} \pm 4.5 \cdot  10^{-3} $\\
\hdashline
HaLRTC & $46 \pm  0$ & $2.2671 \cdot  10^{-3} \pm 3.6 \cdot  10^{-5} $ & $ 48 \pm  0$ & $1.3880 \cdot 10^{-3} \pm 2.7 \cdot  10^{-5} $\\
\hdashline
TOpt & $80 \pm 32$ & $1.7854 \cdot  10^{-3} \pm 3.8 \cdot  10^{-4} $ & ${\bf 27 \pm 21}$ & $2.1259 \cdot 10^{-3} \pm 3.8 \cdot  10^{-4} $\\
\hdashline
Latent & $553 \pm 3$ & $2.9296 \cdot 10^{-3} \pm 6.4 \cdot  10^{-5} $ & $558 \pm 3$ & $1.6339  \cdot  10^{-3} \pm 2.3 \cdot  10^{-5} $\\
\hdashline
Hard & $400 \pm 5$ & $6.5090 \cdot 10^{2}\pm 6.1 \cdot  10^{1} $ & $402 \pm 4$ & $6.5989 \cdot  10^{2} \pm 9.8 \cdot  10^{1}$\\
\hlinewd{1.0pt}
\textbf{MovieLens-10M} & \multicolumn{2}{c|}{Proposed}  & \multicolumn{2}{c}{geomCG}  \\
\hdashline
\qquad  {\bf r} & \qquad Time & \qquad \qquad  MSE on $\Gamma$ &  \qquad  Time & \qquad \qquad MSE on $\Gamma$  \\
\hline
$(4,4,4)$ & ${\bf 1748 \pm 441}$ & ${\bf 0.6762 \pm 1.5 \cdot  10^{-3}}$ &  $2981 \pm 40$ & $ 0.6956 \pm 2.8 \cdot 10^{-3}$ \\
\hdashline
$(6,6,6)$ & ${\bf 6058 \pm 47}$ & ${\bf 0.6913 \pm 3.3 \cdot 10^{-3}}$ &  $6554 \pm 655$ & $0.7398 \pm 7.1 \cdot  10^{-3}$ \\
\hdashline
$(8,8,8)$ & ${\bf 11370 \pm 103}$ & ${\bf 0.7589 \pm 7.1 \cdot 10^{-3}}$ &  $13853 \pm 118 $ & $0.8955 \pm 3.3 \cdot 10^{-2}$ \\
\hdashline
$(10,10,10)$ & ${\bf 32802 \pm 52}$ & ${\bf 1.0107 \pm 2.7 \cdot 10^{-2}}$ &  $38145 \pm 36$ & $1.6550 \pm 8.7 \cdot 10^{-2}$ \\
\hlinewd{1.0pt}
\end{tabular}
}
\end{center}
\end{table*}

{\bf Case S1: comparison with the Euclidean metric.} We first show the benefit of the proposed metric (\ref{Eq:metric}) over the conventional choice of the Euclidean metric that exploits the product structure of $\mathcal{M}$ and symmetry (\ref{Eq:EquivalenceClass_3}). This is defined by combining the individual natural metrics for ${\rm St}(r_d,n_d)$ and $\mathbb{R}^{r_1\times r_2\times r_3}$. For simulations, we randomly generate a tensor of size $200 \times 200 \times 200$ and rank  ${\bf r}=(10,10,10)$. OS is $10$. For simplicity, we compare {\it gradient descent} algorithms with Armijo backtracking linesearch for both the metric choices. Figure \ref{fig:syntheticReal}(a) shows that the algorithm with the metric (\ref{Eq:metric}) gives a superior performance \changeHK{in \emph{train error}} than that of the conventional metric choice. 

{\bf Case S2: small-scale instances.} Small-scale tensors of size $100 \times100 \times100$, $150 \times150\times150$, and $200\times200\times200$ and rank ${\bf r}=(10,10,10)$ are considered. OS is $\{10,20,30\}$. Figure \ref{fig:syntheticReal}(b) shows that our proposed algorithm has faster convergence than others. In Figure \ref{fig:syntheticReal}(c), the lowest test errors are obtained by our proposed algorithm and geomCG. 


{\bf Case S3: large-scale instances.} We consider large-scale tensors of size $3000 \times 3000 \times 3000$, $5000 \times 5000 \times 5000$, and $10000 \times 10000 \times 10000$ and ranks ${\bf r}=(5,5,5)$ and $(10,10,10)$.  OS is $10$. Our proposed algorithm outperforms geomCG in Figure \ref{fig:syntheticReal}(d).


{\bf Case S4: influence of low sampling.} We look into problem instances from scarcely sampled data, e.g., OS is $4$. The test requires completing a tensor of size $10000 \times 10000 \times 10000$ and rank ${\bf r}=(5,5,5)$. Figure \ref{fig:syntheticReal}(e) shows the superior performance of the proposed algorithm against geomCG. Whereas the test error increases for geomCG, it decreases for the proposed algorithm.  

{\bf Case S5: influence of ill-conditioning and low sampling.} We consider the problem instance of {\bf Case S4} with $\text{OS\ } = 5$. Additionally, for generating the instance, we impose a diagonal core $\mathbfcal{G}$ with exponentially decaying \emph{positive} values of condition numbers (CN)  $5$, $50$, and $100$. Figure \ref{fig:syntheticReal}(f) shows that the proposed algorithm outperforms geomCG for all the considered CN values. 

{\bf Case S6: influence of noise.} We evaluate the convergence properties of algorithms under the presence of noise by adding \emph{scaled} Gaussian noise $\mathbfcal{P}_{\Omega}(\mathbfcal{E})$ to $\mathbfcal{P}_{\Omega}(\mathbfcal{X}^\star)$ as in 
\cite{Kressner_BIT_2014_s}. 
The different noise levels are $\epsilon= \{10^{-4}, 10^{-6}, 10^{-8}, 10^{-10}, 10^{-12}\}$. In order to evaluate for $\epsilon=10^{-12}$, the stopping threshold on the MSE of the train set is lowered to $10^{-24}$. The tensor size and rank are same as in {\bf Case S4} and OS is $10$. Figure \ref{fig:syntheticReal}(g) shows that the test error for each $\epsilon$ is almost identical to the $\epsilon^2 \| \mathbfcal{P}_{\Omega} (\mathbfcal{X}^{\star})\|_F ^2$ 
\cite{Kressner_BIT_2014_s}, 
but our proposed algorithm converges faster than geomCG.

{\bf Case S7: \changeHK{rectangular} instances.} We consider instances where the dimensions and ranks along certain modes are different than others. Two cases are considered. Case (7.a) considers tensors size $20000 \times7000 \times 7000$, $30000 \times 6000 \times 6000$, and $40000 \times 5000\times 5000$ with rank ${\bf r}=(5,5,5)$. Case (7.b) considers a tensor of size $10000 \times10000 \times10000$ with ranks $(7,6,6)$, $(10,5,5)$, and $ (15,4,4)$. In all the cases, the proposed algorithm converges faster than geomCG as shown in Figure \ref{fig:syntheticReal}(h). 


{\bf Case R1: hyperspectral image.} We consider the hyperspectral image ``Ribeira" \cite{Foster_VisNero_2007_s} discussed in \cite{Signoretto_IEEESigProLetter_2011_s, Kressner_BIT_2014_s}. The tensor size is $1017 \times 1340 \times 33$, where each slice corresponds to a particular image measured at a different wavelength. As suggested in \cite{Signoretto_IEEESigProLetter_2011_s, Kressner_BIT_2014_s}, we resize it to $203 \times 268 \times 33$. We perform five random samplings of the pixels based on the OS values $11$ and $22$, corresponding to the rank {\bf r}\changeHK{=}$(15,15,6)$ adopted in \cite{Kressner_BIT_2014_s}. This set is further randomly split into $80$/$10$/$10$--train/validation/test partitions. The algorithms are stopped when the MSE on the validation set starts to increase. While $\text{OS}=22$ corresponds to the observation ratio of $10\%$ studied in \cite{Kressner_BIT_2014_s}, $\text{OS}=11$ considers a challenging scenario with the observation ratio of $5\%$. Figures \ref{fig:syntheticReal}(i) shows the good performance of our 
algorithm. Table \ref{tab:R1R2} compiles the results.

{\bf Case R2: MovieLens-10M\footnote{\url{http://grouplens.org/datasets/movielens/}.}.} This dataset contains $10000054$ ratings corresponding to $71567$ users and $10681$ movies. We split the time into $7$-days wide bins results, and finally, get a tensor of size $71567 \times 10681 \times 731$. The fraction of known entries is less than $0.002\%$. The completion task on this dataset reveals {\it periodicity} of the {\it latent} genres. We perform five random $80$/$10$/$10$--train/validation/test partitions. The maximum iteration threshold is set to $500$. In Table \ref{tab:R1R2}, our proposed algorithm consistently gives lower test errors than geomCG across different ranks. 
\begin{figure}[t]
\begin{tabular}{cc}
\hspace*{-0.2cm}\begin{minipage}{0.48\hsize}
\begin{center}
\includegraphics[width=\hsize]{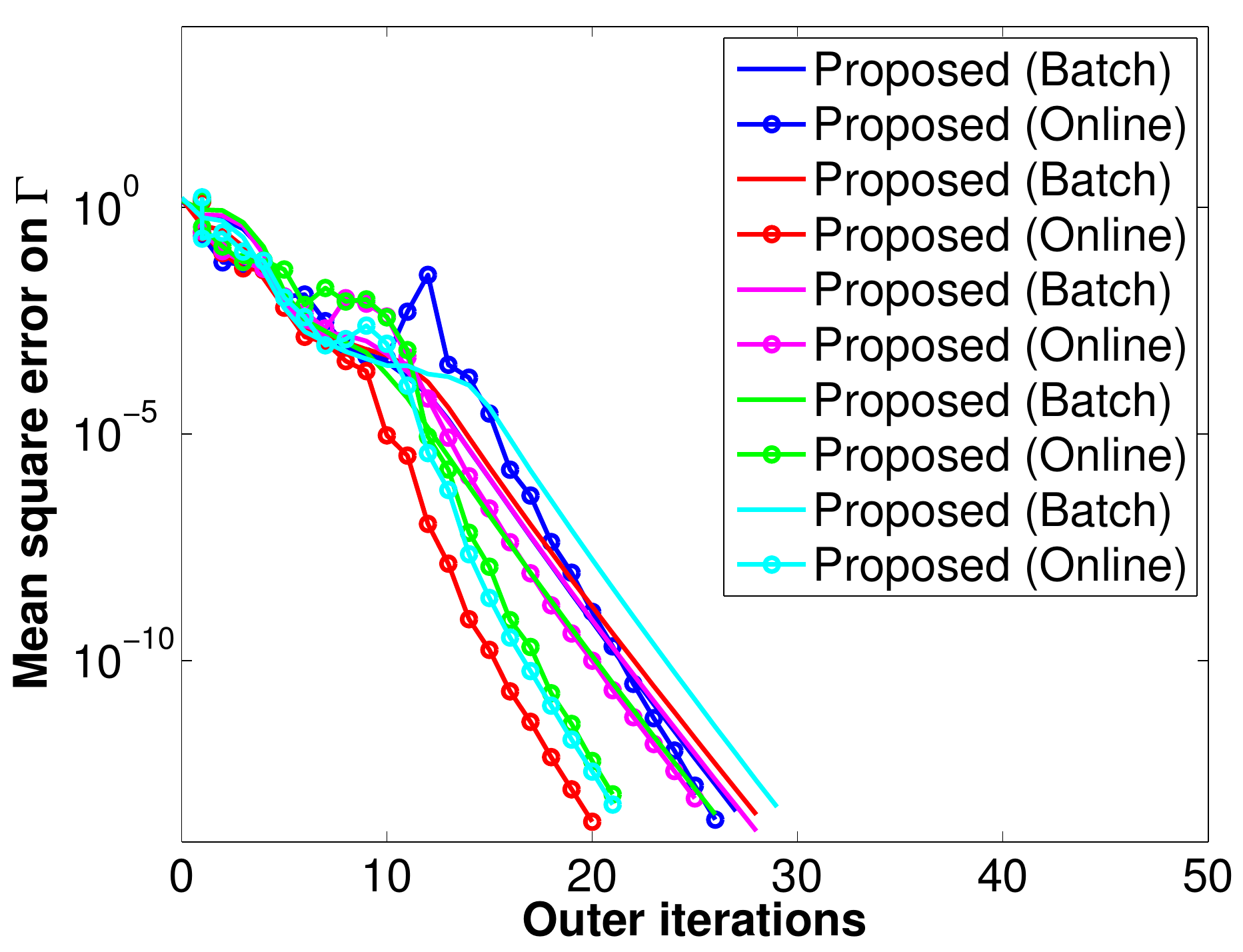}\\
\changeHK{{\scriptsize (a) {\bf Case O:} synthetic dataset.}}
\end{center}
\end{minipage}
\begin{minipage}{0.48\hsize}
\begin{center}
\includegraphics[width=\hsize]{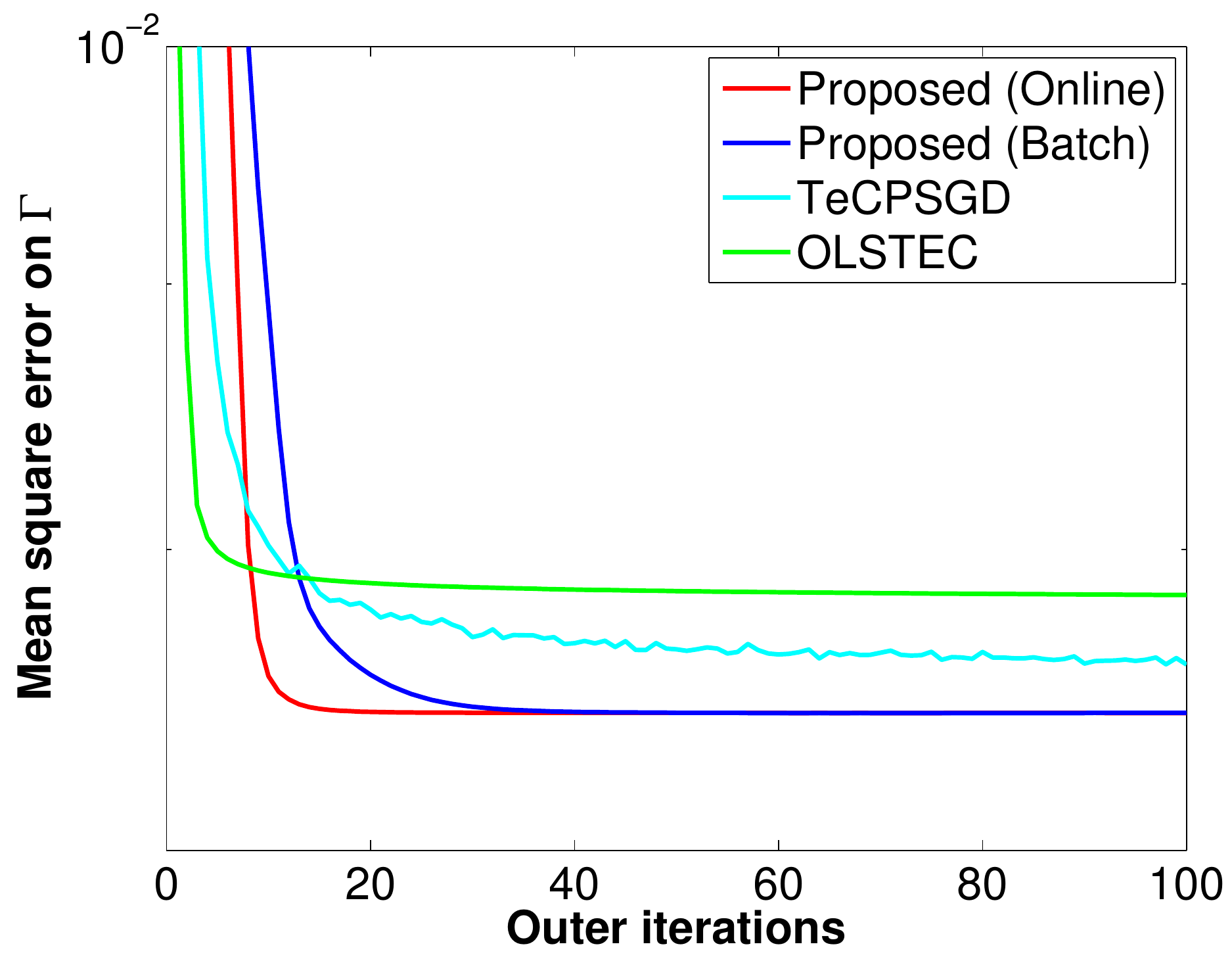}\\
\changeHK{{\scriptsize (b) {\bf Case O:} Airport Hall dataset.}}
\end{center}
\end{minipage}
\end{tabular}
\caption{\changeHK{Experiments on online \changeBM{instances}.}}
\label{fig:online}
\end{figure}

\changeHK{
{\bf Case O: online instances.}
We compare the proposed stochastic gradient descent algorithm with its batch counterpart gradient descent algorithm and with TeCPSGD \cite{Mardani_IEEETransSP_2015} and OLSTEC \cite{Kasai_IEEEICASSP_2016_s}. \changeBM{As the implementations of TeCPSGD and OLSTEC are computationally more intensive than ours, our plots only show test MSE against the number of \emph{outer iterations}, \changeBM{i.e.,} the number of the passes through the data}.}

 \changeBM{Figure \ref{fig:online}(a) shows comparisons on a synthetic instance of tensor size $100\times 100\times 10000$ with rank ${\bf r}=(5,5,5)$. $\gamma_0$ is selected from the step-size list $\{8,9,10,11,12\}$ in the pre-training phase. \changeBM{$10\%$ entries are randomly observed}. The pre-training uses $10\%$ frontal slices of all the slices. The maximum number of outer loops is set to $100$. Figure \ref{fig:online}(a) shows \changeBM{five different runs}, where the online algorithm has the same asymptotic convergence behavior as the batch counterpart on a test dataset $\Gamma$.} \changeBM{Figure \ref{fig:online}(b) shows comparisons on the Airport Hall surveillance video sequence dataset\footnote{\url{http://perception.i2r.a-star.edu.sg/bk_model/bk_index.html}} of size $176\times 144$ with $1000$ frames. $\gamma_0$ is selected from $\{30, 40, 50, 60, 70\}$ and \changeBM{$10\%$ frontal slices are selected for pre-training}. $2\%$ of the entries are observed. In Figure \ref{fig:online}(b), both the proposed online and batch algorithms achieve lower test errors than TeCPSGD and OLSTEC.}


\section{Conclusion}
\label{sec:Conclusion}
We have proposed preconditioned batch (conjugate gradient) and online (stochastic gradient descent) algorithms for the tensor completion problem. The algorithms stem from the Riemannian preconditioning approach that exploits the fundamental structures of symmetry (due to non-uniqueness of Tucker decomposition) and least-squares of the cost function. A novel Riemannian metric (inner product) is proposed that enables to use the versatile Riemannian optimization framework. Numerical comparisons suggest that our proposed algorithms have a superior performance on different benchmarks. 

\subsubsection*{Acknowledgments}
We thank Rodolphe Sepulchre, Paul Van Dooren, and Nicolas Boumal for useful discussions on the paper. This paper presents research results of the Belgian Network DYSCO (Dynamical Systems, Control, and Optimization), funded by the Interuniversity Attraction Poles Programme, initiated by the Belgian State, Science Policy Office. The scientific responsibility rests with its authors. Hiroyuki Kasai is (partly) supported by the Ministry of Internal Affairs and Communications, Japan, as the SCOPE Project (150201002). This work was initiated while Bamdev Mishra was with the Department of Electrical Engineering and Computer Science, University of Li\`ege, 4000 Li\`ege, Belgium and was visiting the Department of Engineering (Control Group), University of Cambridge, Cambridge, UK. He was supported as a research fellow (aspirant) of the Belgian National Fund for Scientific Research (FNRS).

\bibliographystyle{unsrt}
\bibliography{KMarXiv2016}

\begin{thebibliography}{27}
\providecommand{\natexlab}[1]{#1}
\providecommand{\url}[1]{\texttt{#1}}
\expandafter\ifx\csname urlstyle\endcsname\relax
  \providecommand{\doi}[1]{doi: #1}\else
  \providecommand{\doi}{doi: \begingroup \urlstyle{rm}\Url}\fi

\bibitem[Absil et~al.(2008)Absil, Mahony, and
  Sepulchre]{Absil_OptAlgMatManifold_2008}
Absil, P.-A., Mahony, R., and Sepulchre, R.
\newblock \emph{Optimization Algorithms on Matrix Manifolds}.
\newblock Princeton University Press, 2008.

\bibitem[Bonnabel(2013)]{bonnabel13a}
Bonnabel, S.
\newblock Stochastic gradient descent on {Riemannian} manifolds.
\newblock \emph{IEEE Trans. Autom. Control}, 58\penalty0 (9):\penalty0
  2217--2229, 2013.

\bibitem[Bottou(2012)]{Bottou_SGDTricks_2012_s}
Bottou, L.
\newblock Stochastic gradient descent tricks.
\newblock \emph{Neural Networks: Tricks of the Trade (2nd ed.)}, pp.\
  421--436, 2012.

\bibitem[Boumal \& Absil(2015)Boumal and Absil]{boumal15a}
Boumal, N. and Absil, P.-A.
\newblock Low-rank matrix completion via preconditioned optimization on the
  {Grassmann} manifold.
\newblock \emph{Linear Algebra Appl.}, 475:\penalty0 200--239, 2015.

\bibitem[Boumal et~al.(2014)Boumal, Mishra, Absil, and
  Sepulchre]{Boumal_Manopt_2014_s}
Boumal, N., Mishra, B., Absil, P.-A., and Sepulchre, R.
\newblock Manopt: a {Matlab} toolbox for optimization on manifolds.
\newblock \emph{J. Mach. Learn. Res.}, 15\penalty0 (1):\penalty0 1455--1459,
  2014.

\bibitem[Cand\`es \& Recht(2009)Cand\`es and
  Recht]{Candes_FoundCompuMath_2009_s}
Cand\`es, E.~J. and Recht, B.
\newblock Exact matrix completion via convex optimization.
\newblock \emph{Found. Comput. Math.}, 9\penalty0 (6):\penalty0 717--772, 2009.

\bibitem[Edelman et~al.(1998)Edelman, Arias, and Smith]{Edelman98a}
Edelman, A., Arias, T.A., and Smith, S.T.
\newblock The geometry of algorithms with orthogonality constraints.
\newblock \emph{SIAM J. Matrix Anal. Appl.}, 20\penalty0 (2):\penalty0
  303--353, 1998.

\bibitem[Filipovi\'c \& Juki\'c(2013)Filipovi\'c and
  Juki\'c]{Filipovi_MultiSysSigPro_2013_s}
Filipovi\'c, M. and Juki\'c, A.
\newblock {Tucker} factorization with missing data with application to
  low-n-rank tensor completion.
\newblock \emph{Multidim. Syst. Sign. P.}, 2013.
\newblock Doi: 10.1007/s11045-013-0269-9.

\bibitem[Foster et~al.(2007)Foster, Nascimento, and
  Amano]{Foster_VisNero_2007_s}
Foster, D.~H., Nascimento, S. M.~C., and Amano, K.
\newblock Information limits on neural identification of colored surfaces in
  natural scenes.
\newblock \emph{Visual Neurosci.}, 21\penalty0 (3):\penalty0 331--336, 2007.

\bibitem[Kasai(2016)]{Kasai_IEEEICASSP_2016_s}
Kasai, H.
\newblock Online low-rank tensor subspace tracking from incomplete data by {CP}
  decomposition using recursive least squares.
\newblock In \emph{IEEE ICASSP}, 2016.

\bibitem[Kasai \& Mishra(2015)Kasai and Mishra]{kasai_arXiv_2015}
Kasai, Hiroyuki and Mishra, Bamdev.
\newblock Riemannian preconditioning for tensor completion.
\newblock Technical report, arXiv preprint arXiv:1506.02159, 2015.

\bibitem[Kolda \& Bader(2009)Kolda and Bader]{Kolda_SIAMReview_2009_s}
Kolda, T.~G. and Bader, B.~W.
\newblock Tensor decompositions and applications.
\newblock \emph{SIAM Rev.}, 51\penalty0 (3):\penalty0 455--500, 2009.

\bibitem[Kressner et~al.(2014)Kressner, Steinlechner, and
  Vandereycken]{Kressner_BIT_2014_s}
Kressner, D., Steinlechner, M., and Vandereycken, B.
\newblock Low-rank tensor completion by {Riemannian} optimization.
\newblock \emph{BIT Numer. Math.}, 54\penalty0 (2):\penalty0 447--468, 2014.

\bibitem[Lee(2003)]{Lee03a}
Lee, J.~M.
\newblock \emph{Introduction to smooth manifolds}, volume 218 of \emph{Graduate
  Texts in Mathematics}.
\newblock Springer-Verlag, New York, second edition, 2003.

\bibitem[Liu et~al.(2013)Liu, Musialski, Wonka, and
  Ye]{Liu_IEEETransPAMI_2013_s}
Liu, J., Musialski, P., Wonka, P., and Ye, J.
\newblock Tensor completion for estimating missing values in visual data.
\newblock \emph{IEEE Trans. Pattern Anal. Mach. Intell.}, 35\penalty0
  (1):\penalty0 208--220, 2013.

\bibitem[Mardani et~al.(2015)Mardani, Mateos, and
  Giannakis]{Mardani_IEEETransSP_2015}
Mardani, M., Mateos, G., and Giannakis, G.B.
\newblock Subspace learning and imputation for streaming big data matrices and
  tensors.
\newblock \emph{IEEE Trans. Signal Process.}, 63\penalty0 (10):\penalty0 266
  --2677, 2015.

\bibitem[Mishra \& Sepulchre(2014)Mishra and Sepulchre]{Mishra_ICDC_2014_s}
Mishra, B. and Sepulchre, R.
\newblock {R3MC}: A {Riemannian} three-factor algorithm for low-rank matrix
  completion.
\newblock In \emph{IEEE CDC}, pp.\  1137--1142, 2014.

\bibitem[Mishra \& Sepulchre(2016)Mishra and Sepulchre]{Bamdev_arXiv_2014_s}
Mishra, B. and Sepulchre, R.
\newblock {Riemannian} preconditioning.
\newblock \emph{SIAM J. Optim.}, 26\penalty0 (1):\penalty0 635--660, 2016.

\bibitem[Ngo \& Saad(2012)Ngo and Saad]{Ngo_NIPS_2012_s}
Ngo, T. and Saad, Y.
\newblock Scaled gradients on {Grassmann} manifolds for matrix completion.
\newblock In \emph{NIPS}, pp.\  1421--1429, 2012.

\bibitem[Nocedal \& Wright(2006)Nocedal and
  Wright]{Nocedal_NumericalOpt_2006_s}
Nocedal, J. and Wright, S.~J.
\newblock \emph{Numerical Optimization}, volume Second Edition.
\newblock Springer, 2006.

\bibitem[Ring \& Wirth(2012)Ring and Wirth]{Ring_SIAMJOptim_2012_s}
Ring, W. and Wirth, B.
\newblock Optimization methods on {Riemannian} manifolds and their application
  to shape space.
\newblock \emph{SIAM J. Optim.}, 22\penalty0 (2):\penalty0 596--627, 2012.

\bibitem[Sato \& Iwai(2015)Sato and Iwai]{Sato15a}
Sato, H. and Iwai, T.
\newblock A new, globally convergent {R}iemannian conjugate gradient method.
\newblock \emph{Optimization}, 64\penalty0 (4):\penalty0 1011--1031, 2015.

\bibitem[Signoretto et~al.(2011)Signoretto, Plas, Moor, and
  Suykens]{Signoretto_IEEESigProLetter_2011_s}
Signoretto, M., Plas, R. V.~d., Moor, B.~D., and Suykens, J. A.~K.
\newblock Tensor versus matrix completion: A comparison with application to
  spectral data.
\newblock \emph{IEEE Signal Process. Lett.}, 18\penalty0 (7):\penalty0
  403--406, 2011.

\bibitem[Signoretto et~al.(2014)Signoretto, Dinh, Lathauwer, and
  Suykens]{Signoretto_MachineLearning_2014_s}
Signoretto, M., Dinh, Q.~T., Lathauwer, L.~D., and Suykens, J. A.~K.
\newblock Learning with tensors: a framework based on convex optimization and
  spectral regularization.
\newblock \emph{Mach. Learn.}, 94\penalty0 (3):\penalty0 303--351, 2014.

\bibitem[Tomioka et~al.(2011)Tomioka, Hayashi, and
  Kashima]{Tomioka_Latent_2011_s}
Tomioka, R., Hayashi, K., and Kashima, H.
\newblock Estimation of low-rank tensors via convex optimization.
\newblock Technical report, arXiv preprint arXiv:1010.0789, 2011.

\bibitem[Vandereycken(2013)]{Vandereycken_SIAMOpt_2013_s}
Vandereycken, B.
\newblock Low-rank matrix completion by {Riemannian} optimization.
\newblock \emph{SIAM J. Optim.}, 23\penalty0 (2):\penalty0 1214--1236, 2013.

\bibitem[Wen et~al.(2012)Wen, Yin, and Zhang]{Wen_MPC_2012_s}
Wen, Z., Yin, W., and Zhang, Y.
\newblock Solving a low-rank factorization model for matrix completion by a
  nonlinear successive over-relaxation.
\newblock \emph{Math Program. Comput.}, 4\penalty0 (4):\penalty0 333--361,
  2012.

\end{thebibliography}

\clearpage

%
%



\normalsize

\appendix

\renewcommand\thefigure{A.\arabic{figure}}  
\setcounter{figure}{0} 

\renewcommand\thetable{A.\arabic{table}}  
\setcounter{table}{0} 

\renewcommand{\theequation}{A.\arabic{equation}}
\setcounter{equation}{0}


%
%

\section{\changeHK{Proof and} derivation of manifold-related ingredients}
\label{Sup_sec:Derivationofmanifold-relatedingredients}
The concrete computations of the optimization-related ingredients presented in the paper are discussed below. 

The total space is $\mathcal{M}: = {\rm St}(r_1, n_1) \times {\rm St}(r_2, n_2) \times {\rm St}(r_3, n_3) \times \mathbb{R}^{r_1 \times r_2 \times r_3}$. Each element $x \in \mathcal{M}$ has the matrix representation $(\mat{U}_{1}, \mat{U}_{2}, \mat{U}_{3}, \mathbfcal{G})$. Invariance of Tucker decomposition under the transformation $(\mat{U}_{1}, \mat{U}_{2}, \mat{U}_{3}, \mathbfcal{G}) \mapsto  (  \mat{U}_{1}\mat{O}_{1}, \mat{U}_{2}\mat{O}_{2}, \mat{U}_{3}\mat{O}_{3}, \mathbfcal{G} {\times_1} \mat{O}^T_{1} {\times_2} \mat{O}^T_{2} {\times_3} \mat{O}^T_{3})$ for all $\mat{O}_{d} \in \mathcal{O}(r_d)$, the set of orthogonal matrices of size of $r_d \times r_d$ results in equivalence classes of the form $[x] = [(\mat{U}_{1}, \mat{U}_{2}, \mat{U}_{3}, \mathbfcal{G})] := \{( \mat{U}_{1}\mat{O}_{1}, \mat{U}_{2}\mat{O}_{2}, \mat{U}_{3}\mat{O}_{3}, \mathbfcal{G} {\times_1} \mat{O}^T_{1} {\times_2} \mat{O}^T_{2} {\times_3} \mat{O}^T_{3}) : \mat{O}_{d} \in \mathcal{O}(r_d)\}$.

\subsection{Tangent space characterization and the Riemannian metric}
\label{Sup_Sec:TangentspaceandanewRiemannian metric}

The tangent space, $T_{{x}} \mathcal{M}$, at $x$ given by $(\mat{U}_{1}, \mat{U}_{2}, \mat{U}_{3}, \mathbfcal{G})$ in the total space $\mathcal{M}$ is the product space of the tangent spaces of the individual manifolds. From 
\cite{Absil_OptAlgMatManifold_2008}, 
the tangent space has the matrix characterization 
\begin{equation}
\begin{array}{lll}
\label{Sup_Eq:tangent_space}
T_{{x}} {\mathcal{M}} 
& = &  \{ (\mat{Z}_{{\bf U}_{1}}, \mat{Z}_{{\bf U}_{2}}, \mat{Z}_{{\bf U}_{3}}, \mat{Z}_{\mathbfcal{G}}) \in 
\mathbb{R}^{n_1 \times r_1} \times  
\mathbb{R}^{n_2 \times r_2} \times 
\mathbb{R}^{n_3 \times r_3} \times 
\mathbb{R}^{r_1 \times r_2 \times r_3}  \\
& &
: \mat{U}_{d}^T \mat{Z}_{{\bf U}_{d}} +  \mat{Z}_{{\bf U}_{d}}^T \mat{U}_{d} = 0,\ {\rm for\ } d \in \{1,2,3 \} \}.
\end{array}
\end{equation}
The proposed metric ${g}_{x}:T_x \mathcal{M} \times T_x \mathcal{M} \rightarrow \mathbb{R}$ is
\begin{eqnarray}
\label{Sup_Eq:metric}
{g}_{x}(\xi_{x}, \eta_{x}) 
&=&  
\langle \xi_{\scriptsize \mat{U}_{1}},
{\eta}_{\scriptsize\mat{U}_{1}} (\mat{G}_{1} \mat{G}_{1}^T) \rangle +
\langle \xi_{\scriptsize \mat{U}_{2}},
{\eta}_{\scriptsize\mat{U}_{2}} (\mat{G}_{2} \mat{G}_{2}^T) \rangle
+
\langle \xi_{\scriptsize \mat{U}_{3}},
{\eta}_{\scriptsize\mat{U}_{3}} (\mat{G}_{3} \mat{G}_{3}^T) \rangle 
+ \langle {\xi}_{\scriptsize \mathbfcal{G}}, {\eta}_{\scriptsize \mathbfcal{G}}\rangle ,
\end{eqnarray}
where ${\xi}_{{x}}, {\eta}_{{x}} \in T_{{x}} {\mathcal{M}}$ are 
tangent vectors with matrix characterizations 
$({\xi}_{\scriptsize \mat{U}_{1}}, {\xi}_{\scriptsize \mat{U}_{2}}
, {\xi}_{\scriptsize \mat{U}_{3}}, {\xi}_{\scriptsize \mathbfcal{G}})$ and
$({\eta}_{\scriptsize \mat{U}_{1}}, {\eta}_{\scriptsize \mat{U}_{2}}, \\{\eta}_{\scriptsize \mat{U}_{3}}, {\eta}_{\scriptsize \mathbfcal{G}})$, respectively and $\langle \cdot, \cdot \rangle$ is the Euclidean inner product.

\subsection{Characterization of the normal space}
\label{Sup_Sec:normalspace}

Given a vector in $\mathbb{R}^{ n_1 \times r_1} \times \mathbb{R}^{ n_2 \times r_2} \times \mathbb{R}^{ n_3 \times r_3} \times \mathbb{R}^{r_1 \times r_2 \times r_3}$, its projection onto the tangent space $T_{{x}} \mathcal{M}$ is obtained by extracting the component {\it normal}, in the metric sense, to the tangent space. This section describes the characterization of the {\it normal space}, $N_x \mathcal{M}$.

Let $\zeta_{x} = (\zeta_{\scriptsize \mat{U}_1}, \zeta_{\scriptsize \mat{U}_2}, \zeta_{\scriptsize \mat{U}_3}, \zeta_{\scriptsize \mathbfcal{G}}) \in N_{x} \mathcal{M}$, and $\eta_{{x}} = (\eta_{\scriptsize \mat{U}_1}, \eta_{\scriptsize \mat{U}_2}, \eta_{\scriptsize \mat{U}_3}, \eta_{\scriptsize \mathbfcal{G}}) \in T_{{x}} \mathcal{M}$. Since $\zeta_{{x}}$ is orthogonal to $\eta_{{x}}$, i.e., $g_{{x}}(\zeta_{{x}}, \eta_{{x}})=0$, the conditions
\begin{eqnarray}
\label{Sup_Eq:TraceConditions}
{\rm Trace}( \mat{G}_{d}  \mat{G}_{d}^T  \zeta_{\scriptsize \mat{U}_d}^T  \eta_{\scriptsize \mat{U}_d}) = 0, \ {\rm for\ } d \in \{1,2,3 \}
\end{eqnarray}
must hold for all $\eta_{{x}}$ in the tangent space. Additionally from 
\cite{Absil_OptAlgMatManifold_2008}, 
$\eta_{\scriptsize \mat{U}_d}$ has the characterization
\begin{eqnarray}
\label{Sup_Eq:alternateTangentSpace}
\eta_{\scriptsize \mat{U}_d} & = & \mat{U}_d {\bf \Omega} + {\mat{U}_d}_{\perp} \mat{K},
\end{eqnarray}
where ${\bf \Omega}$ is any skew-symmetric matrix, $\mat{K}$ is a any matrix of size ${(n_d-r_d) \times r_d}$, and ${\mat{U}_d}_{\perp}$ is any $n_d \times (n_d - r_d)$ that is orthogonal complement of $\mat{U}_d$. Let $\tilde{\zeta}_{\scriptsize \mat{U}_d} =  \zeta_{\scriptsize \mat{U}_d} \mat{G}_{d}  \mat{G}_{d}^T $ and let $\tilde{\zeta}_{\scriptsize \mat{U}_d}$ is defined as 
\begin{eqnarray}
\label{Sup_Eq:DefVariableChange}
\tilde{\zeta}_{\scriptsize \mat{U}_d} &=& \mat{U}_d \mat{A} + {\mat{U}_d}_{\perp} \mat{B}
\end{eqnarray}
without loss of generality, where $\mat{A} \in \mathbb{R}^{r_d \times r_d}$ and $\mat{B} \in \mathbb{R}^{(n_d-r_d) \times r_d}$ are to be characterized from (\ref{Sup_Eq:TraceConditions}) and (\ref{Sup_Eq:alternateTangentSpace}). A few standard computations show that $\mat{A}$ has to be symmetric and $\mat{B} = \mat{0}$. Consequently,
$\tilde{\zeta}_{\scriptsize \mat{U}_d} = \mat{U}_d \mat{S}_{\scriptsize \mat{U}_d}$, where $\mat{S}_{\scriptsize \mat{U}_d} = \mat{S}_{\scriptsize \mat{U}_d}^T$. Equivalently, $\zeta_{\scriptsize \mat{U}_d}  = \mat{U}_d \mat{S}_{\scriptsize \mat{U}_d} (\mat{G}_{d}  \mat{G}_{d}^T)^{-1}$ for a symmetric matrix $\mat{S}_{\scriptsize \mat{U}_d}$. Finally, the normal space $N_x \mathcal{M}$ has the characterization
\begin{equation}
\begin{array}{lll}
\label{Sup_Eq:NormalSpace}
N_x \mathcal{M}
 & = & \{(\mat{U}_{1}\mat{S}_{\scriptsize \mat{U}_{1}}(\mat{G}_{1}  \mat{G}_{1}^T)^{-1},
\mat{U}_{2}\mat{S}_{\scriptsize \mat{U}_{2}}(\mat{G}_{2}  \mat{G}_{2}^T)^{-1}, 
\mat{U}_{3}\mat{S}_{\scriptsize \mat{U}_{3}}(\mat{G}_{3}  \mat{G}_{3}^T)^{-1}, 0)\\
&  &  :  \mat{S}_{\scriptsize \mat{U}_{d}}  \in \mathbb{R}^{r_d \times r_d},  \mat{S}_{\scriptsize \mat{U}_{d}}^T  = \mat{S}_{\scriptsize \mat{U}_{d}},
  \text{ for\ } d \in \{1, 2, 3\} \}.
\end{array}
\end{equation}

\subsection{Characterization of the vertical space}
\label{Sup_Sec:verticalspace}

The horizontal space projector of a tangent vector is obtained by removing the component along the vertical direction. This section shows the matrix characterization of the vertical space $\mathcal{V}_x$.

$\mathcal{V}_{x}$ is the defined as the linearization of the equivalence class $ [(\mat{U}_1, \mat{U}_2, \mat{U}_3, \mathbfcal{G})]$ at $x = [(\mat{U}_1, \mat{U}_2, \mat{U}_3, \mathbfcal{G})]$. Equivalently, $\mathcal{V}_{x}$ is the linearization of $(
 \mat{U}_1\mat{O}_{1}, \mat{U}_2\mat{O}_{2}, \mat{U}_3\mat{O}_{3}, \mathbfcal{G}_{\times 1} \mat{O}^T_{1}{_{\times 2}} \mat{O}^T_{2}{_{\times 3}} \mat{O}^T_{3})$ along $\mat{O}_{d} \in \mathcal{O}(r_d)$ at the \emph{identity element} for $d \in \{1, 2,3 \}$. From the characterization of linearization of an orthogonal matrix 
 \cite{Absil_OptAlgMatManifold_2008}, 
 we have the characterization for the vertical space as
\begin{equation}
\begin{array}{lll}
\label{Sup_Eq:VerticalComponents}
\mathcal{V}_{x} & = &\{ (\mat{U}_1{\bf \Omega}_1, \mat{U}_2 {\bf \Omega}_2, \mat{U}_3 {\bf \Omega}_3, 
  - (\mathbfcal{G}{\times_1} {\bf \Omega}_1  + 
\mathbfcal{G}{{\times_2}} {\bf \Omega}_2 +
\mathbfcal{G}{{\times_3}} {\bf \Omega}_3)):\\
&   & {\bf \Omega}_d \in \mathbb{R}^{r_d \times r_d}, {\bf \Omega}_d ^T = -{\bf \Omega}_d \text{ for\ } d \in \{1, 2, 3\} \}.
\end{array}
\end{equation}

\subsection{Characterization of the horizontal space}
\label{Sup_Sec:horizontalspace}

The characterization of the horizontal space $\mathcal{H}_{{x}}$ is derived from its orthogonal relationship with the vertical space $\mathcal{V}_{{x}}$.

Let $\xi_{{x}}=(\xi_{{\bf U}_1}, \xi_{{\bf U}_2}, \xi_{{\bf U}_3}, \xi_{\mathbfcal{G}})$  $\in \mathcal{H}_{{x}}$, and $\zeta_{{x}}=(\zeta_{{\bf U}_1}, \zeta_{{\bf U}_2}, \zeta_{{\bf U}_3}, \zeta_{\mathbfcal{G}})$ $\in \mathcal{V}_{{x}}$. Since $\xi_{{x}}$ must be orthogonal to $\zeta_{{x}}$, which is equivalent to $g_{{x}}(\xi_{{x}}, \zeta_{{x}})=0$ in (\ref{Sup_Eq:metric}), the characterization for $\xi_{{x}}$ is derived from (\ref{Sup_Eq:metric}) and  (\ref{Sup_Eq:VerticalComponents}).

\begin{eqnarray*}
{g}_{x}(\xi_{{x}}, \zeta_{{x}}) 
&=& 
\langle \xi_{\scriptsize \mat{U}_1},
\zeta_{\scriptsize \mat{U}_1} (\mat{G}_{1} \mat{G}_{1}^T) \rangle
+\langle \xi_{\scriptsize \mat{U}_2}, 
\zeta_{\scriptsize \mat{U}_2} (\mat{G}_{2} \mat{G}_{2}^T) \rangle 
 + \langle \xi_{\scriptsize \mat{U}_3}, 
\zeta_{\scriptsize \mat{U}_3} (\mat{G}_{3} \mat{G}_{3}^T) \rangle
+ \langle \xi_{\mathbfcal{G}},  \zeta_{\mathbfcal{G}} \rangle
 \\
&=& \langle {\xi}_{\scriptsize \mat{U}},
( \mat{U}_1 {\bf \Omega}_1) (\mat{G}_{1} \mat{G}_{1}^T) \rangle
 +\langle {\xi}_{\scriptsize \mat{U}_2}, 
( \mat{U}_2 {\bf \Omega}_2) (\mat{G}_{2} \mat{G}_{2}^T) \rangle 
 + \langle {\xi}_{\scriptsize \mat{U}_3}, 
( \mat{U}_3 {\bf \Omega}_3)(\mat{G}_{3} \mat{G}_{3}^T) \rangle  \\
&  & \hspace{0cm}+ \langle \xi_{\mathbfcal{G}}, 
 - (\mathbfcal{G}{\times_1} {\bf \Omega}_1  + 
\mathbfcal{G}{{\times_2}} {\bf \Omega}_2 +
\mathbfcal{G}{{\times_3}} {\bf \Omega}_3) \rangle \\
 & & \text{(Switch to unfoldings of\ } \mathbfcal{G}.) \\
 &=& 
{\rm Trace}( (\mat{G}_{1}  \mat{G}_{1}^T) \xi_{\scriptsize  \mat{U}_1}^T ( \mat{U}_1 {\bf \Omega}_1) ) 
+{\rm Trace}( (\mat{G}_{2}  \mat{G}_{2}^T) \xi_{\scriptsize  \mat{U}_2}^T ( \mat{U}_2 {\bf \Omega}_2) )  
+{\rm Trace}( (\mat{G}_{3}  \mat{G}_{3}^T) \xi_{\scriptsize  \mat{U}_3}^T ( \mat{U}_3 {\bf \Omega}_3) )  \\
&& + {\rm Trace} (\xi_{\scriptsize \mat{G}_{1}} (- {\bf \Omega}_1 \mat{G}_{1})^T ) 
+ {\rm Trace} (\xi_{\scriptsize \mat{G}_{2}} (- {\bf \Omega}_2 \mat{G}_{2})^T )
+ {\rm Trace} (\xi_{\scriptsize \mat{G}_{3}} (- {\bf \Omega}_3 \mat{G}_{3})^T ) \\
&=&
 {\rm Trace} \left[ \left\{(\mat{G}_{1}  \mat{G}_{1}^T) \xi_{\scriptsize  \mat{U}_1}^T \mat{U}_1 + \xi_{\scriptsize \mat{G}_{1}}  \mat{G}_{1}^T \right\} {\bf \Omega}_1 \right] 
 +{\rm Trace} \left[ \left\{(\mat{G}_{2}  \mat{G}_{2}^T) \xi_{\scriptsize  \mat{U}_2}^T  \mat{U}_2  + \xi_{\scriptsize \mat{G}_{2}}  \mat{G}_{2}^T \right\} {\bf \Omega}_2 \right]    \\
&&  \hspace{0cm} +{\rm Trace} \left[ \left\{(\mat{G}_{3}  \mat{G}_{3}^T) \xi_{\scriptsize  \mat{U}_3}^T  \mat{U}_3 + \xi_{\scriptsize \mat{G}_{3}}  \mat{G}_{3}^T \right\} {\bf \Omega}_3 \right] , 
\end{eqnarray*}
where $\xi_{\scriptsize \mat{G}_{d}}$ is the mode-$d$ unfolding of $\xi_{\mathbfcal{G}}$.
Since ${g}_{{x}}(\xi_{{x}}, \zeta_{{x}})$ above should be zero for all skew-matrices ${\bf \Omega}_d$, $\xi_{{x}}=(\xi_{{\bf U}_1}, \xi_{{\bf U}_2}, \xi_{{\bf U}_3}, \xi_{\mathbfcal{G}})$  $\in \mathcal{H}_{{x}}$ must satisfy 
\begin{equation}
\begin{array}{l}
\label{Sup_Eq:horizontal_space_reqrements}
\hspace*{0cm}(\mat{G}_{d}  \mat{G}_{d}^T) \xi_{\scriptsize \mat{U}_d}^T \mat{U}_d  + \xi_{\scriptsize \mat{G}_{d}}  \mat{G}_{d}^T \quad \text{is symmetric for } 
 d \in  \{ 1, 2, 3 \} .
\end{array}
\end{equation}

%
%
%
%




\subsection{\changeHK{Proof of Proposition 1}}
\label{Sup_sec:}

We first introduce the following lemma:
\begin{lemma}
Let 
$(\mat{U}_{1}, \mat{U}_{2}, \mat{U}_{3}, \mathbfcal{G}) \in {\rm St}(r_1, n_1) \times {\rm St}(r_2, n_2) \times {\rm St}(r_3, n_3) \times \mathbb{R}^{r_1 \times r_2 \times r_3}$ and 
$\xi_{\scriptsize [(\mat{U}_{1}, \mat{U}_{2}, \mat{U}_{3}, \mathbfcal{G})]}$ 
be a tangent vector to the quotient manifold at $[(\mat{U}_{1}, \mat{U}_{2}, \mat{U}_{3}, \mathbfcal{G})]$. 
The horizontal lifts of $\xi_{\scriptsize [(\mat{U}_{1}, \mat{U}_{2}, \mat{U}_{3}, \mathbfcal{G})]}$ at $(\mat{U}_{1}, \mat{U}_{2}, \mat{U}_{3}, \mathbfcal{G})$ 
and $(\mat{U}_{1}\mat{O}_{1}, \mat{U}_{2}\mat{O}_{2}, \mat{U}_{3}\mat{O}_{3}, \mathbfcal{G} {\times_1} \mat{O}^T_{1} {\times_2} \mat{O}^T_{2} {\times_3} \mat{O}^T_{3})$ are related for $\mat{O}_{d} \in \mathcal{O}(r_d)$ as follows,
\begin{equation}
\label{Sup_Eq:HoritaonLiftRelation}
\begin{array}{lll}
\hspace*{-0.5cm}(\xi_{\scriptsize \mat{U}_{1}\mat{O}_{1}}, \xi_{\scriptsize \mat{U}_{2}\mat{O}_{2}}, \xi_{\scriptsize \mat{U}_{3}\mat{O}_{3}}, 
\xi_{\scriptsize \mathbfcal{G} {\times_1} \mat{O}^T_{1} {\times_2} \mat{O}^T_{2} {\times_3} \mat{O}^T_{3}}) 
&= &(\xi_{\scriptsize \mat{U}_{1}}\mat{O}_{1}, \xi_{\scriptsize \mat{U}_{2}}\mat{O}_{2}, \xi_{\scriptsize \mat{U}_{3}}\mat{O}_{3}, 
\xi_{\scriptsize \mathbfcal{G}} {\times_1} \mat{O}^T_{1} {\times_2} \mat{O}^T_{2} {\times_3} \mat{O}^T_{3}).
\end{array}
\end{equation}
\end{lemma}
\begin{proof}
Let $f: ({\rm St}(r_1, n_1) \times {\rm St}(r_2, n_2) \times {\rm St}(r_3, n_3) \times \mathbb{R}^{r_1 \times r_2 \times r_3}/(\mathcal{O}{(r_1)} \times \mathcal{O}{(r_2)} \times \mathcal{O}{(r_3)})) \rightarrow \mathbb{R}$ be an arbitrary smooth function, and define 
\begin{equation*}
\begin{array}{ll}
	\bar{f} &  :=   f \circ \pi :
	 ({\rm St}(r_1, n_1) \times {\rm St}(r_2, n_2) \times {\rm St}(r_3, n_3) \times \mathbb{R}^{r_1 \times r_2 \times r_3}
	/(\mathcal{O}{(r_1)} \times \mathcal{O}{(r_2)} \times \mathcal{O}{(r_3)})) \rightarrow \mathbb{R},
\end{array}	
\end{equation*}
\changeHKKK{where $\pi$ is the mapping $\pi: \mathcal{M} \rightarrow \mathcal{M}/\sim$ defined by $x \mapsto [x]$.}

Consider the mapping 
\begin{equation*}
\begin{array}{l}
h : (\mat{U}_{1}, \mat{U}_{2}, \mat{U}_{3}, \mathbfcal{G}) 
\mapsto  (\mat{U}_{1}\mat{O}_{1}, \mat{U}_{2}\mat{O}_{2}, \mat{U}_{3}\mat{O}_{3}, \mathbfcal{G} {\times_1} \mat{O}^T_{1} {\times_2} \mat{O}^T_{2} {\times_3} \mat{O}^T_{3}), 
\end{array}
\end{equation*}
where $\mat{O}_{d} \in \mathcal{O}(r_d)$. Since $\pi(h(\mat{U}_{1}, \mat{U}_{2}, \mat{U}_{3}, \mathbfcal{G}))=\pi(\mat{U}_{1}, \mat{U}_{2}, \mat{U}_{3}, \mathbfcal{G})$ for all $(\mat{U}_{1}, \mat{U}_{2}, \mat{U}_{3}, \mathbfcal{G})$, we have
\begin{equation*}
	\bar{f}(h(\mat{U}_{1}, \mat{U}_{2}, \mat{U}_{3}, \mathbfcal{G}))= \bar{f}(\mat{U}_{1}, \mat{U}_{2}, \mat{U}_{3}, \mathbfcal{G}).
\end{equation*}

By taking the differential of both sides,
\begin{equation}
\label{Sup_Eq:Diff}
\begin{array}{ll}
&\hspace*{-1cm} \changeHKKK{{\rm D}} \bar{f}(h(\mat{U}_{1}, \mat{U}_{2}, \mat{U}_{3}, \mathbfcal{G}))
[\changeHKKK{{\rm D}} h(\mat{U}_{1}, \mat{U}_{2}, \mat{U}_{3}, \mathbfcal{G})[(\xi_{\scriptsize \mat{U}_{1}}, \xi_{\scriptsize \mat{U}_{2}}, \xi_{\scriptsize \mat{U}_{3}}, \xi_{\scriptsize \mathbfcal{G}})]]
= \changeHKKK{{\rm D}} \bar{f}(\mat{U}_{1}, \mat{U}_{2}, \mat{U}_{3}, \mathbfcal{G})[(\xi_{\scriptsize \mat{U}_{1}}, \xi_{\scriptsize \mat{U}_{2}}, \xi_{\scriptsize \mat{U}_{3}}, \xi_{\scriptsize \mathbfcal{G}})].
\end{array}
\end{equation}

By \changeHKKK{noting the} definition of $(\xi_{\scriptsize \mat{U}_{1}}, \xi_{\scriptsize \mat{U}_{2}}, \xi_{\scriptsize \mat{U}_{3}}, \xi_{\scriptsize \mathbfcal{G}})$, i.e., $\changeHKKK{{\rm D}} \pi(\mat{U}_{1}, \mat{U}_{2}, \mat{U}_{3}, \mathbfcal{G})[\xi_{\scriptsize \mat{U}_{1}}, \xi_{\scriptsize \mat{U}_{2}}, \xi_{\scriptsize \mat{U}_{3}}, \xi_{\scriptsize \mathbfcal{G}}] = \xi_{\scriptsize [(\xi_{\scriptsize \mat{U}_{1}}, \xi_{\scriptsize \mat{U}_{2}},\\ \xi_{\scriptsize \mat{U}_{3}}, \xi_{\scriptsize \mathbfcal{G}})]}$, \changeHKKK{the right side of (\ref{Sup_Eq:Diff}) is}
\begin{equation*}
\begin{array}{lll}
	\changeHKKK{{\rm D}} \bar{f}(\mat{U}_{1}, \mat{U}_{2}, \mat{U}_{3}, \mathbfcal{G})
	[(\xi_{\scriptsize \mat{U}_{1}}, \xi_{\scriptsize \mat{U}_{2}}, \xi_{\scriptsize \mat{U}_{3}}, \xi_{\scriptsize \mathbfcal{G}})]
	&=&\changeHKKK{{\rm D}} f ( \pi (\mat{U}_{1}, \mat{U}_{2}, \mat{U}_{3}, \mathbfcal{G}))[\changeHKKK{{\rm D}} \pi (\mat{U}_{1}, \mat{U}_{2}, \mat{U}_{3}, \mathbfcal{G})[(\xi_{\scriptsize \mat{U}_{1}}, \xi_{\scriptsize \mat{U}_{2}}, \xi_{\scriptsize \mat{U}_{3}}, \xi_{\scriptsize \mathbfcal{G}})]\\
	&=& \changeHKKK{{\rm D}} f ( \pi (\mat{U}_{1}, \mat{U}_{2}, \mat{U}_{3}, \mathbfcal{G}))[\xi_{\scriptsize [(\mat{U}_{1}, \mat{U}_{2}, \mat{U}_{3}, \mathbfcal{G})]}],
\end{array}
\end{equation*}
where \changeHKKK{the chain rule} is applied to the \changeHKKK{first} equality. 

Moreover, from the directional derivatives of the mapping $h$, \changeHKKK{the bracket of the left side of (\ref{Sup_Eq:Diff}) is obtained as}
\begin{equation*}
\begin{array}{lll}
	\changeHKKK{{\rm D}} h(\mat{U}_{1}, \mat{U}_{2}, \mat{U}_{3}, \mathbfcal{G})
	[(\xi_{\scriptsize \mat{U}_{1}}, \xi_{\scriptsize \mat{U}_{2}}, \xi_{\scriptsize \mat{U}_{3}}, \xi_{\scriptsize \mathbfcal{G}})]
	&=& (\xi_{\scriptsize \mat{U}_{1}}\mat{O}_{1}, \xi_{\scriptsize \mat{U}_{2}}\mat{O}_{2}, \xi_{\scriptsize \mat{U}_{3}}\mat{O}_{3}, \xi_{\scriptsize \mathbfcal{G}} {\times_1} \mat{O}^T_{1} {\times_2} \mat{O}^T_{2} {\times_3} \mat{O}^T_{3}).
\end{array}
\end{equation*}
Therefore, (\ref{Sup_Eq:Diff}) yields
\begin{equation}
\begin{array}{l}
\label{Sup_Eq:Diff_new}
\changeHKKK{{\rm D}} \bar{f}(\mat{U}_{1}\mat{O}_{1}, \mat{U}_{2}\mat{O}_{2}, \mat{U}_{3}\mat{O}_{3}, \mathbfcal{G} {\times_1} \mat{O}^T_{1} {\times_2} \mat{O}^T_{2} {\times_3} \mat{O}^T_{3}) 
[(\xi_{\scriptsize \mat{U}_{1}}\mat{O}_{1}, \xi_{\scriptsize \mat{U}_{2}}\mat{O}_{2}, \xi_{\scriptsize \mat{U}_{3}}\mat{O}_{3}, \xi_{\scriptsize \mathbfcal{G}} {\times_1} \mat{O}^T_{1} {\times_2} \mat{O}^T_{2} {\times_3} \mat{O}^T_{3})]\\
\hspace*{1cm}= \changeHKKK{{\rm D}} f(\pi (\mat{U}_{1}\mat{O}_{1}, \mat{U}_{2}\mat{O}_{2}, \mat{U}_{3}\mat{O}_{3}, \mathbfcal{G} {\times_1} \mat{O}^T_{1} {\times_2} \mat{O}^T_{2} {\times_3} \mat{O}^T_{3}))
[\xi_{\scriptsize [(\mat{U}_{1}, \mat{U}_{2}, \mat{U}_{3}, \mathbfcal{G})]}],
\end{array}
\end{equation}
\changeHKKK{where we address the equivalence class $\pi(\mat{U}_{1}, \mat{U}_{2}, \mat{U}_{3}, \mathbfcal{G})=\pi(\mat{U}_{1}\mat{O}_{1}, \mat{U}_{2}\mat{O}_{2}, \mat{U}_{3}\mat{O}_{3}, \mathbfcal{G} {\times_1} \mat{O}^T_{1} {\times_2} \mat{O}^T_{2} {\times_3} \mat{O}^T_{3})$.}
\changeHKKK{The left side of (\ref{Sup_Eq:Diff_new}) is further transformed by the chain rule as
\begin{equation}
\begin{array}{l}
\label{Sup_Eq:Diff_new_left}
{\rm D} \bar{f}(\mat{U}_{1}\mat{O}_{1}, \mat{U}_{2}\mat{O}_{2}, \mat{U}_{3}\mat{O}_{3}, \mathbfcal{G} {\times_1} \mat{O}^T_{1} {\times_2} \mat{O}^T_{2} {\times_3} \mat{O}^T_{3}) 
[(\xi_{\scriptsize \mat{U}_{1}}\mat{O}_{1}, \xi_{\scriptsize \mat{U}_{2}}\mat{O}_{2}, \xi_{\scriptsize \mat{U}_{3}}\mat{O}_{3}, \xi_{\scriptsize \mathbfcal{G}} {\times_1} \mat{O}^T_{1} {\times_2} \mat{O}^T_{2} {\times_3} \mat{O}^T_{3})]\\
\hspace*{0.1cm}= {\rm D} f( \pi(\mat{U}_{1}\mat{O}_{1}, \mat{U}_{2}\mat{O}_{2}, \mat{U}_{3}\mat{O}_{3}, \mathbfcal{G} {\times_1} \mat{O}^T_{1} {\times_2} \mat{O}^T_{2} {\times_3} \mat{O}^T_{3})) 
[{\rm D}  \pi(\mat{U}_{1}\mat{O}_{1}, \mat{U}_{2}\mat{O}_{2}, \mat{U}_{3}\mat{O}_{3}, \mathbfcal{G} {\times_1} \mat{O}^T_{1} {\times_2} \mat{O}^T_{2} {\times_3} \mat{O}^T_{3})]\\
\hspace*{2cm}[(\xi_{\scriptsize \mat{U}_{1}}\mat{O}_{1}, \xi_{\scriptsize \mat{U}_{2}}\mat{O}_{2}, \xi_{\scriptsize \mat{U}_{3}}\mat{O}_{3}, \xi_{\scriptsize \mathbfcal{G}} {\times_1} \mat{O}^T_{1} {\times_2} \mat{O}^T_{2} {\times_3} \mat{O}^T_{3})].
\end{array}
\end{equation}
By comparing the right sides of (\ref{Sup_Eq:Diff_new}) and (\ref{Sup_Eq:Diff_new_left})}, 
since this equality holds for any smooth function $f$, it implies that 
\begin{equation}
\begin{array}{l}
\changeHKKK{{\rm D}} \pi (\mat{U}_{1}\mat{O}_{1}, \mat{U}_{2}\mat{O}_{2}, \mat{U}_{3}\mat{O}_{3}, \mathbfcal{G} {\times_1} \mat{O}^T_{1} {\times_2} \mat{O}^T_{2} {\times_3} \mat{O}^T_{3})
[(\xi_{\scriptsize \mat{U}_{1}}\mat{O}_{1}, \xi_{\scriptsize \mat{U}_{2}}\mat{O}_{2}, \xi_{\scriptsize \mat{U}_{3}}\mat{O}_{3}, \xi_{\scriptsize \mathbfcal{G}} {\times_1} \mat{O}^T_{1} {\times_2} \mat{O}^T_{2} {\times_3} \mat{O}^T_{3})]\\
\hspace*{1cm}= \xi_{\scriptsize [(\mat{U}_{1}, \mat{U}_{2}, \mat{U}_{3}, \mathbfcal{G})]}.
\end{array}
\end{equation}
Finally, we check whether $(\xi_{\scriptsize \mat{U}_{1}}\mat{O}_{1}, \xi_{\scriptsize \mat{U}_{2}}\mat{O}_{2}, \xi_{\scriptsize \mat{U}_{3}}\mat{O}_{3}, \xi_{\scriptsize \mathbfcal{G}} {\times_1} \mat{O}^T_{1} {\times_2} \mat{O}^T_{2} {\times_3} \mat{O}^T_{3})$ is an element of \\ $\mathcal{H}_{\scriptsize (\mat{U}_{1}\mat{O}_{1}, \mat{U}_{2}\mat{O}_{2}, \mat{U}_{3}\mat{O}_{3}, \mathbfcal{G} {\times_1} \mat{O}^T_{1} {\times_2} \mat{O}^T_{2} {\times_3} \mat{O}^T_{3})}$. 
Addressing that the mode-$1$ unfolding of $\mathbfcal{G} {\times_1} \mat{O}^T_{1} {\times_2} \mat{O}^T_{2} {\times_3} \mat{O}^T_{3}$ is $ \mat{O}_1^T \mat{G}_1 (\mat{O}_3^T \otimes \mat{O}_2^T )^T$, plugging $(\xi_{\scriptsize \mat{U}_{1}}\mat{O}_{1}, \xi_{\scriptsize \mat{U}_{2}}\mat{O}_{2}, \xi_{\scriptsize \mat{U}_{3}}\mat{O}_{3}, \xi_{\scriptsize \mathbfcal{G}} {\times_1} \mat{O}^T_{1} {\times_2} \mat{O}^T_{2} {\times_3} \mat{O}^T_{3})$ into $(\mat{G}_{d}  \mat{G}_{d}^T) \xi_{\scriptsize \mat{U}_d}^T \mat{U}_d  + \xi_{\scriptsize \mat{G}_{d}}  \mat{G}_{d}^T $ in (\ref{Sup_Eq:horizontal_space_reqrements}) yields
\begin{equation}
\begin{array}{l}
\hspace*{-1.5cm}( \mat{O}_1^T \mat{G}_{1} (\mat{O}_3^T \otimes \mat{O}_2^T)^T) 
\left( \mat{O}_1^T \mat{G}_{1} (\mat{O}_3^T \otimes \mat{O}_2^T)^T\right)^T
(\xi_{\scriptsize \mat{U}_{1}}\mat{O}_{1})^T(\mat{U}_{1}\mat{O}_{1}) \\
\hspace*{0cm}+(\mat{O}_1^T)^T \xi_{\scriptsize \mat{G}_{1}} (\mat{O}_3^T \otimes \mat{O}_2^T)
(\mat{O}_1^T \mat{G}_1 (\mat{O}_3^T \otimes \mat{O}_2^T )^T)^T\\
\hspace*{3cm}= \mat{O}_1^T \mat{G}_1 \mat{G}_1^T \mat{O}_1 + \mat{O}_1^T  \xi_{\scriptsize \mat{G}_{1}} \mat{G}_1^T  \mat{O}_1 \\
\hspace*{3cm}=  \mat{O}_1^T ((\mat{G}_{1}  \mat{G}_{1}^T) \xi_{\scriptsize \mat{U}_1}^T \mat{U}_1  + \xi_{\scriptsize \mat{G}_{1}}  \mat{G}_{1}^T) \mat{O}_1.
\end{array}
\end{equation}
Since $(\xi_{\scriptsize \mat{U}_{1}}, \xi_{\scriptsize \mat{U}_{2}}, \xi_{\scriptsize \mat{U}_{3}}, \xi_{\scriptsize \mathbfcal{G}})$ is a symmetric matrix, the obtained result is also symmetric. Therefore, $(\xi_{\scriptsize \mat{U}_{1}}\mat{O}_{1}, \\ \xi_{\scriptsize \mat{U}_{2}}\mat{O}_{2}, \xi_{\scriptsize \mat{U}_{3}}\mat{O}_{3}, \xi_{\scriptsize \mathbfcal{G}} {\times_1} \mat{O}^T_{1} {\times_2} \mat{O}^T_{2} {\times_3} \mat{O}^T_{3})$ is a horizontal vector at $(\mat{U}_{1}\mat{O}_{1}, \mat{U}_{2}\mat{O}_{2}, \mat{U}_{3}\mat{O}_{3}, \mathbfcal{G} {\times_1} \mat{O}^T_{1} {\times_2} \mat{O}^T_{2} {\times_3} \mat{O}^T_{3})$. This implies that $(\xi_{\scriptsize \mat{U}_{1}}\mat{O}_{1}, \xi_{\scriptsize \mat{U}_{2}}\mat{O}_{2}, \xi_{\scriptsize \mat{U}_{3}}\mat{O}_{3}, \xi_{\scriptsize \mathbfcal{G}} {\times_1} \mat{O}^T_{1} {\times_2} \mat{O}^T_{2} {\times_3} \mat{O}^T_{3})$ is the horizontal lift of $\xi$ at $(\mat{U}_{1}\mat{O}_{1}, \mat{U}_{2}\mat{O}_{2}, \\ \mat{U}_{3}\mat{O}_{3}, \mathbfcal{G} {\times_1} \mat{O}^T_{1} {\times_2} \mat{O}^T_{2} {\times_3} \mat{O}^T_{3})$, and the proof is completed.
%
\end{proof}

Now, the proof of {\bf Proposition 1} is given below using the result (\ref{Sup_Eq:HoritaonLiftRelation}) in {\bf Lemma 1}.
\begin{proof}
Plugging $\xi_{\scriptsize \mat{U}_{1}}^{'}=\xi_{\scriptsize \mat{U}_{1}\mat{O}_{1}}$, ${\eta}_{\scriptsize\mat{U}_{1}}^{'}=\eta_{\scriptsize \mat{U}_{1}\mat{O}_{1}}$, and $\mat{G}_{1}^{'}=\mat{O}_1^T \mat{G}_{1} (\mat{O}_3^T \otimes \mat{O}_2^T)$ into the first term of (\ref{Sup_Eq:metric}) yields 
\begin{equation}
\begin{array}{lcl}
\label{}
\langle \xi_{\scriptsize \mat{U}_{1}\mat{O}_1}, \eta_{\scriptsize \mat{U}_{1}\mat{O}_1} (\mat{G}_{1}^{'}  \mat{G}_{1}^{'T} )) \rangle  
 &= & {\rm Trace}( \xi_{\scriptsize \mat{U}_{1}\mat{O}_1}^T \eta_{\scriptsize \mat{U}_{1}\mat{O}_1} (\mat{G}_{1}^{'}  \mat{G}_{1}^{'T} ) ) \nonumber \\
&\overset{\tiny (\mathrm{\ref{Sup_Eq:HoritaonLiftRelation}})}{=} &{\rm Trace}( (\xi_{\scriptsize \mat{U}_{1}}\mat{O}_1)^T \eta_{\scriptsize \mat{U}_{1}}\mat{O}_1 (\mat{G}_{1}^{'}  \mat{G}_{1}^{'T} ) ) \nonumber \\ 
 &= &  {\rm Trace}\left[ (\xi_{\scriptsize \mat{U}_{1}}\mat{O}_1)^T ({\eta}_{\scriptsize\mat{U}_{1}}\mat{O}_1)
  ( \mat{O}_1^T \mat{G}_{1} (\mat{O}_3^T \otimes \mat{O}_2^T)^T) 
\left( \mat{O}_1^T \mat{G}_{1} (\mat{O}_3^T \otimes \mat{O}_2^T)^T\right)^T \right]\nonumber \\
 &= &  {\rm Trace}\left[ (\xi_{\scriptsize \mat{U}_{1}}\mat{O}_1)^T ({\eta}_{\scriptsize\mat{U}_{1}}\mat{O}_1)   \mat{O}_1^T \mat{G}_{1} (\mat{O}_3^T \otimes \mat{O}_2^T)^T 
(\mat{O}_3^T \otimes \mat{O}_2^T) \mat{G}_{1}^T \mat{O}_1  \right]\nonumber \\
 &= &  {\rm Trace}\left[ \mat{O}_1^T \xi_{\scriptsize \mat{U}_{1}}^T {\eta}_{\scriptsize\mat{U}_{1}}\mat{O}_1 
\mat{O}_1^T \mat{G}_{1} \mat{G}_{1}^T \mat{O}_1  \right]\nonumber \\
 &= & {\rm Trace}\left[\xi_{\scriptsize \mat{U}_{1}}^T {\eta}_{\scriptsize\mat{U}_{1}} \mat{G}_{1} \mat{G}_{1}^T  \right]\nonumber \\
 &= & \langle \xi_{\scriptsize \mat{U}_{1}},{\eta}_{\scriptsize\mat{U}_{1}} (\mat{G}_{1} \mat{G}_{1}^T) \rangle.
 \end{array}
\end{equation}	
Since the same equalities against the each term in the metric (\ref{Sup_Eq:metric}) corresponding to $\mat{U}_2$, $\mat{U}_3$ and $\mathbfcal{G}$ hold, we finally obtain the invariant property that the proposition claims; 
\begin{equation*}
\label{Sup_Eq:InvariantMetric}
\begin{array}{l}
\hspace*{-1cm}g_{\scriptsize (\mat{U}_{1}, \mat{U}_{2}, \mat{U}_{3}, \mathbfcal{G})}((\xi_{\scriptsize \mat{U}_{1}}, \xi_{\scriptsize \mat{U}_{2}}, \xi_{\scriptsize \mat{U}_{3}}, \xi_{\scriptsize \mathbfcal{G}}), (\eta_{\scriptsize \mat{U}_{1}}, \eta_{\scriptsize \mat{U}_{2}}, \eta_{\scriptsize \mat{U}_{3}}, \eta_{\scriptsize \mathbfcal{G}})) \\
= g_{\scriptsize (\mat{U}_{1}\mat{O}_{1}, \mat{U}_{2}\mat{O}_{2}, \mat{U}_{3}\mat{O}_{3}, \mathbfcal{G} {\times_1} \mat{O}^T_{1} {\times_2} \mat{O}^T_{2} {\times_3} \mat{O}^T_{3})}((\xi_{\scriptsize \mat{U}_{1}\mat{O}_{1}}, \xi_{\scriptsize \mat{U}_{2}\mat{O}_{2}}, \xi_{\scriptsize \mat{U}_{3}\mat{O}_{3}}, 
\xi_{\scriptsize \mathbfcal{G} {\times_1} \mat{O}^T_{1} {\times_2} \mat{O}^T_{2} {\times_3} \mat{O}^T_{3}}), \\\hspace*{0.5cm}(\eta_{\scriptsize \mat{U}_{1}\mat{O}_{1}}, \eta_{\scriptsize \mat{U}_{2}\mat{O}_{2}}, \eta_{\scriptsize \mat{U}_{3}\mat{O}_{3}}, \eta_{\scriptsize \mathbfcal{G} {\times_1} \mat{O}^T_{1} {\times_2} \mat{O}^T_{2} {\times_3} \mat{O}^T_{3}})).
\end{array}
\end{equation*} 
\end{proof}

\subsection{\changeHK{Proof of Proposition 2 (}derivation of the tangent space projector\changeHK{)}}
\label{Sup_Sec:tangentspaceprojector}

\begin{proof}
The tangent space $T_{x} \mathcal{M}$ projector is obtained by extracting the component normal to $T_{x} \mathcal{M}$ in the ambient space. The normal space $N_{x} \mathcal{M}$ has the matrix characterization shown in (\ref{Sup_Eq:NormalSpace}).
The operator $\Psi_{{x}}: 
\mathbb{R}^{n_1 \times r_1} \times  
\mathbb{R}^{n_2 \times r_2} \times 
\mathbb{R}^{n_3 \times r_3} \times 
\mathbb{R}^{r_1 \times r_2 \times r_3} \rightarrow T_{{x}} {\mathcal{M}} :(\mat{Y}_{\scriptsize \mat{U}_{1}}, \mat{Y}_{\scriptsize \mat{U}_{2}}, \mat{Y}_{\scriptsize \mat{U}_{3}}, \mat{Y}_{\scriptsize \mathbfcal{G}} )$
$\mapsto \Psi_{{x}}(\mat{Y}_{\scriptsize \mat{U}_{1}}, \mat{Y}_{\scriptsize \mat{U}_{2}}, \mat{Y}_{\scriptsize \mat{U}_{3}}, \mat{Y}_{\scriptsize \mathbfcal{G}} )$ has the expression
\begin{equation}
\begin{array}{lll}
\label{Sup_Eq:Tangent_Projection}
\Psi_{{x}}(\mat{Y}_{\scriptsize \mat{U}_{1}}, \mat{Y}_{\scriptsize \mat{U}_{2}}, \mat{Y}_{\scriptsize \mat{U}_{3}}, \mat{Y}_{\scriptsize \mathbfcal{G}} )
&=&  (\mat{Y}_{\scriptsize \mat{U}_{1}} - \mat{U}_{1} \mat{S}_{\scriptsize \mat{U}_{1}} (\mat{G}_{1}  \mat{G}_{1}^T)^{-1}, \mat{Y}_{\scriptsize \mat{U}_{2}} - \mat{U}_{2} \mat{S}_{\scriptsize \mat{U}_{2}} (\mat{G}_{2}  \mat{G}_{2}^T)^{-1}, \\
&&\mat{Y}_{\scriptsize \mat{U}_{3}} - \mat{U}_{3}\mat{S}_{\scriptsize \mat{U}_{3}} (\mat{G}_{3}  \mat{G}_{3}^T)^{-1},
\mat{Y}_{\scriptsize \mathbfcal{G}}).
\end{array}
\end{equation}

From the definition of the tangent space in (\ref{Sup_Eq:tangent_space}), 
$\mat{U}_d$ should satisfy
\begin{eqnarray*}
 \eta_{{\bf U}_d}^T \mat{U}_d + \mat{U}_d^T \eta_{{\bf U}_d} 
 & = & 
( \mat{Y}_{\scriptsize \mat{U}_d} - \mat{U}_d \mat{S}_{\scriptsize \mat{U}_d} (\mat{G}_{d}  \mat{G}_{d}^T)^{-1})^T \mat{U}_d+ \mat{U}_d^T (\mat{Y}_{\scriptsize \mat{U}_d} - \mat{U}_d \mat{S}_{\scriptsize \mat{U}_d} (\mat{G}_{d}  \mat{G}_{d}^T)^{-1}) \\
& = & 
\mat{Y}_{\scriptsize \mat{U}_d}^T \mat{U}_d  - (\mat{G}_{d}  \mat{G}_{d}^T)^{-1} 
 \mat{S}_{\scriptsize \mat{U}_d}^T \mat{U}_d ^T \mat{U}_d + \mat{U}_d^T \mat{Y}_{\scriptsize \mat{U}_d} - \mat{U}_d^T \mat{U}_d \mat{S}_{\scriptsize \mat{U}_d} (\mat{G}_{d}  \mat{G}_{d}^T)^{-1}  \\
& = & 
\mat{Y}_{\scriptsize \mat{U}_d}^T \mat{U}_d  - (\mat{G}_{d}  \mat{G}_{d}^T)^{-1} 
 \mat{S}_{\scriptsize \mat{U}_d} + \mat{U}_d^T \mat{Y}_{\scriptsize \mat{U}_d} - \mat{S}_{\scriptsize \mat{U}_d} (\mat{G}_{d}  \mat{G}_{d}^T)^{-1} \\
& = & 0.
\end{eqnarray*}
Multiplying $(\mat{G}_{d}  \mat{G}_{d}^T)$ from the right and left sides results in 
\begin{eqnarray*}
(\mat{G}_{d}  \mat{G}_{d}^T)^{-1} \mat{S}_{\scriptsize \mat{U}_d}  + \mat{S}_{\scriptsize \mat{U}_d} (\mat{G}_{d}  \mat{G}_{d}^T)^{-1}  
& = & \mat{Y}_{\scriptsize \mat{U}_d}^T \mat{U}_d + \mat{U}_d^T \mat{Y}_{\scriptsize \mat{U}_d} \\
\mat{S}_{\scriptsize \mat{U}_d} \mat{G}_{d}  \mat{G}_{d}^T + \mat{G}_{d}  \mat{G}_{d}^T \mat{S}_{\scriptsize \mat{U}_d}   
& = & \mat{G}_{d}  \mat{G}_{d}^T (\mat{Y}_{\scriptsize \mat{U}_d}^T \mat{U}_d +\ \mat{U}_d^T \mat{Y}_{\scriptsize \mat{U}_d} ) \mat{G}_{d}  \mat{G}_{d}^T.
\end{eqnarray*}

Finally, we obtain the \emph{Lyapunov} equation as
\begin{eqnarray}
\label{Sup_Eq:Req_horizontal_space}
\mat{S}_{\scriptsize \mat{U}_{d}} \mat{G}_{d}  \mat{G}_{d}^T + \mat{G}_{d}  \mat{G}_{d}^T \mat{S}_{\scriptsize \mat{U}_{d}}  
&=&\mat{G}_{d}  \mat{G}_{d}^T (\mat{Y}_{\scriptsize \mat{U}_{d}}^T \mat{U}_{d} + \mat{U}_{d}^T \mat{Y}_{\scriptsize \mat{U}_{d}}) \mat{G}_{d}  \mat{G}_{d}^T\ \ {\rm for\ } d\in \{ 1,2,3\},
\end{eqnarray}
that are solved efficiently with the Matlab's \verb+lyap+ routine.

%
%

\end{proof}

\subsection{\changeHK{Proof of Proposition 3 (}derivation of the horizontal space projector\changeHK{)}}
\label{Sup_Sec:horizontalspaceprojector}

\begin{proof}
We consider the projection of a tangent vector $\eta_{{x}}=(\eta_{\scriptsize \mat{U}_1}, \eta_{\scriptsize \mat{U}_2}, \eta_{\scriptsize \mat{U}_3}, \eta_{\scriptsize \mathbfcal{G}}) \in T_{{x}} \mathcal{M}$ into a vector $\xi_{{x}}=(\xi_{{\bf U}_1}, \xi_{{\bf U}_2}, \xi_{{\bf U}_3}, \xi_{\mathbfcal{G}}) \in H_{{x}}$. This is achieved by subtracting the component in the vertical space $\mathcal{V}_{{x}}$ in (\ref{Sup_Eq:VerticalComponents}) as
\begin{eqnarray*}
\left\{
\begin{array}{lll}
\eta_{\scriptsize \mat{U}_1} & = &  \underbrace{\eta_{\scriptsize \mat{U}_1} - \mat{U}_1 {\bf \Omega}_1}
_{=\xi_{\scriptsize \mat{U}_1} \in \mathcal{H}_{{x}}}
+ \underbrace{\mat{U}_1 {\bf \Omega}_1}_{ \in \mathcal{V}_{{x}}},  \\
\eta_{\scriptsize \mat{U}_2} & = &  \eta_{\scriptsize \mat{U}_2} - \mat{U}_2 {\bf \Omega}_2 + \mat{U}_2 {\bf \Omega}_2, \\
\eta_{\scriptsize \mat{U}_3} & = &  \eta_{\scriptsize \mat{U}_3} - \mat{U}_3 {\bf \Omega}_3 + \mat{U}_3 {\bf \Omega}_3, \\
\eta_{\scriptsize \mathbfcal{G}} & = &  \eta_{\scriptsize \mathbfcal{G}} - ( - (\mathbfcal{G}{\times_1} {\bf \Omega}_1  + 
\mathbfcal{G}{{\times_2}} {\bf \Omega}_2 +
\mathbfcal{G}{{\times_3}} {\bf \Omega}_3)) 
+ ( - (\mathbfcal{G}{\times_1} {\bf \Omega}_1  + 
\mathbfcal{G}{{\times_2}} {\bf \Omega}_2 +
\mathbfcal{G}{{\times_3}} {\bf \Omega}_3)).
\end{array}
\right.
\end{eqnarray*}

As a result, the horizontal operator $\Pi_{{x}}: T_x \mathcal{M} \rightarrow \mathcal{H}_x: \eta_x \mapsto \Pi_{{x}}(\eta_x)$ has the expression
\begin{equation}
\begin{array}{lll}
\label{Sup_Eq:Horizontal_Projection}
\Pi_{{x}}(\eta_{{x}} )
& = &(
\eta_{\scriptsize \mat{U}_{1}} - \mat{U}_{1} {\bf \Omega}_{1},
\eta_{\scriptsize \mat{U}_{2}} - \mat{U}_{2} {\bf \Omega}_{2}, 
\eta_{\scriptsize \mat{U}_{3}} - \mat{U}_{3} {\bf \Omega}_{3}, 
\eta_{\scriptsize \mathbfcal{G}} - (- (\mathbfcal{G}{\times_1} {\bf \Omega}_{1} + 
\mathbfcal{G}{{\times_2}} {\bf \Omega}_{2} +
\mathbfcal{G}{{\times_3}} {\bf \Omega}_{3})) 
),
\end{array}
\end{equation}
where $\eta_x = (\eta_{\scriptsize \mat{U}_{1}}, \eta_{\scriptsize \mat{U}_{2}}, \eta_{\scriptsize \mat{U}_{3}}, \eta_{\scriptsize \mathbfcal{G}}) \in T_x \mathcal{M}$ and  ${\bf \Omega}_{d}$ is a skew-symmetric matrix of size $r_d \times r_d$. The skew-matrices ${\bf \Omega}_{d}$ for $d = \{1, 2,3\}$ that are identified based on the conditions (\ref{Sup_Eq:horizontal_space_reqrements}).

It should be noted that the tensor $\mathbfcal{G}{\times_1} {\bf \Omega}_1  + \mathbfcal{G}{{\times_2}} {\bf \Omega}_2 + \mathbfcal{G}{{\times_3}} {\bf \Omega}_3$ in (\ref{Sup_Eq:VerticalComponents}) has the following equivalent unfoldings.
\begin{eqnarray*}
\begin{array}{lll}
\label{Sup_Eq:VerticalComponentsMatrixRep}
\mathbfcal{G}{\times_1} {\bf \Omega}_1  + 
\mathbfcal{G}{{\times_2}} {\bf \Omega}_2 +
\mathbfcal{G}{{\times_3}} {\bf \Omega}_3 
&\xLeftrightarrow{\text{mode\ }-1}  &
{\bf \Omega}_{1} \mat{G}_{1} + \mat{G}_{1}(\mat{I}_{r_3} \otimes {\bf \Omega}_2)^T
+ \mat{G}_{1}( {\bf \Omega}_3 \otimes \mat{I}_{r_2} )^T \\
&\xLeftrightarrow{\text{mode\ }-2}   &
\mat{G}_{2}(\mat{I}_{r_3} \otimes {\bf \Omega}_1)^T + {\bf \Omega}_{2} \mat{G}_{2}  
+ \mat{G}_{2}( {\bf \Omega}_3 \otimes \mat{I}_{r_1} )^T \\
&\xLeftrightarrow{\text{mode\ }-3}   &
\mat{G}_{3}(\mat{I}_{r_2} \otimes {\bf \Omega}_1)^T 
+ \mat{G}_{3}( {\bf \Omega}_2 \otimes \mat{I}_{r_1} )^T + {\bf \Omega}_{3} \mat{G}_{3}.
\end{array}
\end{eqnarray*}

Plugging $\xi_{\scriptsize \mat{U}_1} = \eta_{\scriptsize \mat{U}_{1}} - \mat{U}_{1} {\bf \Omega}_{1}$ and $\xi_{\scriptsize \mat{G}_1} = \eta{\scriptsize \mat{G}_{1}} +{\bf \Omega}_{1} \mat{G}_{1} + \mat{G}_{1}(\mat{I}_{r_3} \otimes {\bf \Omega}_2)^T
+ \mat{G}_{1}( {\bf \Omega}_3 \otimes \mat{I}_{r_2} )^T$ into (\ref{Sup_Eq:horizontal_space_reqrements}) and using the relation $(\mat{A} \otimes \mat{B})^T = \mat{A}^T \otimes \mat{B}^T$ results in
\begin{eqnarray*}
\begin{array}{lll}
\label{Sup_Eq:ConditionA}
(\mat{G}_{1}  \mat{G}_{1}^T) \xi_{\scriptsize \mat{U}_1}^T \mat{U}_  + \xi_{\scriptsize \mat{G}_{1}}  \mat{G}_{1}^T   &= &(\mat{G}_{1}  \mat{G}_{1}^T) (\eta_{\scriptsize \mat{U}_1} - \mat{U}_1 {\bf \Omega}_1)^T \mat{U}_1 \\
&& + \left\{ \eta{\scriptsize \mat{G}_{1}} +({\bf \Omega}_{1} \mat{G}_{1} + \mat{G}_{1}(\mat{I}_{r_3} \otimes {\bf \Omega}_2)^T
+ \mat{G}_{1}( {\bf \Omega}_3 \otimes \mat{I}_{r_2} )^T)\right\} \mat{G}_{1}^T  \\
&=& (\mat{G}_{1}  \mat{G}_{1}^T) \eta_{\scriptsize \mat{U}_1} ^T \mat{U}_1 
- (\mat{G}_{1}  \mat{G}_{1}^T) ( \mat{U}_1 {\bf \Omega}_1)^T \mat{U}_1 + \eta_{\scriptsize \mat{G}_{1}}\mat{G}_{1}^T  
+ {\bf \Omega}_{1} \mat{G}_{1}\mat{G}_{1}^T \\
&&+\mat{G}_{1}(\mat{I}_{r_3} \otimes {\bf \Omega}_2)^T \mat{G}_{1}^T 
+ \mat{G}_{1}( {\bf \Omega}_3 \otimes \mat{I}_{r_2} )^T\mat{G}_{1}^T\\
&= &(\mat{G}_{1}  \mat{G}_{1}^T) \eta_{\scriptsize \mat{U}_1} ^T \mat{U}_1 
+ (\mat{G}_{1}  \mat{G}_{1}^T) {\bf \Omega}_1 + \eta_{\scriptsize \mat{G}_{1}}\mat{G}_{1}^T  
+  {\bf \Omega}_{1} \mat{G}_{1}\mat{G}_{1}^T \\
&&
- \mat{G}_{1}(\mat{I}_{r_3} \otimes {\bf \Omega}_2) \mat{G}_{1}^T 
- \mat{G}_{1}( {\bf \Omega}_3 \otimes \mat{I}_{r_2} ) \mat{G}_{1}^T,
\end{array}
\end{eqnarray*}
which should be a symmetric matrix due to (\ref{Sup_Eq:horizontal_space_reqrements}), i.e., $
(\mat{G}_{1}  \mat{G}_{1}^T) \xi_{\scriptsize \mat{U}_1}^T \mat{U}_  + \xi_{\scriptsize \mat{G}_{1}}  \mat{G}_{1}^T =  
((\mat{G}_{1}  \mat{G}_{1}^T) \xi_{\scriptsize \mat{U}_1}^T \mat{U}_  + \xi_{\scriptsize \mat{G}_{1}}  \mat{G}_{1}^T)^T$.


Subsequently,  
\begin{equation*}
\begin{array}{l}
(\mat{G}_{1}  \mat{G}_{1}^T) \eta_{\scriptsize \mat{U}_1} ^T \mat{U}_1 
+ (\mat{G}_{1}  \mat{G}_{1}^T) {\bf \Omega}_1 
+ \eta_{\scriptsize \mat{G}_{1}}\mat{G}_{1}^T  
+ {\bf \Omega}_{1} \mat{G}_{1}\mat{G}_{1}^T - \mat{G}_{1}(\mat{I}_{r_3} \otimes {\bf \Omega}_2) \mat{G}_{1}^T 
- \mat{G}_{1}( {\bf \Omega}_3 \otimes \mat{I}_{r_2} ) \mat{G}_{1}^T \\
\hspace*{1cm}= \mat{U}_1^T \eta_{\scriptsize \mat{U}_1} (\mat{G}_{1}  \mat{G}_{1}^T) 
- {\bf \Omega}_1 \mat{G}_{1}  \mat{G}_{1}^T  
+ \mat{G}_{1} \eta_{\scriptsize \mat{G}_{1}}^T  
-  \mat{G}_{1}\mat{G}_{1}^T {\bf \Omega}_{1}+\mat{G}_{1}(\mat{I}_{r_3} \otimes {\bf \Omega}_2) \mat{G}_{1}^T 
+ \mat{G}_{1}( {\bf \Omega}_3 \otimes \mat{I}_{r_2} ) \mat{G}_{1}^T,
\end{array}
\end{equation*}
which is equivalent to
\begin{equation*}
\begin{array}{l}
\mat{G}_{1}  \mat{G}_{1}^T {\bf \Omega}_1 
+ {\bf \Omega}_{1} \mat{G}_{1}\mat{G}_{1}^T  - \mat{G}_{1}(\mat{I}_{r_3} \otimes {\bf \Omega}_2) \mat{G}_{1}^T 
- \mat{G}_{1}( {\bf \Omega}_3 \otimes \mat{I}_{r_2} ) \mat{G}_{1}^T  
= {\rm Skew}(\mat{U}_1^T \eta_{\scriptsize \mat{U}_1} \mat{G}_{1}  \mat{G}_{1}^T)
+ {\rm Skew}(\mat{G}_{1} \eta_{\scriptsize \mat{G}_{1}}^T ).
\end{array}
\end{equation*}
Here ${\rm Skew}(\cdot)$ extracts the skew-symmetric part of a square matrix, i.e., ${\rm Skew}(\mat{D})=(\mat{D}-\mat{D}^T)/2$.

Finally, we obtain the \emph{coupled} Lyapunov equations 
\begin{equation}
\begin{array}{lll}
\label{Sup_Eq:OmegaRequirements}
\hspace{-0.3cm}\left\{
\begin{array}{l}
\mat{G}_{1}  \mat{G}_{1}^T {\bf \Omega}_{1} + {\bf \Omega}_{1} \mat{G}_{1}  \mat{G}_{1}^T -\mat{G}_{1}(\mat{I}_{r_3} \otimes {\bf \Omega}_{2}) \mat{G}_{1}^T 
- \mat{G}_{1}( {\bf \Omega}_{3} \otimes \mat{I}_{r_2} )\mat{G}_{1}^T  \\
\hspace{4.8cm} = {\rm Skew}(\mat{U}_1^T\eta_{\scriptsize \mat{U}_1}\mat{G}_{1}  \mat{G}_{1}^T) + {\rm Skew}(\mat{G}_{1}\eta_{\scriptsize \mat{G}_{1}}^T), \\
\mat{G}_{2}  \mat{G}_{2}^T {\bf \Omega}_{2} + {\bf \Omega}_{2} \mat{G}_{2}  \mat{G}_{2}^T -\mat{G}_{2}(\mat{I}_{r_3} \otimes {\bf \Omega}_{1}) \mat{G}_{2}^T 
- \mat{G}_{2}( {\bf \Omega}_{3} \otimes \mat{I}_{r_1} )\mat{G}_{2}^T  \\
\hspace{4.8cm} = {\rm Skew}(\mat{U}_2^T\eta_{\scriptsize \mat{U}_2}\mat{G}_{2}  \mat{G}_{2}^T) + {\rm Skew}(\mat{G}_{2}\eta_{\scriptsize \mat{G}_{2}}^T), \\
\mat{G}_{3}  \mat{G}_{3}^T {\bf \Omega}_{3} + {\bf \Omega}_{3} \mat{G}_{3}  \mat{G}_{3}^T -\mat{G}_{3}(\mat{I}_{r_2} \otimes {\bf \Omega}_{1}) \mat{G}_{3}^T 
- \mat{G}_{3}( {\bf \Omega}_{2} \otimes \mat{I}_{r_1} )\mat{G}_{3}^T  \\
\hspace{4.8cm} = {\rm Skew}(\mat{U}_3^T\eta_{\scriptsize \mat{U}_3}\mat{G}_{3}  \mat{G}_{3}^T) + {\rm Skew}(\mat{G}_{3}\eta_{\scriptsize \mat{G}_{3}}^T),
\end{array}
\right.
\end{array}
\end{equation}
that are solved efficiently with the Matlab's \verb+pcg+ routine that is combined with a specific preconditioner resulting from the Gauss-Seidel approximation of (\ref{Sup_Eq:OmegaRequirements}).
 
\end{proof}

\subsection{\changeHK{Proof of Proposition 4 (}derivation of the Riemannian gradient formula\changeHK{)}}
\label{Sup_Sec:Riemanniangradientformula}

\begin{proof}
Let $f(\mathbfcal{X})=\| \mathbfcal{P}_{\Omega}(\mathbfcal{X}) - \mathbfcal{P}_{\Omega}(\mathbfcal{X}^{\star}) \|^2_F/|\Omega |$ and 
$\mathbfcal{S} = 2 (\mathbfcal{P}_{\Omega}(\mathbfcal{G}{\times_1} \mat{U}_{1} {\times_2} \mat{U}_{2}{\times_3} \mat{U}_{3}) - 
\mathbfcal{P}_{\Omega}(\mathbfcal{X}^{\star}))/|{\Omega}|$ 
be an auxiliary sparse tensor variable that is interpreted as the Euclidean gradient of $f$ in $\mathbb{R}^{n_1 \times n_2 \times n_3}$. 

The partial derivatives of $f(\mat{U}_1, \mat{U}_2, \mat{U}_3, \mathbfcal{G})$ are
\begin{eqnarray*}
\label{Sup_Eq:Egradient}
 \left\{
\begin{array}{lll}
\displaystyle{\frac{\partial f_{1}(\mat{U}_1, \mat{U}_2, \mat{U}_3, \mat{G}_{1})}{\partial \mat{U}_1}}  &=&  \displaystyle{\frac{2}{|{\Omega} |}} ( {\mathbfcal{P}}_{\Omega} (\mat{U}_1 \mat{G}_{1} (\mat{U}_3 \otimes \mat{U}_2)^T)  -  {\mathbfcal{P}}_{\Omega}(\mat{X}^{\star}_{1}) ) (\mat{U}_3 \otimes \mat{U}_2) \mat{G}_{1}^T\\
   &=&  \mat{S}_{1} (\mat{U}_3 \otimes \mat{U}_2) \mat{G}_{1}^T, \\
\displaystyle{\frac{\partial f_{2}(\mat{U}_1, \mat{U}_2, \mat{U}_3, \mat{G}_{2})}{\partial \mat{U}_2}}   &=&  \displaystyle{\frac{2}{|{\Omega} |}} ( {\mathbfcal{P}}_{\Omega} (\mat{U}_2 \mat{G}_{2} (\mat{U}_3 \otimes \mat{U}_1)^T)   - {\mathbfcal{P}}_{\Omega} (\mat{X}^{\star}_{2}) ) (\mat{U}_3 \otimes \mat{U}_1) \mat{G}_{2}^T\\
   &=&  \mat{S}_{2} (\mat{U}_2 \otimes \mat{U}_1) \mat{G}_{2}^T, \\
\displaystyle{\frac{\partial f_{3}(\mat{U}_1, \mat{U}_2, \mat{U}_3, \mat{G}_{3})}{\partial \mat{U}_3}}   &=&  \displaystyle{\frac{2}{|{\Omega} |}} ( {\mathbfcal{P}}_{\Omega} (\mat{U}_3 \mat{G}_{3} (\mat{U}_2 \otimes \mat{U}_1)^T) - {\mathbfcal{P}}_{\Omega} (\mat{X}^{\star}_{3}) ) (\mat{U}_2 \otimes \mat{U}_1) \mat{G}_{3}^T\\
   &=&  \mat{S}_{3} (\mat{U}_2 \otimes \mat{U}_1) \mat{G}_{3}^T, \\
\displaystyle{\frac{\partial f(\mat{U}_1, \mat{U}_2, \mat{U}_3, \mathbfcal{G})}{\partial \mathbfcal{G}}}   &=&  \displaystyle{\frac{2}{|{\Omega} |} (\mathbfcal{P}_{\Omega}(\mathbfcal{G}_{\times 1}\mat{U}_1 {\times_2} \mat{U}_2 {\times_3} \mat{U}_3) - 
\mathbfcal{P}_{\Omega}(\mathbfcal{X}^{\star}))}  \times_1 \mat{U}_1^T \times_2 \mat{U}_2^T \times_3 \mat{U}_3^T\\
 &=&  \mathbfcal{S} \times_1 \mat{U}_1^T \times_2 \mat{U}_2\changeHK{^T} \times_3 \mat{U}_3^T,
\end{array}
\right.
\end{eqnarray*}
where $\mat{X}^{\star}_{d}$ is mode-$d$ unfolding of $\mathbfcal{X}^{\star}$ and 
\begin{eqnarray*}
 \left\{
\begin{array}{lll}
\mat{S}_{1} & =  & \displaystyle{\frac{2}{|{\Omega} |} ( {\mathbfcal{P}}_{\Omega}(\mat{U}_1 \mat{G}_{1} (\mat{U}_3 \otimes \mat{U}_2)^T) - {\mathbfcal{P}}_{\Omega}(\mat{X}^{\star}_{1}) )}\\ 
\mat{S}_{2} & =  & \displaystyle{\frac{2}{|{\Omega} |} ( {\mathbfcal{P}}_{\Omega}(\mat{U}_2 \mat{G}_{2} (\mat{U}_3 \otimes \mat{U}_1)^T) - {\mathbfcal{P}}_{\Omega}(\mat{X}^{\star}_{2}) )}\\ 
\mat{S}_{3} & =  & \displaystyle{\frac{2}{|{\Omega} |} ( {\mathbfcal{P}}_{\Omega}(\mat{U}_3 \mat{G}_{3} (\mat{U}_2 \otimes \mat{U}_1)^T) - {\mathbfcal{P}}_{\Omega}(\mat{X}^{\star}_{3}) )}\\ 
\mathbfcal{S} & =  & \displaystyle{\frac{2}{|{\Omega} |} (\mathbfcal{P}_{\Omega}(\mathbfcal{G} {\times_1} \mat{U}_1 {\times_2} \mat{U}_2 {\times_3} \mat{U}_3) - 
\mathbfcal{P}_{\Omega}(\mathbfcal{X}^{\star}))}.
\end{array}
\right.
\end{eqnarray*}


Due to the specific scaled metric (\ref{Sup_Eq:metric}), the partial derivatives of $f$ are further scaled by $((\mat{G}_{1}\mat{G}_{1}^T)^{-1}, \\
(\mat{G}_{2}\mat{G}_{2}^T)^{-1}, (\mat{G}_{3}\mat{G}_{3}^T)^{-1}, \mathbfcal{I})$, denoted as ${\rm egrad}_{x} f$ (after scaling), i.e.,

\[
\begin{array}{lcl}
{\rm egrad}_{x} f 
& = &( \mat{S}_{1} (\mat{U}_3 \otimes \mat{U}_2) \mat{G}_{1}^T (\mat{G}_{1}\mat{G}_{1}^T)^{-1}, 
 \mat{S}_{2} (\mat{U}_3 \otimes \mat{U}_1) \mat{G}_{2}^T (\mat{G}_{2}\mat{G}_{2}^T)^{-1}, 
 \hspace{0cm} \mat{S}_{3} (\mat{U}_2 \otimes \mat{U}_1) \mat{G}_{3}^T(\mat{G}_{3}\mat{G}_{3}^T)^{-1},  \\
& & \mathbfcal{S} \times_1 \mat{U}_1^T \times_2 \mat{U}_2^T \times_3 \mat{U}_3^T). \\
\end{array}
\]

Consequently, from the relationship that horizontal lift of ${\rm grad}_{[x]}f$ is equal to ${\rm grad}_{x}f \ =\ \Psi({\rm egrad}_{x} f )$, we obtain that, using (\ref{Sup_Eq:Tangent_Projection}),
\begin{equation*}
\label{Sup_Eq:RiemannianGradientDerive}
\begin{array}{lll}
\text{the horizontal lift of\ }{\rm grad}_{[x]} f  &=&   
(\mat{S}_{1} (\mat{U}_{3} \otimes \mat{U}_{2}) \mat{G}_{1}^T (\mat{G}_{1}\mat{G}_{1}^T)^{-1} 
- \mat{U}_{1} \mat{B}_{\scriptsize \mat{U}_{1}}(\mat{G}_{1}\mat{G}_{1}^T)^{-1}, \\
&&\mat{S}_{2} (\mat{U}_{3} \otimes \mat{U}_{1}) \mat{G}_{2}^T (\mat{G}_{2}\mat{G}_{2}^T)^{-1}
- \mat{U}_{2} \mat{B}_{\scriptsize \mat{U}_{2}}(\mat{G}_{2}\mat{G}_{2}^T)^{-1}, \\
 && \mat{S}_{3} (\mat{U}_{2} \otimes \mat{U}_{1}) \mat{G}_{3}^T(\mat{G}_{3}\mat{G}_{3}^T)^{-1}
- \mat{U}_{3} \mat{B}_{\scriptsize \mat{U}_{3}}(\mat{G}_{3}\mat{G}_{3}^T)^{-1}, \\
 && \mathbfcal{S} \times_1 \mat{U}_{1}^T \times_2 \mat{U}_{2}^T \times_3 \mat{U}_{3}^T), 
\end{array}
\end{equation*}


From the requirements in (\ref{Sup_Eq:Req_horizontal_space}) for a vector to be in the tangent space, we have the following relationship for mode-$1$.
\begin{eqnarray*}
\mat{B}_{\scriptsize \mat{U}_1} \mat{G}_{1}  \mat{G}_{1}^T + \mat{G}_{1}  \mat{G}_{1}^T \mat{B}_{\scriptsize \mat{U}_1}   
 =  \mat{G}_{1}  \mat{G}_{1}^T (\mat{Y}_{\scriptsize \mat{U}_1}^T \mat{U}_1 + \mat{U}_1^T \mat{Y}_{\scriptsize \mat{U}_1} ) \mat{G}_{1}  \mat{G}_{1}^T,
\end{eqnarray*}
where $\mat{Y}_{\scriptsize \mat{U}_1}  = (\mat{S}_{1} (\mat{U}_3 \otimes \mat{U}_2) \mat{G}_{1}^T (\mat{G}_{1}\mat{G}_{1}^T)^{-1}$.

Subsequently,
\[
\begin{array}{lll}
\mat{G}_{1}  \mat{G}_{1}^T (\mat{Y}_{\scriptsize \mat{U}_1}^T \mat{U}_1 + \mat{U}_1^T \mat{Y}_{\scriptsize \mat{U}_1} ) \mat{G}_{1}  \mat{G}_{1}^T 
 &=  &
\mat{G}_{1}  \mat{G}_{1}^T \left\{ 
((\mat{S}_{1} (\mat{U}_3 \otimes \mat{U}_2) \mat{G}_{1}^T (\mat{G}_{1}\mat{G}_{1}^T)^{-1})^T  \mat{U}_1 \right.  \\ 
 && + \left.  \mat{U}_1^T (\mat{S}_{1} (\mat{U}_3 \otimes \mat{U}_2) \mat{G}_{1}^T (\mat{G}_{1}\mat{G}_{1}^T)^{-1}  \right\} \mat{G}_{1}  \mat{G}_{1}^T \\ 
 &= & ((\mat{S}_{1} (\mat{U}_3 \otimes \mat{U}_2) \mat{G}_{1}^T)^T \mat{U}_1 \mat{G}_{1}  \mat{G}_{1}^T +\ \mat{G}_{1}  \mat{G}_{1}^T\mat{U}\changeHK{_1}^T (\mat{S}_{1} (\mat{U}_3\otimes \mat{U}_2) \mat{G}_{1}^T \\ 
 &= & (\mat{G}_{1}  \mat{G}_{1}^T\mat{U}_1^T (\mat{S}_{1} (\mat{U}_3 \otimes \mat{U}_2) \mat{G}_{1}^T)^T +\ \mat{G}_{1}  \mat{G}_{1}^T\mat{U}_1^T (\mat{S}_{1} (\mat{U}_3 \otimes \mat{U}_2) \mat{G}_{1}^T \\ 
 &= & 2 {\rm Sym} (\mat{G}_{1}  \mat{G}_{1}^T\mat{U}_1^T (\mat{S}_{1} (\mat{U}_3 \otimes \mat{U}_2) \mat{G}_{1}^T).
\end{array}
\]

Finally, $\mat{B}_{\scriptsize \mat{U}_{d}}$ for $d \in \{1, 2,3\} $ are obtained by solving the Lyapunov equations
\begin{equation*}
\label{Sup_Eq:BUBVBWRequirementGradient}
\left\{
\begin{array}{lll}
\mat{B}_{\scriptsize \mat{U}_{1}} \mat{G}_{1}  \mat{G}_{1}^T + \mat{G}_{1}  \mat{G}_{1}^T \mat{B}_{\scriptsize \mat{U}_{1}}
 &=&  2 {\rm Sym} (\mat{G}_{1}  \mat{G}_{1}^T\mat{U}_{1}^T (\mat{S}_{1} (\mat{U}_{3} \otimes \mat{U}_{2}) \mat{G}_{2}^T), \\ 
\mat{B}_{\scriptsize \mat{U}_{2}} \mat{G}_{2}  \mat{G}_{2}^T + \mat{G}_{2}  \mat{G}_{2}^T \mat{B}_{\scriptsize \mat{U}_{2}} &=&  2 {\rm Sym} (\mat{G}_{2}  \mat{G}_{2}^T\mat{U}_{2}^T (\mat{S}_{2} (\mat{U}_{3} \otimes \mat{U}_{1}) \mat{G}_{2}^T), \\ 
\mat{B}_{\scriptsize \mat{U}_{3}} \mat{G}_{3}  \mat{G}_{3}^T + \mat{G}_{3}  \mat{G}_{3}^T \mat{B}_{\scriptsize \mat{U}_{3}} &=&  2 {\rm Sym} (\mat{G}_{3}  \mat{G}_{3}^T\mat{U}_{3}^T (\mat{S}_{3} (\mat{U}_{2} \otimes \mat{U}_{1}) \mat{G}_{3}^T),
\end{array}
\right.
\end{equation*}
where ${\rm Sym}(\cdot)$ extracts the symmetric part of a square matrix, i.e., ${\rm Sym}(\mat{D})=(\mat{D}+\mat{D}^T)/2$. The above Lyapunov equations are solved efficiently with the Matlab's \verb+lyap+ routine.
\end{proof}


\section{Additional numerical comparisons}
\label{sec:AdditionalNumericalComparisons}

In addition to the representative numerical comparisons in the paper, we show additional numerical experiments spanning synthetic and real-world datasets. 

{\bf Experiments on synthetic datasets:}

{\bf Case S1: comparison with the Euclidean metric.} We first show the benefit of the proposed metric (\ref{Sup_Eq:metric}) over the conventional choice of the Euclidean metric that exploits the product structure of $\mathcal{M}$ and symmetry. We compare {\it steepest descent} algorithms with Armijo backtracking linesearch for both the metric choices. Figure \ref{appnfig:comp_euclideian-test} shows that the algorithm with the metric (\ref{Sup_Eq:metric}) gives a superior performance \changeHK{in \emph{test error}} than that of the conventional metric choice. 

{\bf Case S2: small-scale instances.} We consider tensors of size $100 \times 100 \times 100$, $150 \times 150 \times 150$, and  $200 \times 200 \times 200$ and ranks $(5,5,5)$, $(10,10,10)$, and  $(15,15,15)$. OS is $\{10,20,30\}$. Figures \ref{appnfig:small-scale-train}(a)-(c) and Figures \ref{appnfig:small-scale-test}(a)-(c) show the convergence behavior of different algorithms \changeHK{on a train set $\Omega$ and on a test set $\Gamma$}, where \changeHK{Figures \ref{appnfig:small-scale-test}}(b) is identical to the figure in the manuscript paper. Figures \ref{appnfig:small-scale-train}(d)-(f) and \ref{appnfig:small-scale-test}(d)-(f) show the mean square error on \changeHK{$\Omega$ and} $\Gamma$ on each algorithm. Furthermore, Figure \ref{appnfig:small-scale-train}(g)-(i) and Figure \ref{appnfig:small-scale-test}(g)-(i) show the mean square error on \changeHK{$\Omega$ and} $\Gamma$ when OS is $10$ in all the five runs. From \changeHK{Figures \ref{appnfig:small-scale-train} and }Figures \ref{appnfig:small-scale-test}, our proposed algorithm is consistently competitive or faster than geomCG, HalRTC, and TOpt. In addition, the mean square errors on \changeHK{a train set $\Omega$ and} a test set $\Gamma$ are consistently competitive or lower than those of geomCG and HalRTC, especially for lower sampling ratios, e.g, for OS $10$.

{\bf Case S3: large-scale instances.} We consider large-scale tensors of size $3000 \times 3000 \times 3000$, $5000 \times 5000 \times 5000$, and $10000 \times 10000 \times 10000$ and ranks {\bf r}=$(5,5,5)$ and $(10,10,10)$. OS is $10$. We compare our proposed algorithm to geomCG. \changeHK{Figure \ref{appnfig:large-scale-train} and} Figure \ref{appnfig:large-scale-test} show the convergence behavior of the algorithms. The proposed algorithm outperforms geomCG in all the cases. 

{\bf Case S4: influence of low sampling.} We look into problem instances which result from scarcely sampled data. The test requires completing a tensor of size $10000 \times 10000 \times 10000$ and rank {\bf r}=$(5,5,5)$. \changeHK{Figure \ref{appnfig:low-sampling-train}} and Figure \ref{appnfig:low-sampling-test} show the convergence behavior when OS is $\{8,6,5\}$. The case of  $\text{OS}=5$ is particularly interesting. In this case, while the mean square errors on \changeHK{$\Omega$ and} $\Gamma$ increase for geomCG, the proposed algorithm stably decreases the error in all the five runs.  

\changeHK{{\bf Case S5: influence of ill-conditioning and low sampling.} We consider the problem instance of {\bf Case S4} with $\text{OS\ } = 5$. Additionally, for generating the instance, we impose a diagonal core $\mathbfcal{G}$ with exponentially decaying \emph{positive} values of condition numbers (CN)  $5$, $50$, and $100$. Figure \ref{appnfig:S5-train} shows that the proposed algorithm outperforms geomCG for all the considered CN values on a train set $\Omega$}. 

\changeHK{
{\bf Case S6: influence of noise.} We evaluate the convergence properties of algorithms under the presence of noise The tensor size and rank are same as in {\bf Case S4} and OS is $10$. Figure \ref{appnfig:S6-train} shows that the train error on a train set $\Omega$ for each $\epsilon$ is almost identical to the $\epsilon^2 \| \mathbfcal{P}_{\Omega} (\mathbfcal{X}^{\star})\|_F ^2$, but our proposed algorithm converges faster than geomCG.
}

{\bf Case S7: rectangular instances.} We consider instances where dimensions and ranks along certain modes are different than others. Two cases are considered. Case (7.a) considers tensors size $20000 \times7000 \times 7000$, $30000 \times 6000 \times 6000$, and $40000 \times 5000\times 5000$ and rank ${\bf r}=(5,5,5)$. Case (7.b) considers a tensor of size $10000 \times10000 \times10000$ with ranks ${\bf r}=(7,6,6)$, $(10,5,5)$, and $ (15,4,4)$. \changeHK{Figures \ref{appnfig:asymmetric-train}(a)-(c) and} Figures \ref{appnfig:asymmetric-test}(a)-(c) show that the convergence behavior of our proposed algorithm is superior to that of geomCG \changeHK{on $\Omega$ and $\Gamma$, respectively}. Our proposed algorithm also outperforms geomCG for the asymmetric rank cases as shown in \changeHK{Figure \ref{appnfig:asymmetric-train}(d)-(f) and} Figure \ref{appnfig:asymmetric-test}(d)-(f).

{\bf Case S8: medium-scale instances.} We additionally consider medium-scale tensors of size $500 \times 500 \times 500$, $1000 \times 1000 \times 1000$, and $1500 \times 1500 \times 1500$ and ranks ${\bf r}=(5,5,5), (10,10,10)$, and $(15,15,15)$. OS is $\{10,20,30,40\}$. Our proposed algorithm and geomCG are only compared as the other algorithms cannot handle these scales efficiently. Figures \changeHK{\ref{appnfig:middle-scale-train}(a)-(c) and} \ref{appnfig:middle-scale-test}(a)-(c) show the convergence behavior \changeHK{on $\Omega$ and $\Gamma$, respectively}. \changeHK{Figures \ref{appnfig:middle-scale-train}(d)-(f) and} Figures \ref{appnfig:middle-scale-test}(d)-(f) also show the mean square error on \changeHK{$\Omega$ and } $\Gamma$ of rank ${\bf r}=(15,15,15)$ in all the five runs. The proposed algorithm performs better than geomCG in all the cases.

{\bf Experiments on real-world datasets:}

{\bf Case R1: hyperspectral image.} We also show the performance of our algorithm on the hyperspectral image ``Ribeira". We show the mean square error on \changeHK{$\Omega$ and } $\Gamma$ when OS is $\{11, 22\}$ in Figure \ref{appnfig:R1-train} and Figure \ref{appnfig:R1-test}, where \changeHK{Figure \ref{appnfig:R1-test}}(a) is identical to the figure in the manuscript paper. Our proposed algorithm gives lower test errors than those obtained by the other algorithms. We also show the image recovery results. Figures \ref{appnfig:R1-reconstructedimage_OS_11} and \ref{appnfig:R1-reconstructedimage_OS_22} show the reconstructed images when OS is $\{11, 22\}$, respectively. From these figures, we find that the proposed algorithm shows a good performance, especially for the lower sampling ratio.  

{\bf Case R2: MovieLens-10M.} \changeHK{Figure \ref{appnfig:R2-train} and} Figure \ref{appnfig:R2-test} show the convergence plots for all the five runs of ranks ${\bf r}=(4,4,4)$, $(6,6,6)$, $(8,8,8)$ and $(10,10,10) $ \changeHK{on $\Omega$ and $\Gamma$, respectively}. 
These figures show the superior performance of our proposed algorithm.

\changeHK{
{\bf Experiments for online algorithms:}
}

\changeBM{
{\bf Case O: online instances.} Figure \ref{appnfig:O1-5000} and \ref{appnfig:O1-10000}  show the convergence plots for all the five runs on tensors of ranks $100 \times 100 \times 5000$, and $100 \times 100 \times 10000$ with rank ${\bf r}=(5,5,5)$ \changeHK{on $\Omega$ and $\Gamma$, respectively}. These figures show that the proposed stochastic gradient descent algorithm gives similar or faster convergence than the proposed batch gradient descent algorithm.
}

\changeBM{
 Figure \ref{appnfig:O2-train} and \ref{appnfig:O2-test} show the convergence speed comparisons in the train error and the test error of the proposed online and batch algorithms with TeCPSGD and OLSTEC with rank ${\bf r}=(5,5,5)$ on the real-world video sequence Airport Hall dataset. These figures show that the proposed stochastic gradient descent algorithm gives similar or faster convergence than the proposed batch algorithm. In addition, Table \ref{appntbl:O2} shows that the final train and test MSEs show the superior performance of the proposed algorithms.}

\begin{figure}[htbp]
\begin{center}
\includegraphics[width=0.48\hsize]{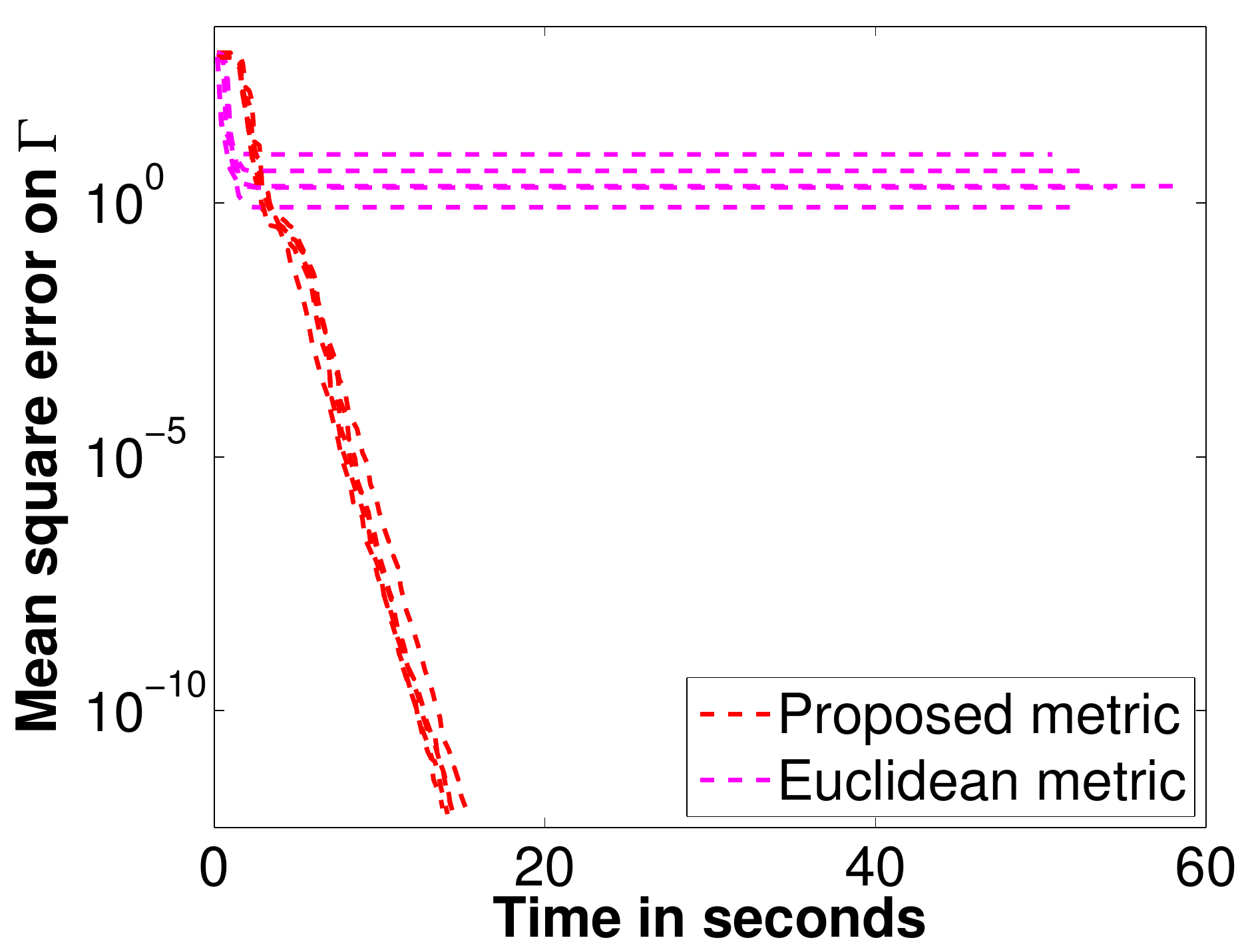}\\
\caption{\changeHK{{\bf Case S1:} comparison between metrics (test error)．}}
\label{appnfig:comp_euclideian-test}
\end{center}
\end{figure}

\begin{figure*}[htbp]
\begin{tabular}{cc}
\begin{minipage}{0.32\hsize}
\begin{center}
\includegraphics[width=\hsize]{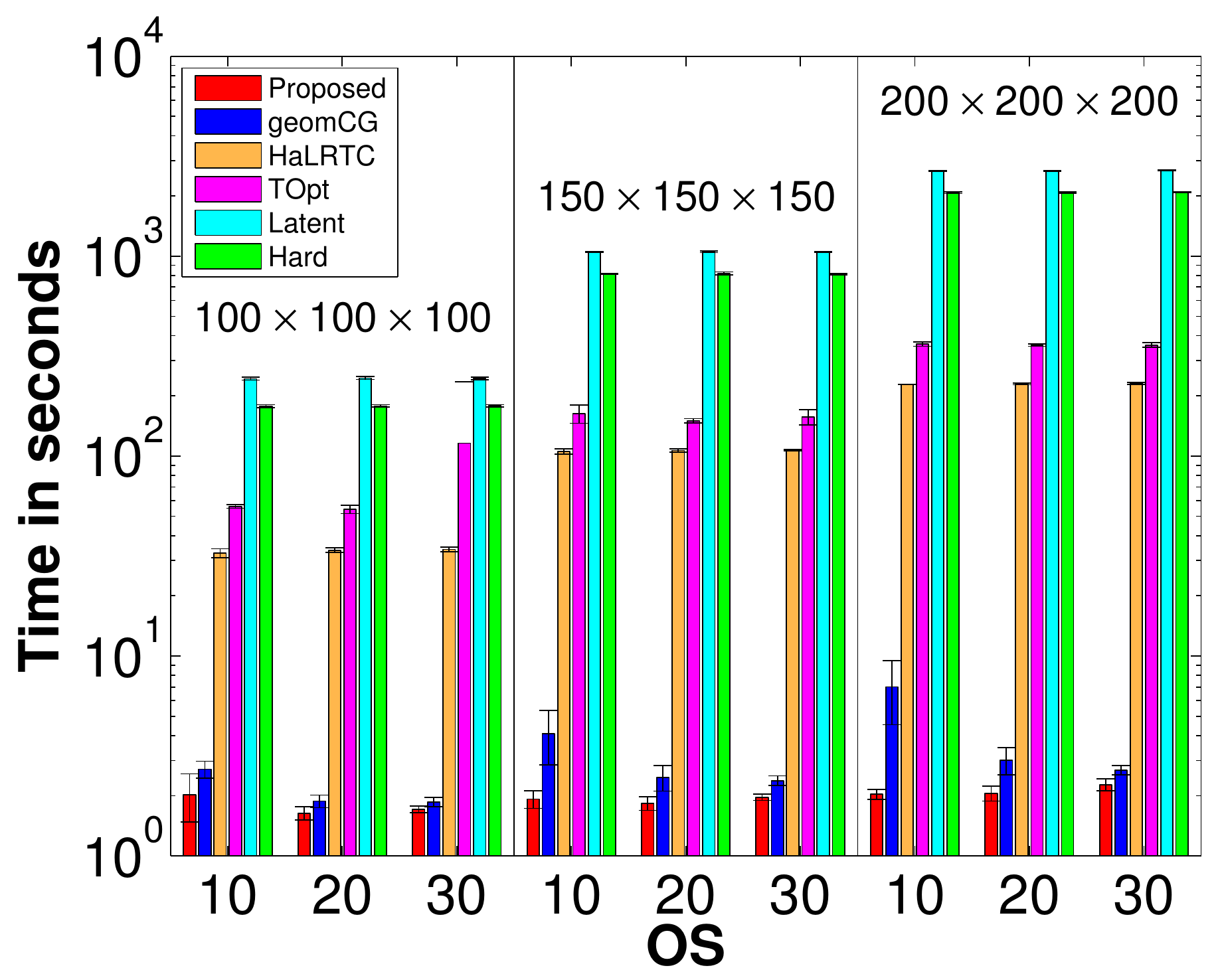}\\
{\scriptsize(a) {\bf r} = ($5,5,5$).}
\end{center}
\end{minipage}
\begin{minipage}{0.32\hsize}
\begin{center}
\includegraphics[width=\hsize]{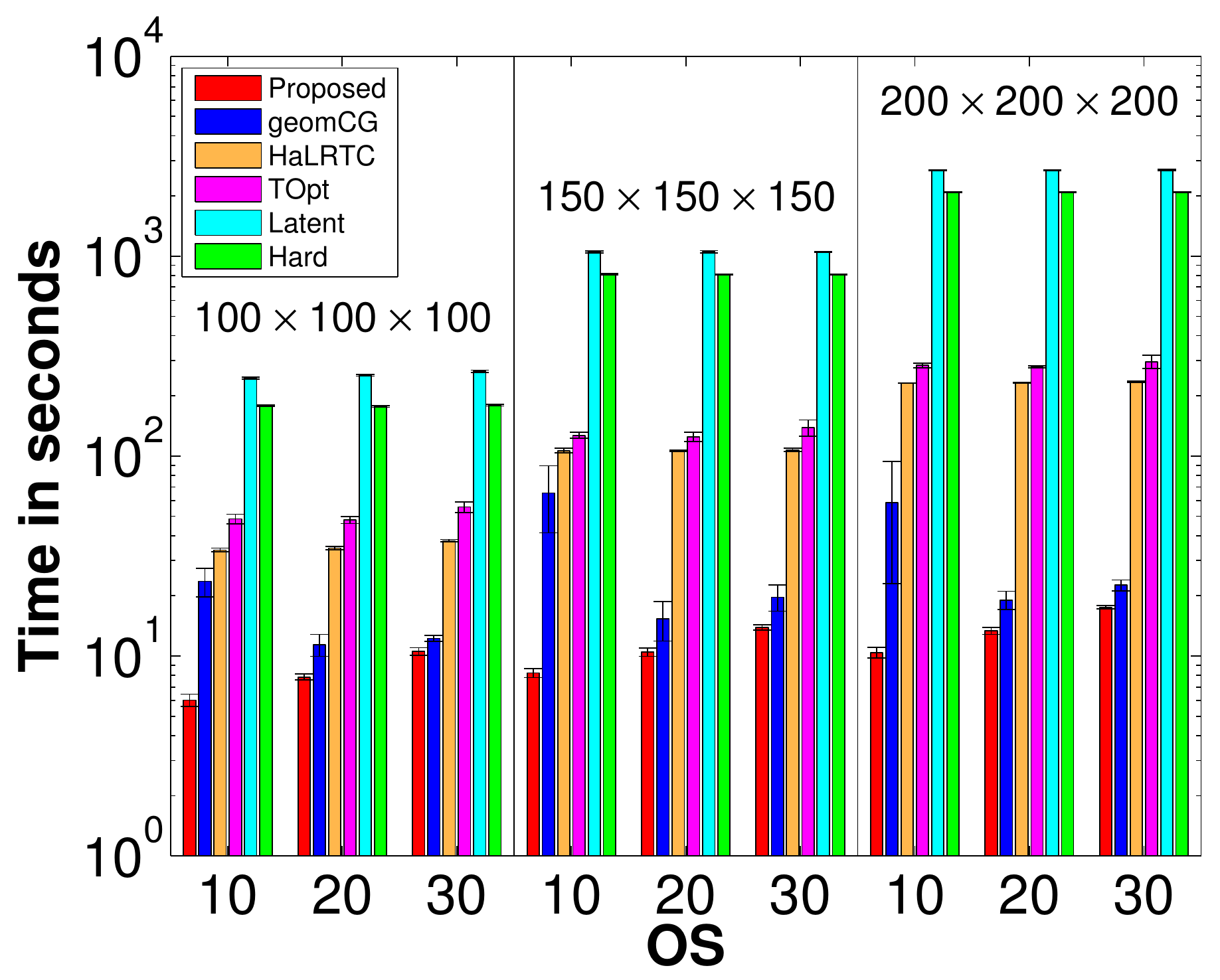}\\
{\scriptsize(b) {\bf r} = ($10,10,10$).}
\end{center}
\end{minipage}
\begin{minipage}{0.32\hsize}
\begin{center}
\includegraphics[width=\hsize]{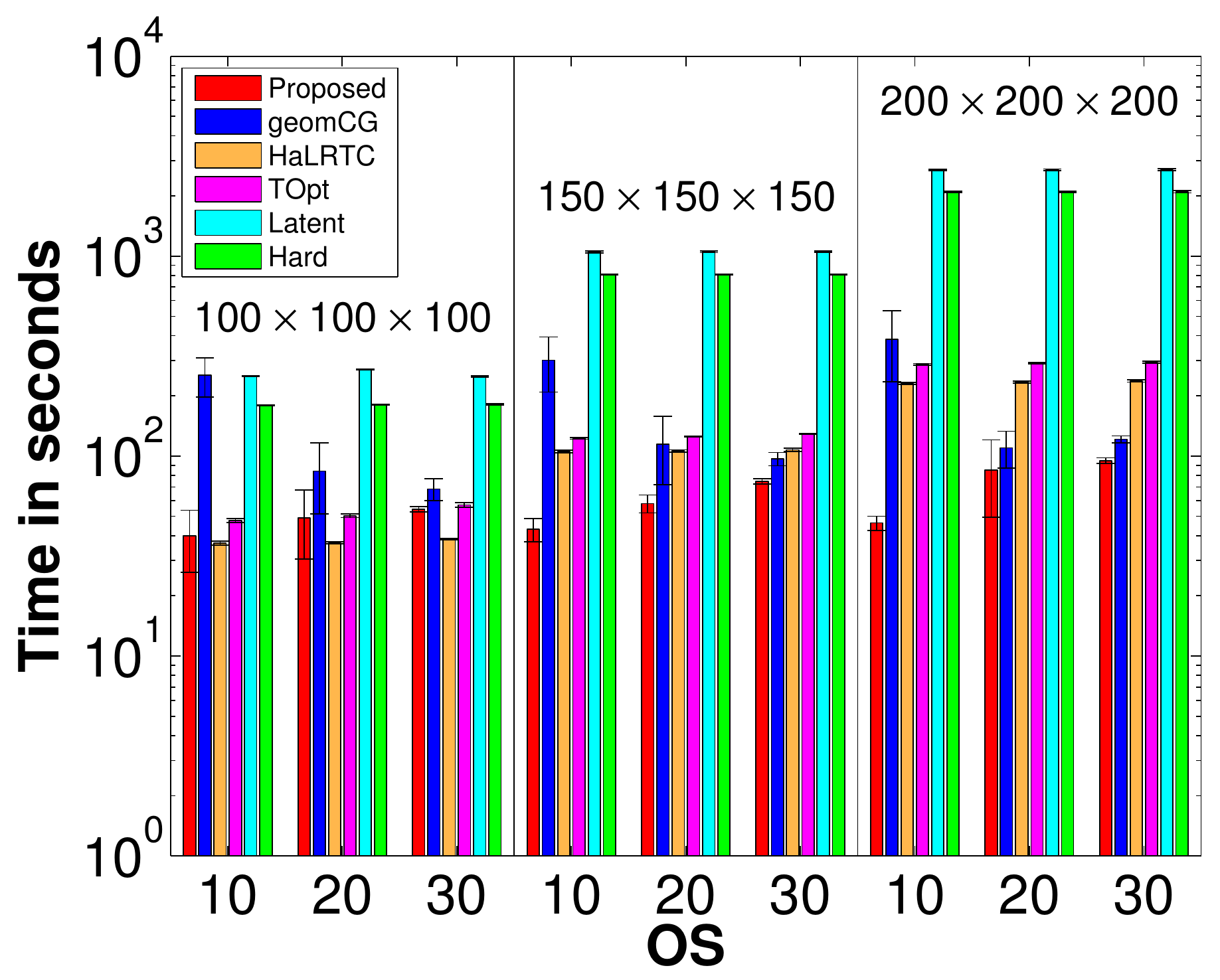}\\
{\scriptsize(c) {\bf r} = ($15,15,15$).}
\end{center}
\end{minipage}\\
\begin{minipage}{0.32\hsize}
\begin{center}
\includegraphics[width=\hsize]{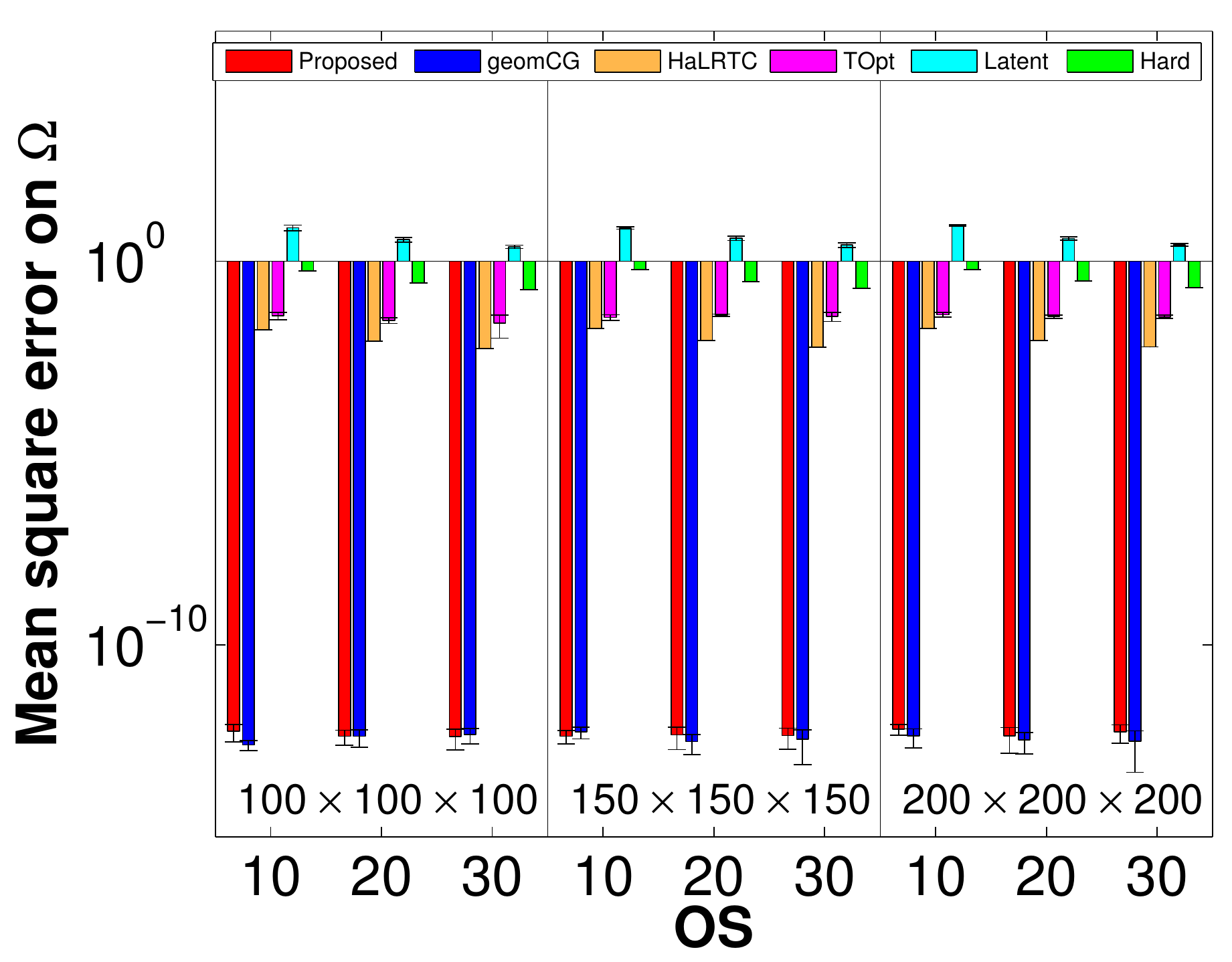}\\
{\scriptsize(d) {\bf r} = ($5,5,5$).}
\end{center}
\end{minipage}
\begin{minipage}{0.32\hsize}
\begin{center}
\includegraphics[width=\hsize]{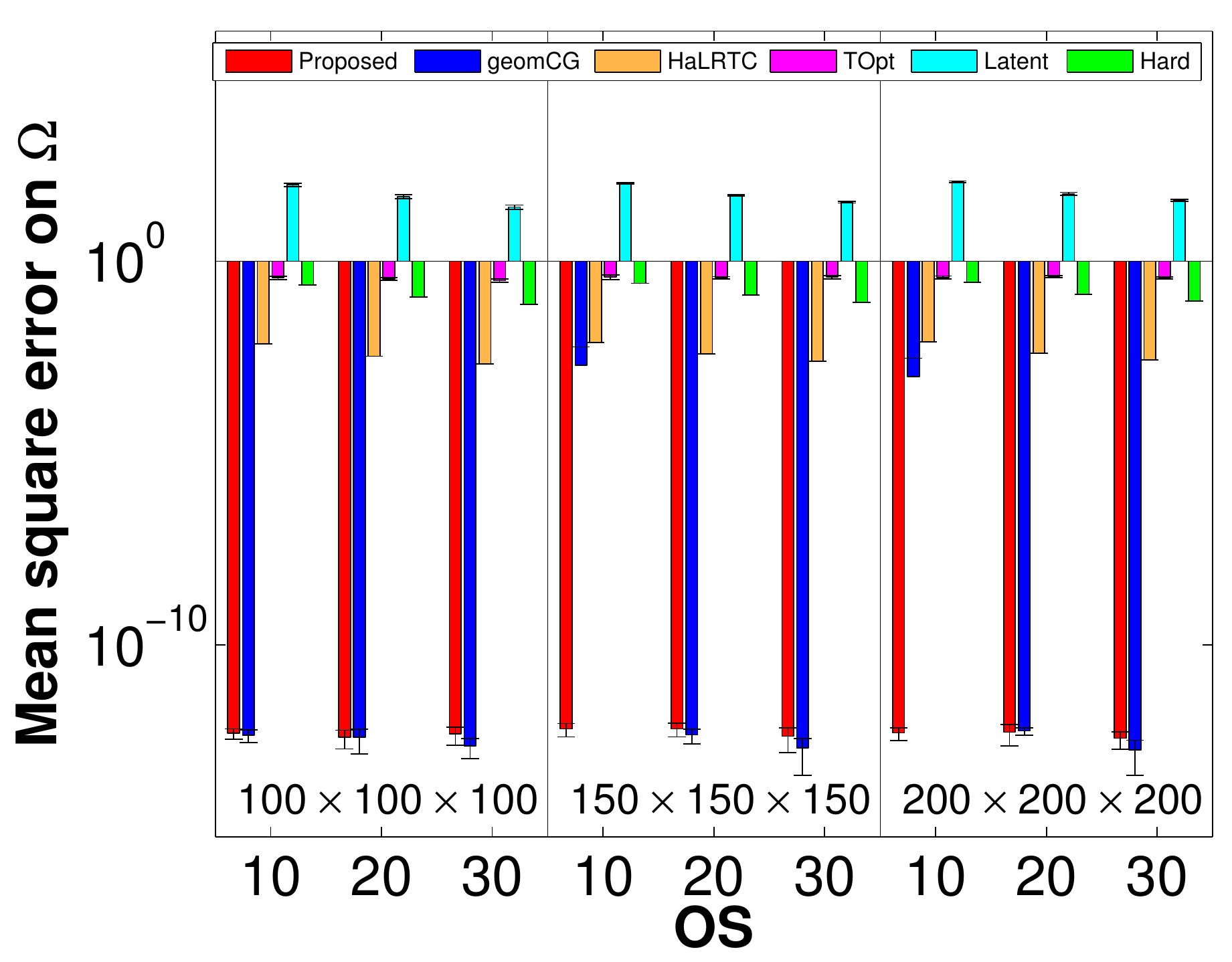}\\
{\scriptsize(e) {\bf r} = ($10,10,10$).}
\end{center}
\end{minipage}
\begin{minipage}{0.32\hsize}
\begin{center}
\includegraphics[width=\hsize]{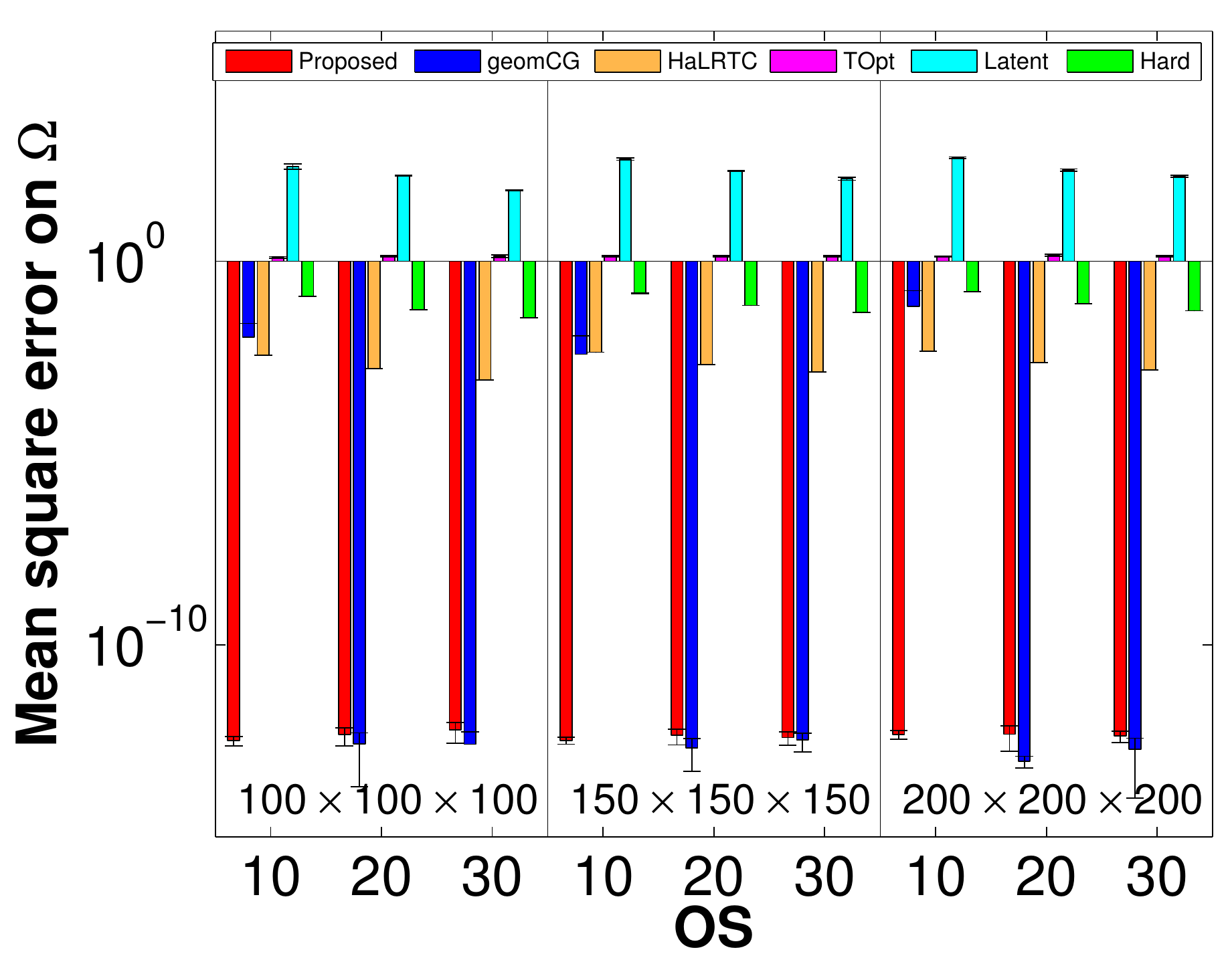}\\
{\scriptsize(f) {\bf r} = ($15,15,15$).}
\end{center}
\end{minipage}\\
\begin{minipage}{0.32\hsize}
\begin{center}
\includegraphics[width=\hsize]{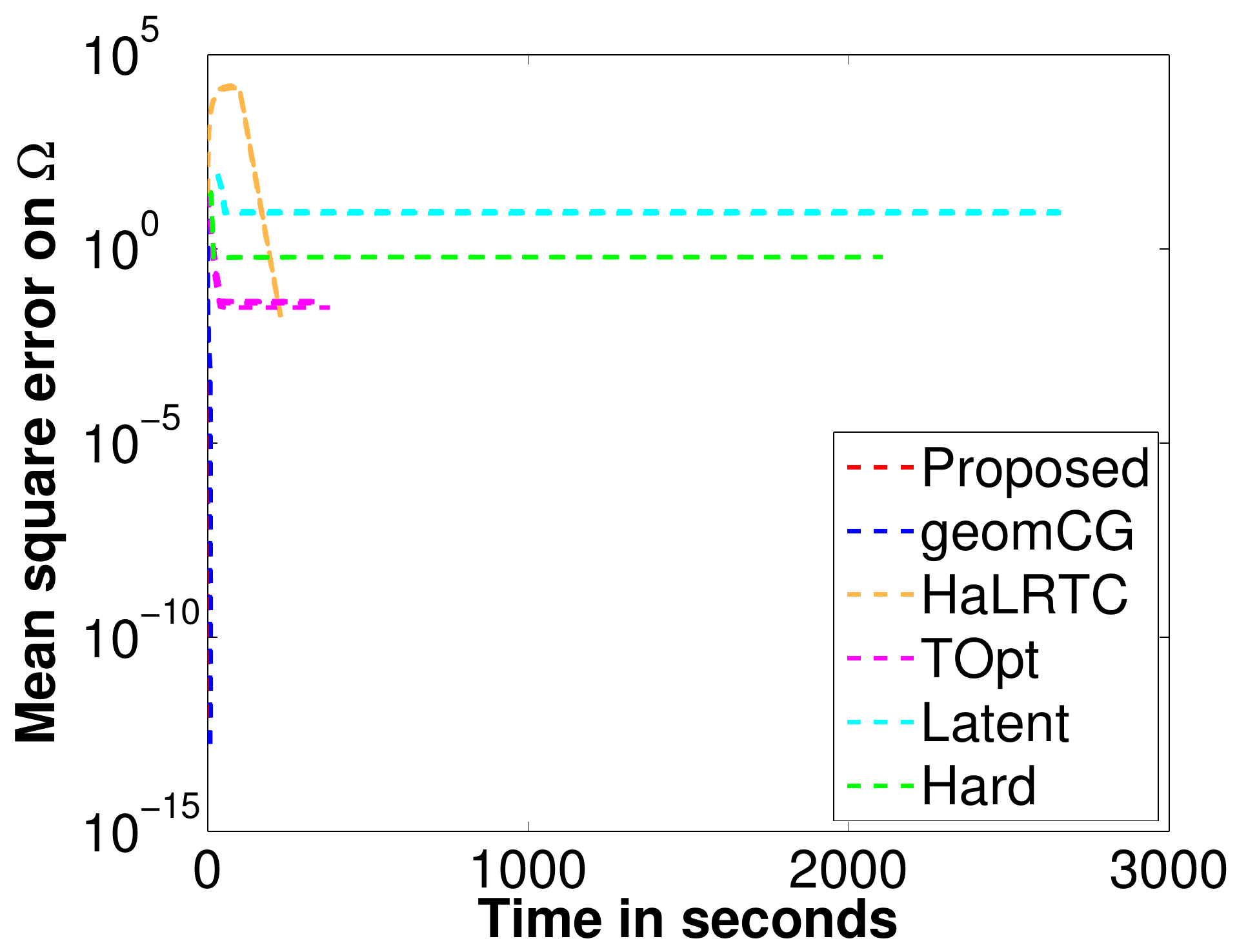}\\
{\scriptsize(g) $200\times 200\times 200$, OS = $10$, \\ {\bf r} = ($5,5,5$).}
\end{center}
\end{minipage}
\begin{minipage}{0.32\hsize}
\begin{center}
\includegraphics[width=\hsize]{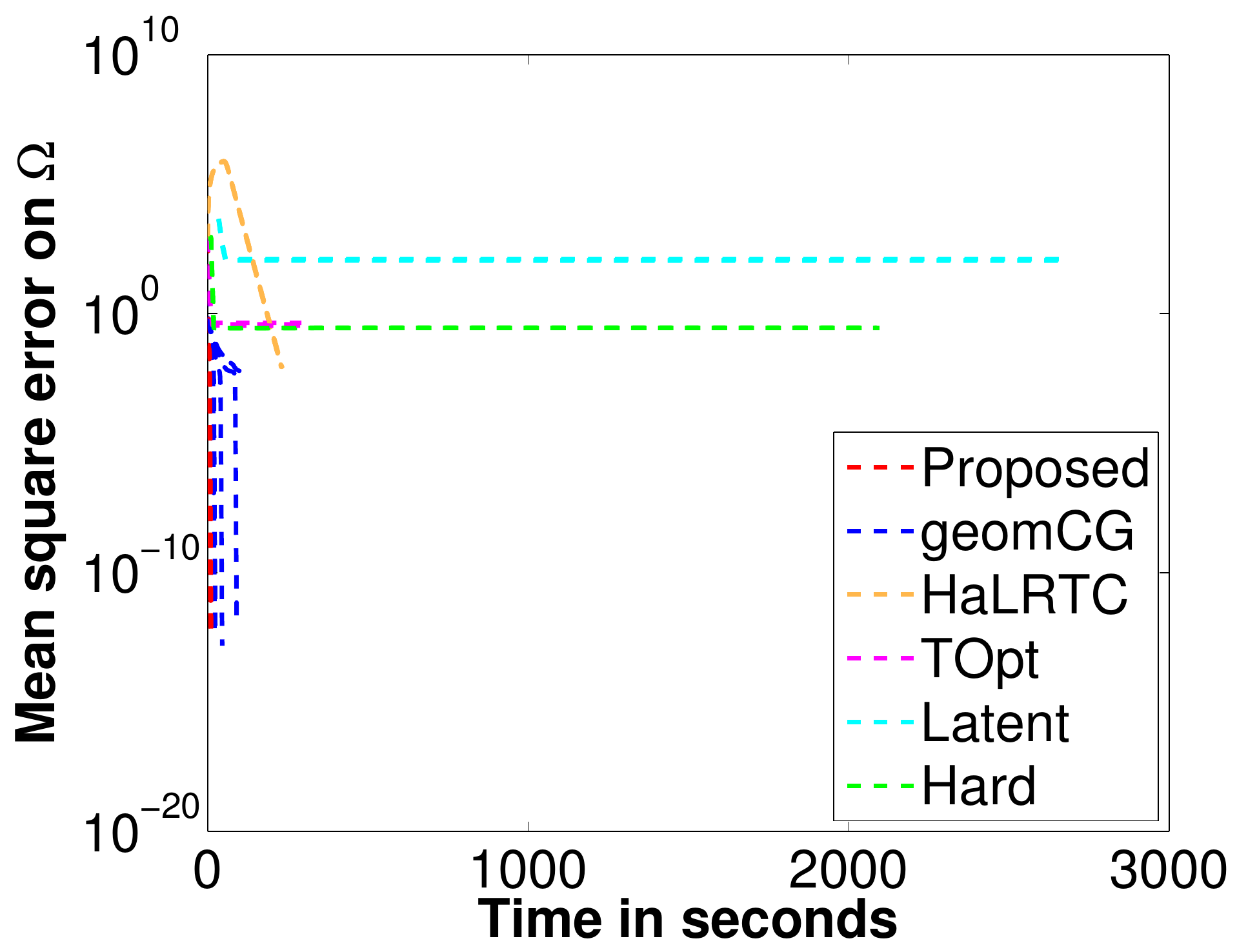}\\
{\scriptsize(h) $200\times 200\times 200$, OS = $10$, \\ {\bf r} = ($10,10,10$).}
\end{center}
\end{minipage}
\begin{minipage}{0.32\hsize}
\begin{center}
\includegraphics[width=\hsize]{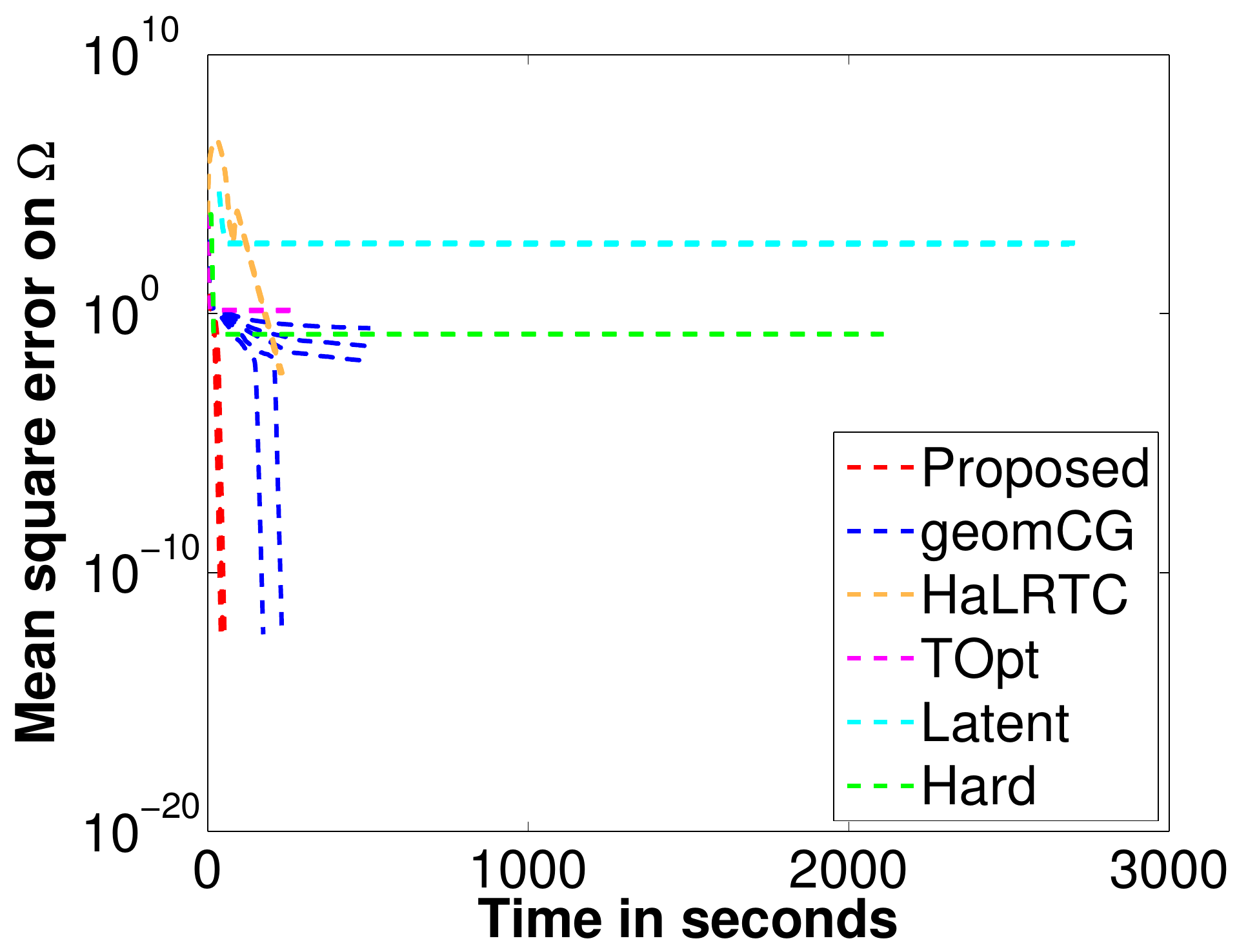}\\
{\scriptsize(i) $200\times 200\times 200$, OS = $10$, \\ {\bf r} = ($15,15,15$).}
\end{center}
\end{minipage}\\
\end{tabular}
\caption{\changeHK{{\bf Case S2:} small-scale comparisons on $\Omega$ (train error). }}
\label{appnfig:small-scale-train}
\end{figure*}

\begin{figure*}[htbp]
\begin{tabular}{cc}
\begin{minipage}{0.32\hsize}
\begin{center}
\includegraphics[width=\hsize]{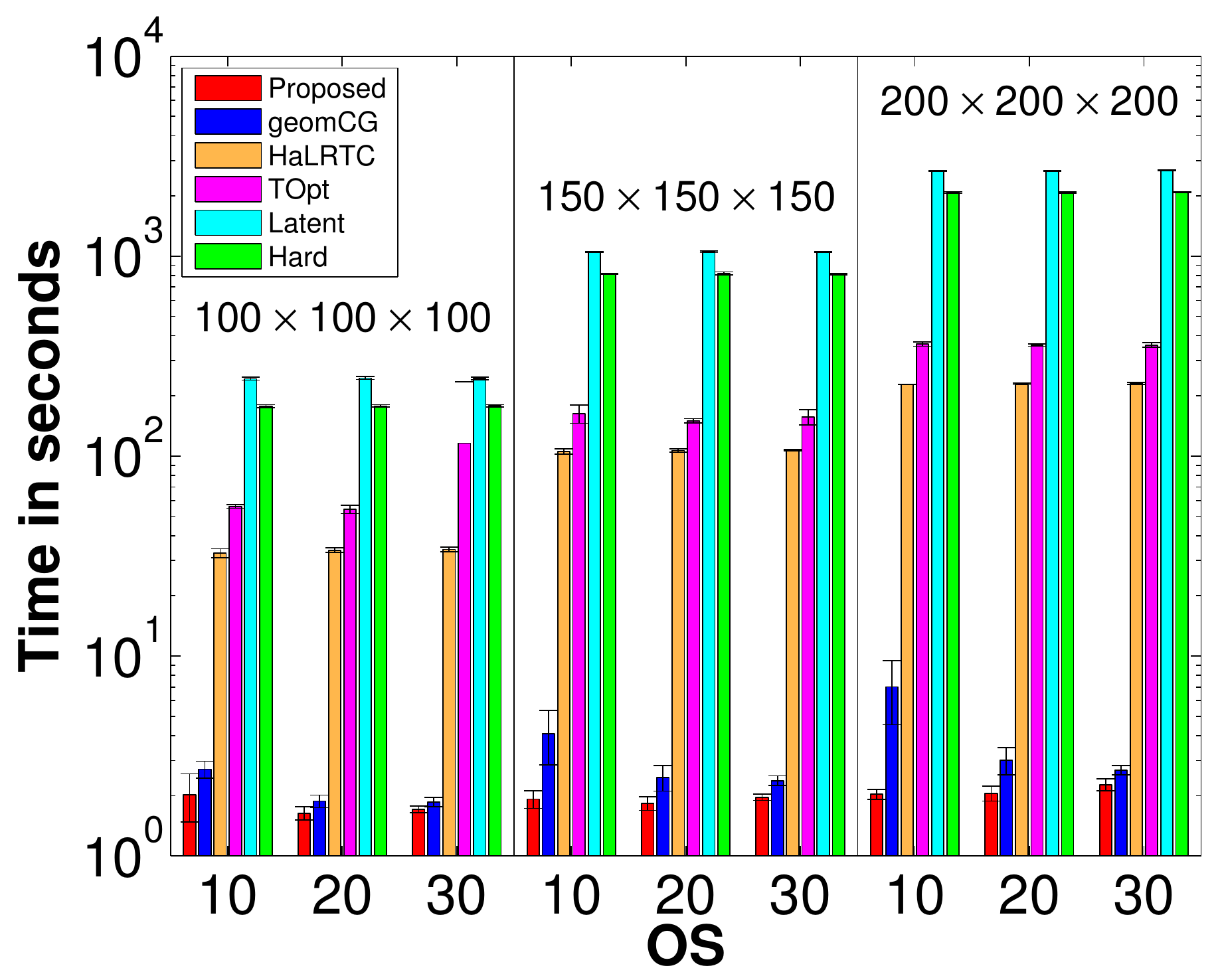}\\
{\scriptsize(a) {\bf r} = ($5,5,5$).}
\end{center}
\end{minipage}
\begin{minipage}{0.32\hsize}
\begin{center}
\includegraphics[width=\hsize]{figures/caseS2_core_10_time_bar-eps-converted-to.pdf}\\
{\scriptsize(b) {\bf r} = ($10,10,10$).}
\end{center}
\end{minipage}
\begin{minipage}{0.32\hsize}
\begin{center}
\includegraphics[width=\hsize]{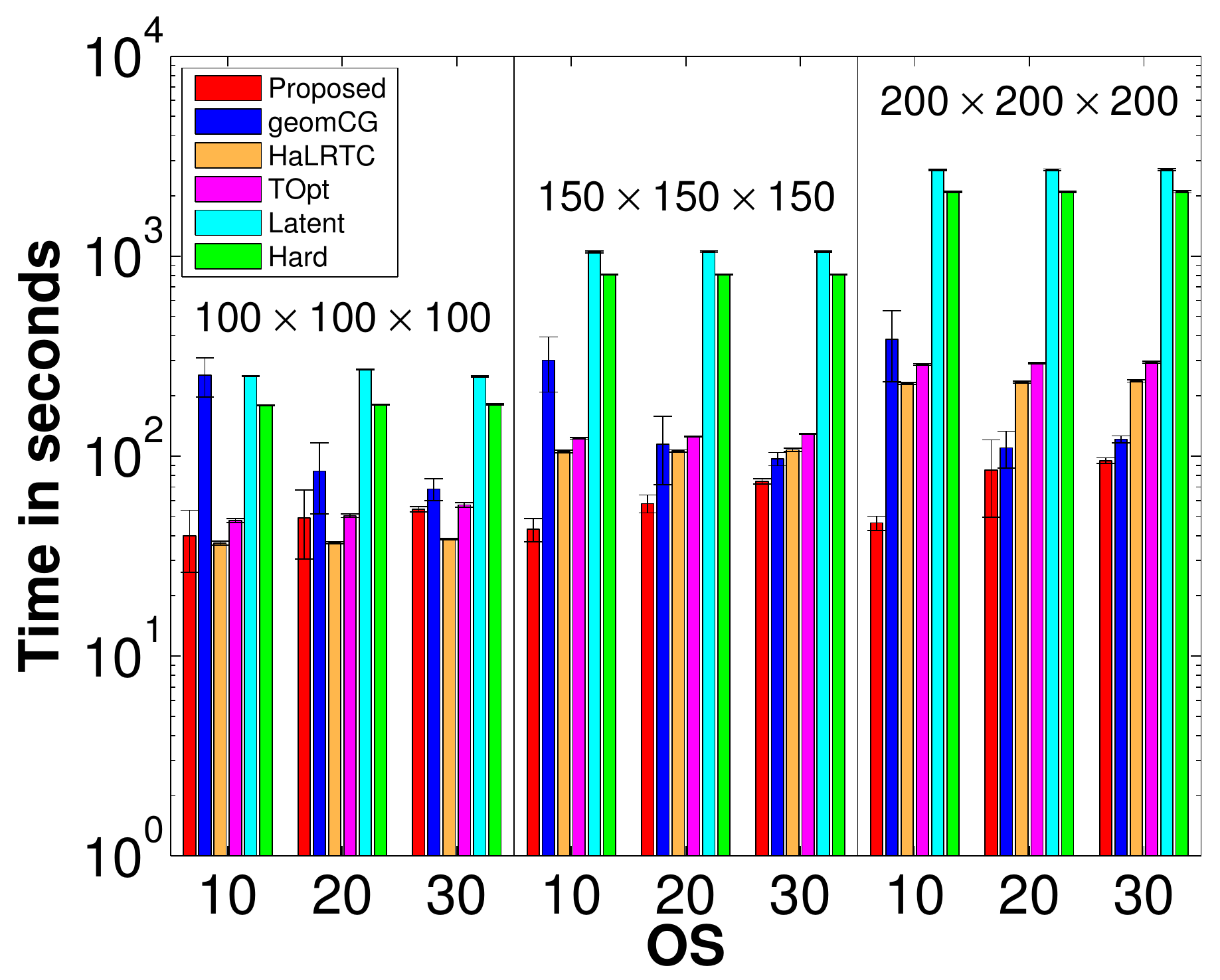}\\
{\scriptsize(c) {\bf r} = ($15,15,15$).}
\end{center}
\end{minipage}\\
\begin{minipage}{0.32\hsize}
\begin{center}
\includegraphics[width=\hsize]{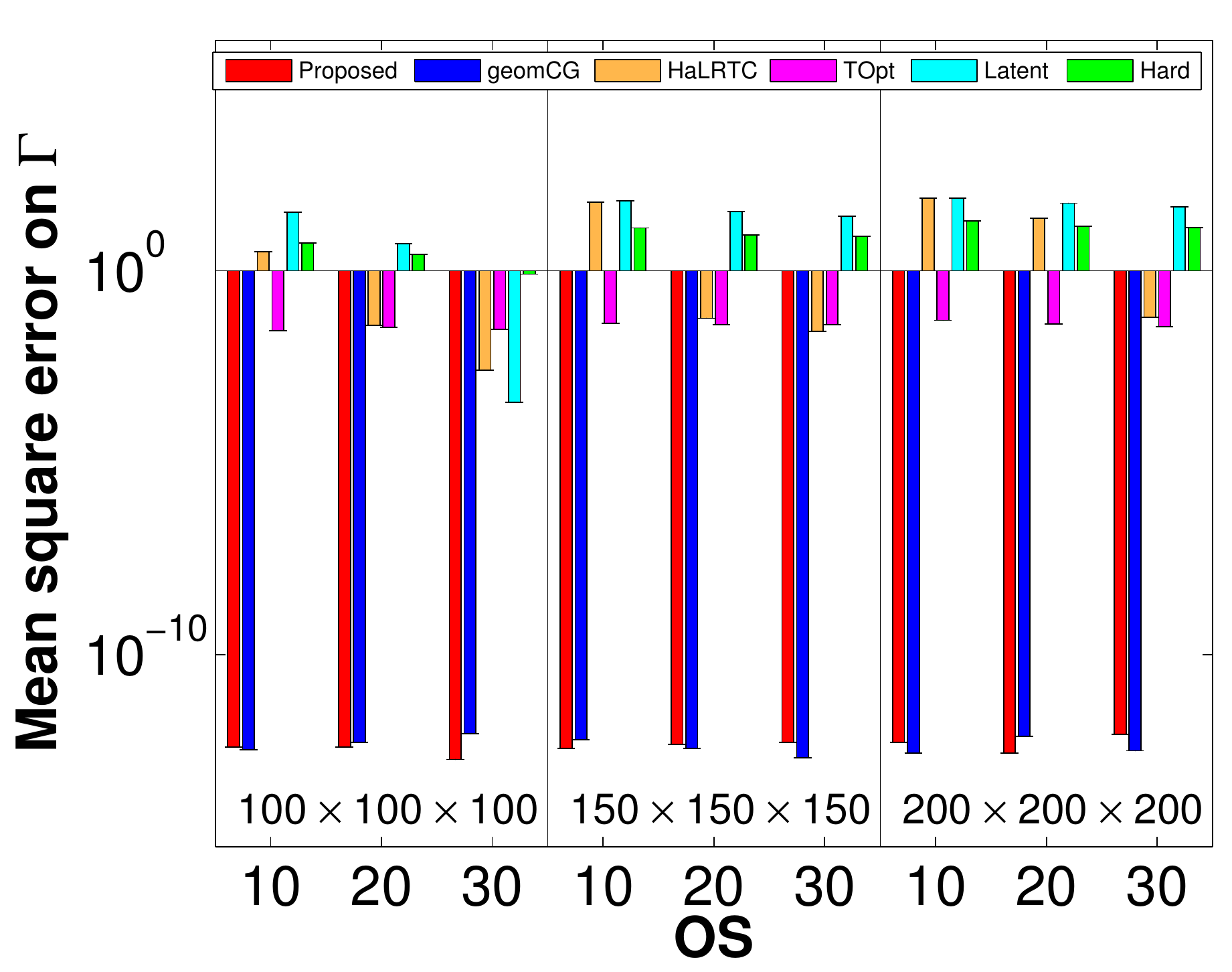}\\
{\scriptsize(d) {\bf r} = ($5,5,5$).}
\end{center}
\end{minipage}
\begin{minipage}{0.32\hsize}
\begin{center}
\includegraphics[width=\hsize]{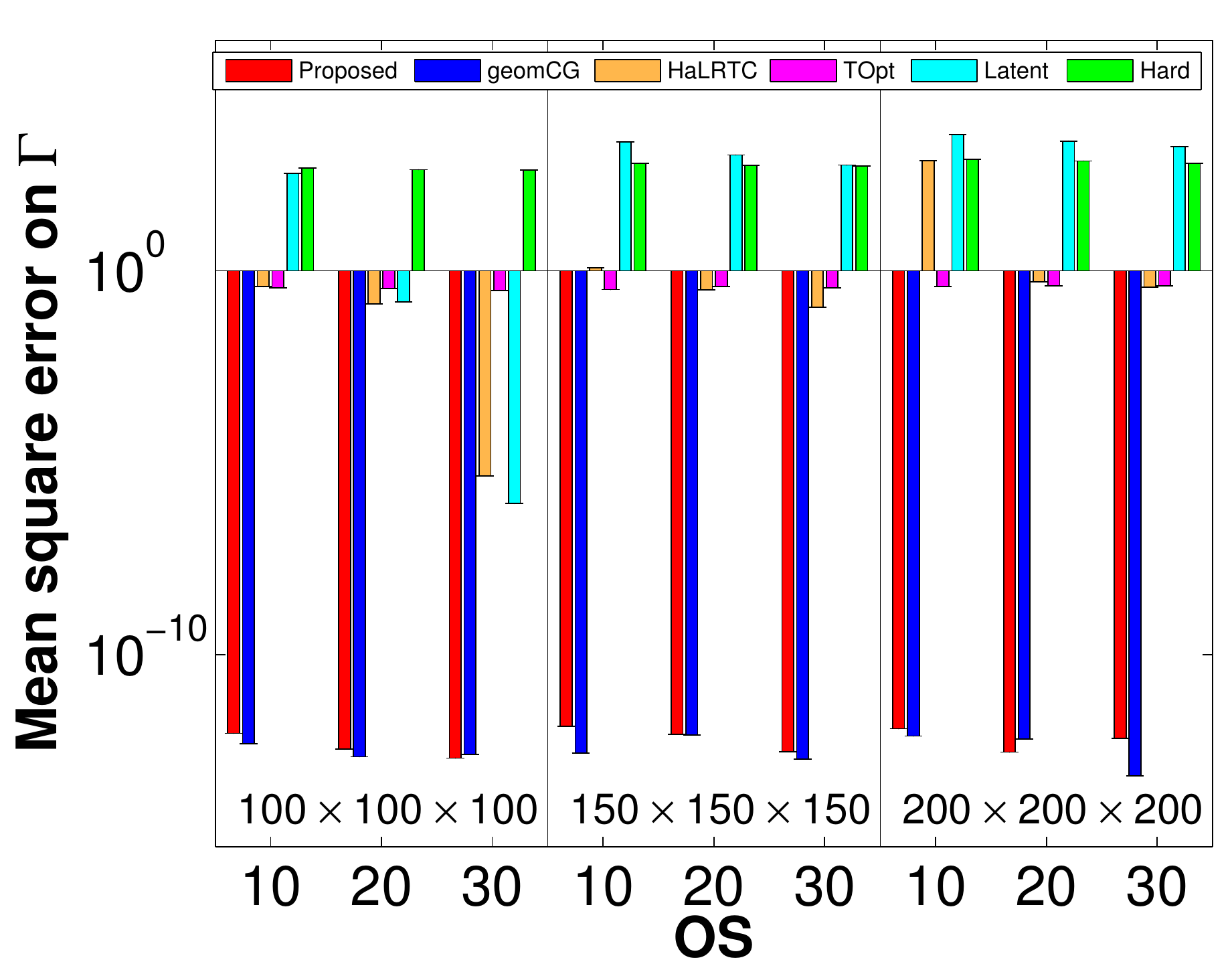}\\
{\scriptsize(e) {\bf r} = ($10,10,10$).}
\end{center}
\end{minipage}
\begin{minipage}{0.32\hsize}
\begin{center}
\includegraphics[width=\hsize]{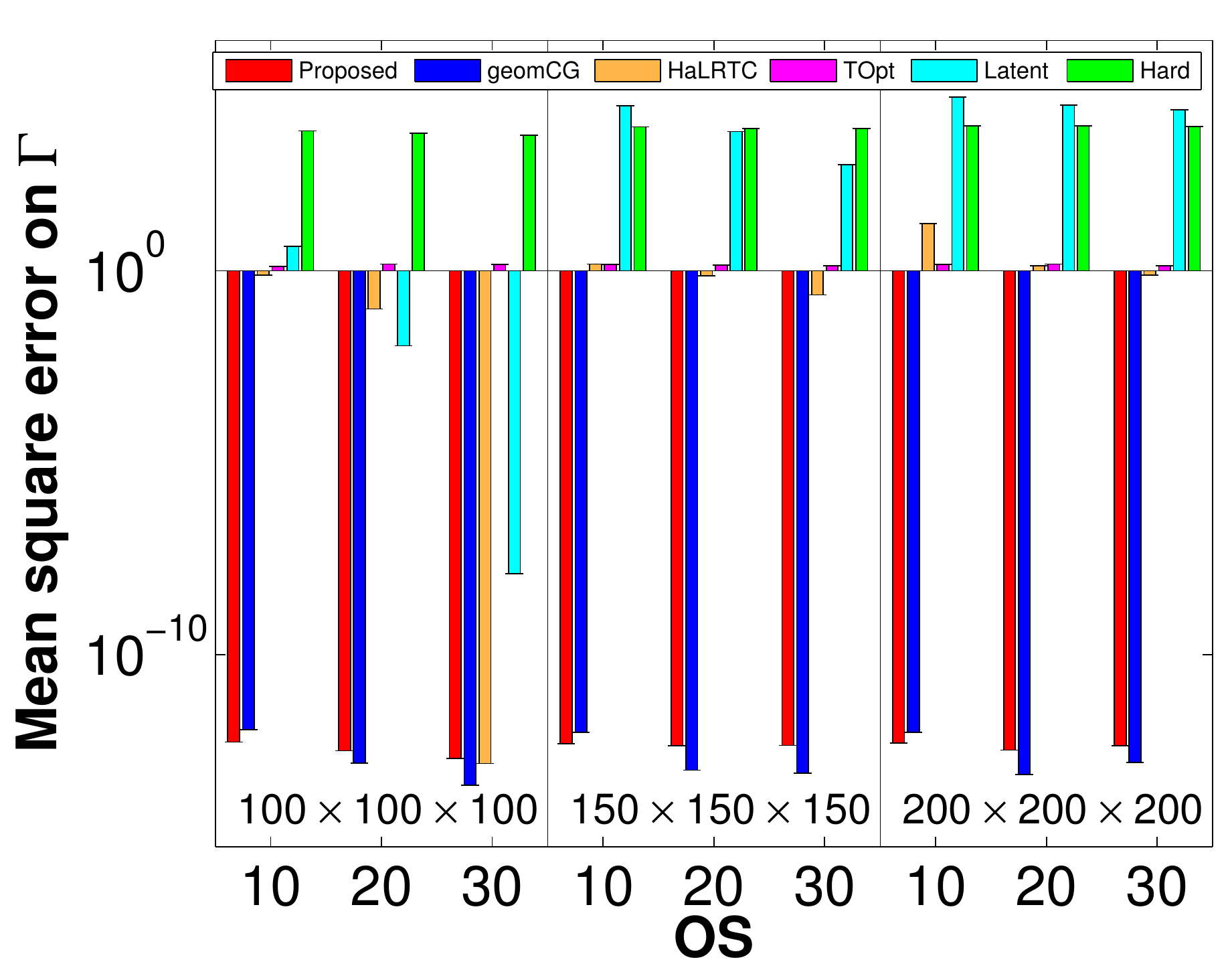}\\
{\scriptsize(f) {\bf r} = ($15,15,15$).}
\end{center}
\end{minipage}\\
\begin{minipage}{0.32\hsize}
\begin{center}
\includegraphics[width=\hsize]{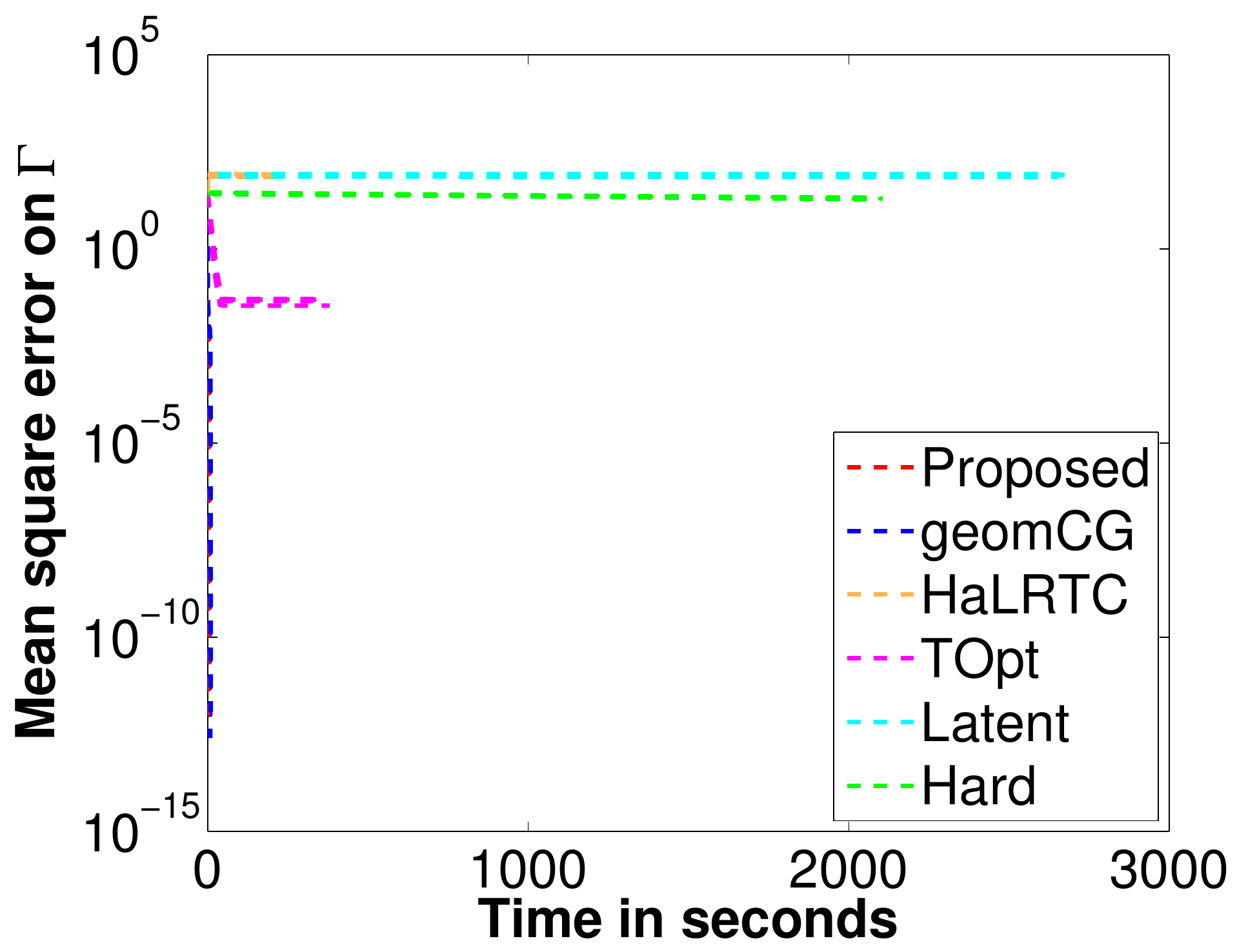}\\
{\scriptsize(g) $200\times 200\times 200$, OS = $10$, \\ {\bf r} = ($5,5,5$).}
\end{center}
\end{minipage}
\begin{minipage}{0.32\hsize}
\begin{center}
\includegraphics[width=\hsize]{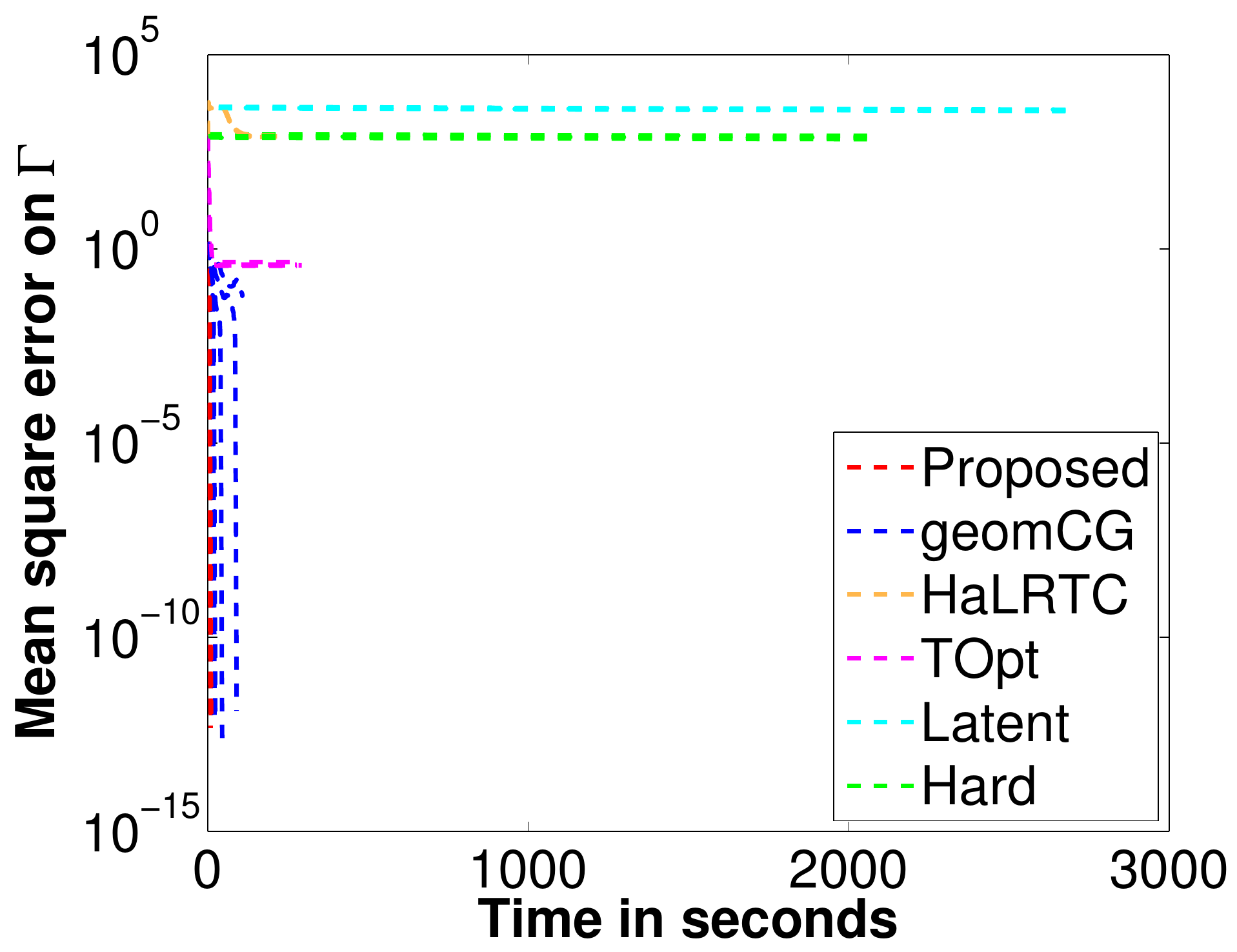}\\
{\scriptsize(h) $200\times 200\times 200$, OS = $10$, \\ {\bf r} = ($10,10,10$).}
\end{center}
\end{minipage}
\begin{minipage}{0.32\hsize}
\begin{center}
\includegraphics[width=\hsize]{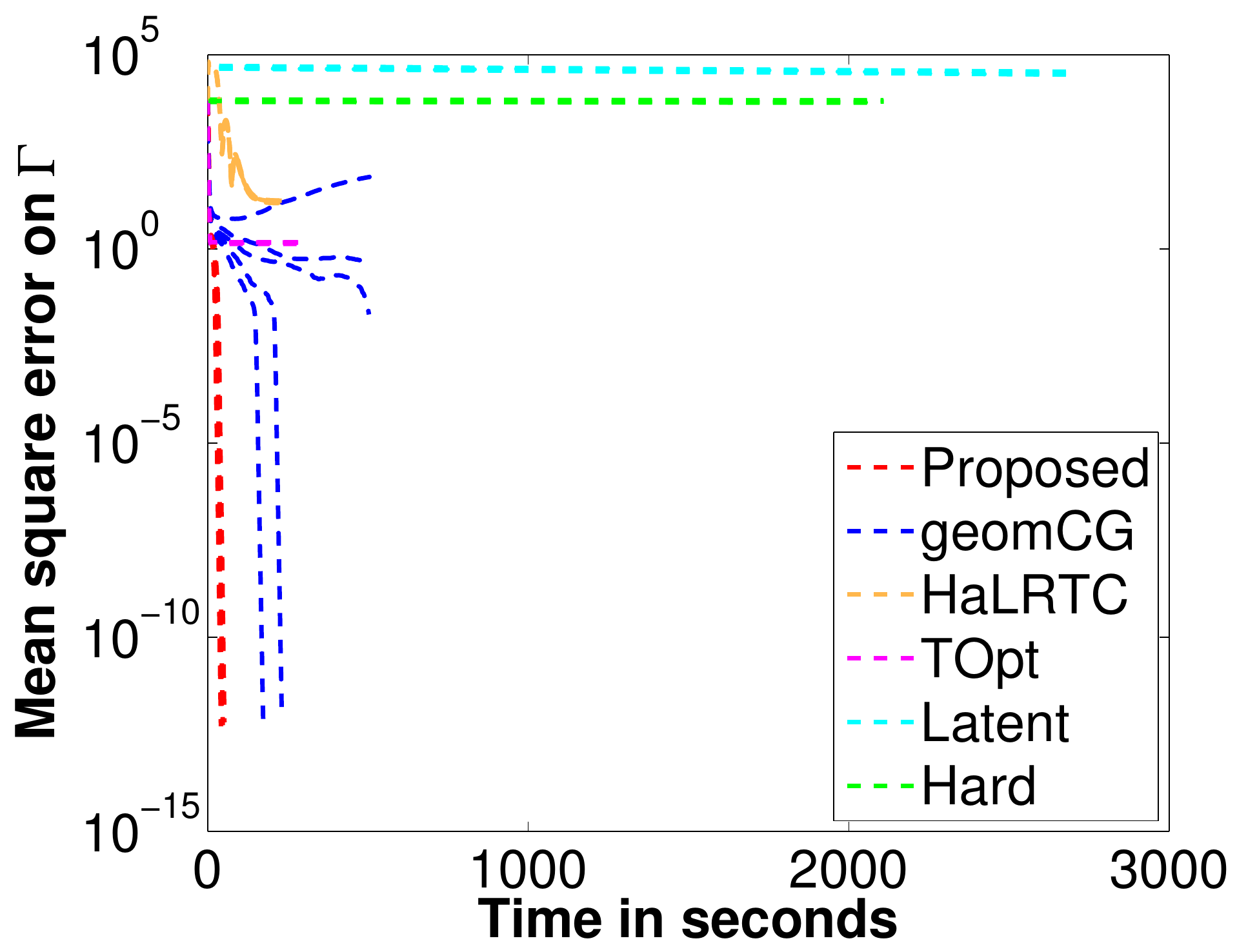}\\
{\scriptsize(i) $200\times 200\times 200$, OS = $10$, \\ {\bf r} = ($15,15,15$).}
\end{center}
\end{minipage}\\
\end{tabular}
\caption{{\bf Case S2:} small-scale comparisons on $\Gamma$ (test error).}
\label{appnfig:small-scale-test}
\end{figure*}

\begin{figure*}[htbp]
\begin{tabular}{ccc}
\begin{minipage}{0.32\hsize}
\begin{center}
\includegraphics[width=\hsize]{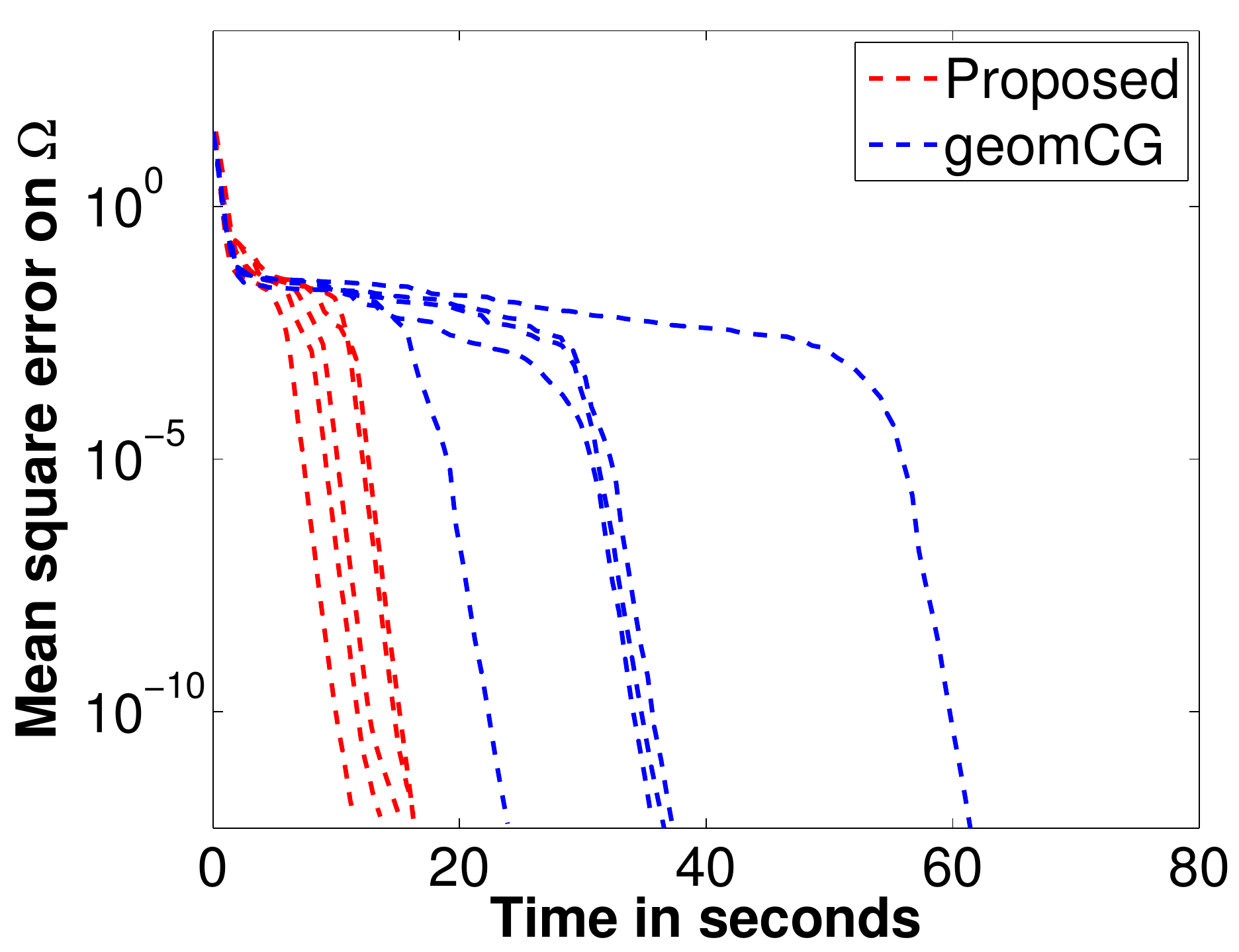}\\
{\scriptsize(a) $3000\times 3000\times 3000$, \\ {\bf r} = ($5\times 5\times 5$).}
\end{center}
\end{minipage}
\begin{minipage}{0.32\hsize}
\begin{center}
\includegraphics[width=\hsize]{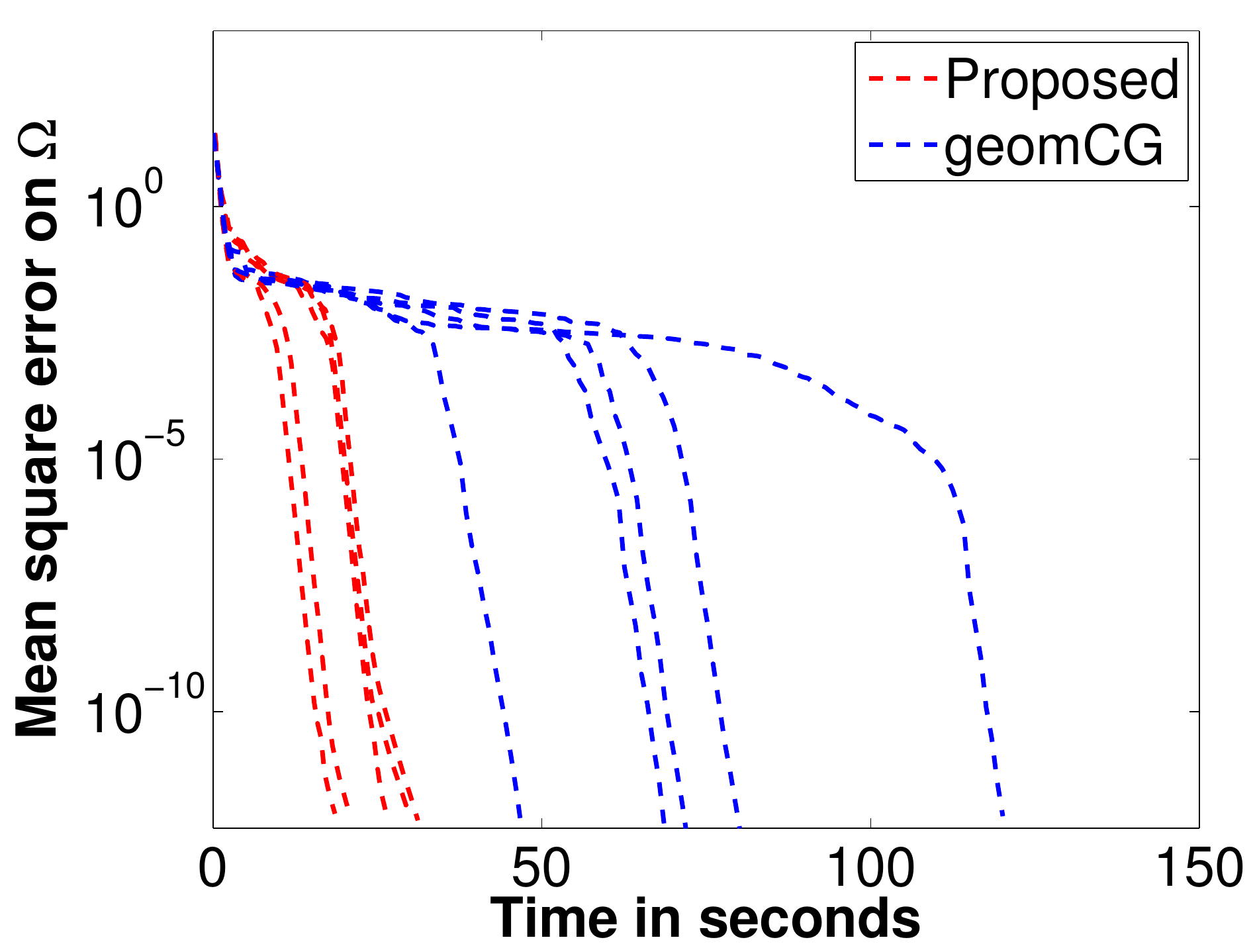}\\
{\scriptsize(b) $5000\times 5000\times 5000$, \\ {\bf r} = ($5\times 5\times 5$).}
\end{center}
\end{minipage}
\begin{minipage}{0.32\hsize}
\begin{center}
\includegraphics[width=\hsize]{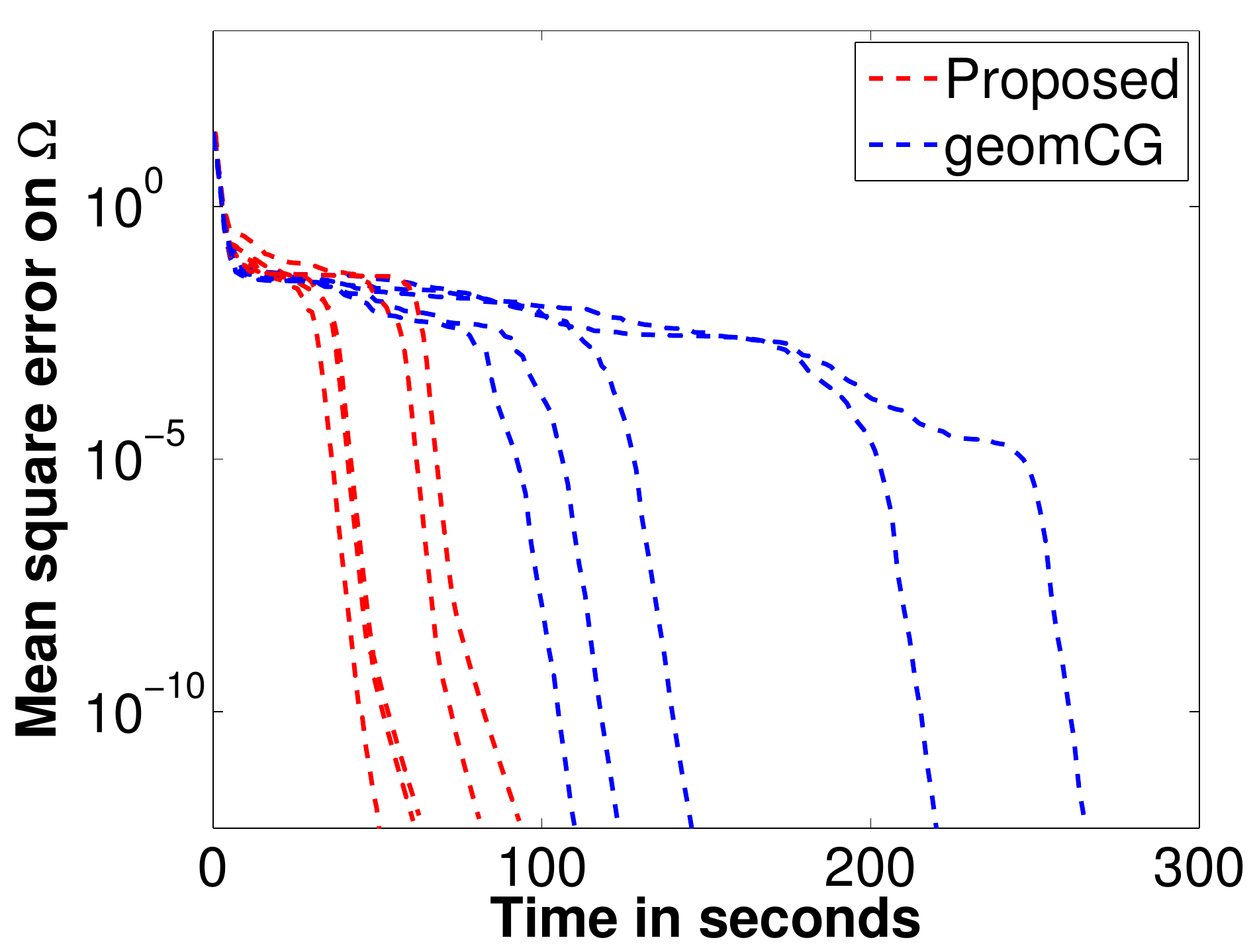}\\
{\scriptsize(c) $10000\times 10000\times 10000$, \\ {\bf r} = ($5\times 5\times 5$).}
\end{center}
\end{minipage}\\
\begin{minipage}{0.32\hsize}
\begin{center}
\includegraphics[width=\hsize]{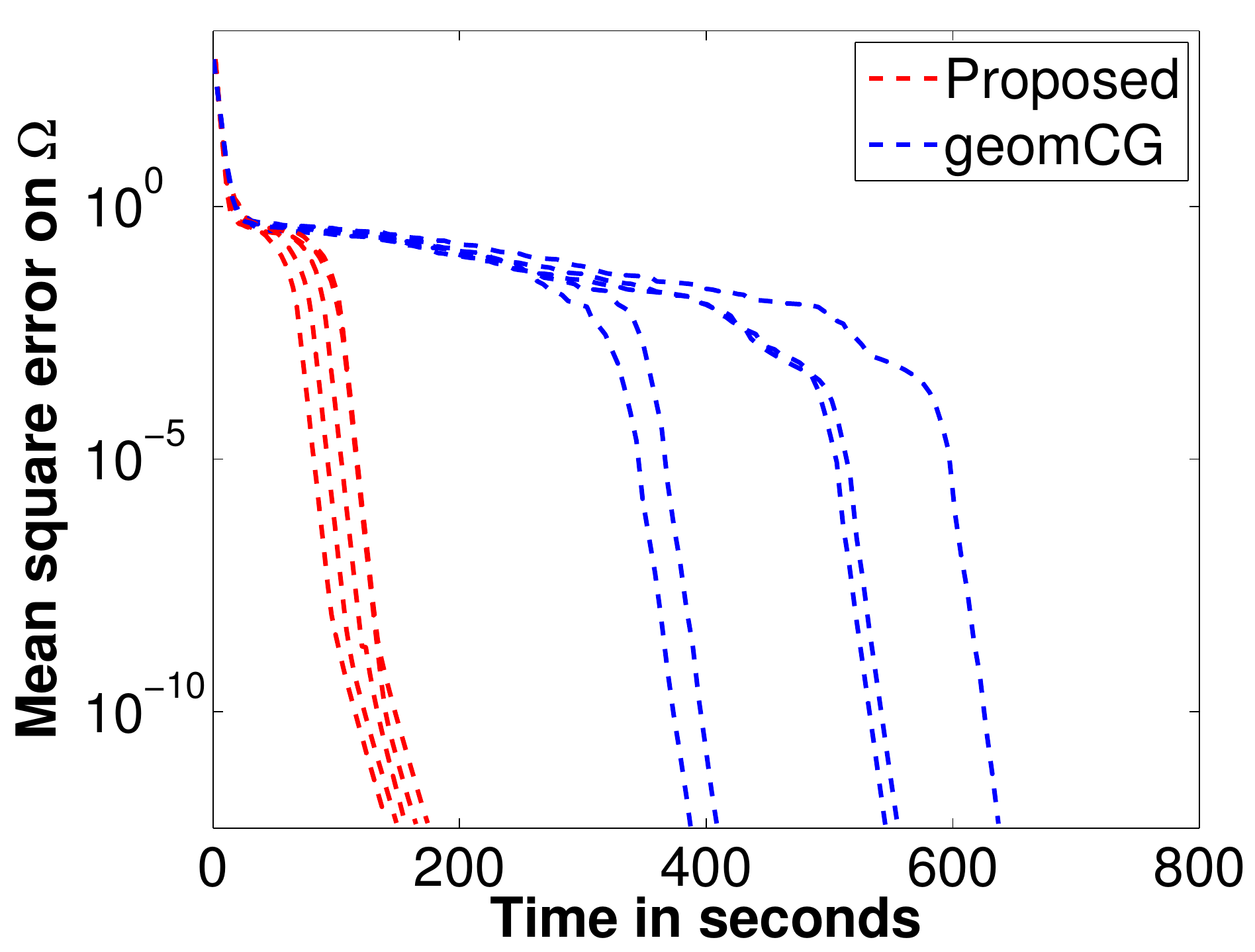}\\
{\scriptsize(d) $3000\times 3000\times 3000$, \\ {\bf r} = ($10\times 10\times 10$).}
\end{center}
\end{minipage}
\begin{minipage}{0.32\hsize}
\begin{center}
\includegraphics[width=\hsize]{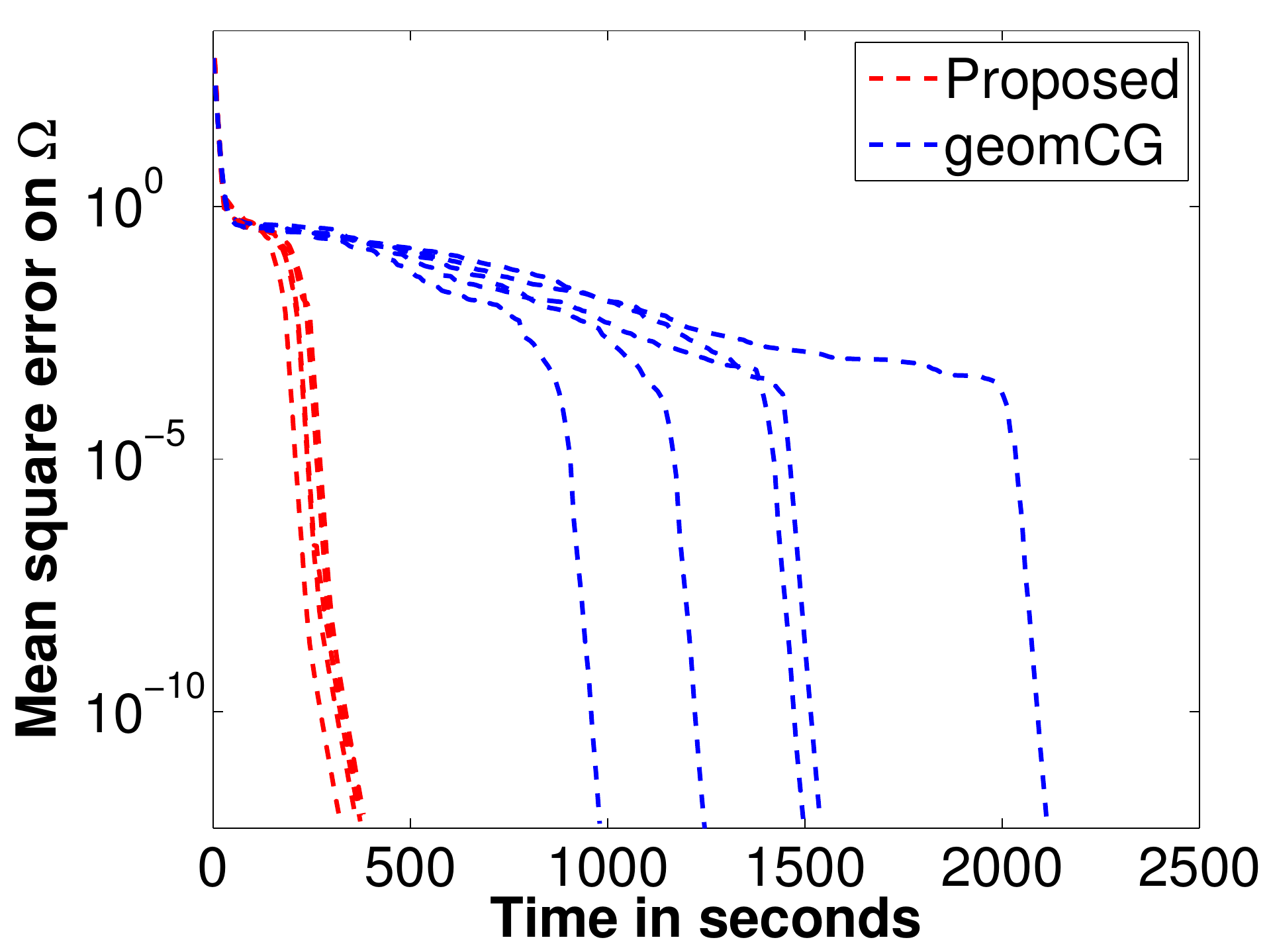}\\
{\scriptsize(e) $5000\times 5000\times 5000$, \\ {\bf r} = ($10\times 10\times 10$).}
\end{center}
\end{minipage}
\begin{minipage}{0.32\hsize}
\begin{center}
\includegraphics[width=\hsize]{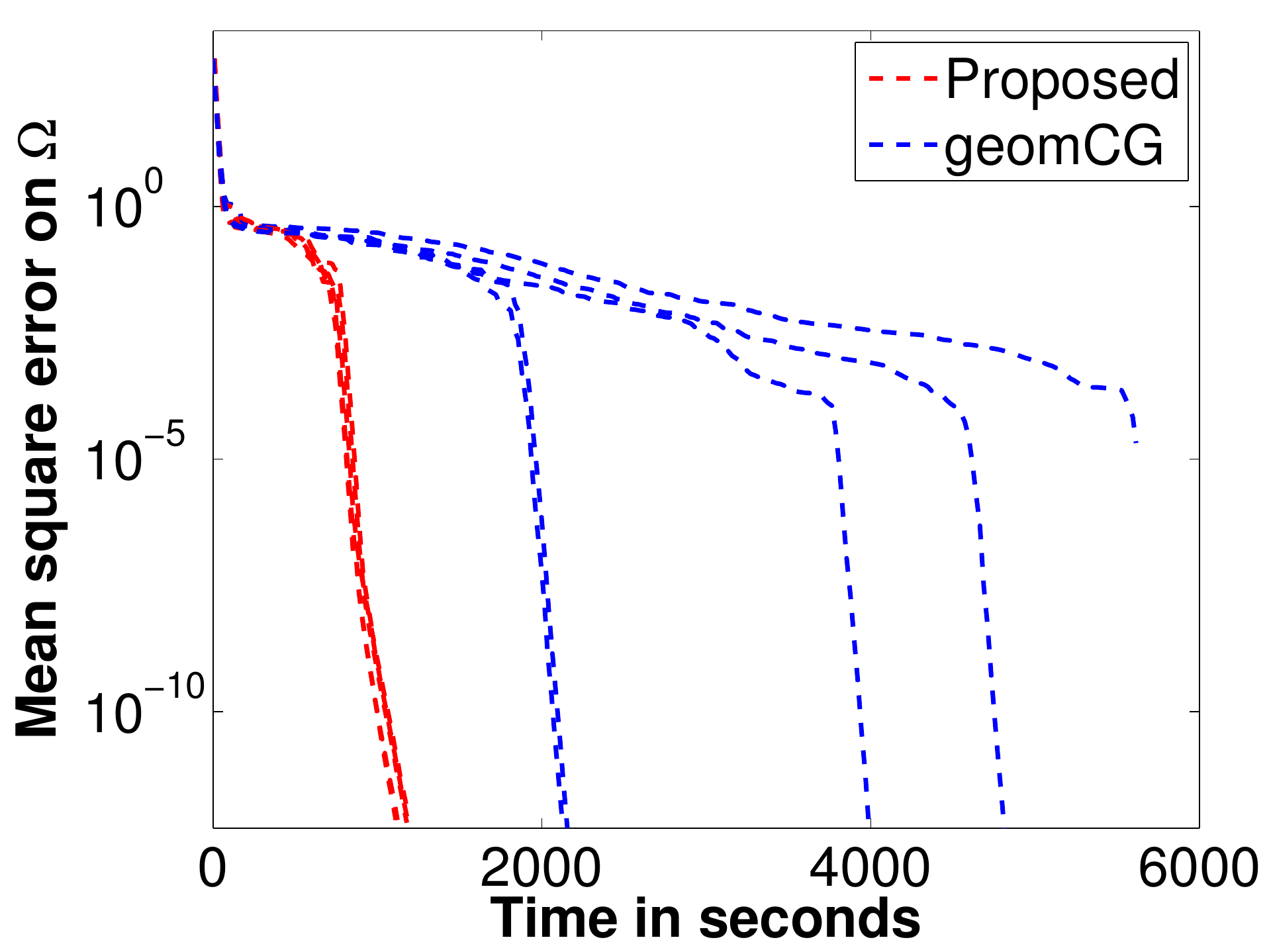}\\
{\scriptsize(f) $10000\times 10000\times 10000$, \\ {\bf r} = ($10\times 10\times 10$).}
\end{center}
\end{minipage}
\end{tabular}
\vspace{-0.1cm}
\caption{\changeHK{{\bf Case S3:} large-scale comparisons on $\Omega$ (train error).}}
\label{appnfig:large-scale-train}
\end{figure*}

\begin{figure*}[htbp]
\begin{tabular}{ccc}
\begin{minipage}{0.32\hsize}
\begin{center}
\includegraphics[width=\hsize]{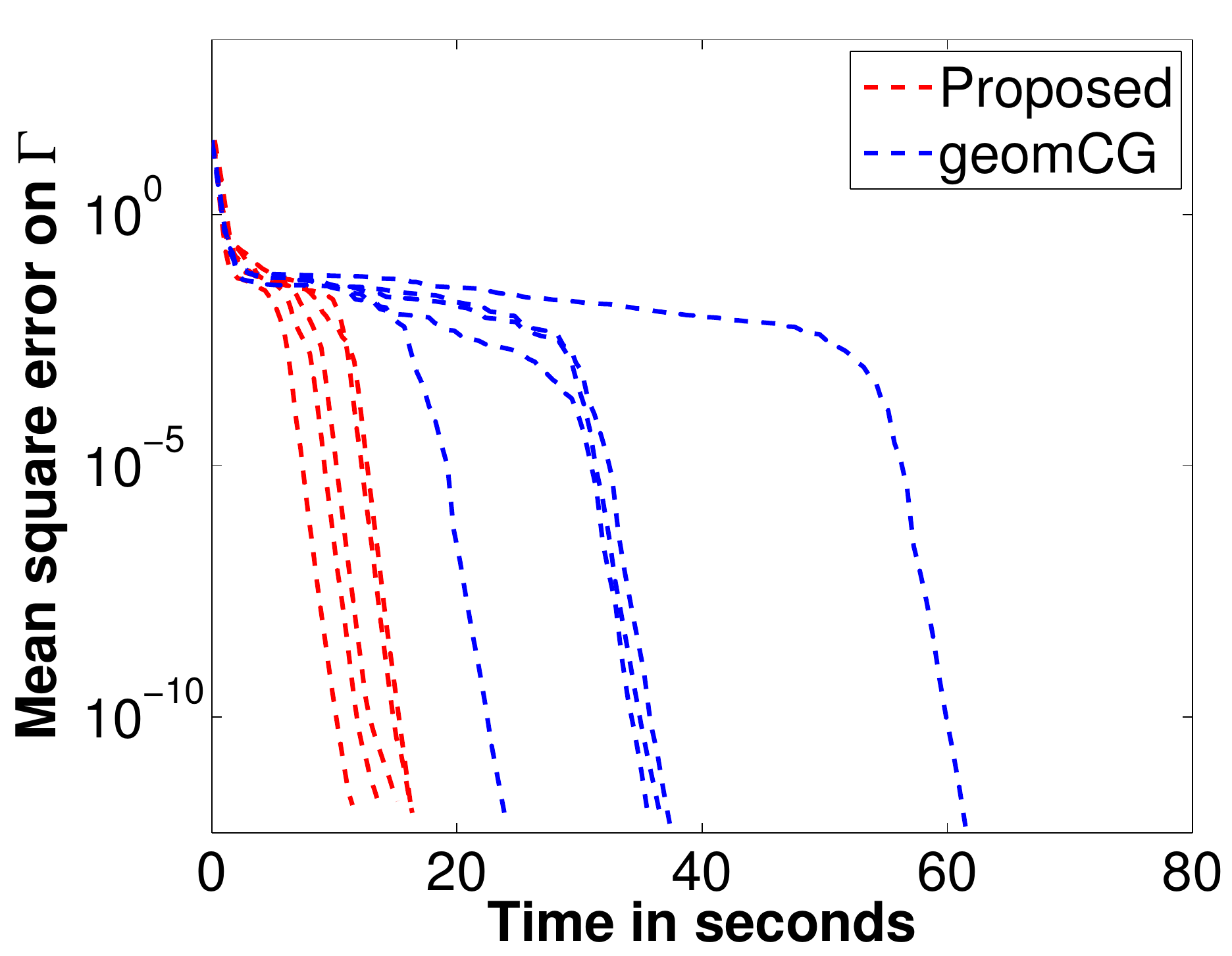}\\
{\scriptsize(a) $3000\times 3000\times 3000$, \\ {\bf r} = ($5\times 5\times 5$).}
\end{center}
\end{minipage}
\begin{minipage}{0.32\hsize}
\begin{center}
\includegraphics[width=\hsize]{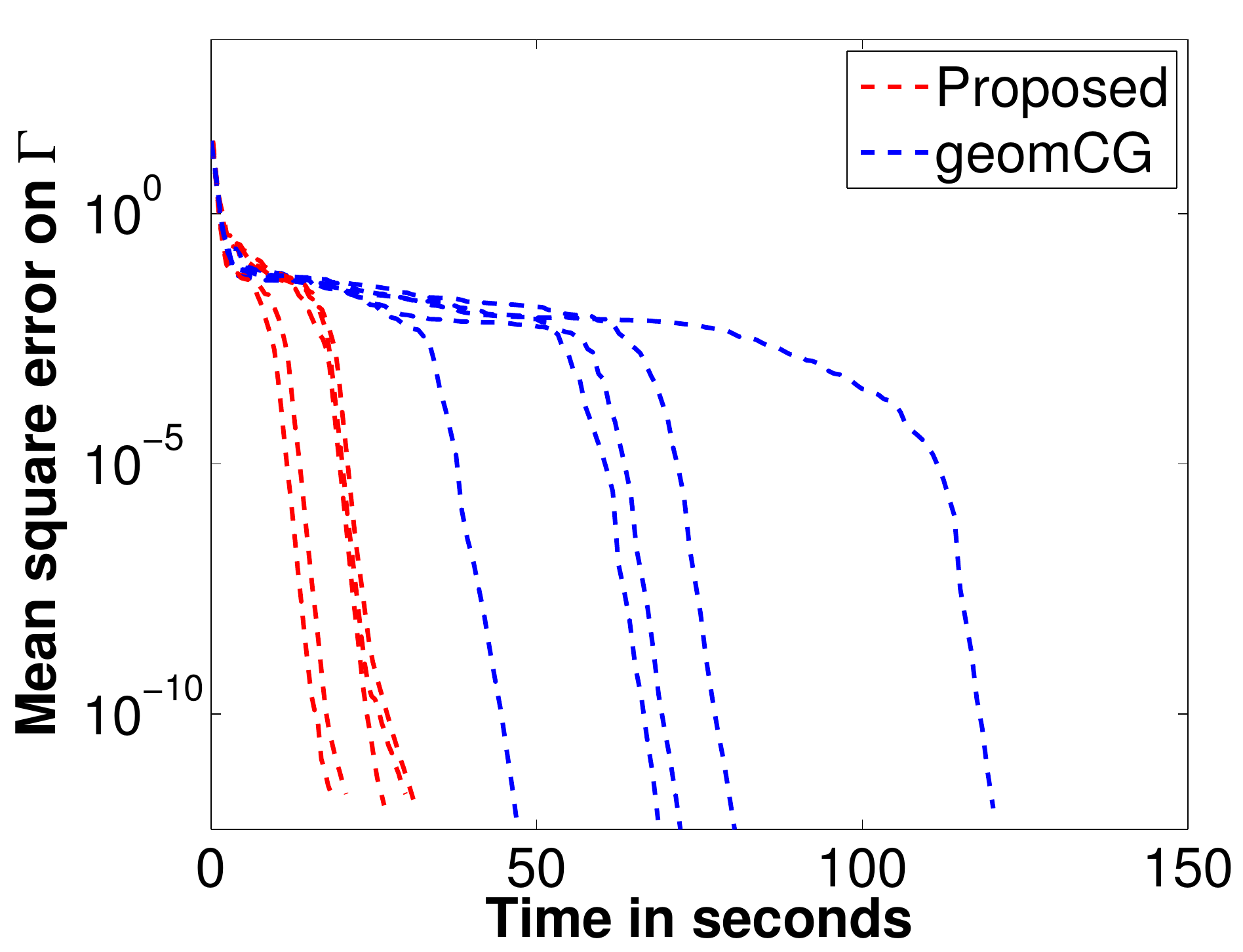}\\
{\scriptsize(b) $5000\times 5000\times 5000$, \\ {\bf r} = ($5\times 5\times 5$).}
\end{center}
\end{minipage}
\begin{minipage}{0.32\hsize}
\begin{center}
\includegraphics[width=\hsize]{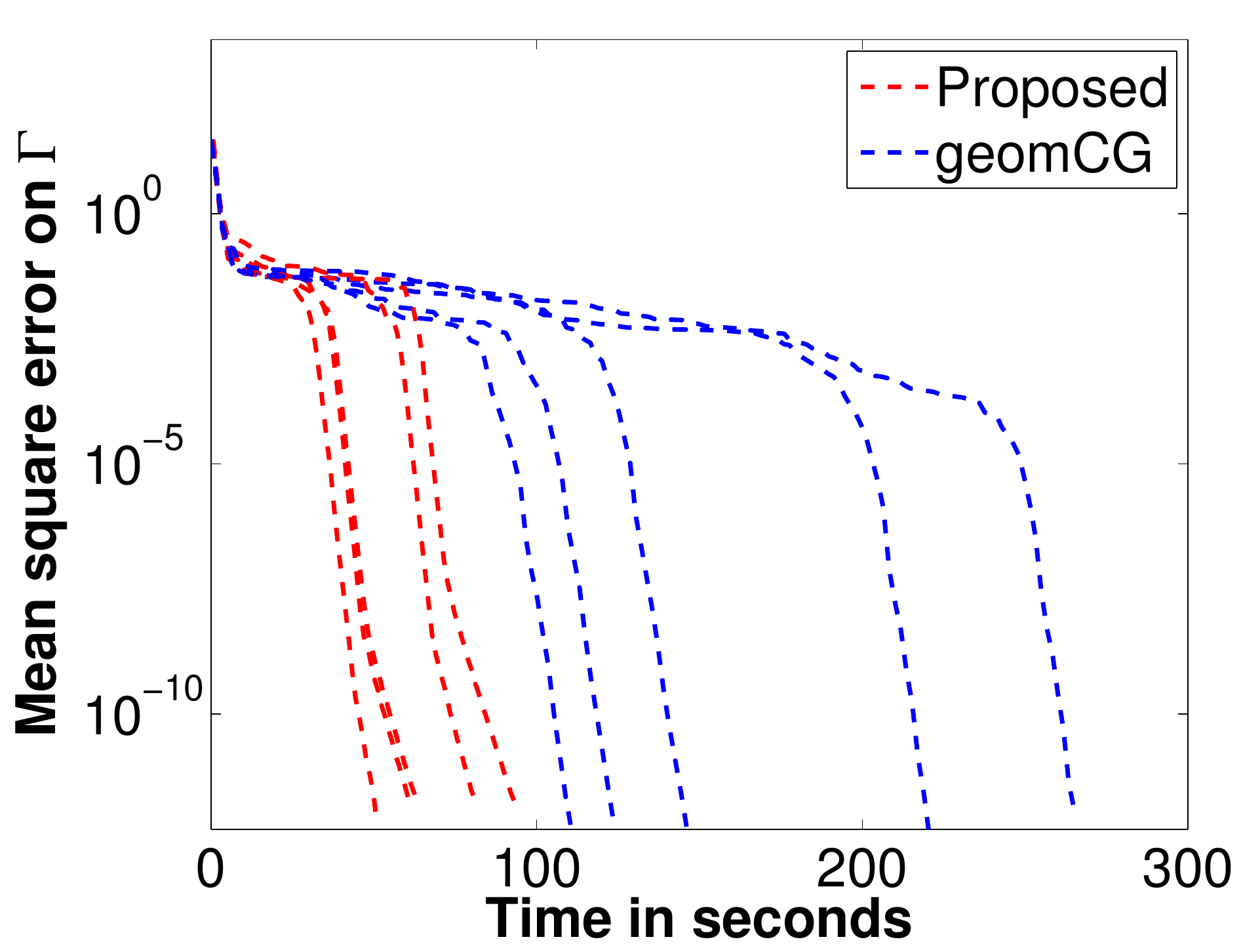}\\
{\scriptsize(c) $10000\times 10000\times 10000$, \\ {\bf r} = ($5\times 5\times 5$).}
\end{center}
\end{minipage}\\
\begin{minipage}{0.32\hsize}
\begin{center}
\includegraphics[width=\hsize]{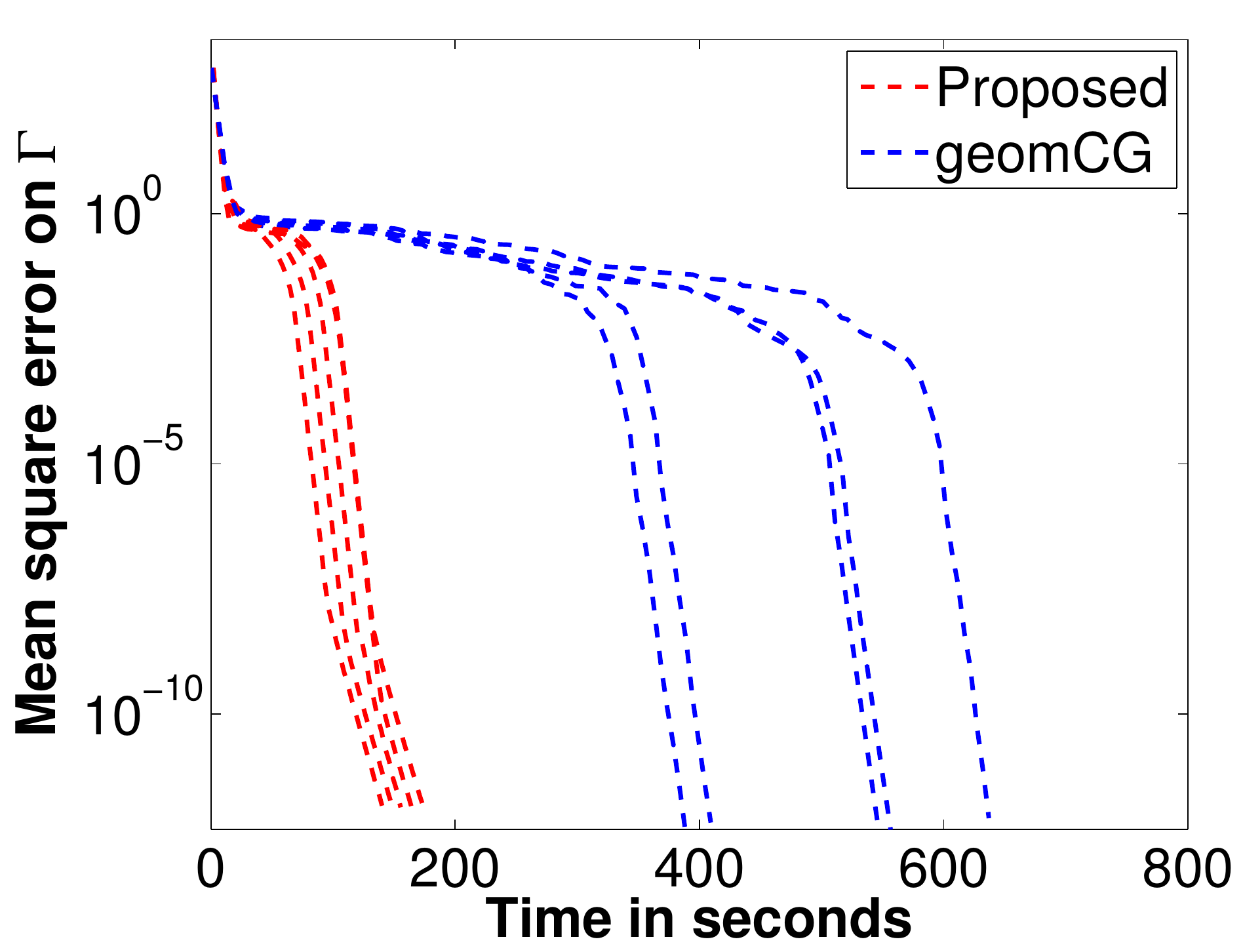}\\
{\scriptsize(d) $3000\times 3000\times 3000$, \\ {\bf r} = ($10\times 10\times 10$).}
\end{center}
\end{minipage}
\begin{minipage}{0.32\hsize}
\begin{center}
\includegraphics[width=\hsize]{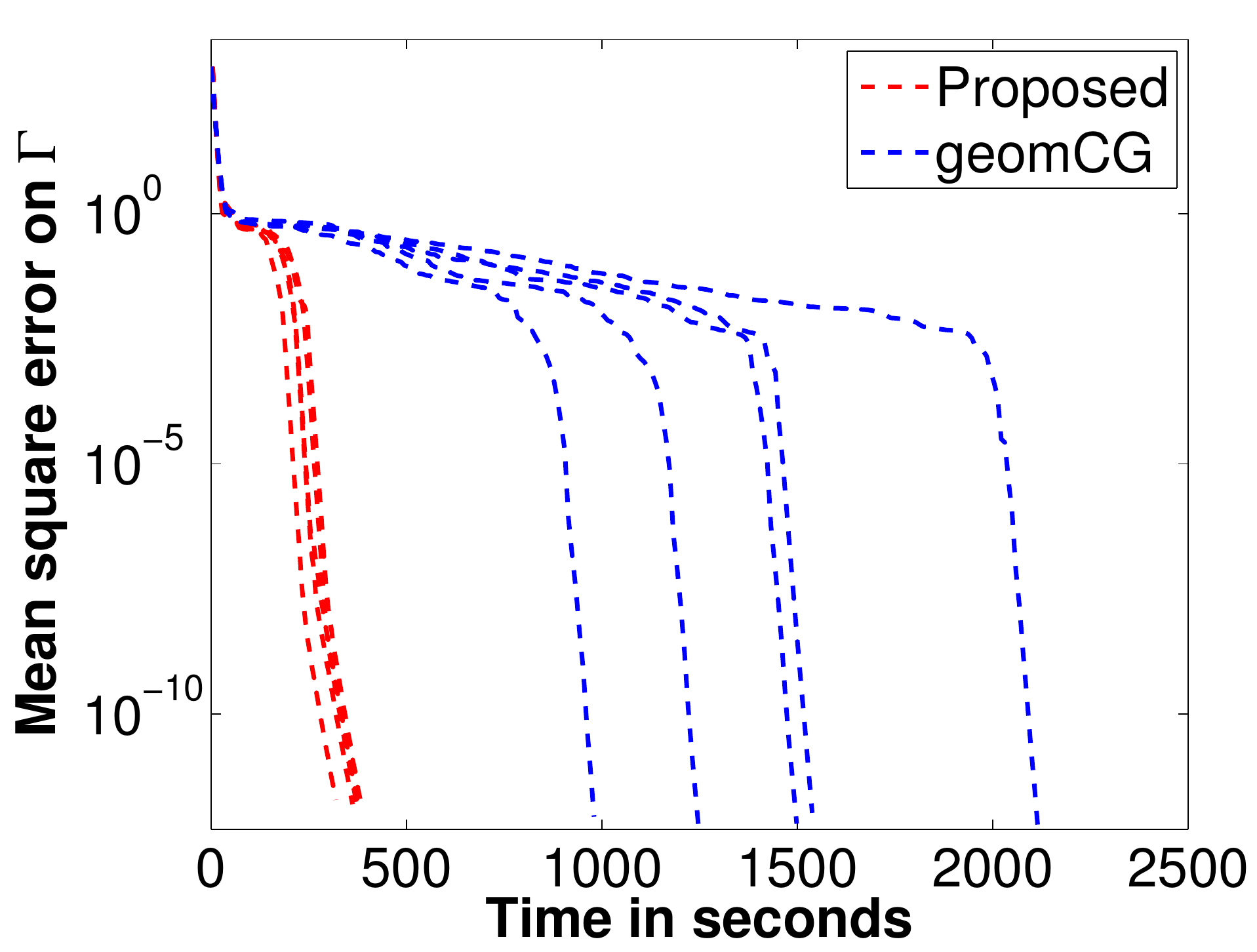}\\
{\scriptsize(e) $5000\times 5000\times 5000$, \\ {\bf r} = ($10\times 10\times 10$).}
\end{center}
\end{minipage}
\begin{minipage}{0.32\hsize}
\begin{center}
\includegraphics[width=\hsize]{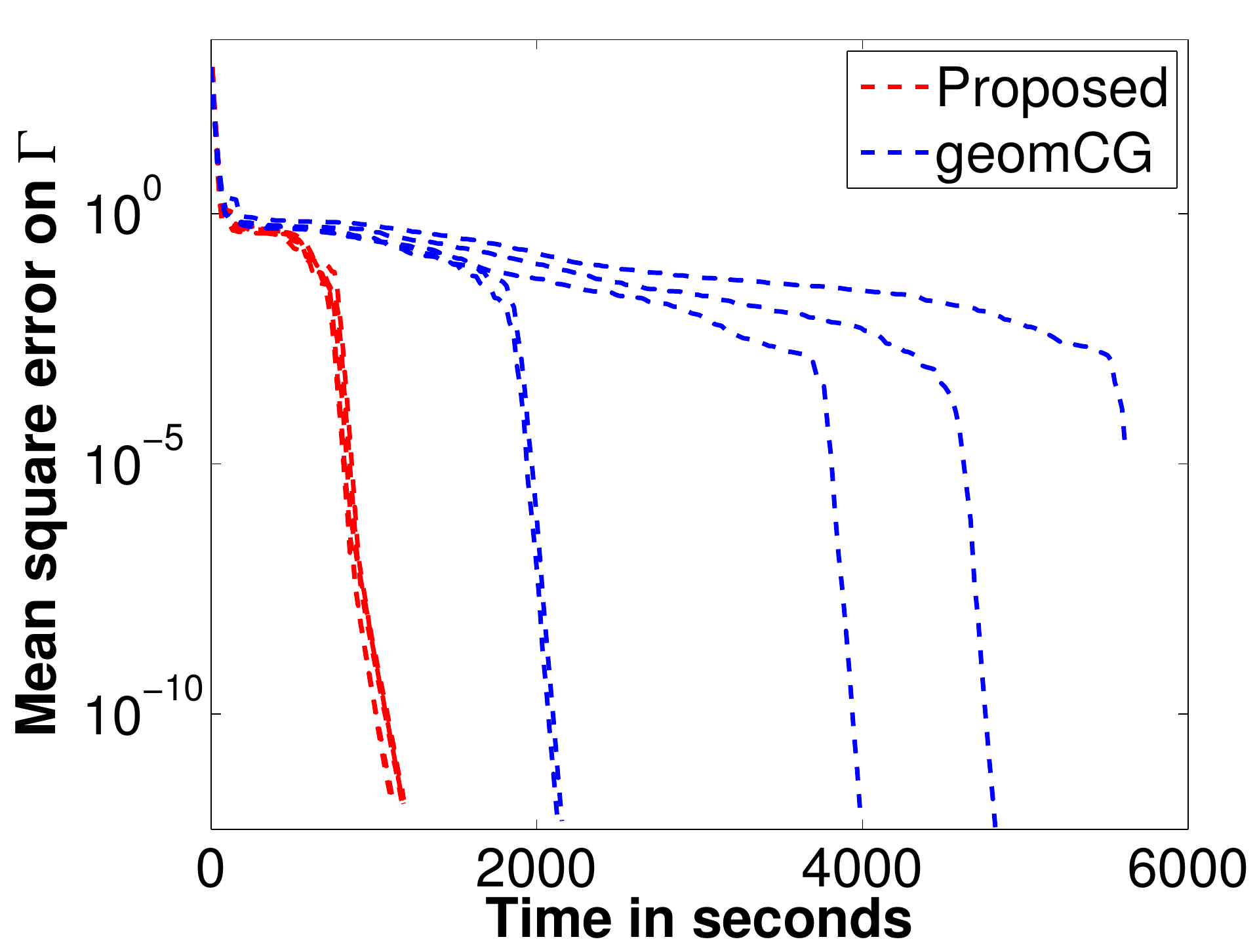}\\
{\scriptsize(f) $10000\times 10000\times 10000$, \\ {\bf r} = ($10\times 10\times 10$).}
\end{center}
\end{minipage}
\end{tabular}
\vspace{-0.1cm}
\caption{{\bf Case S3:} large-scale comparisons on $\Gamma$ (test error).}
\label{appnfig:large-scale-test}
\end{figure*}

\clearpage
\begin{figure*}[htbp]
\vspace{-0.1cm}
\begin{tabular}{ccc}
\begin{minipage}{0.32\hsize}
\begin{center}
\includegraphics[width=\hsize]{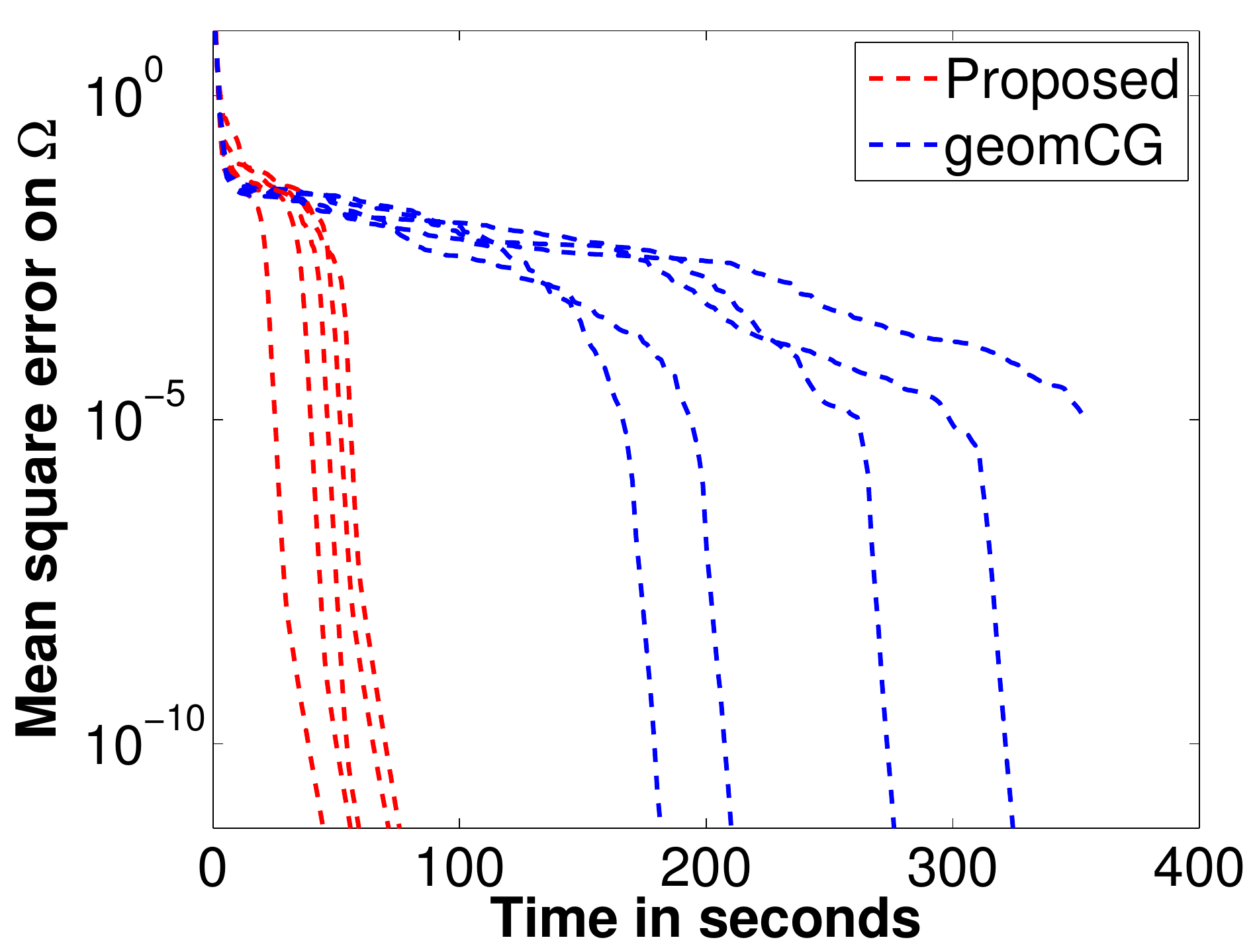}\\
{\scriptsize(a) OS = $8$.}
\end{center}
\end{minipage}
\begin{minipage}{0.32\hsize}
\begin{center}
\includegraphics[width=\hsize]{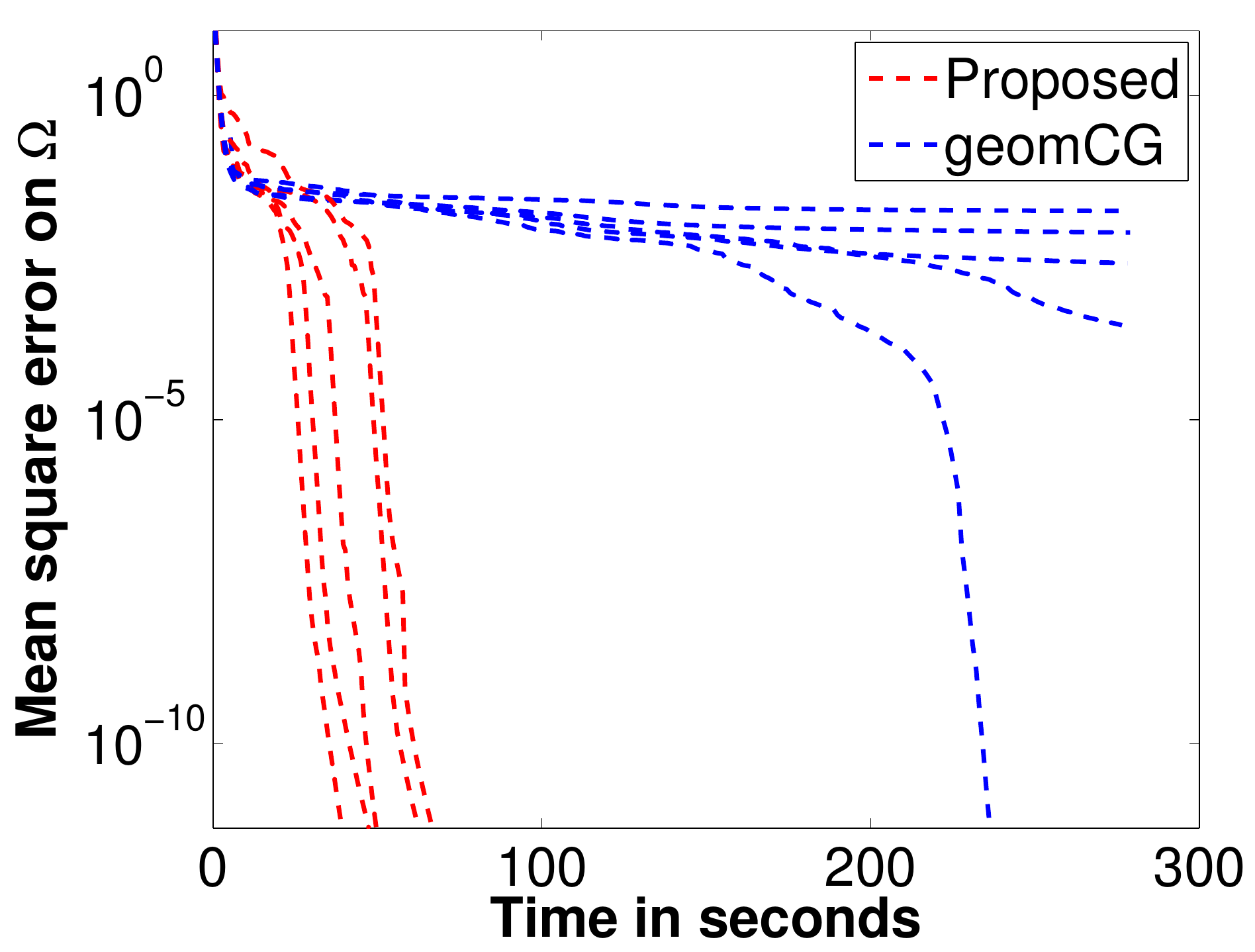}\\
{\scriptsize(b) OS = $6$.}
\end{center}
\end{minipage}
\begin{minipage}{0.32\hsize}
\begin{center}
\includegraphics[width=\hsize]{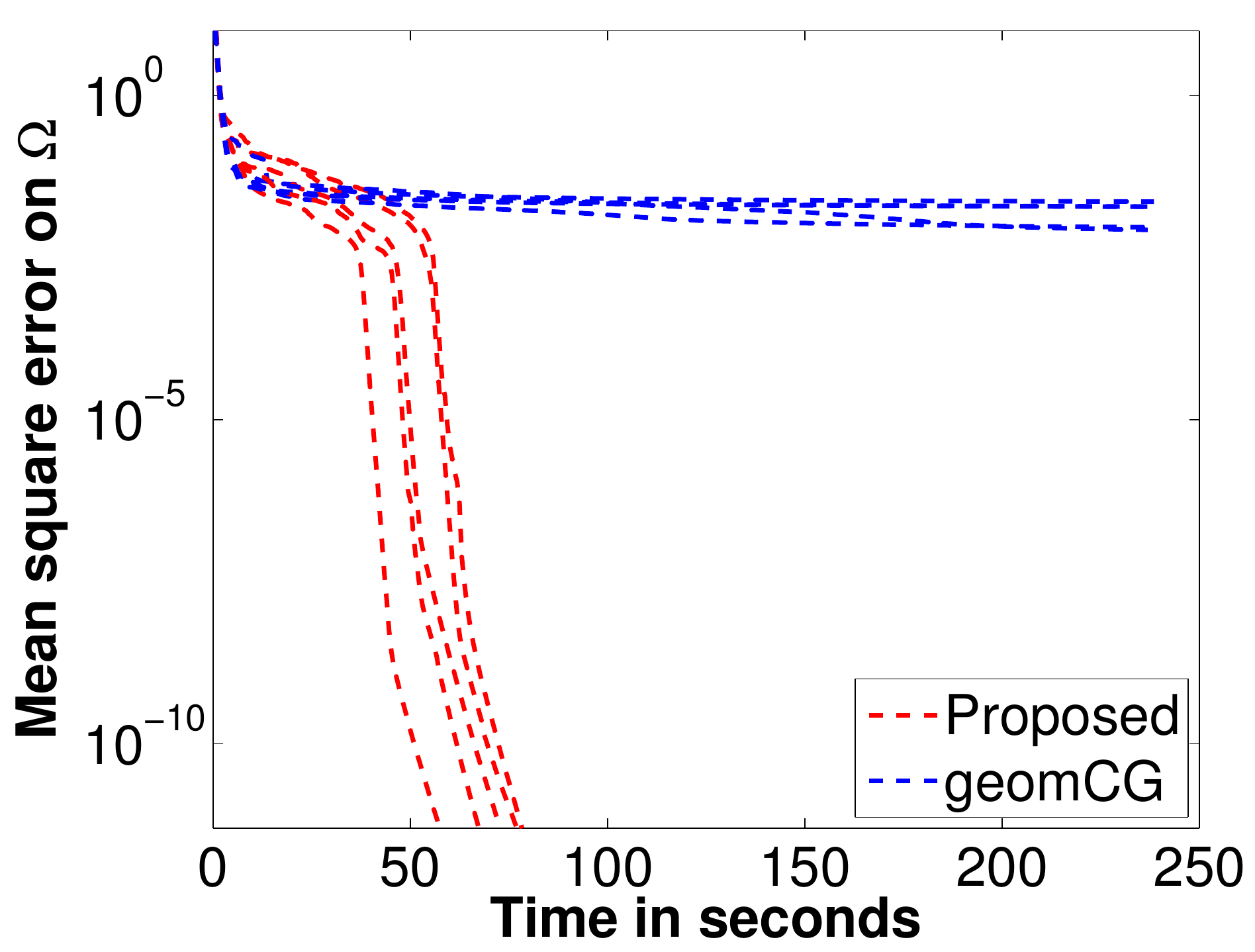}\\
{\scriptsize(c) OS = $5$.}
\end{center}
\end{minipage}
\end{tabular}
\vspace{-0.2cm}
\caption{\changeHK{{\bf Case S4:} low-sampling comparisons on $\Omega$ (train error).}}
\label{appnfig:low-sampling-train}
%
\vspace{1.5cm}
\begin{tabular}{ccc}
\begin{minipage}{0.32\hsize}
\begin{center}
\includegraphics[width=\hsize]{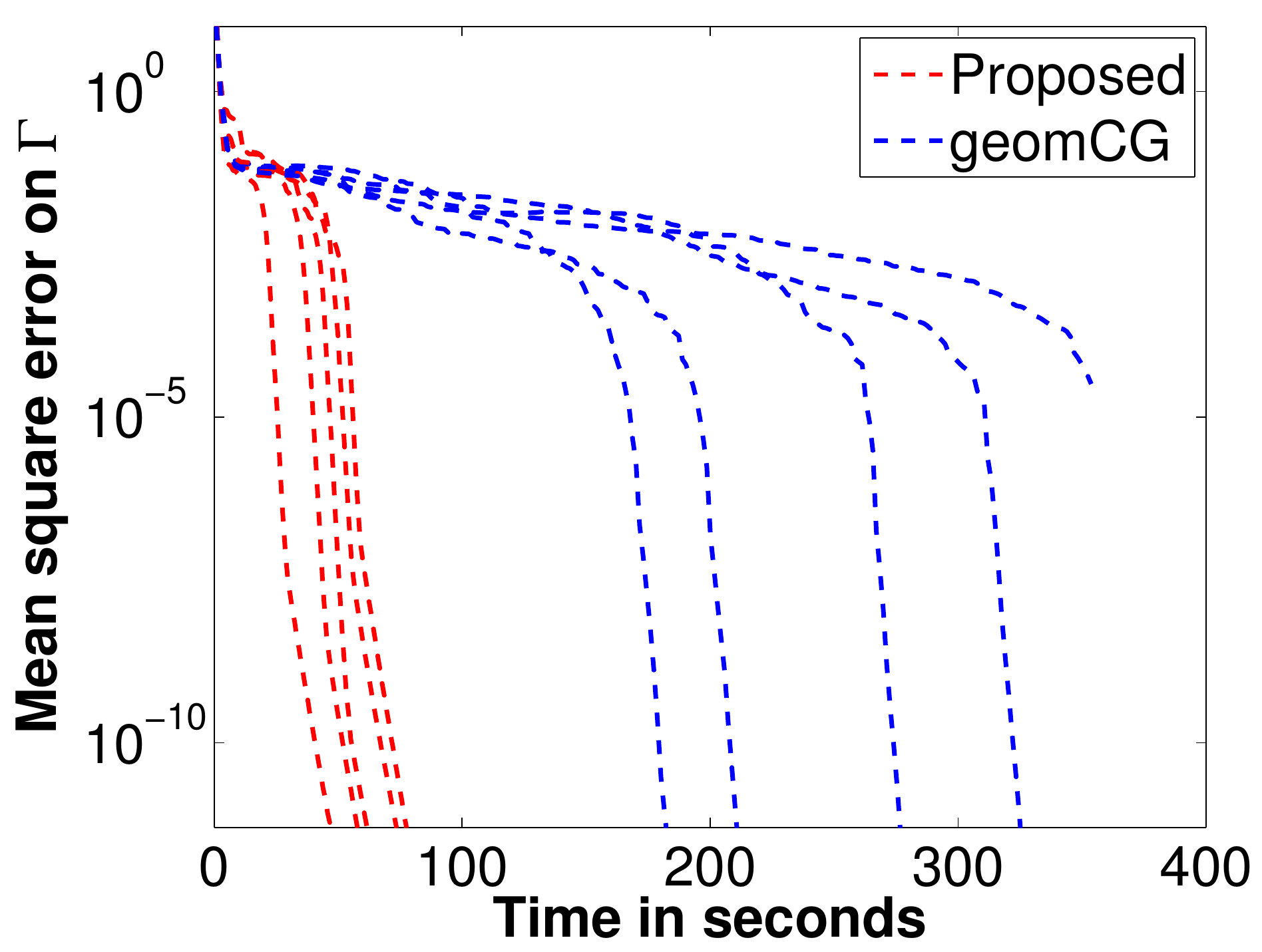}\\
{\scriptsize(a) OS = $8$.}
\end{center}
\end{minipage}
\begin{minipage}{0.32\hsize}
\begin{center}
\includegraphics[width=\hsize]{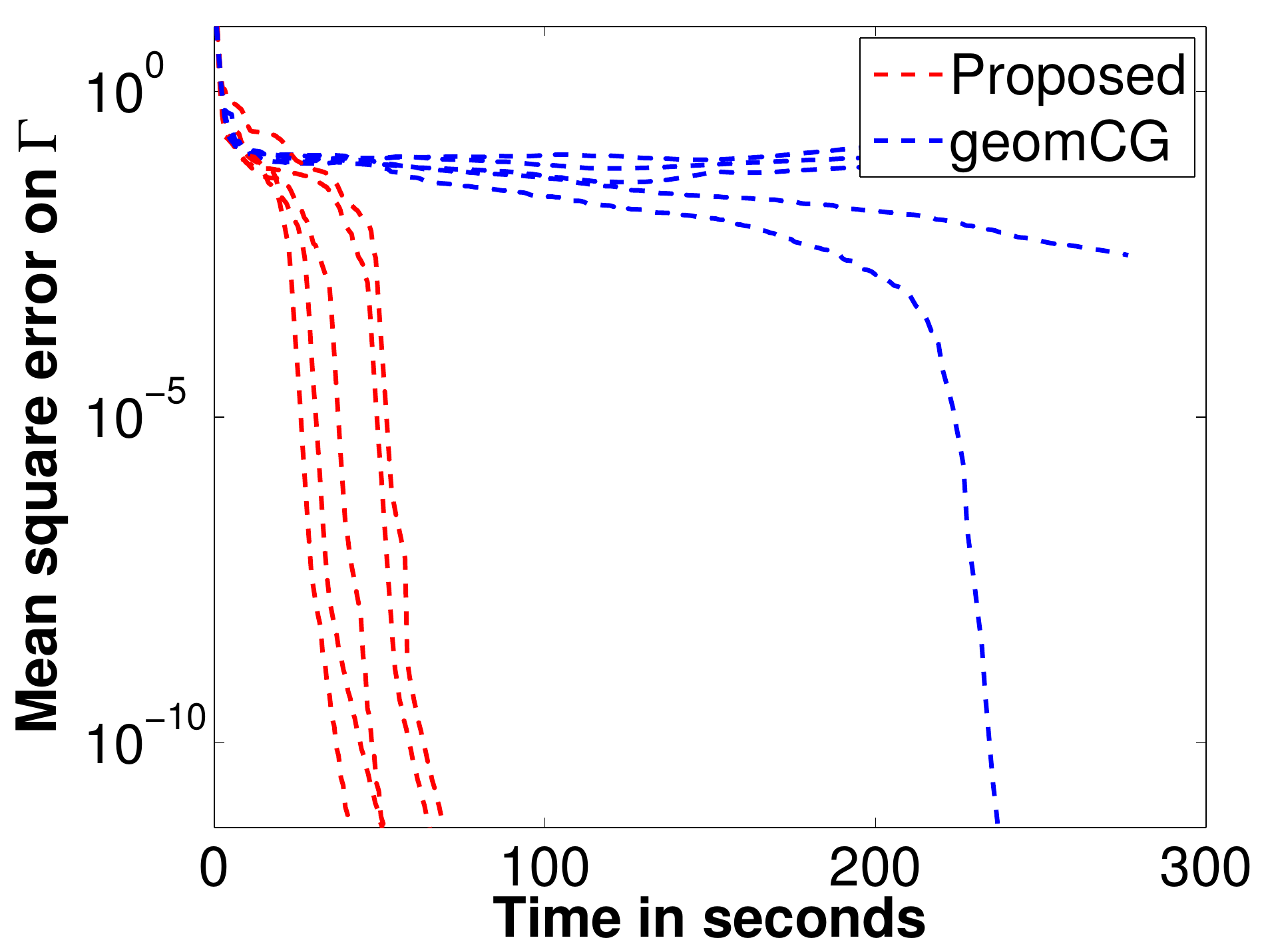}\\
{\scriptsize(b) OS = $6$.}
\end{center}
\end{minipage}
\begin{minipage}{0.32\hsize}
\begin{center}
\includegraphics[width=\hsize]{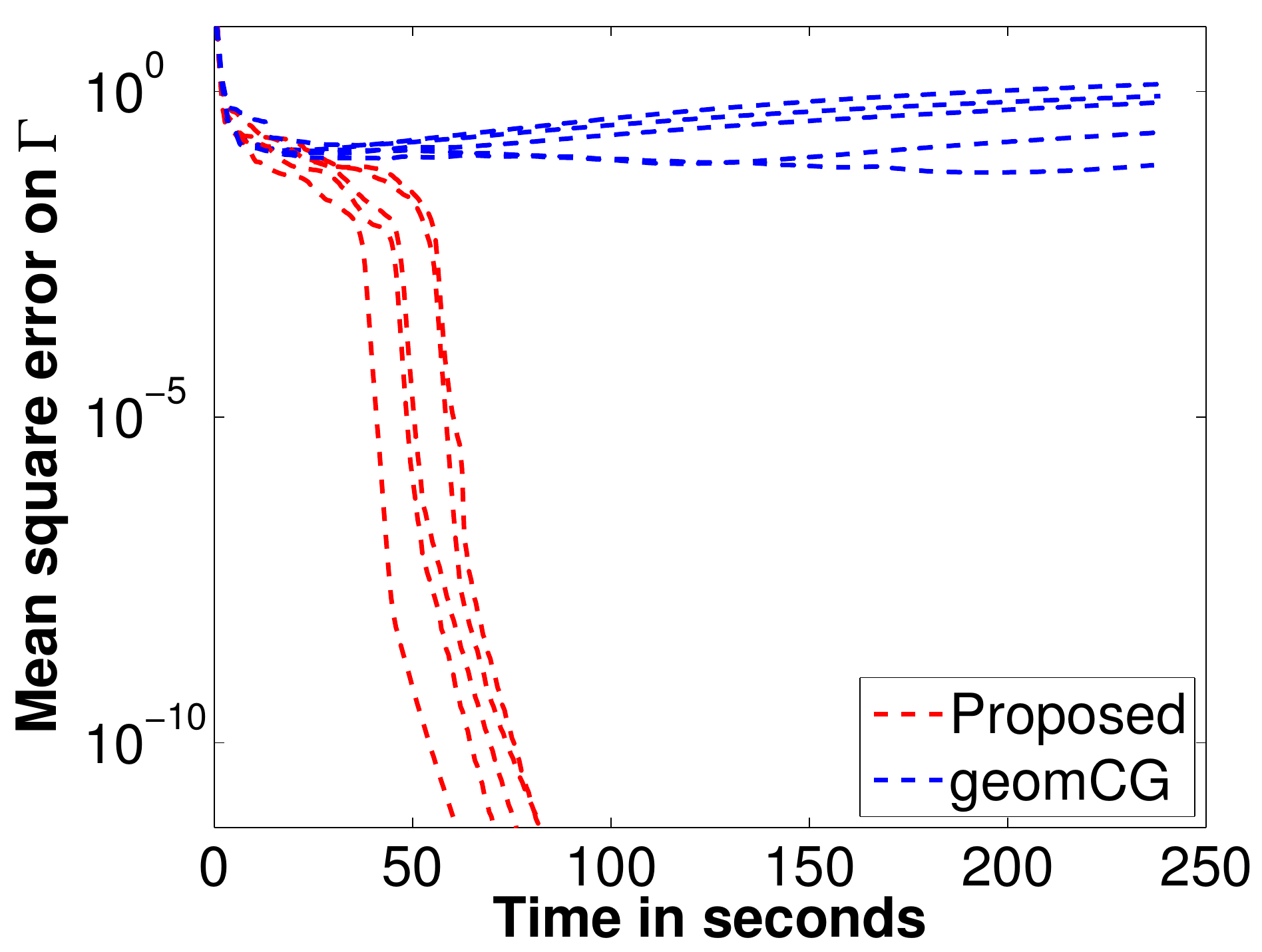}\\
{\scriptsize(c) OS = $5$.}
\end{center}
\end{minipage}
\end{tabular}
\vspace{-0.2cm}
\caption{{\bf Case S4:} low-sampling comparisons on $\Gamma$ (test error).}
\label{appnfig:low-sampling-test}
%
\vspace*{2cm}
\begin{minipage}{0.48\hsize}
\begin{center}
\hspace*{0.2cm}\includegraphics[width=\hsize]{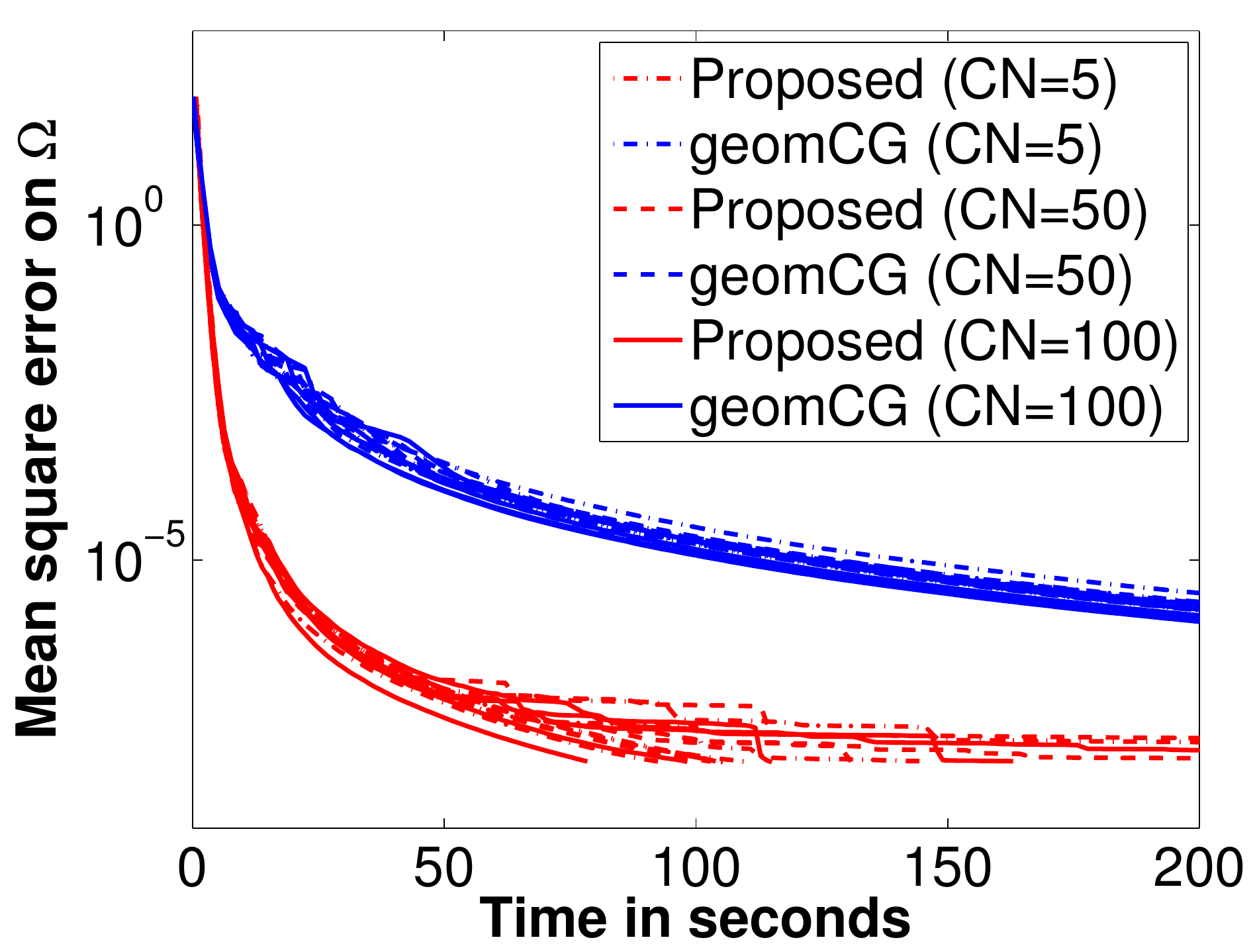}
\caption{\changeHK{{\bf Case S5:} CN = $\{5,50,100\}$ on $\Omega$ (train error).}}
\label{appnfig:S5-train}
\end{center}
\end{minipage}
\hspace*{0.2cm}\begin{minipage}{0.48\hsize}
\begin{center}
\includegraphics[width=\hsize]{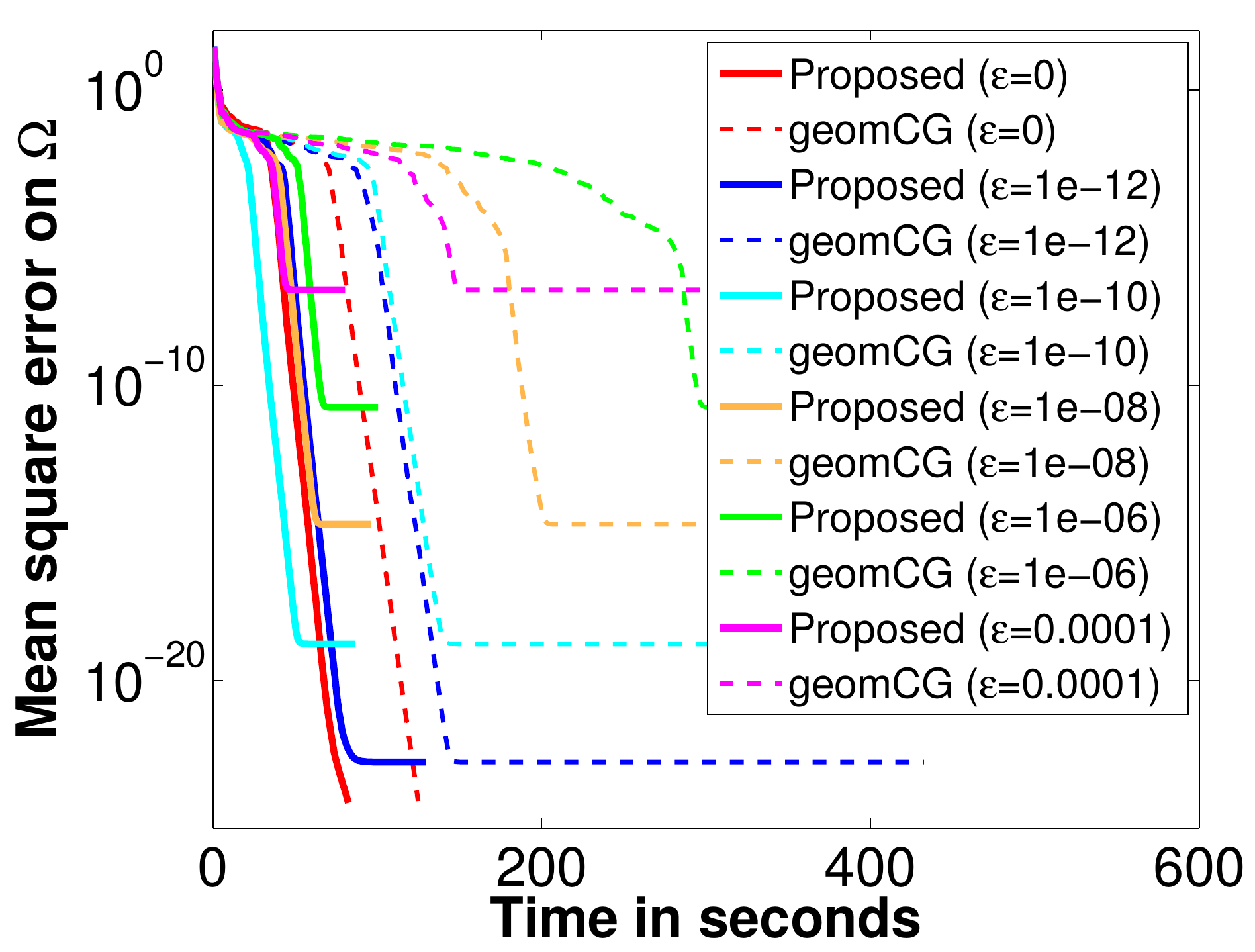}
\caption{\changeHK{{\bf Case S6:} noisy data on $\Omega$ (train error).}}
\label{appnfig:S6-train}
\end{center}
\end{minipage}
\end{figure*}

\clearpage

\begin{figure*}[htbp]
\vspace{-0.01cm}
\begin{tabular}{ccc}
\begin{minipage}{0.32\hsize}
\begin{center}
\includegraphics[width=\hsize]{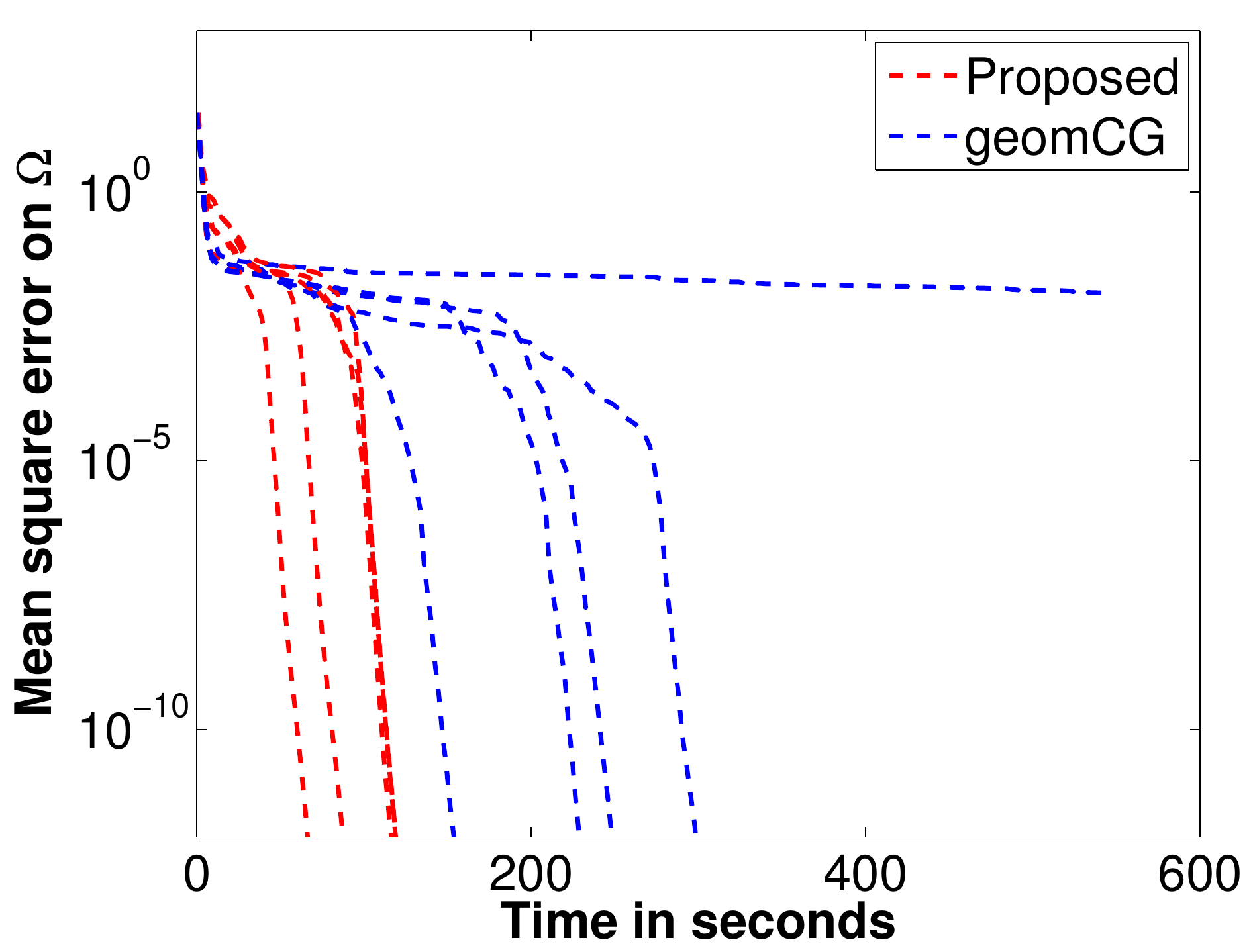}\\
{\scriptsize(a) $20000\times 7000\times 7000$, \\
{\bf r} = ($5\times 5\times 5$).}
\end{center}
\end{minipage}
\begin{minipage}{0.32\hsize}
\begin{center}
\includegraphics[width=\hsize]{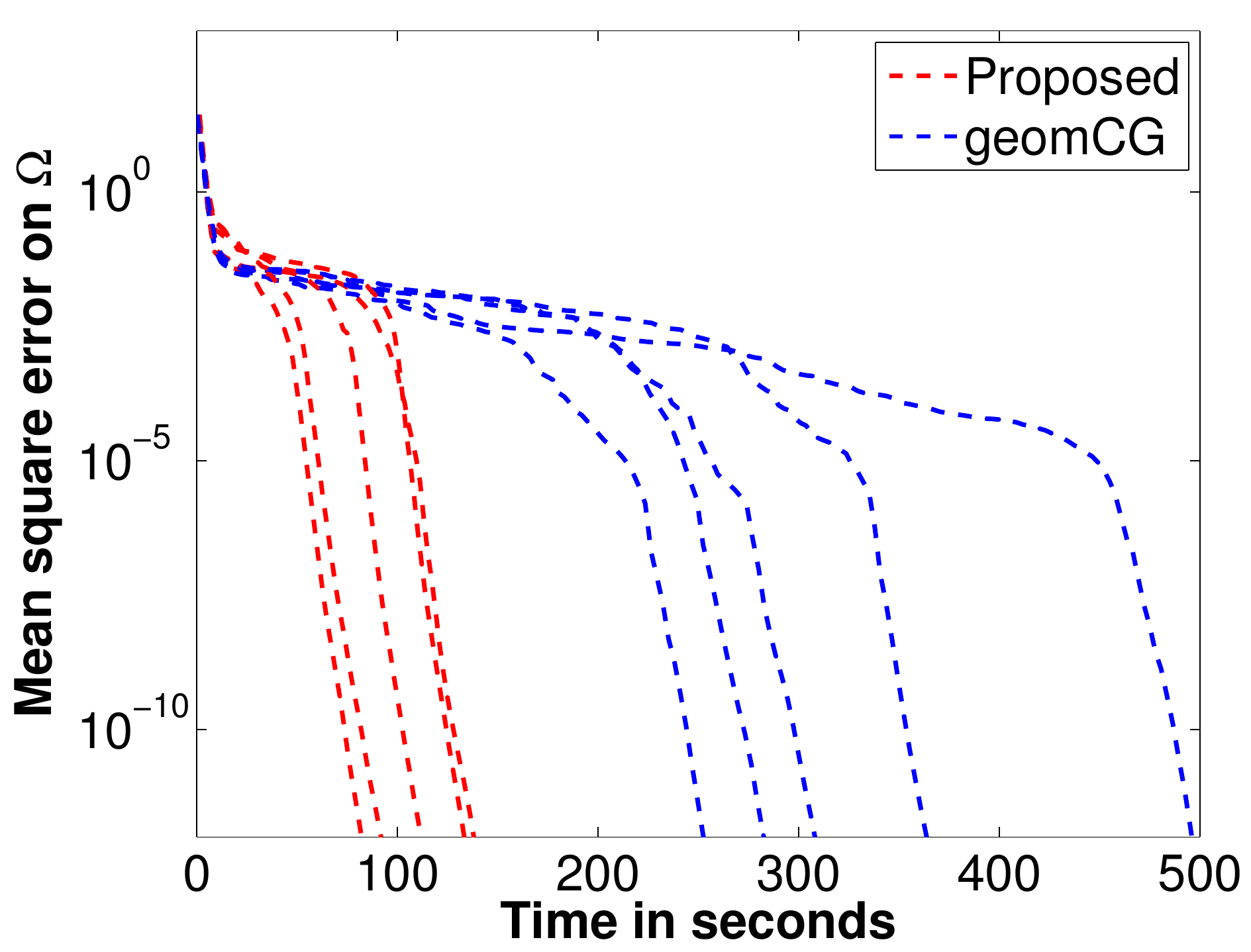}\\
{\scriptsize(b) $30000\times 60000\times 60000$, \\
{\bf r} = ($5\times 5\times 5$).}
\end{center}
\end{minipage}
\begin{minipage}{0.32\hsize}
\begin{center}
\includegraphics[width=\hsize]{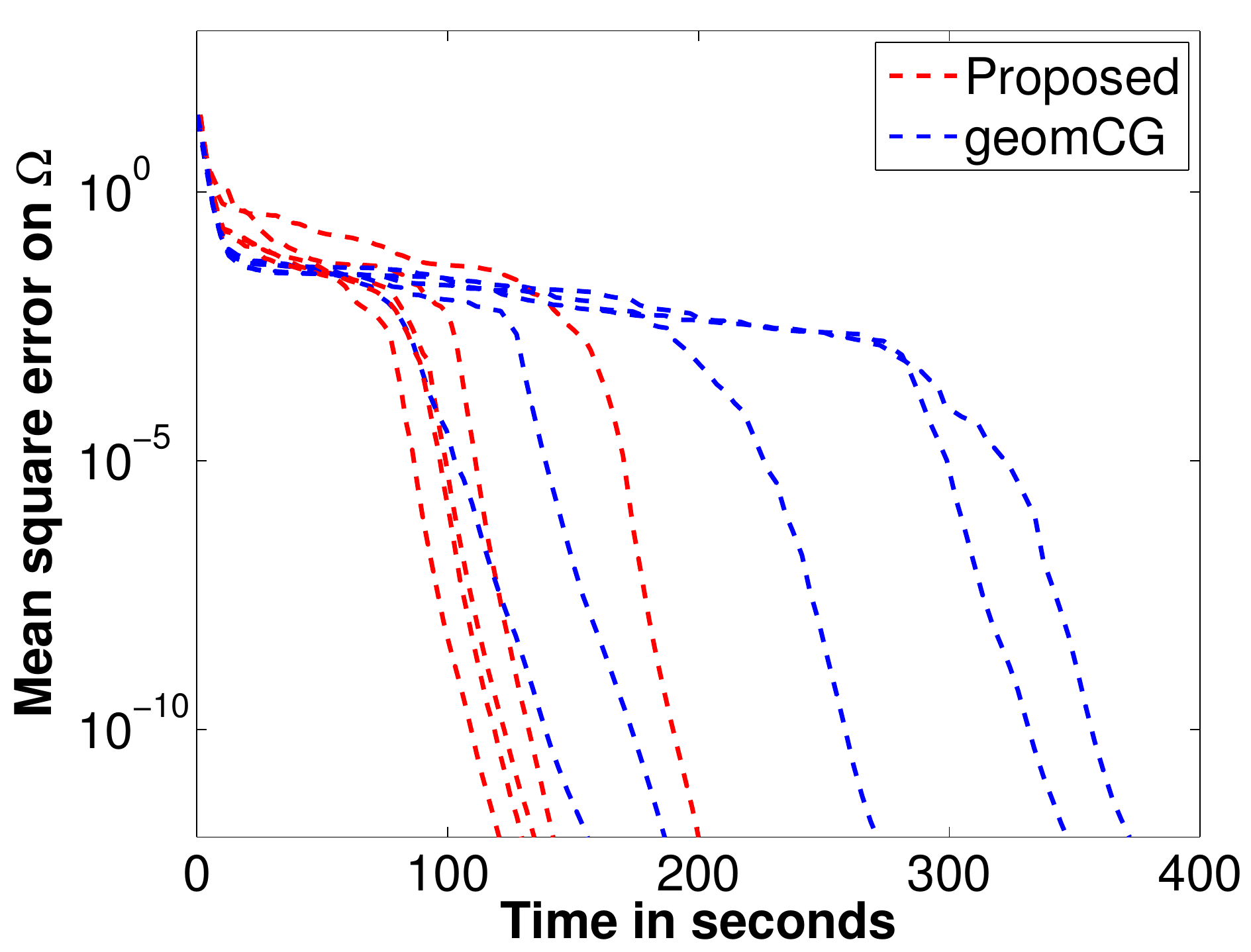}\\
{\scriptsize(c) $40000\times 5000\times 5000$, \\
{\bf r} = ($5\times 5\times 5$).}
\end{center}
\end{minipage}\\

\begin{minipage}{0.32\hsize}
\begin{center}
\includegraphics[width=\hsize]{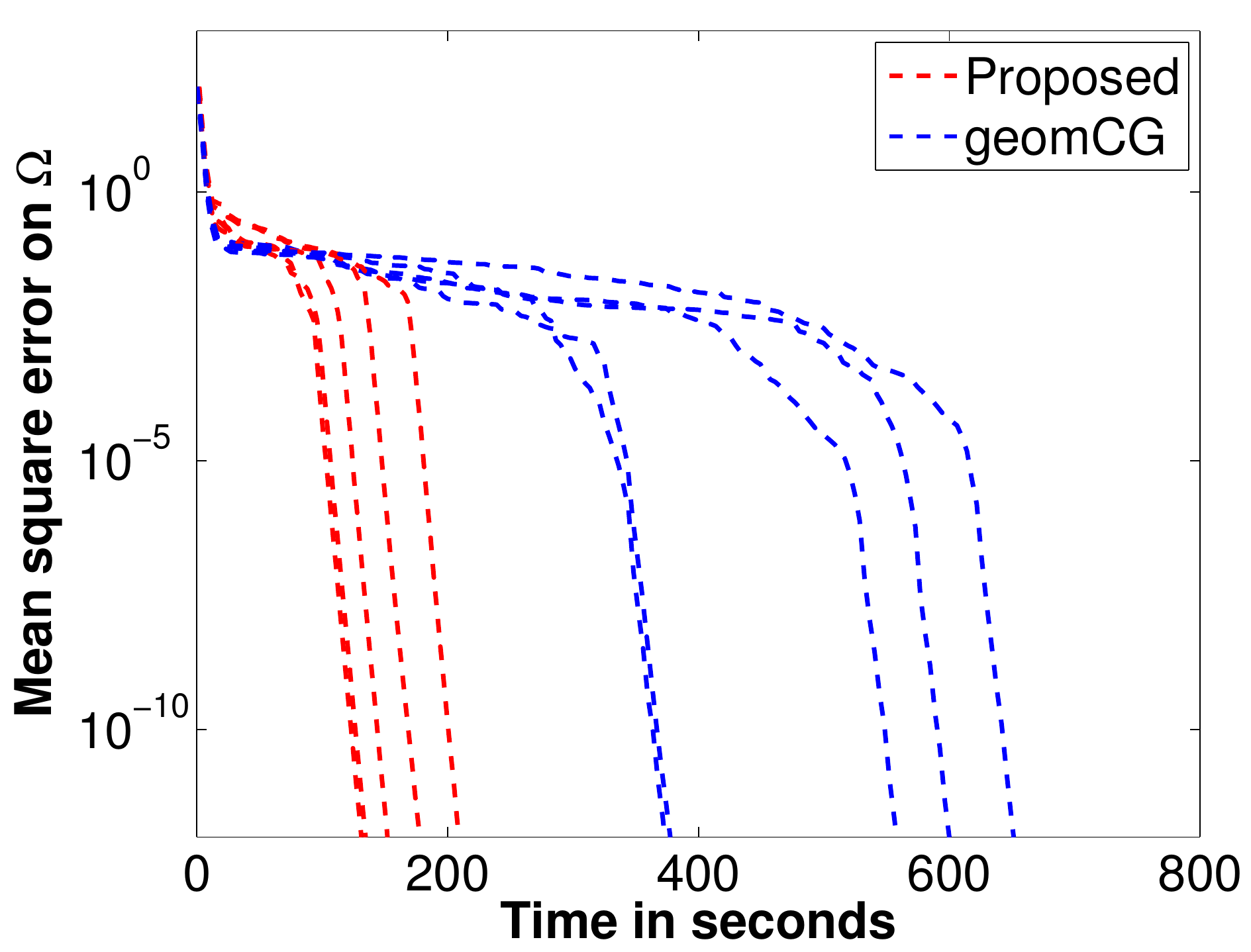}\\
{\scriptsize(d) {\bf r} = ($7\times 6\times 6$),\\
$10000\times 10000\times 10000$.}
\end{center}
\end{minipage}
\begin{minipage}{0.32\hsize}
\begin{center}
\includegraphics[width=\hsize]{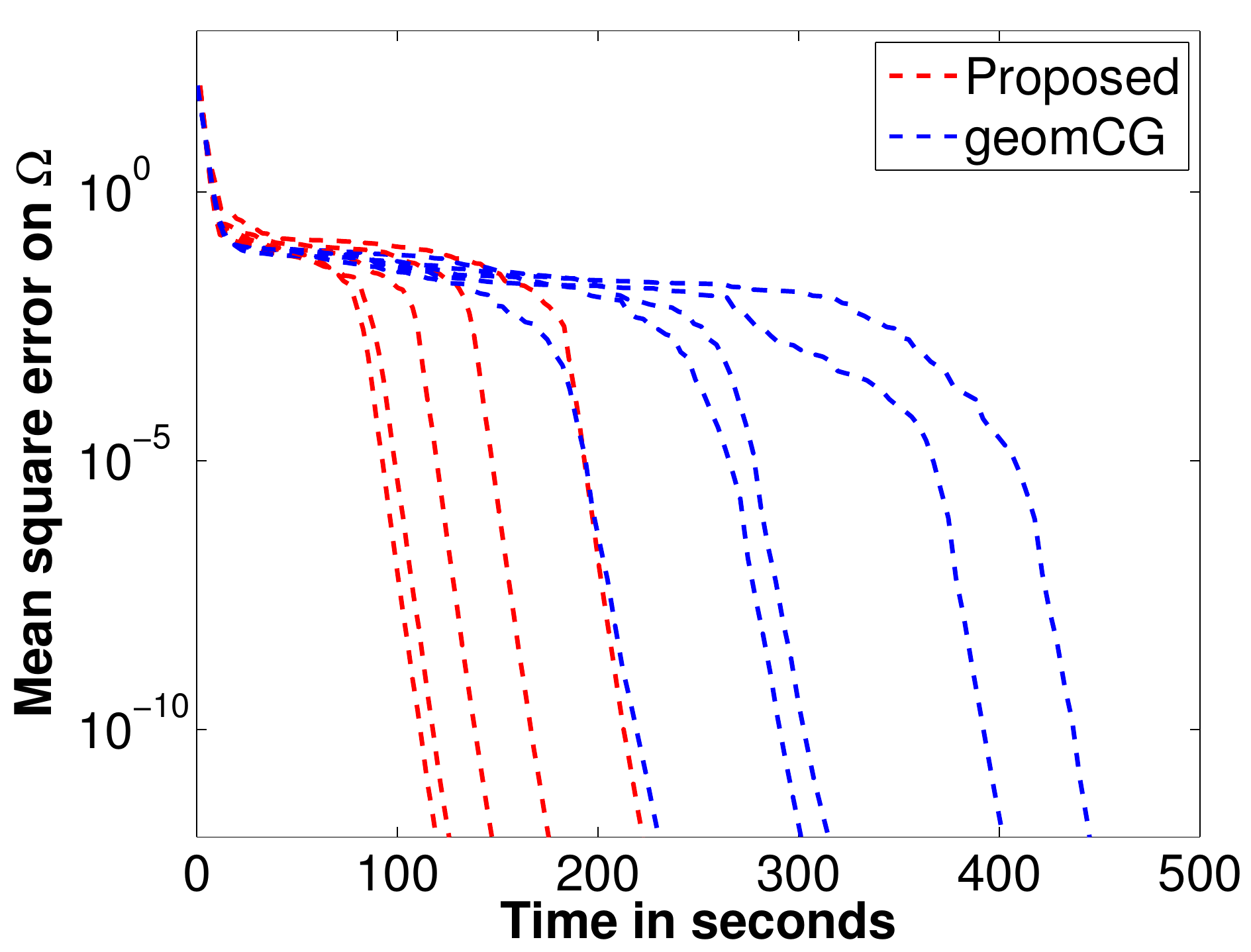}\\
{\scriptsize(e) {\bf r} = ($10\times 5\times 5$),\\
$10000\times 10000\times 10000$.}
\end{center}
\end{minipage}
\begin{minipage}{0.32\hsize}
\begin{center}
\includegraphics[width=\hsize]{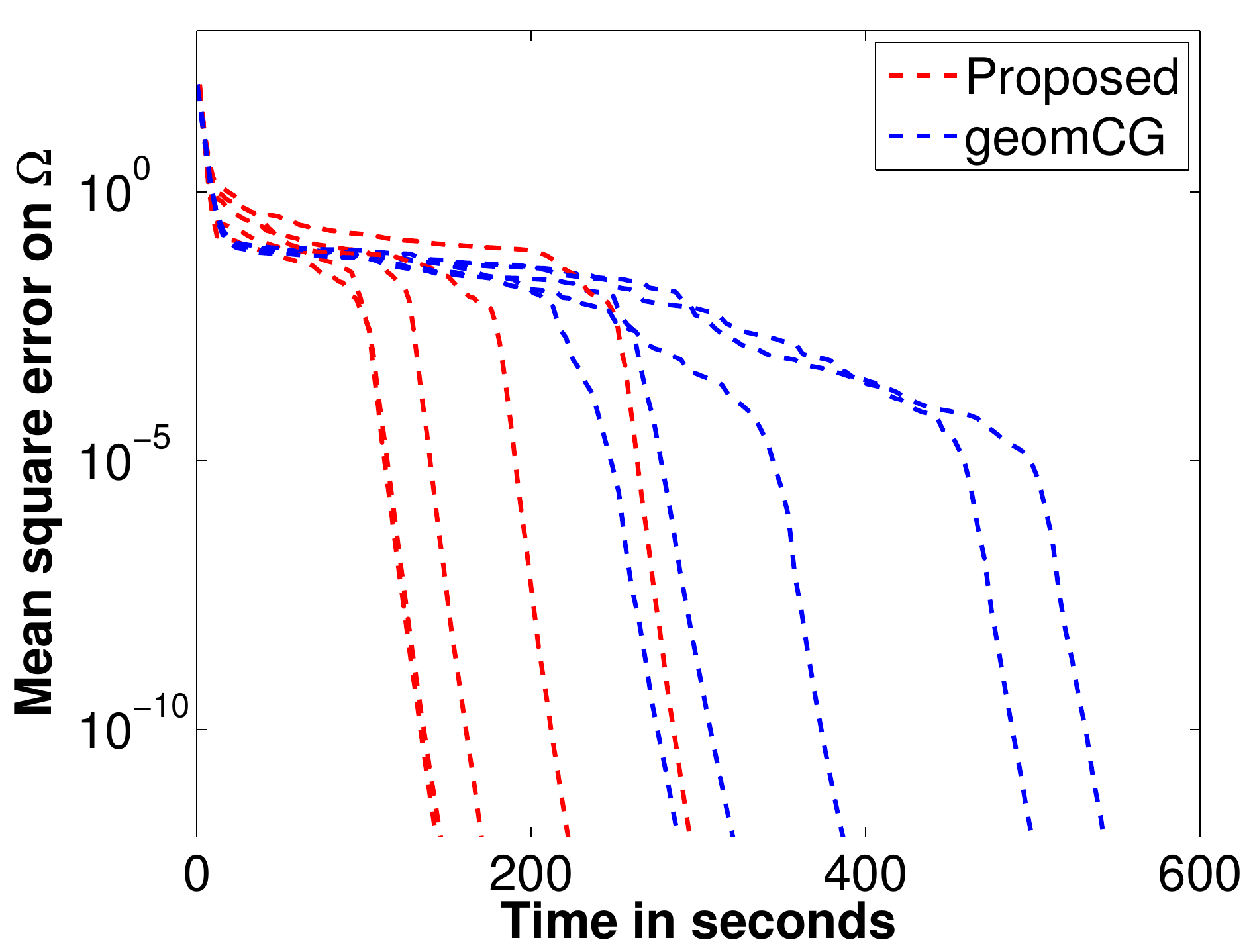}\\
{\scriptsize(f) {\bf r} = ($15\times 4\times 4$),\\
$10000\times 10000\times 10000$.}
\end{center}
\end{minipage}\\
\end{tabular}
\vspace{-0.1cm}
\caption{\changeHK{{\bf Case S7:} rectangular comparisons on $\Omega$ (train error).}}
\label{appnfig:asymmetric-train}
%
\vspace{1.5cm}
\begin{tabular}{ccc}
\begin{minipage}{0.32\hsize}
\begin{center}
\includegraphics[width=\hsize]{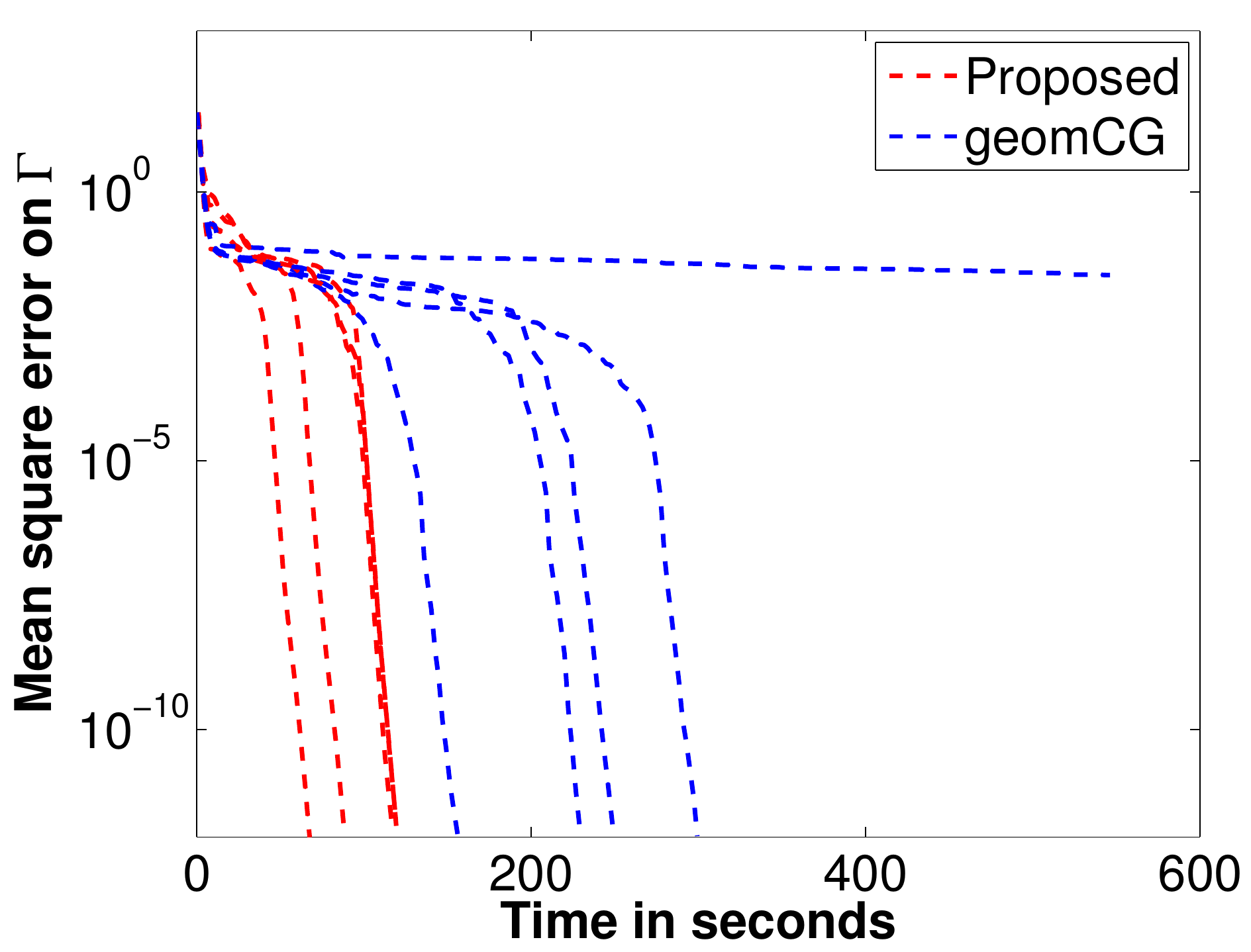}\\
{\scriptsize(a) $20000\times 7000\times 7000$, \\
{\bf r} = ($5\times 5\times 5$).}
\end{center}
\end{minipage}
\begin{minipage}{0.32\hsize}
\begin{center}
\includegraphics[width=\hsize]{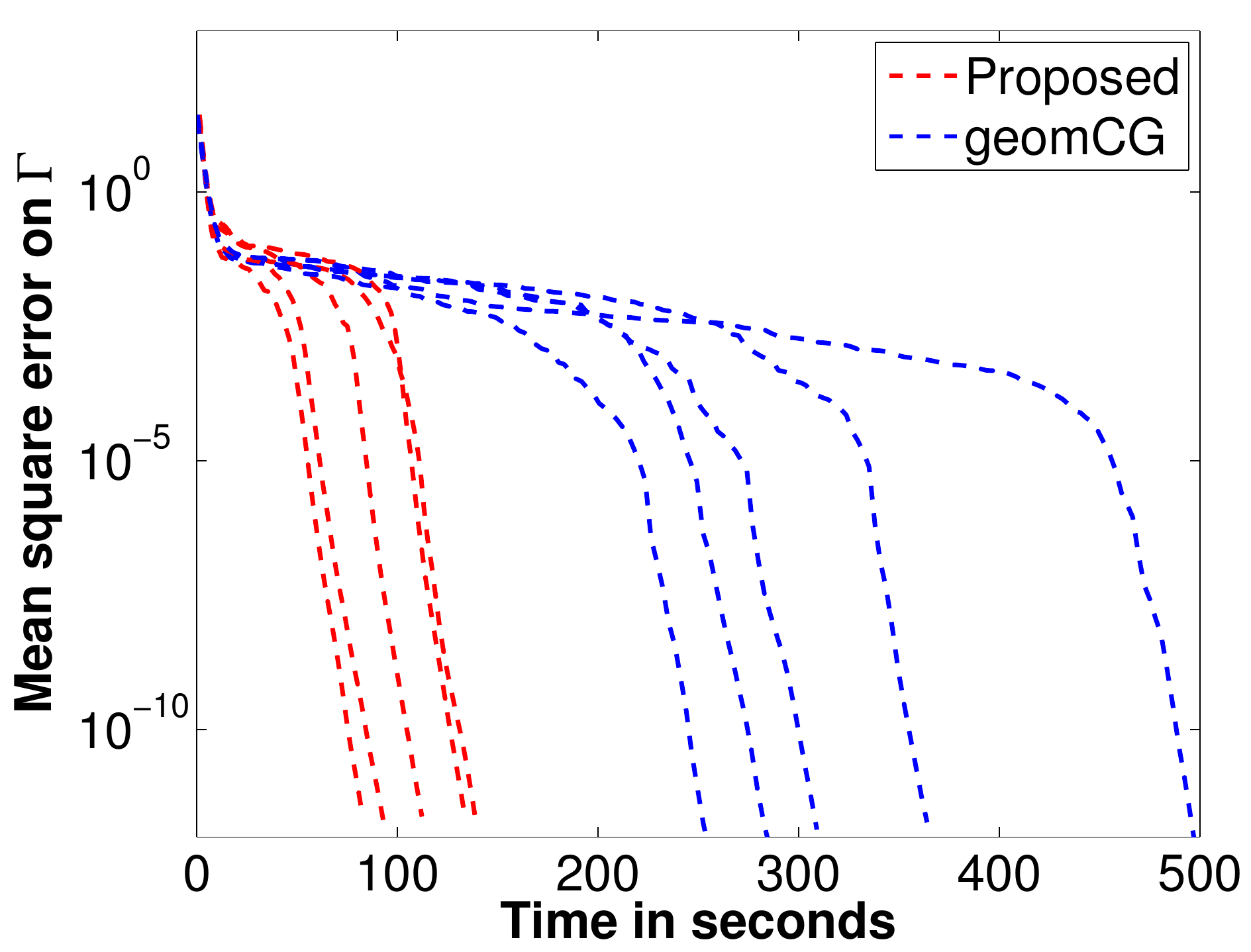}\\
{\scriptsize(b) $30000\times 60000\times 60000$, \\
{\bf r} = ($5\times 5\times 5$).}
\end{center}
\end{minipage}
\begin{minipage}{0.32\hsize}
\begin{center}
\includegraphics[width=\hsize]{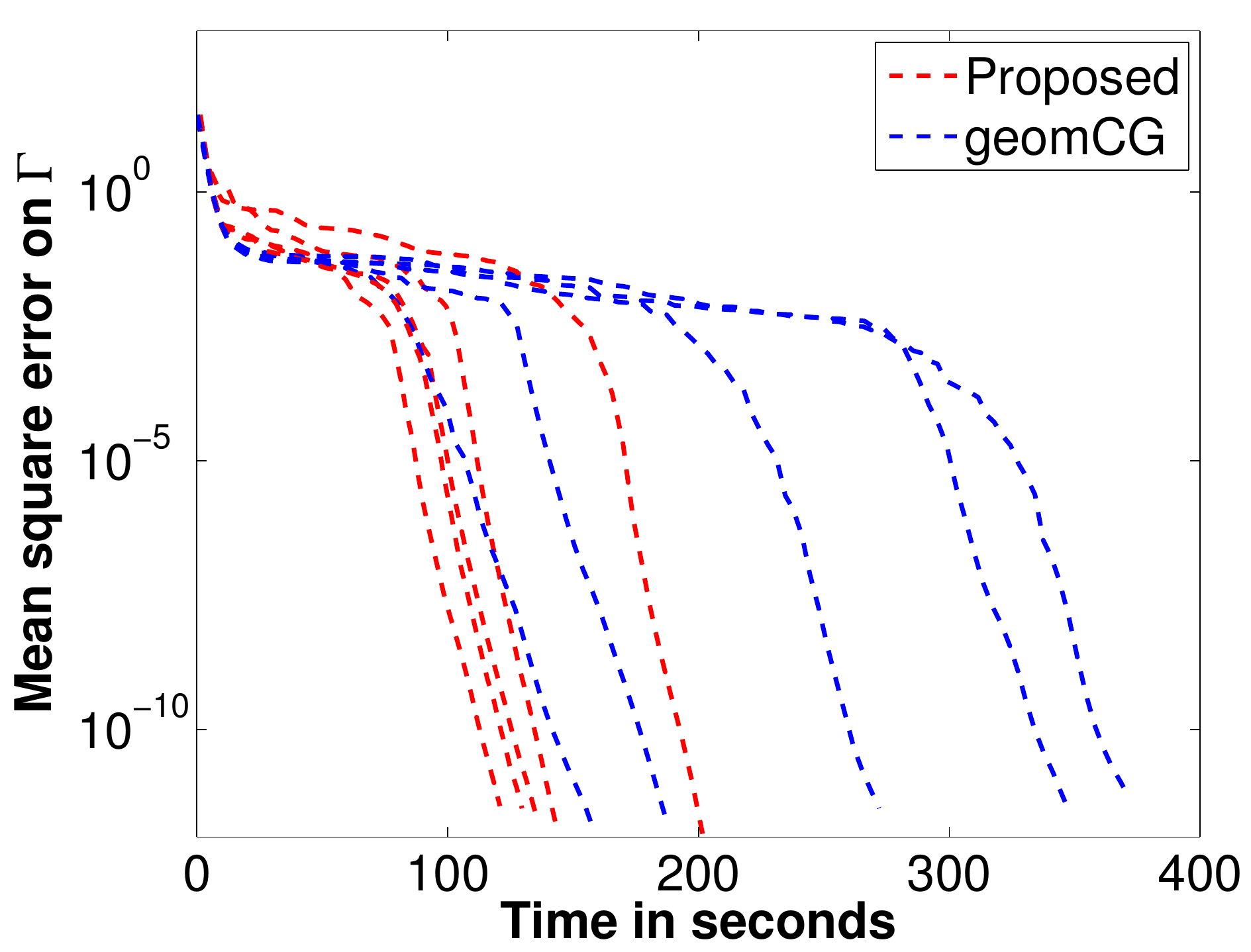}\\
{\scriptsize(c) $40000\times 5000\times 5000$, \\
{\bf r} = ($5\times 5\times 5$).}
\end{center}
\end{minipage}\\

\begin{minipage}{0.32\hsize}
\begin{center}
\includegraphics[width=\hsize]{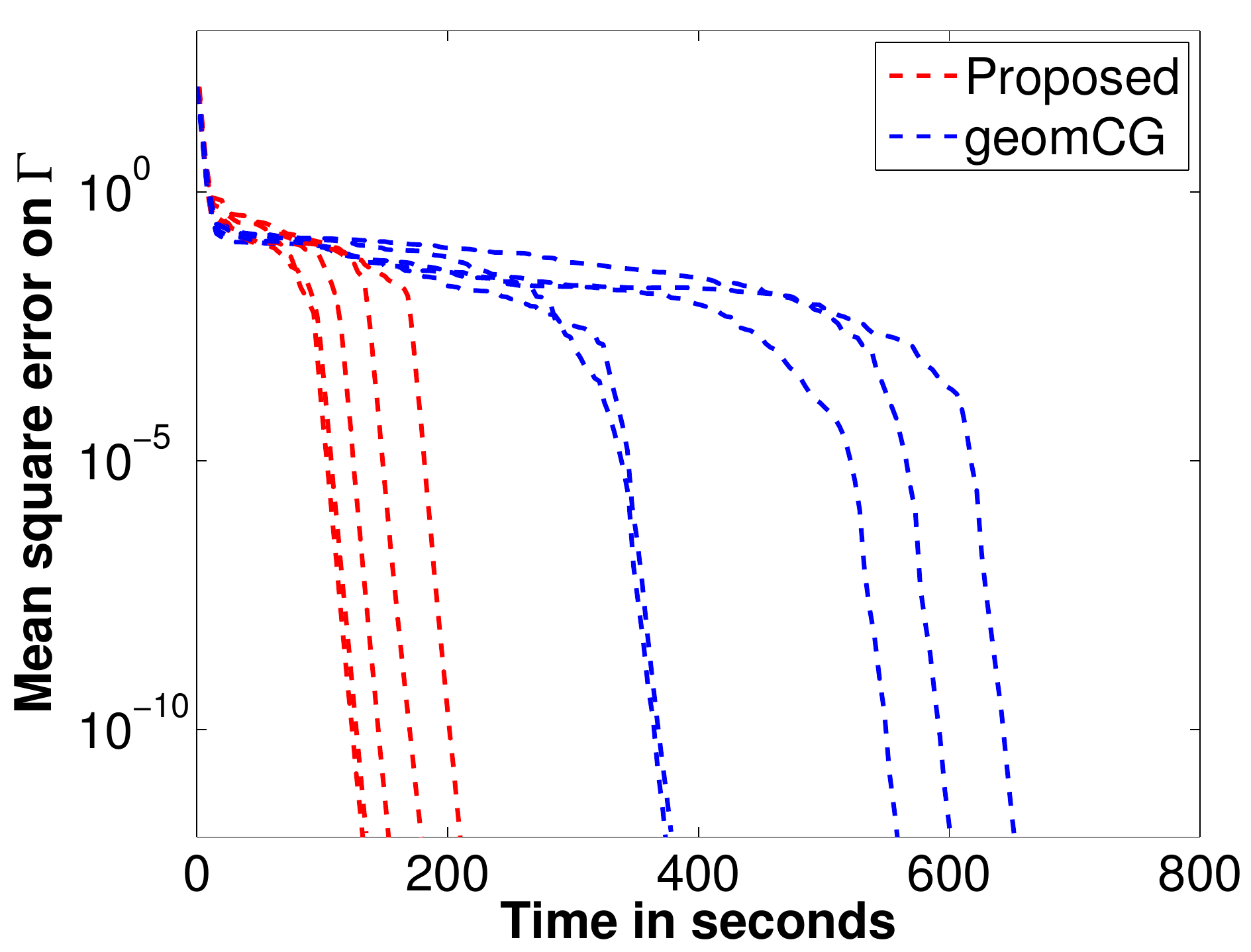}\\
{\scriptsize(d) {\bf r} = ($7\times 6\times 6$),\\
$10000\times 10000\times 10000$.}
\end{center}
\end{minipage}
\begin{minipage}{0.32\hsize}
\begin{center}
\includegraphics[width=\hsize]{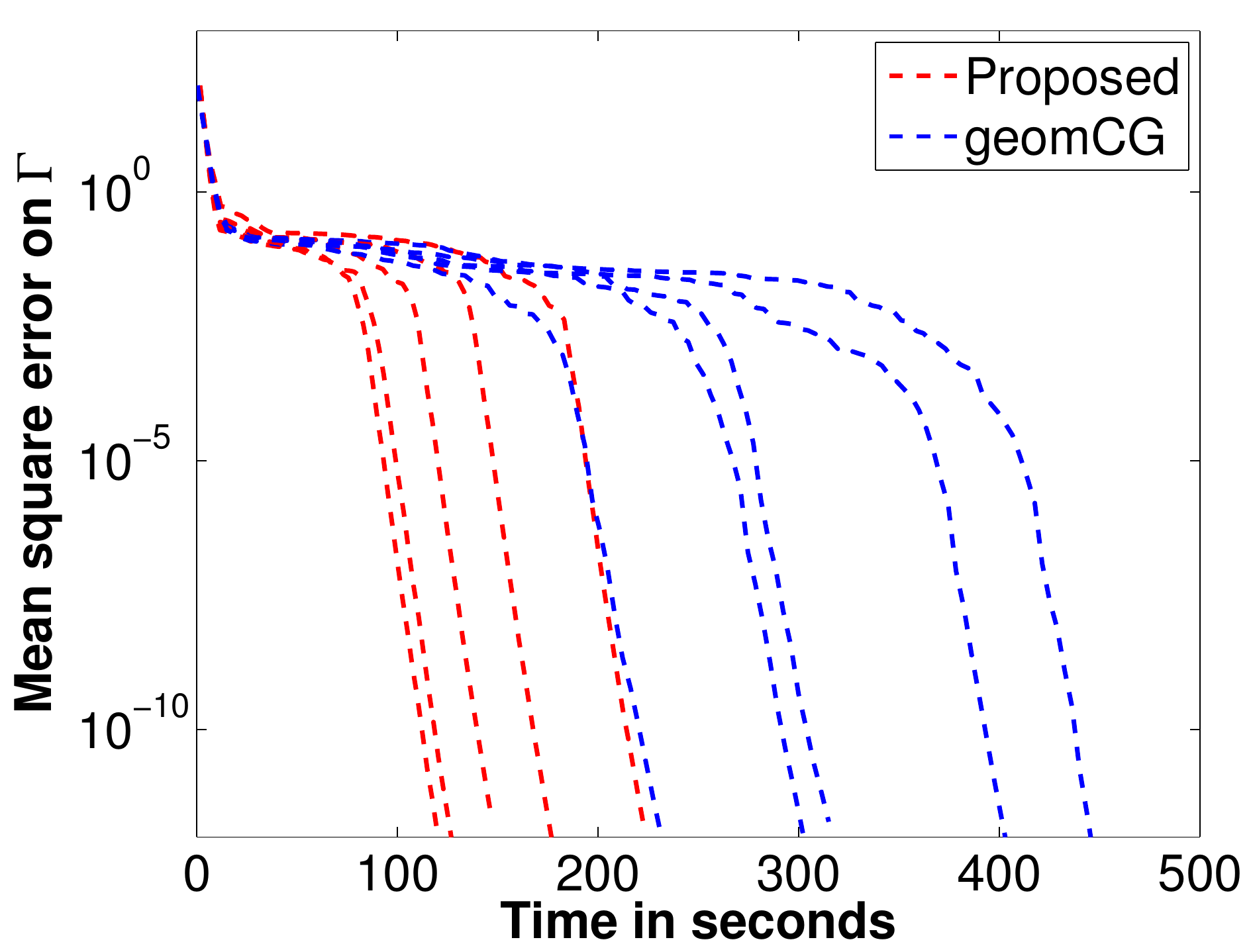}\\
{\scriptsize(e) {\bf r} = ($10\times 5\times 5$),\\
$10000\times 10000\times 10000$.}
\end{center}
\end{minipage}
\begin{minipage}{0.32\hsize}
\begin{center}
\includegraphics[width=\hsize]{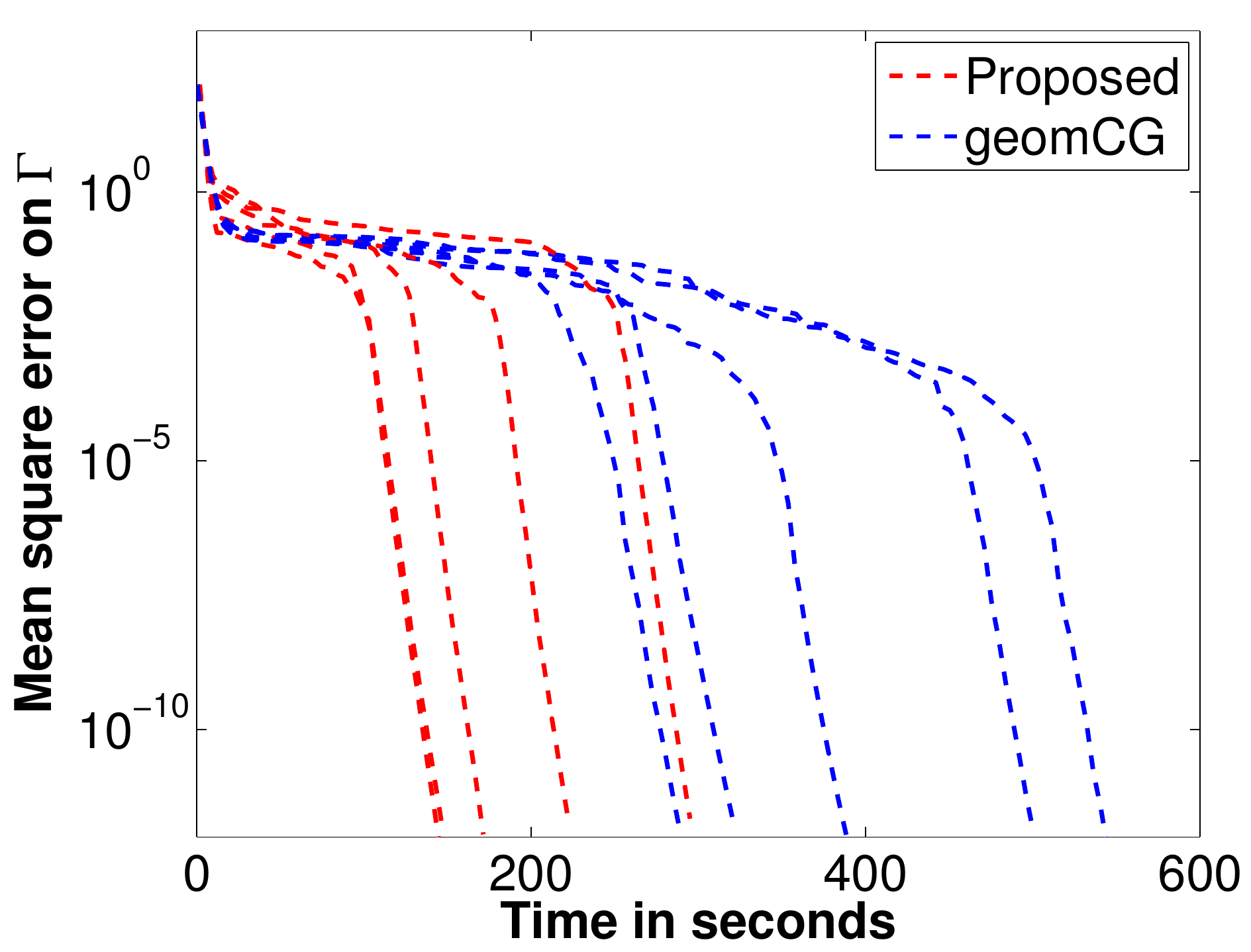}\\
{\scriptsize(f) {\bf r} = ($15\times 4\times 4$),\\
$10000\times 10000\times 10000$.}
\end{center}
\end{minipage}\\
\end{tabular}
\vspace{-0.1cm}
\caption{{\bf Case S7:} rectangular comparisons on $\Gamma$ (test error).}
\label{appnfig:asymmetric-test}
\end{figure*}

\clearpage

\begin{figure*}[htbp]
\vspace{-0.05cm}
\begin{tabular}{cc}
\begin{minipage}{0.32\hsize}
\begin{center}
\includegraphics[width=\hsize]{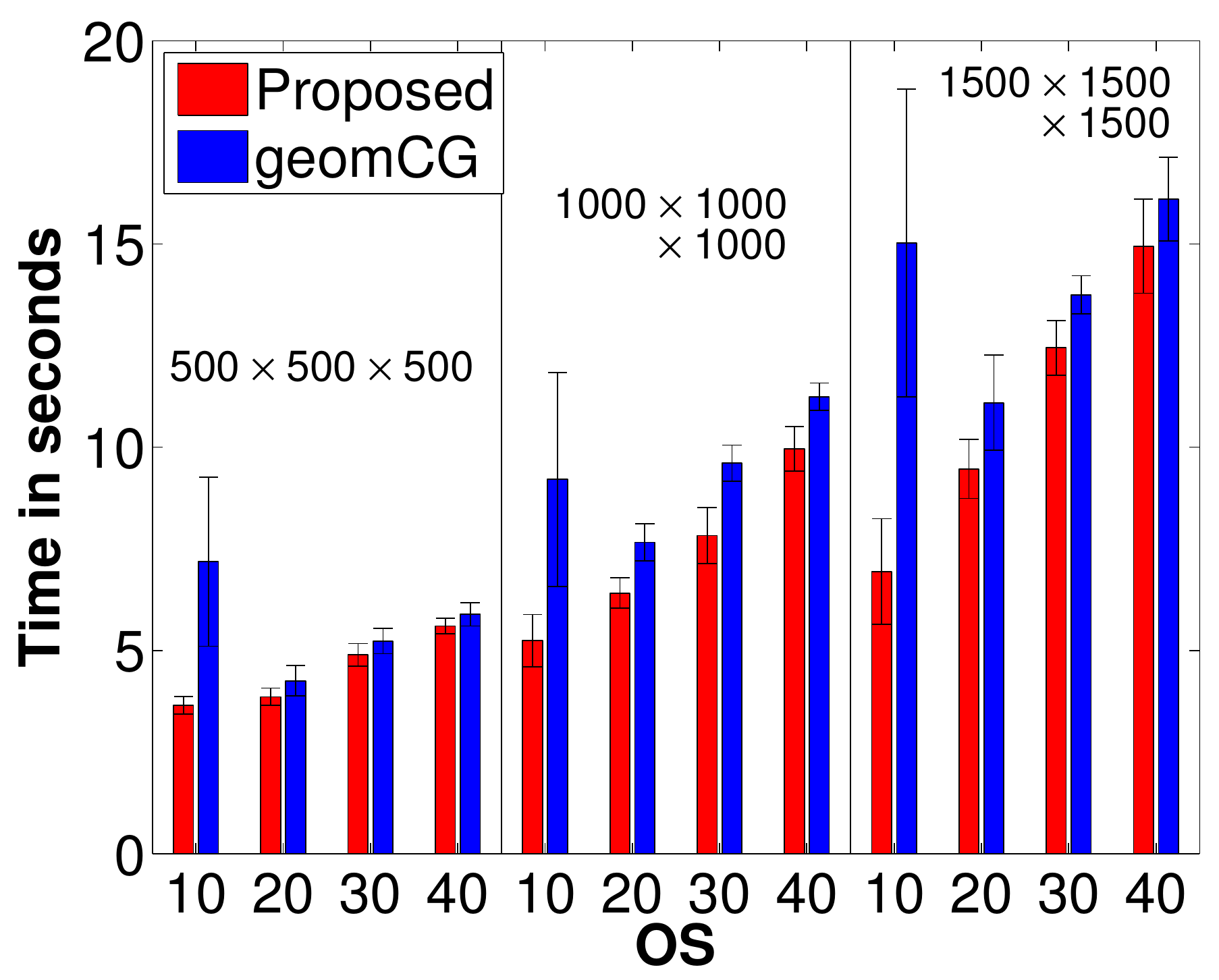}\\
{\scriptsize(a) {\bf r} = ($5\times 5\times 5$).}
\end{center}
\end{minipage}
\begin{minipage}{0.32\hsize}
\begin{center}
\includegraphics[width=\hsize]{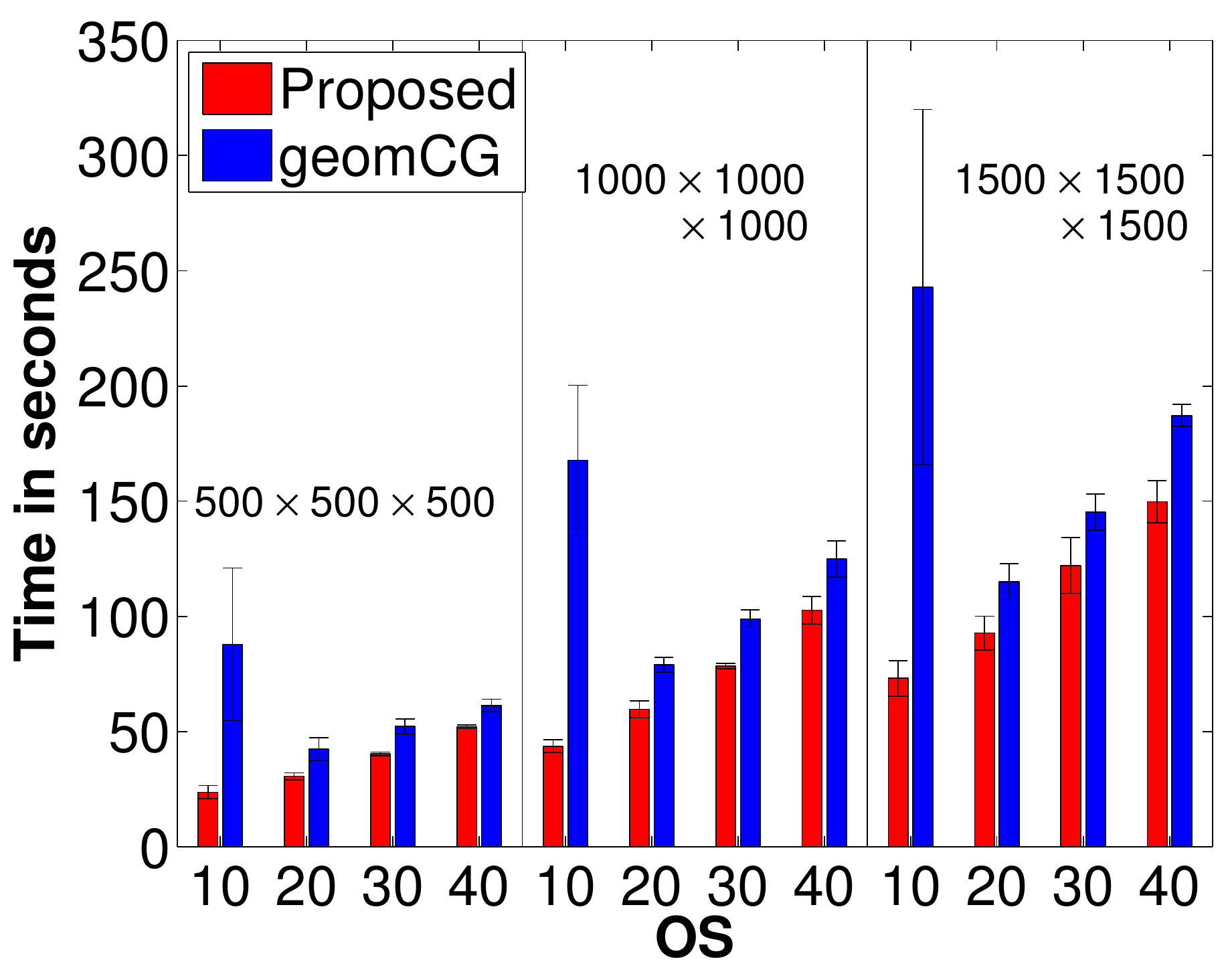}\\
{\scriptsize(b) {\bf r} = ($10\times 10\times 10$).}
\end{center}
\end{minipage}
\begin{minipage}{0.32\hsize}
\begin{center}
\includegraphics[width=\hsize]{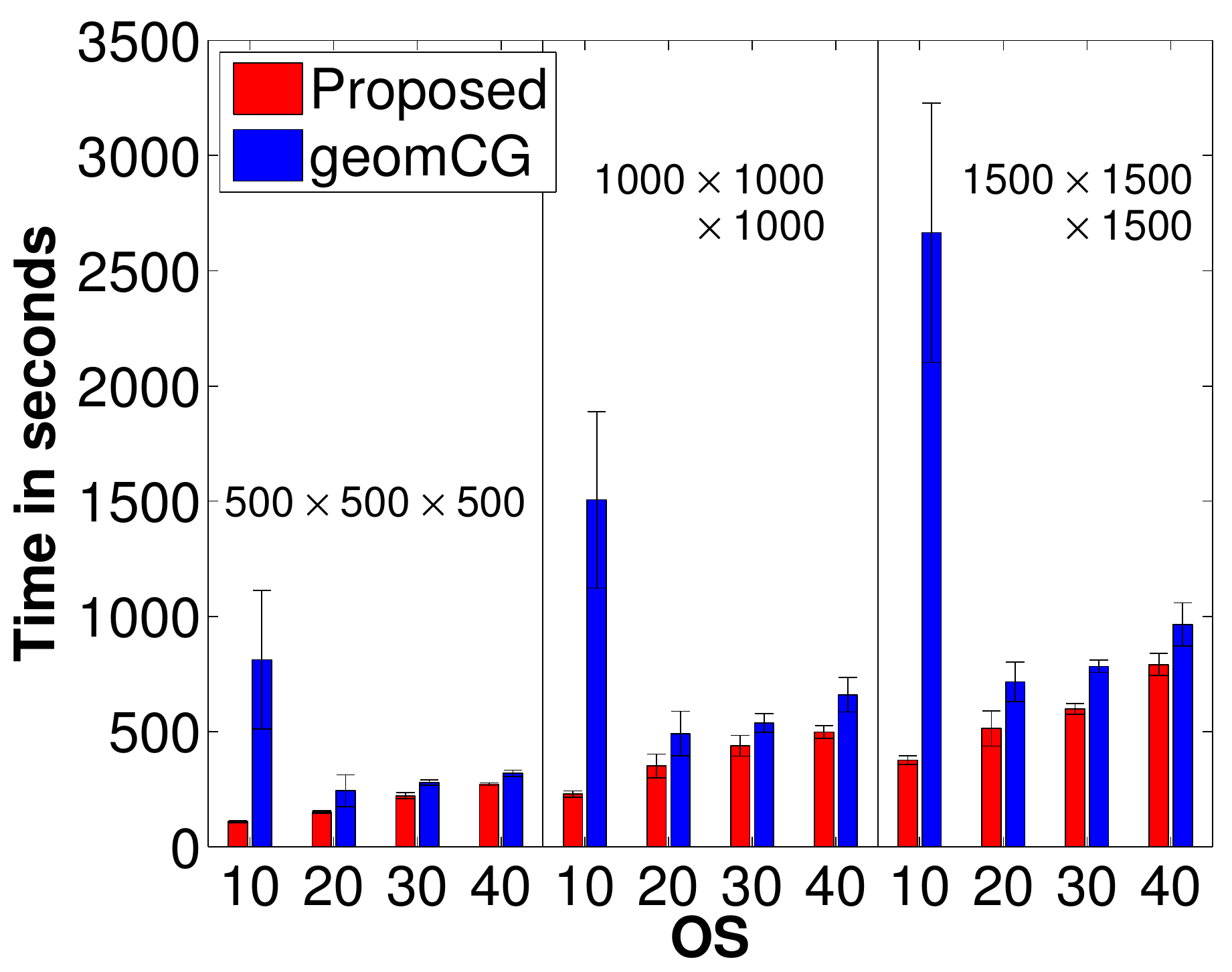}\\
{\scriptsize(c) {\bf r} = ($15\times 15\times 15$).}
\end{center}
\end{minipage}\\
\begin{minipage}{0.32\hsize}
\begin{center}
\includegraphics[width=\hsize]{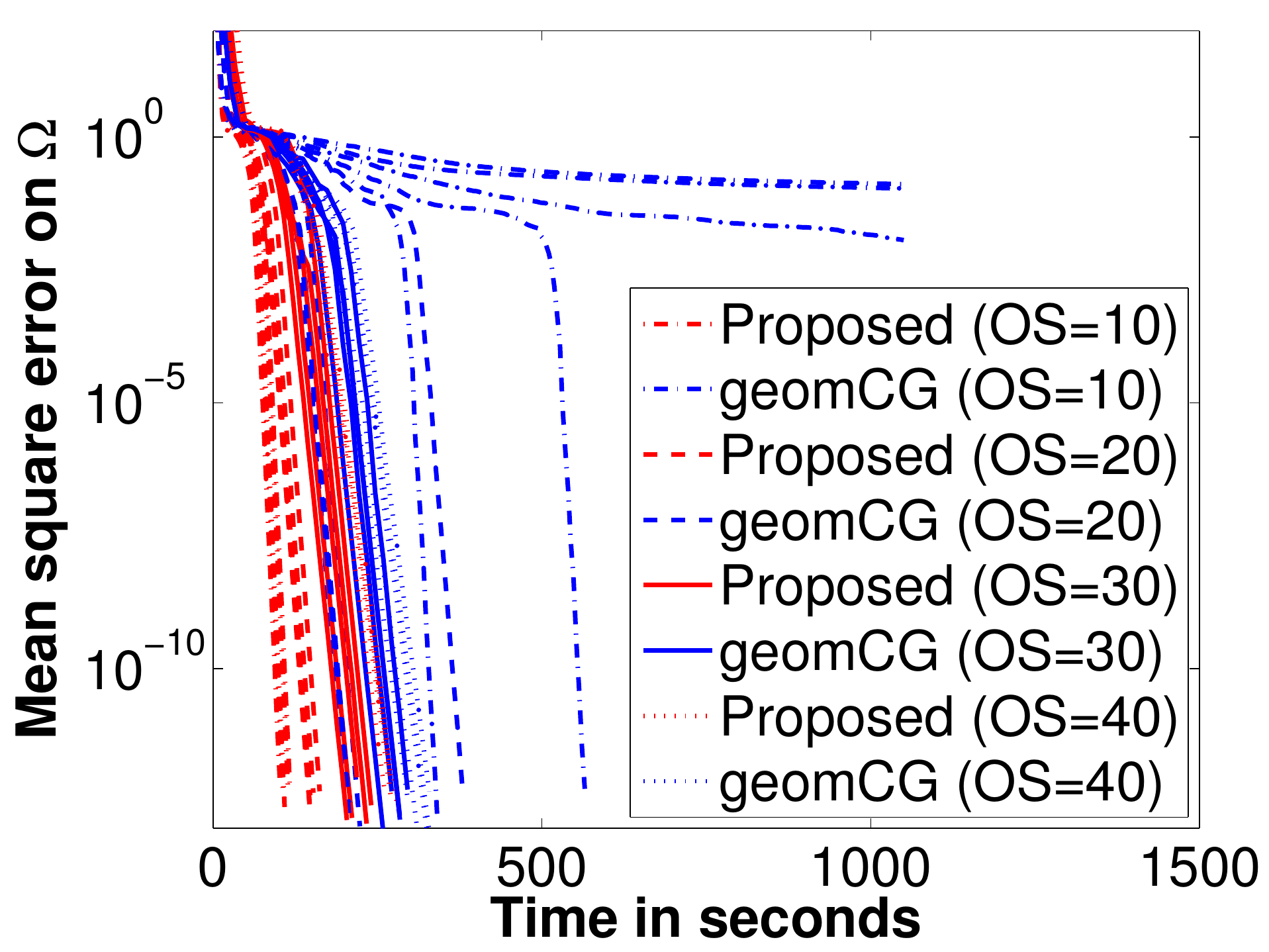}\\
{\scriptsize(d) $500\times 500\times 500$,\\
{\bf r} = ($15\times 15\times 15$).}
\end{center}
\end{minipage}
\begin{minipage}{0.32\hsize}
\begin{center}
\includegraphics[width=\hsize]{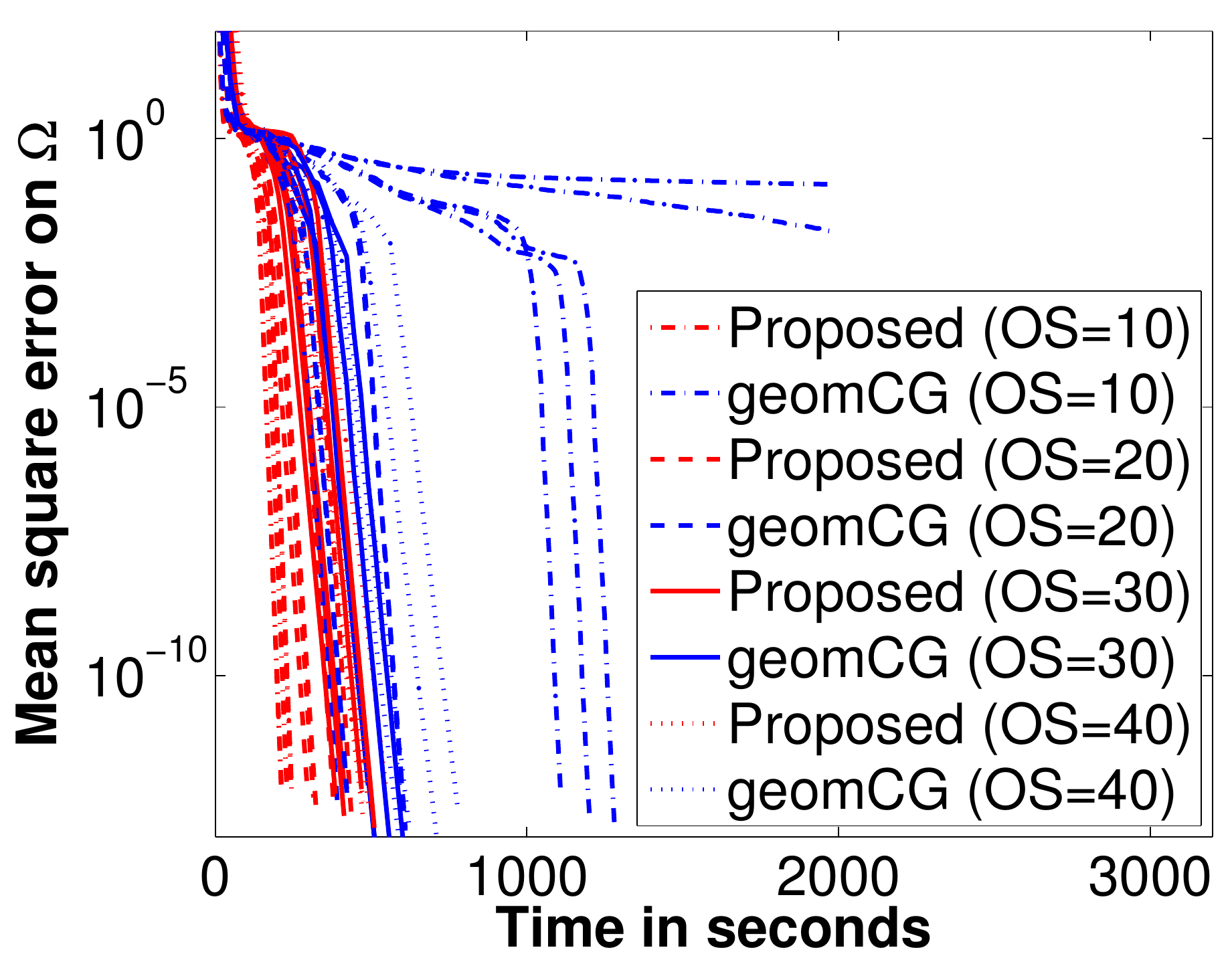}\\
{\scriptsize(e) $1000\times 1000\times 1000$,\\
{\bf r} = ($15\times 15\times 15$).}
\end{center}
\end{minipage}
\begin{minipage}{0.32\hsize}
\begin{center}
\includegraphics[width=\hsize]{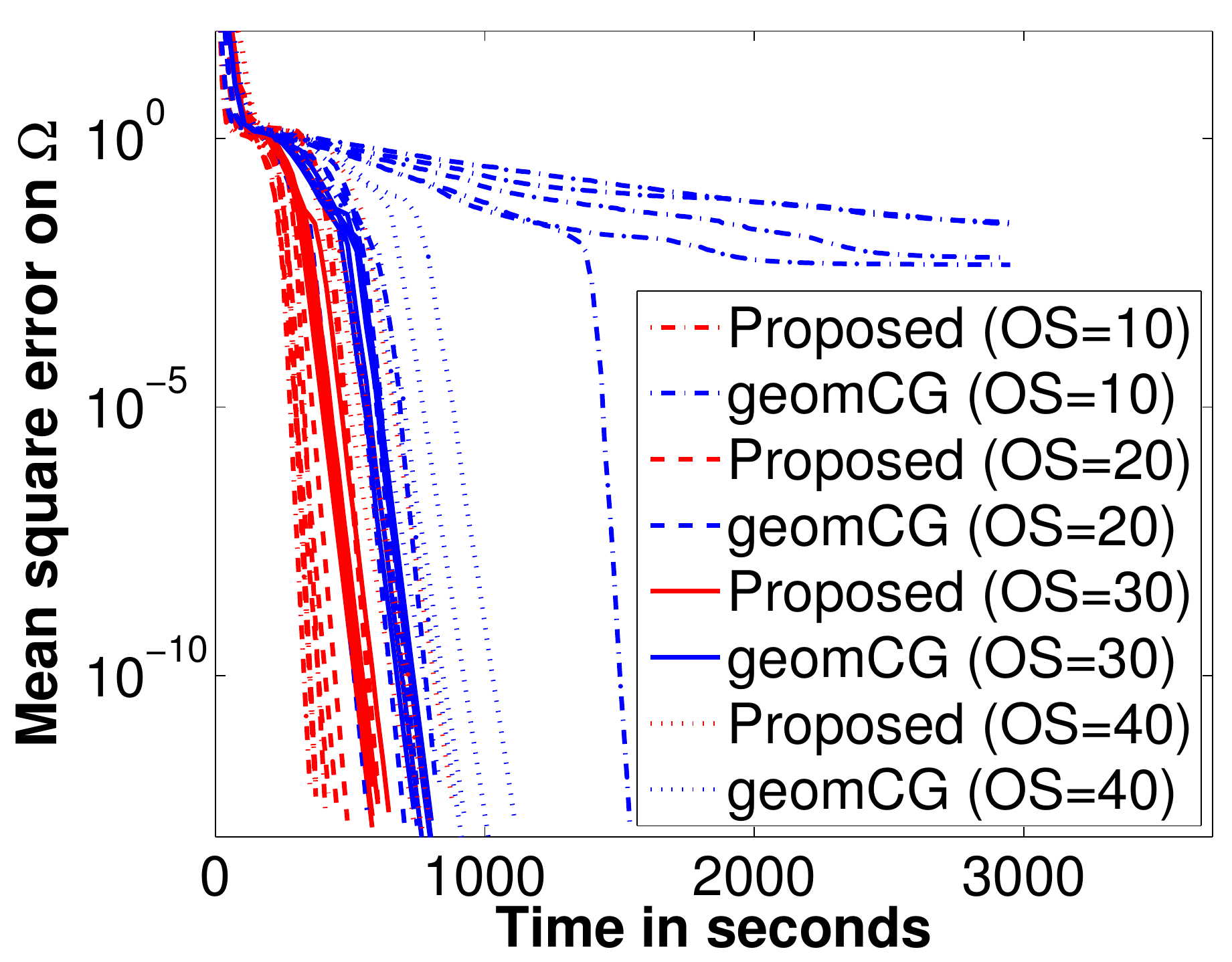}\\
{\scriptsize(f) $1500\times 1500\times 1500$,\\
{\bf r} = ($15\times 15\times 15$).}
\end{center}
\end{minipage}\\
\end{tabular}
\vspace{-0.1cm}
\caption{\changeHK{{\bf Case S8:} medium-scale comparisons on $\Omega$ (train error).}}
\label{appnfig:middle-scale-train}
%
%
\vspace{1.5cm}
\begin{tabular}{cc}
\begin{minipage}{0.32\hsize}
\begin{center}
\includegraphics[width=\hsize]{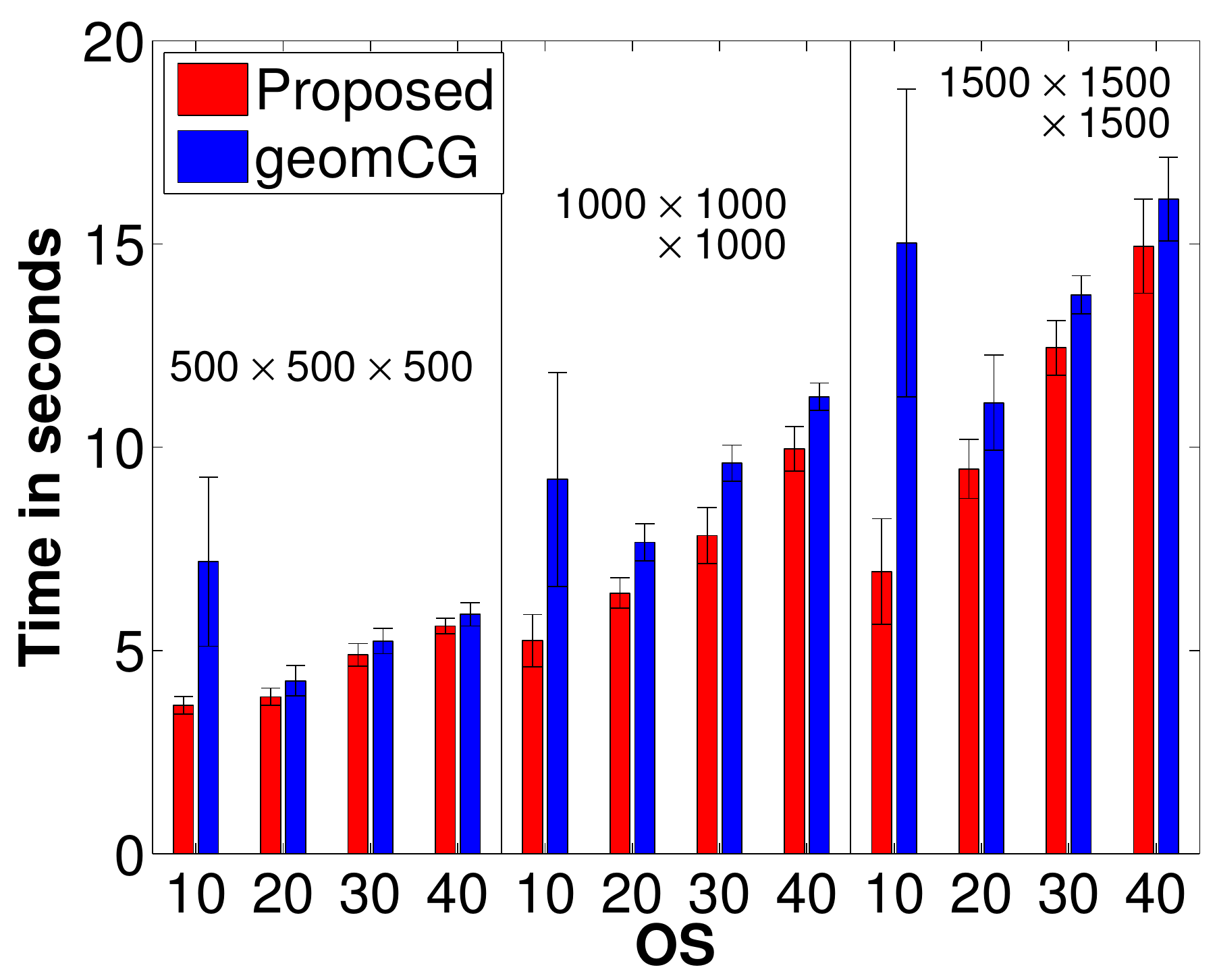}\\
{\scriptsize(a) {\bf r} = ($5\times 5\times 5$).}
\end{center}
\end{minipage}
\begin{minipage}{0.32\hsize}
\begin{center}
\includegraphics[width=\hsize]{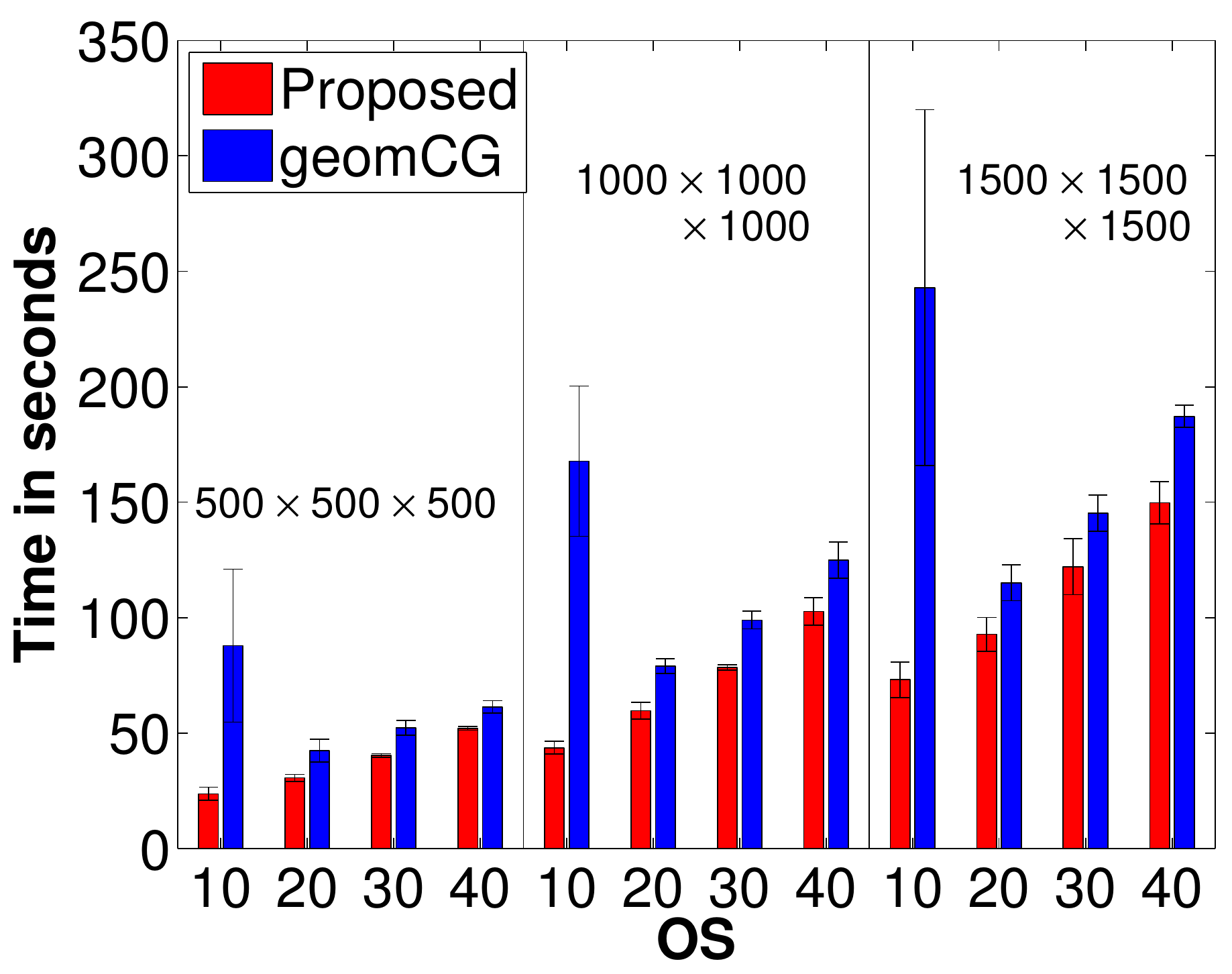}\\
{\scriptsize(b) {\bf r} = ($10\times 10\times 10$).}
\end{center}
\end{minipage}
\begin{minipage}{0.32\hsize}
\begin{center}
\includegraphics[width=\hsize]{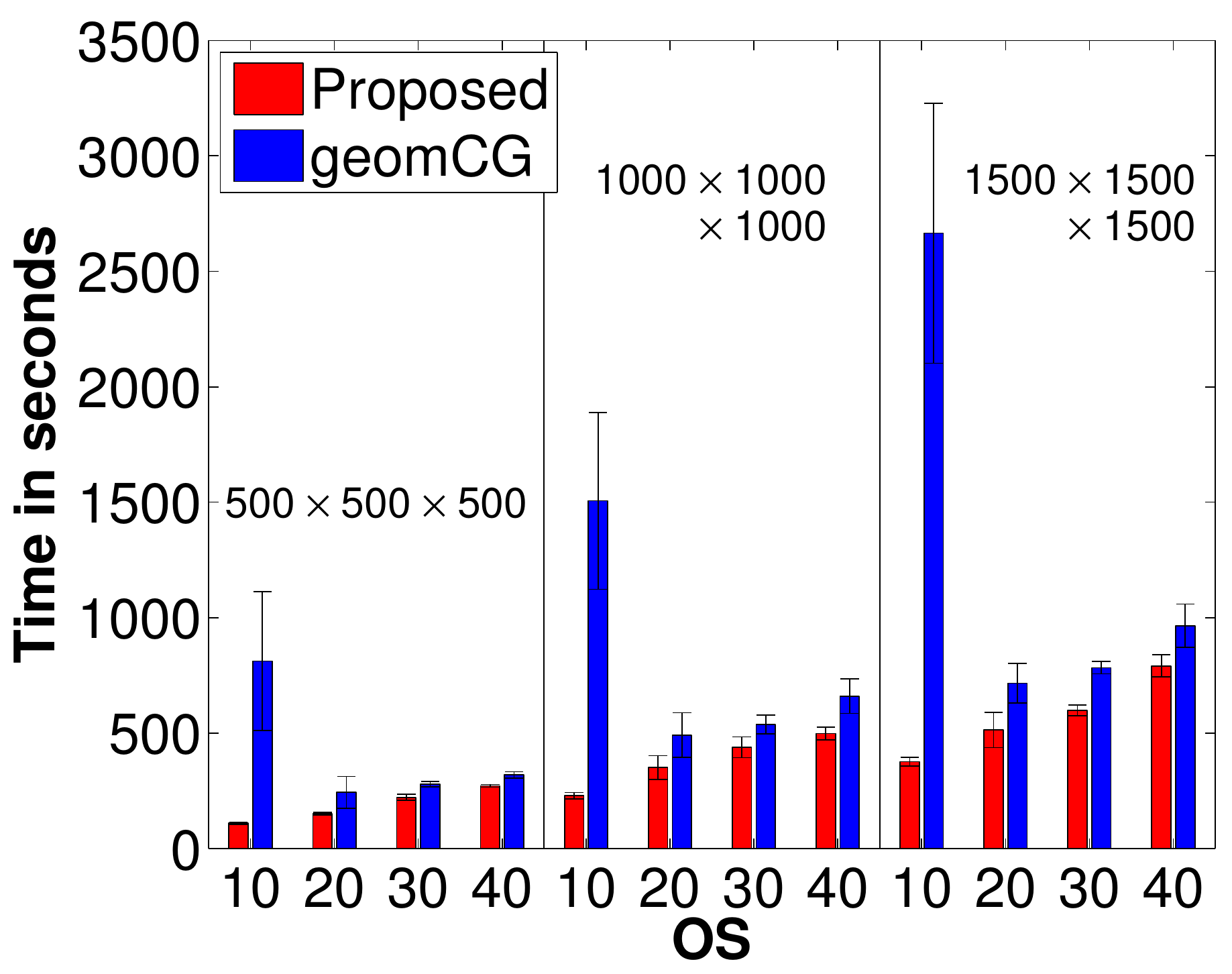}\\
{\scriptsize(c) {\bf r} = ($15\times 15\times 15$).}
\end{center}
\end{minipage}\\
\begin{minipage}{0.32\hsize}
\begin{center}
\includegraphics[width=\hsize]{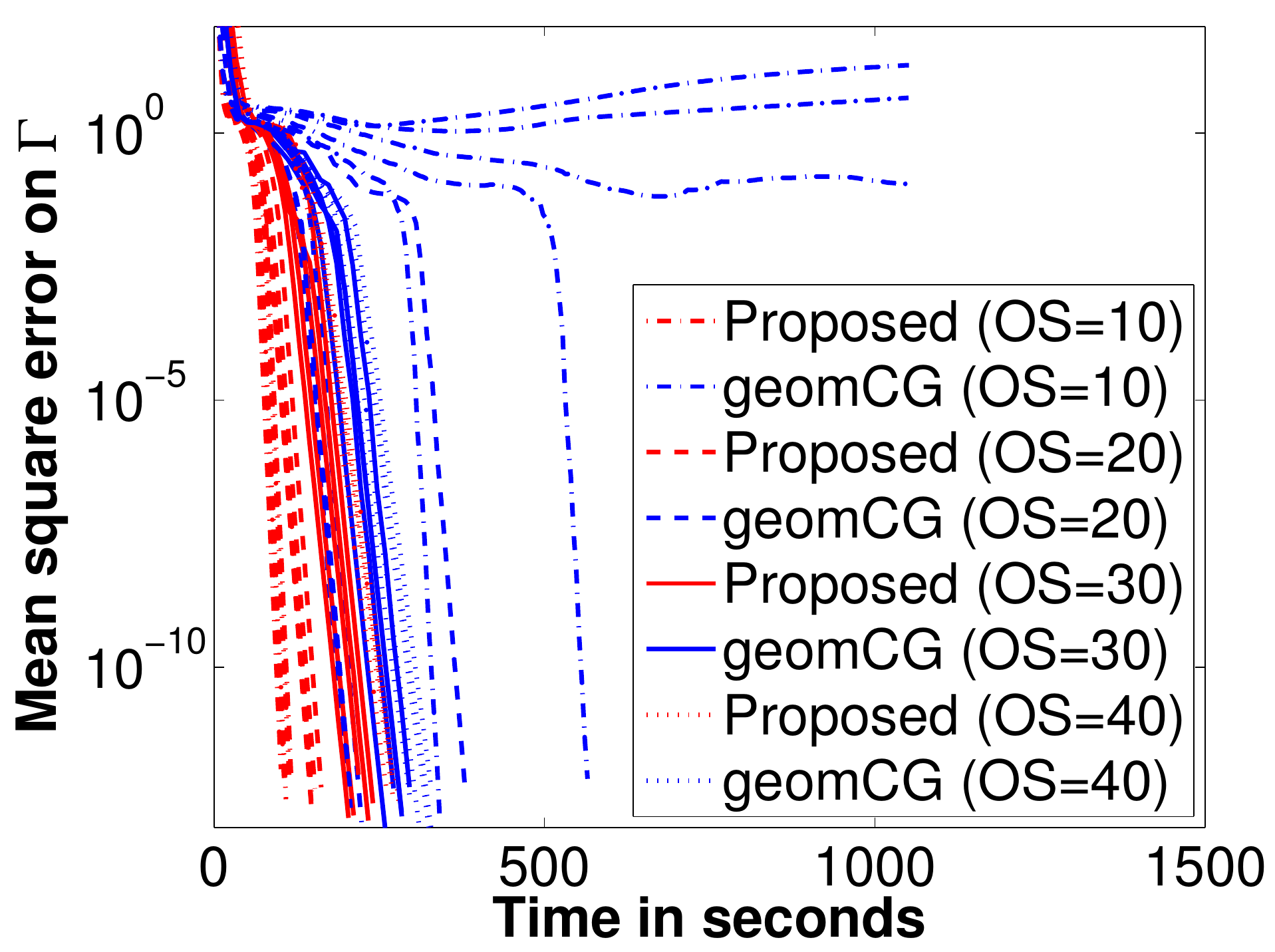}\\
{\scriptsize(d) $500\times 500\times 500$,\\
{\bf r} = ($15\times 15\times 15$).}
\end{center}
\end{minipage}
\begin{minipage}{0.32\hsize}
\begin{center}
\includegraphics[width=\hsize]{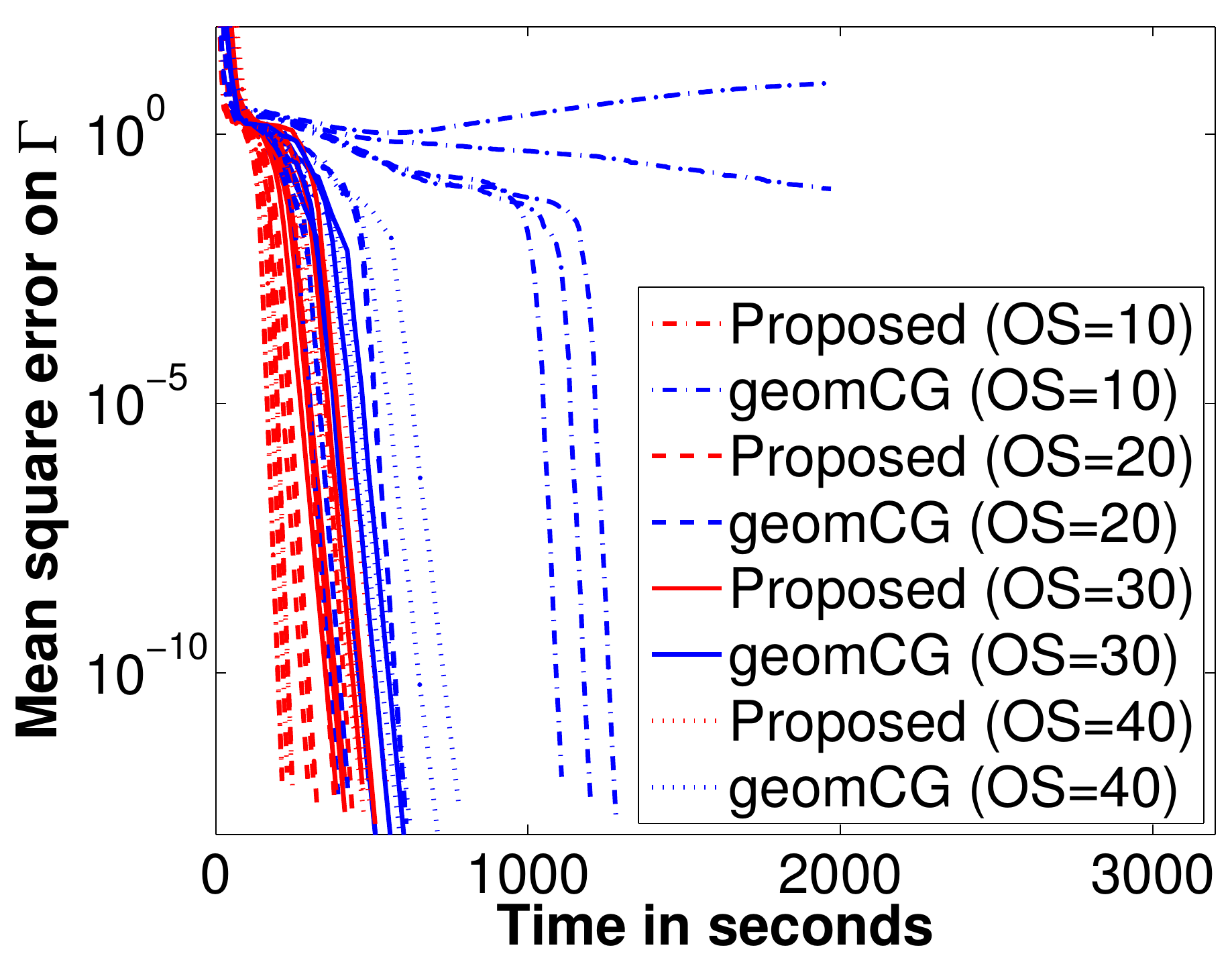}\\
{\scriptsize(e) $1000\times 1000\times 1000$,\\
{\bf r} = ($15\times 15\times 15$).}
\end{center}
\end{minipage}
\begin{minipage}{0.32\hsize}
\begin{center}
\includegraphics[width=\hsize]{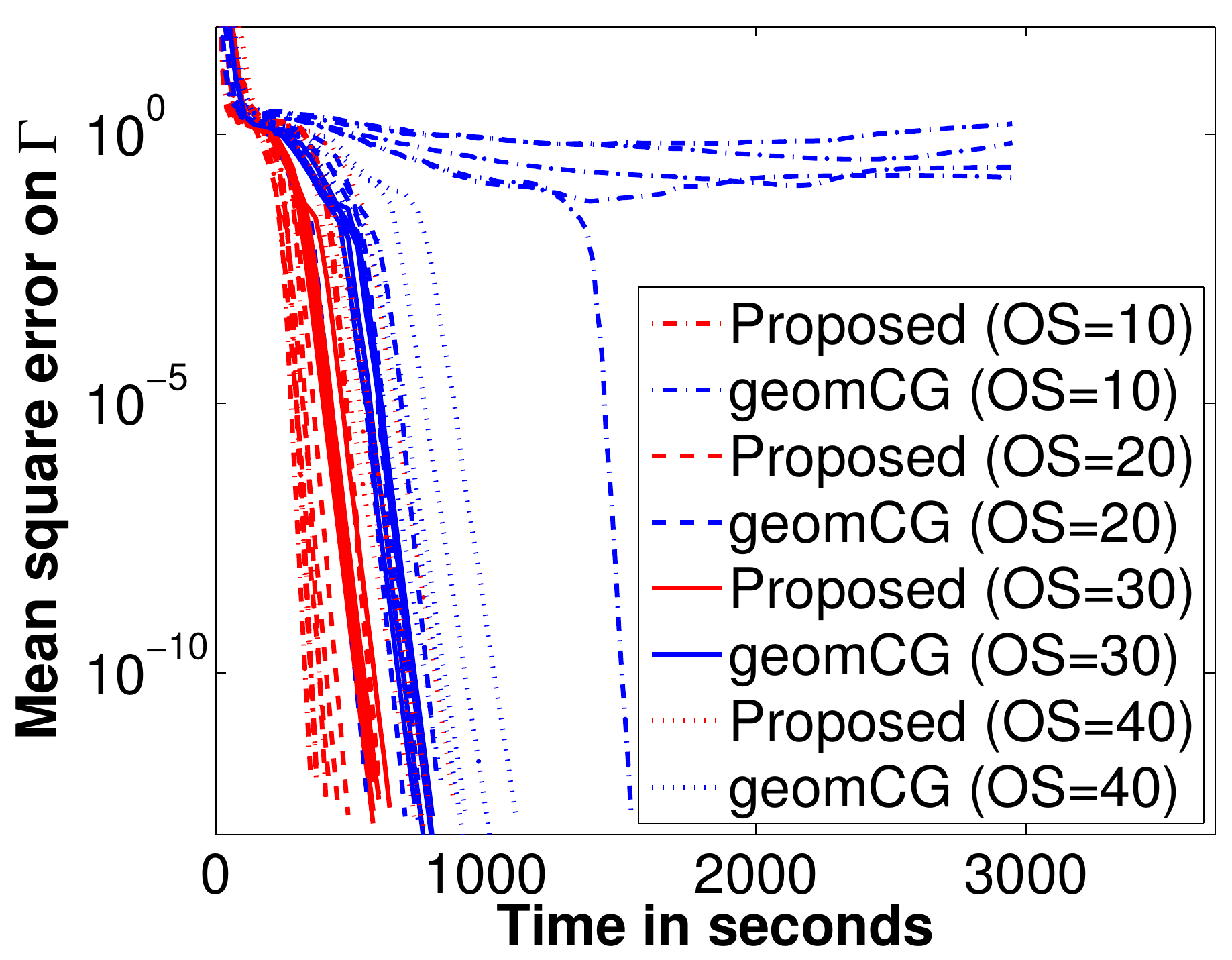}\\
{\scriptsize(f) $1500\times 1500\times 1500$,\\
{\bf r} = ($15\times 15\times 15$).}
\end{center}
\end{minipage}\\
\end{tabular}
\vspace{-0.1cm}
\caption{{\bf Case S8:} medium-scale comparisons on $\Gamma$ (test error).}
\label{appnfig:middle-scale-test}
\end{figure*}

\clearpage
\begin{figure*}[htbp]
\begin{tabular}{cc}
\begin{minipage}{0.48\hsize}
\begin{center}
\includegraphics[width=\hsize]{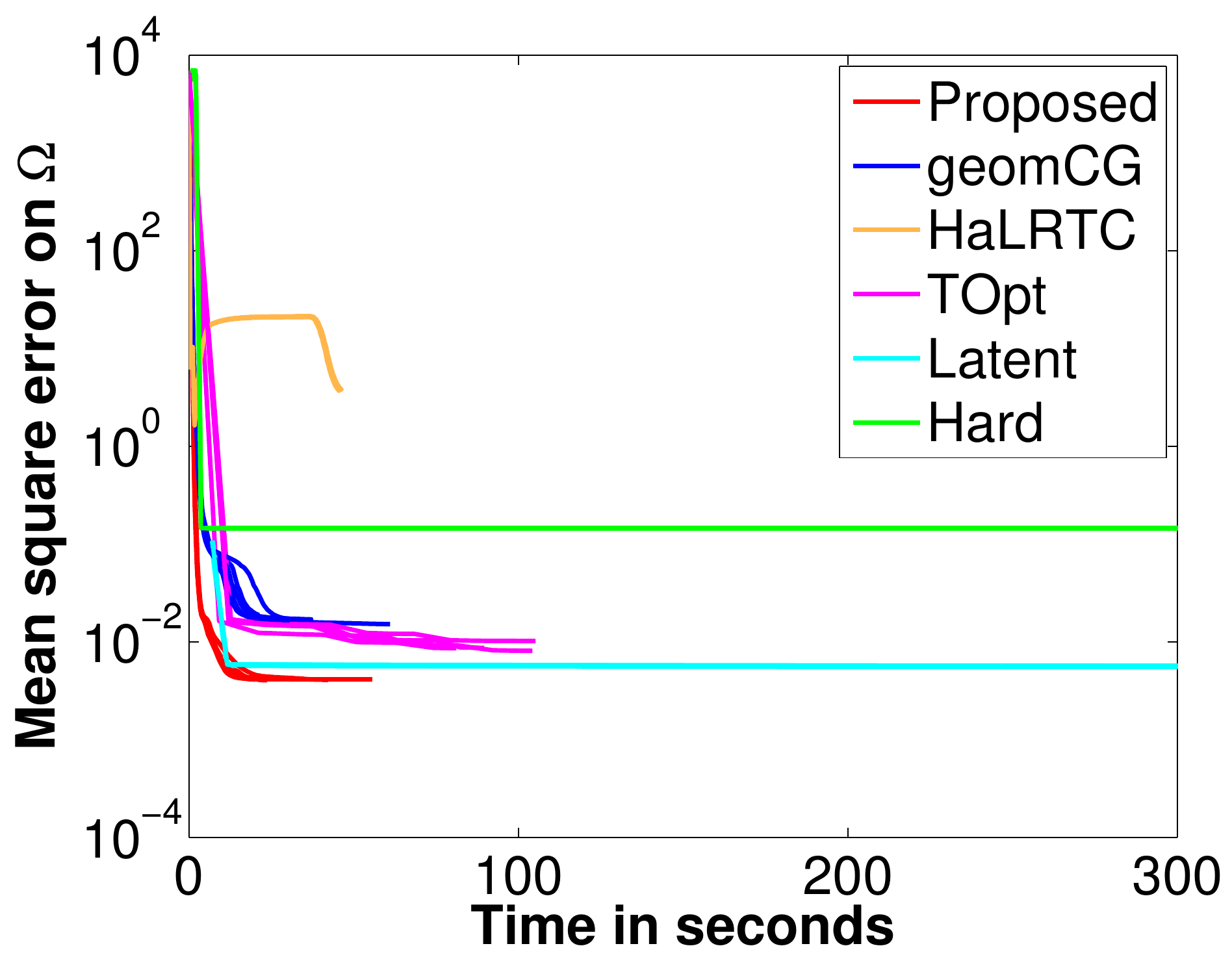}\\
{\scriptsize(a) OS = $11$.} 
\end{center}
\end{minipage}
\begin{minipage}{0.48\hsize}
\begin{center}
\includegraphics[width=\hsize]{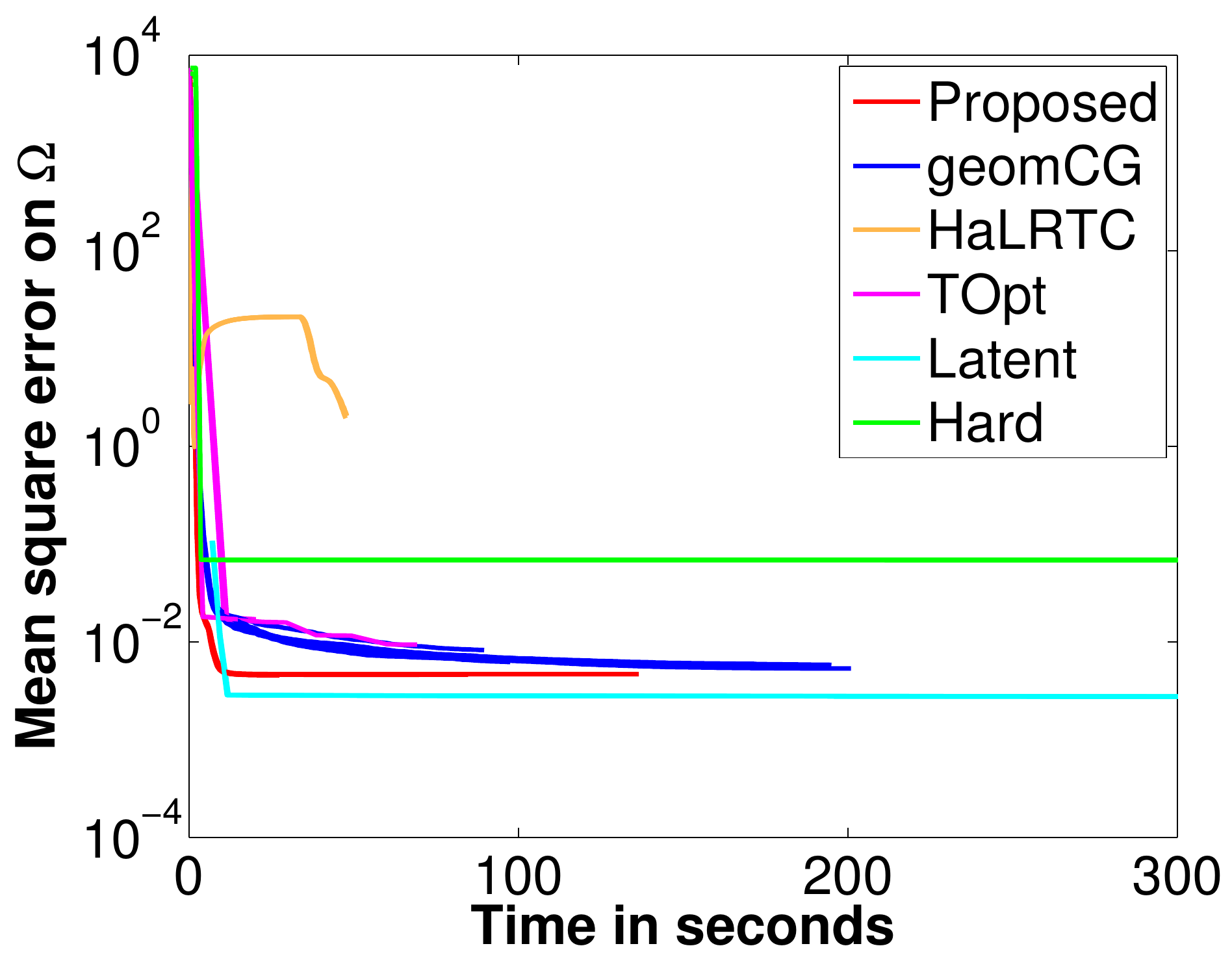}\\
{\scriptsize(b) OS = $22$.} 
\end{center}
\end{minipage}
\end{tabular}
\caption{\changeHK{{\bf Case R1:} mean square error on $\Omega$ (train error).}}
\label{appnfig:R1-train}
\end{figure*}

\begin{figure*}[htbp]
\begin{tabular}{cc}
\begin{minipage}{0.48\hsize}
\begin{center}
\includegraphics[width=\hsize]{figures/caseR1_riebeira_small_OS_11_meansquaretesterror-eps-converted-to.pdf}\\
{\scriptsize(a) OS = $11$.} 
\end{center}
\end{minipage}
\begin{minipage}{0.48\hsize}
\begin{center}
\includegraphics[width=\hsize]{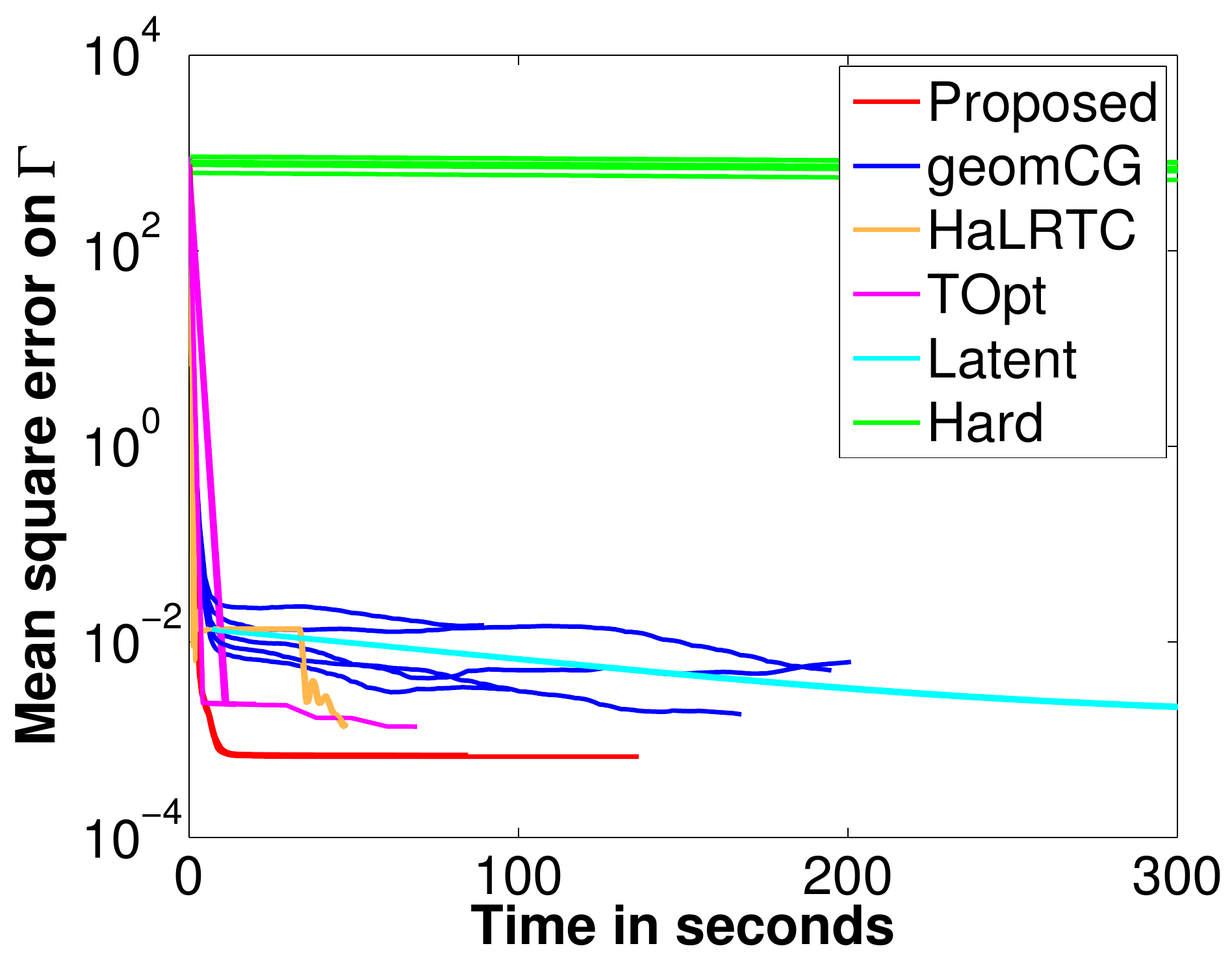}\\
{\scriptsize(b) OS = $22$.} 
\end{center}
\end{minipage}
\end{tabular}
\caption{{\bf Case R1:} mean square error on $\Gamma$ (test error).}
\label{appnfig:R1-test}
\end{figure*}

\begin{figure*}[htbp]
\begin{tabular}{cccc}
\begin{minipage}{0.24\hsize}
\begin{center}
\includegraphics[width=\hsize, bb=0 0 268 203]{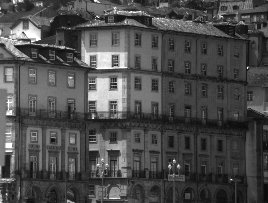}

{\scriptsize (a) Original.}
\label{fig:winter}
\end{center}
\end{minipage}
\begin{minipage}{0.24\hsize}
\begin{center}
\includegraphics[width=\hsize, bb=0 0 268 203]{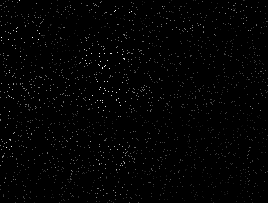}

{\scriptsize (b) Sampled ($4.98$\% observed).}
\label{fig:fall}
\end{center}
\end{minipage}

\begin{minipage}{0.24\hsize}
\begin{center}
\includegraphics[width=\hsize, bb=0 0 268 203]{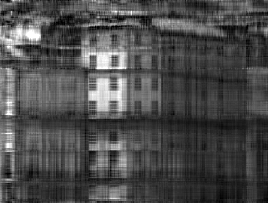}
{\scriptsize (c) Proposed.}
\label{fig:winter}
\end{center}
\end{minipage}
\begin{minipage}{0.24\hsize}
\begin{center}
\includegraphics[width=\hsize, bb=0 0 268 203]{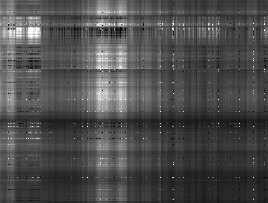}
{\scriptsize (d) geomCG.}
\label{fig:fall}
\end{center}
\end{minipage}\vspace*{0.2cm}\\

\begin{minipage}{0.24\hsize}
\begin{center}
\includegraphics[width=\hsize, bb=0 0 268 203]{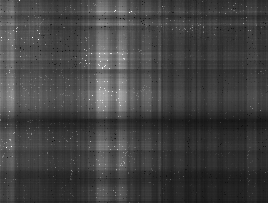}

{\scriptsize (e) HaLRTC.}
\label{fig:winter}
\end{center}
\end{minipage}
\begin{minipage}{0.24\hsize}
\begin{center}
\includegraphics[width=\hsize, bb=0 0 268 203]{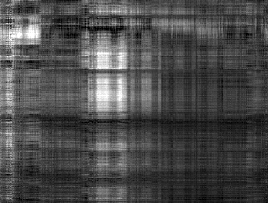}

{\scriptsize (f) TOpt.}
\label{fig:fall}
\end{center}
\end{minipage}

\begin{minipage}{0.24\hsize}
\begin{center}
\includegraphics[width=\hsize, bb=0 0 268 203]{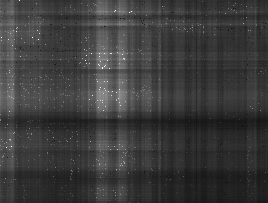}

{\scriptsize (g) Latent.}
\label{fig:winter}
\end{center}
\end{minipage}
\begin{minipage}{0.24\hsize}
\begin{center}
\includegraphics[width=\hsize, bb=0 0 268 203]{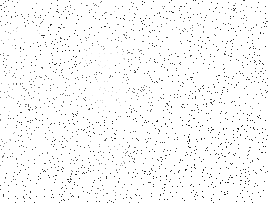}

{\scriptsize (h) Hard.}
\label{fig:fall}
\end{center}
\end{minipage}
\end{tabular}
\caption{{\bf Case R1:} recovery results on the hyperspectral image ``Ribeira" (frame = $16$, OS = $11$).}
\label{appnfig:R1-reconstructedimage_OS_11}
\end{figure*}

\begin{figure*}[htbp]
\begin{tabular}{cccc}
\begin{minipage}{0.24\hsize}
\begin{center}
\includegraphics[width=\hsize, bb=0 0 268 203]{figures/ref_ribeira1bbb_reg1_resize_203x268x33.png}

{\scriptsize (a) Original.}
\label{fig:winter}
\end{center}
\end{minipage}
\begin{minipage}{0.24\hsize}
\begin{center}
\includegraphics[width=\hsize, bb=0 0 268 203]{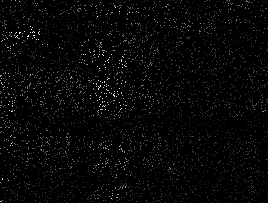}

{\scriptsize (b) Sampled ($9.96\%$ observed).}
\label{fig:fall}
\end{center}
\end{minipage}

\begin{minipage}{0.24\hsize}
\begin{center}
\includegraphics[width=\hsize, bb=0 0 268 203]{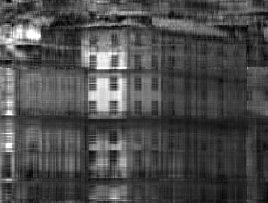}

{\scriptsize (c) Proposed.}
\label{fig:winter}
\end{center}
\end{minipage}
\begin{minipage}{0.24\hsize}
\begin{center}
\includegraphics[width=\hsize, bb=0 0 268 203]{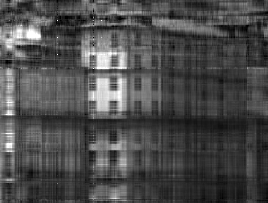}

{\scriptsize (d) geomCG.}
\label{fig:fall}
\end{center}
\end{minipage}\vspace*{0.2cm}\\

\begin{minipage}{0.24\hsize}
\begin{center}
\includegraphics[width=\hsize, bb=0 0 268 203]{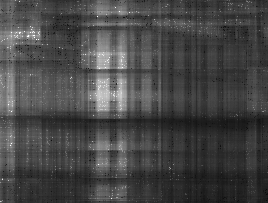}

{\scriptsize (e) HaLRTC.}
\label{fig:winter}
\end{center}
\end{minipage}
\begin{minipage}{0.24\hsize}
\begin{center}
\includegraphics[width=\hsize, bb=0 0 268 203]{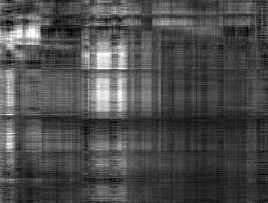}

{\scriptsize (f) TOpt.}
\label{fig:fall}
\end{center}
\end{minipage}

\begin{minipage}{0.24\hsize}
\begin{center}
\includegraphics[width=\hsize, bb=0 0 268 203]{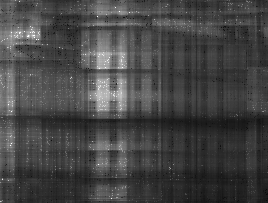}

{\scriptsize (g) Latent.}
\label{fig:winter}
\end{center}
\end{minipage}
\begin{minipage}{0.24\hsize}
\begin{center}
\includegraphics[width=\hsize, bb=0 0 268 203]{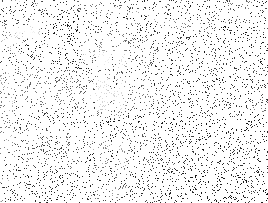}

{\scriptsize (h) Hard.}
\label{fig:fall}
\end{center}
\end{minipage}
\end{tabular}
\caption{{\bf Case R1:} recovery results on the hyperspectral image ``Ribeira" (frame = $16$, OS = $22$).}\label{appnfig:R1-reconstructedimage_OS_22}
\end{figure*}

\begin{figure*}[t]
\begin{tabular}{cc}
\begin{minipage}{0.48\hsize}
\begin{center}
\includegraphics[width=\hsize]{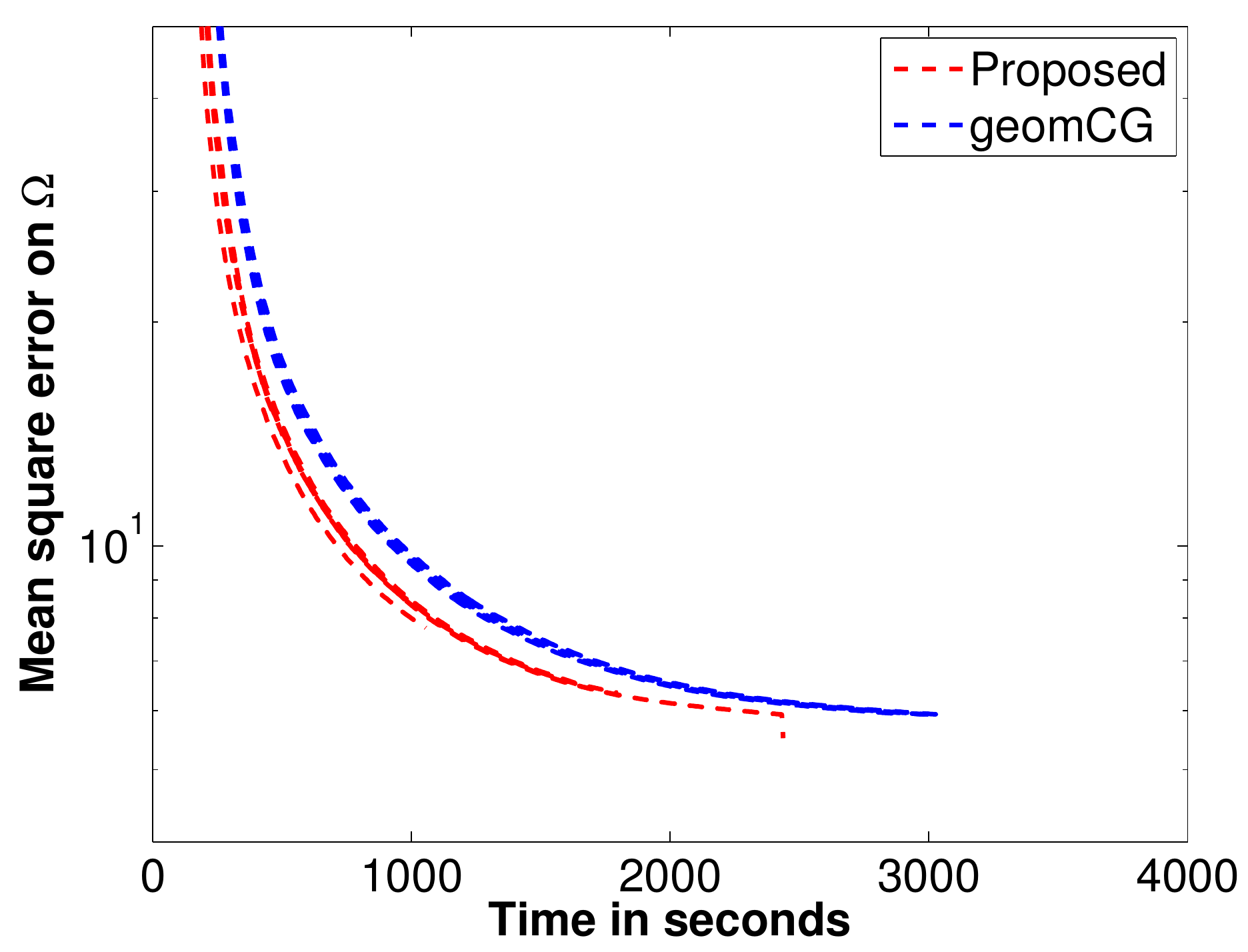}\\
{\scriptsize(a) {\bf r} = ($4\times 4\times 4$).} 
\end{center}
\end{minipage}
\begin{minipage}{0.48\hsize}
\begin{center}
\includegraphics[width=\hsize]{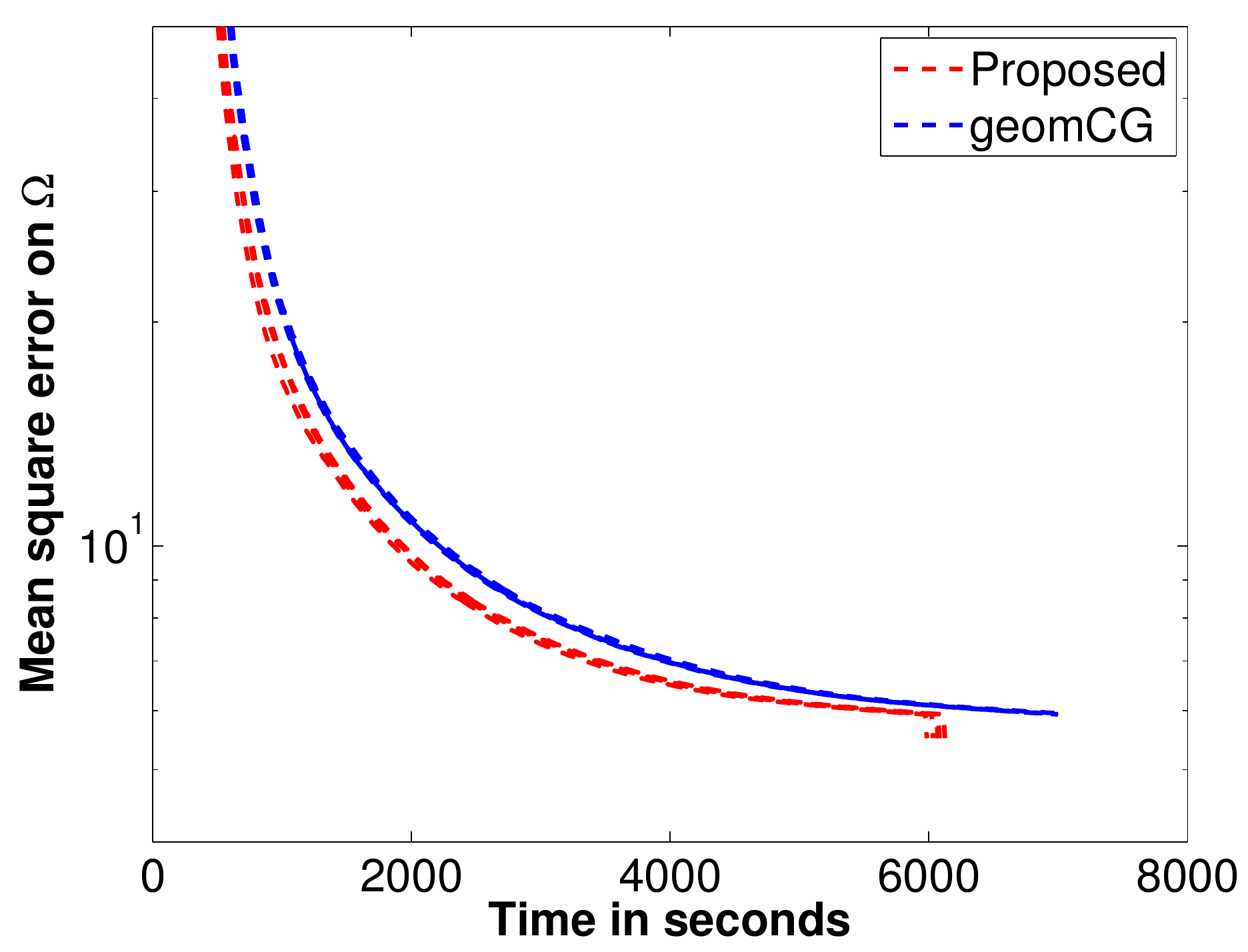}\\
{\scriptsize(b) {\bf r} = ($6\times 6\times 6$).} 
\end{center}
\end{minipage}\\

\begin{minipage}{0.48\hsize}
\begin{center}
\includegraphics[width=\hsize]{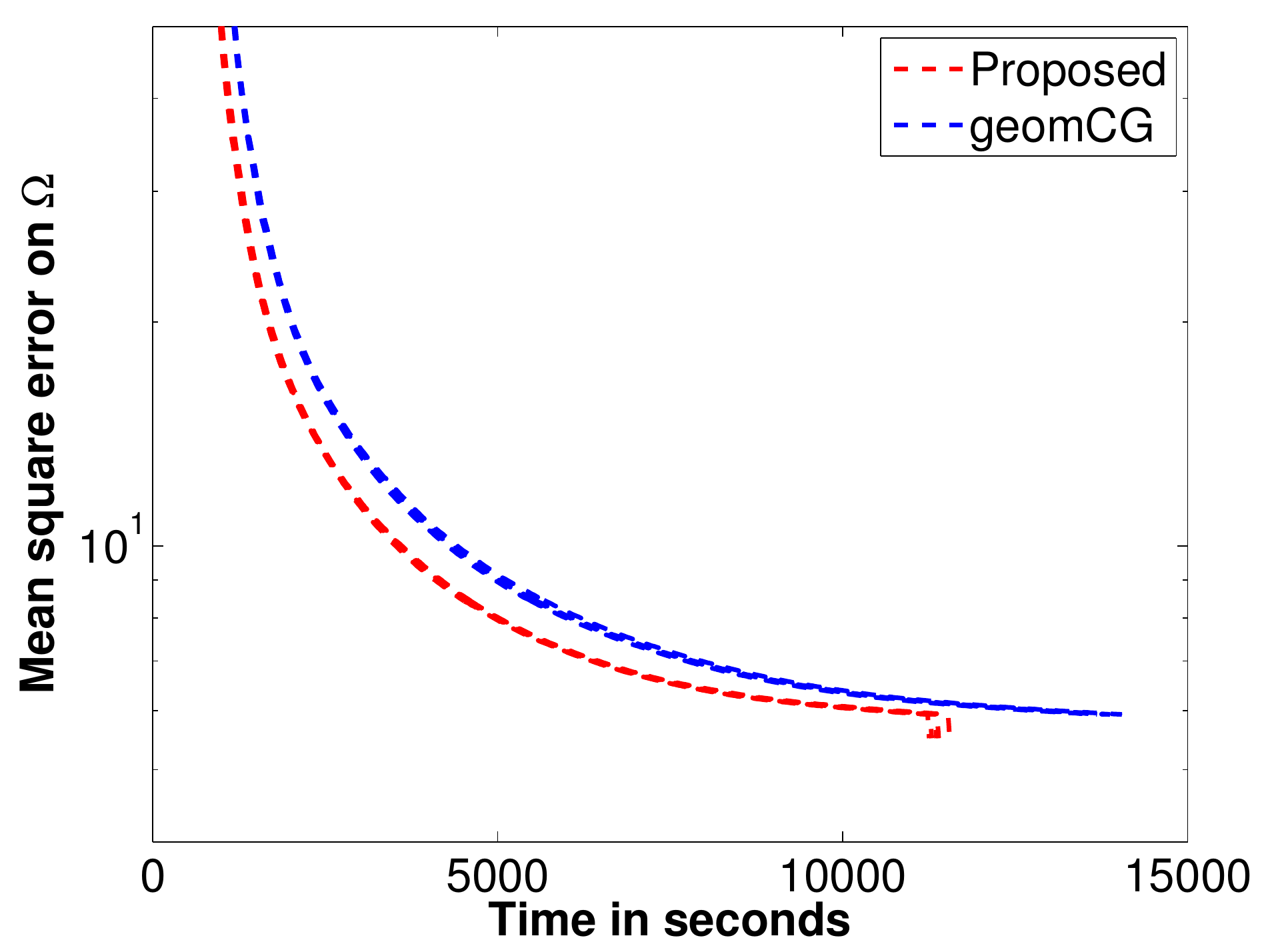}\\
{\scriptsize(c) {\bf r} = ($8\times 8\times 8$).} 
\end{center}
\end{minipage}
\begin{minipage}{0.48\hsize}
\begin{center}
\includegraphics[width=\hsize]{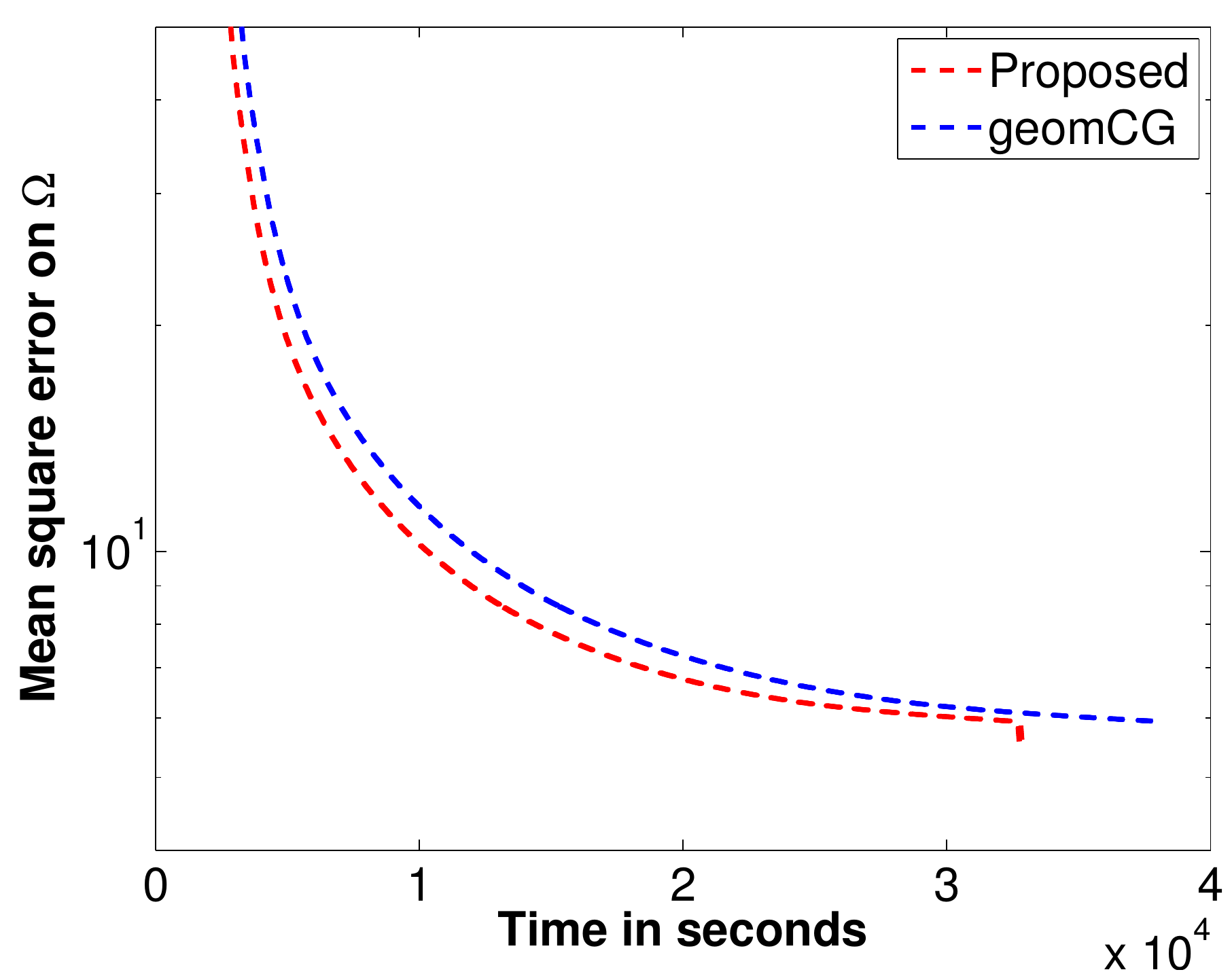}\\
{\scriptsize(d) {\bf r} = ($10\times 10\times 10$).} 
\end{center}
\end{minipage}
\end{tabular}
\caption{\changeHK{{\bf Case R2:} mean square error on $\Omega$ (train error).}}
\label{appnfig:R2-train}
\end{figure*}

\begin{figure*}[t]
\begin{tabular}{cc}
\begin{minipage}{0.48\hsize}
\begin{center}
\includegraphics[width=\hsize]{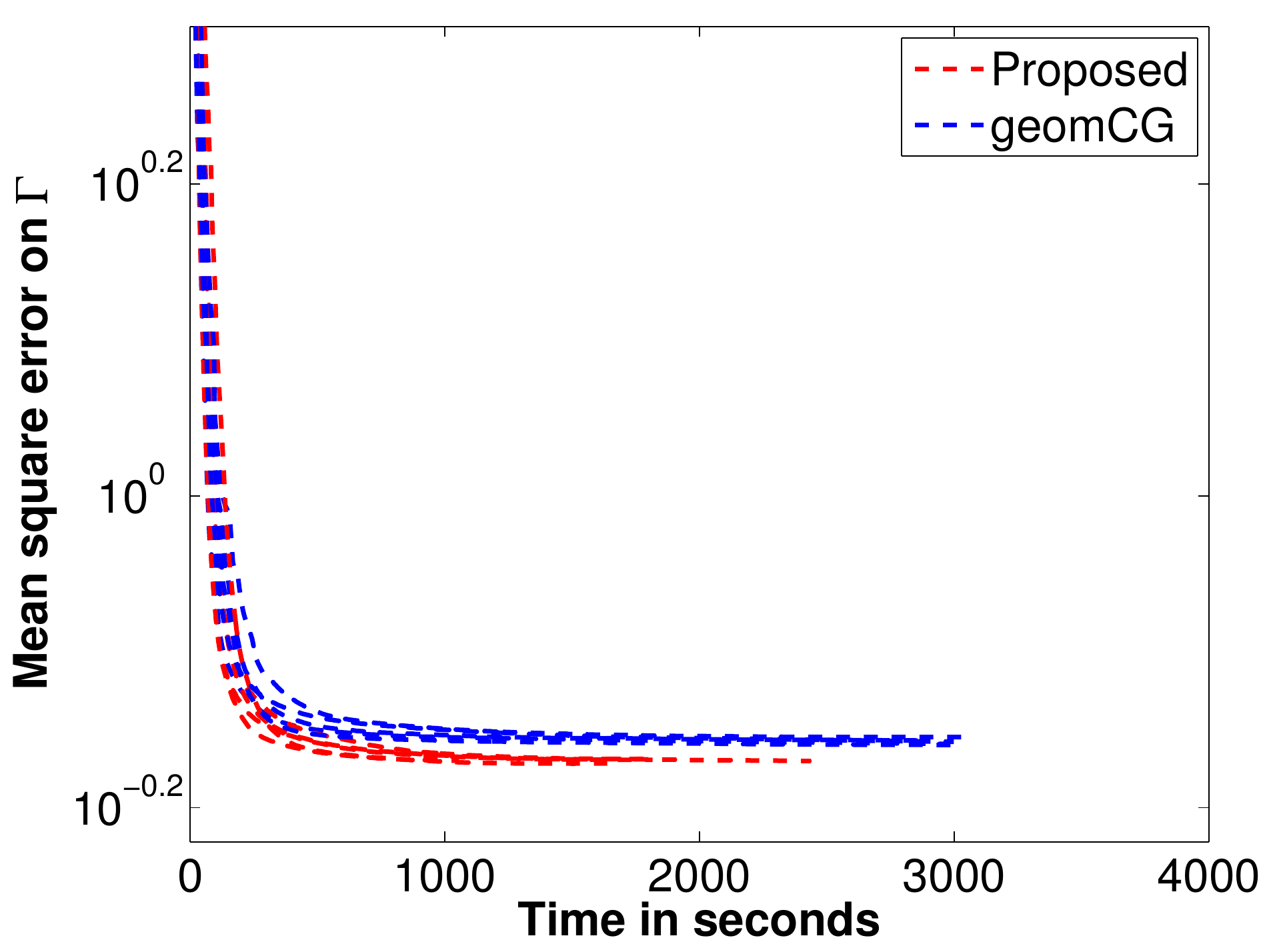}\\
{\scriptsize(a) {\bf r} = ($4\times 4\times 4$).} 
\end{center}
\end{minipage}
\begin{minipage}{0.48\hsize}
\begin{center}
\includegraphics[width=\hsize]{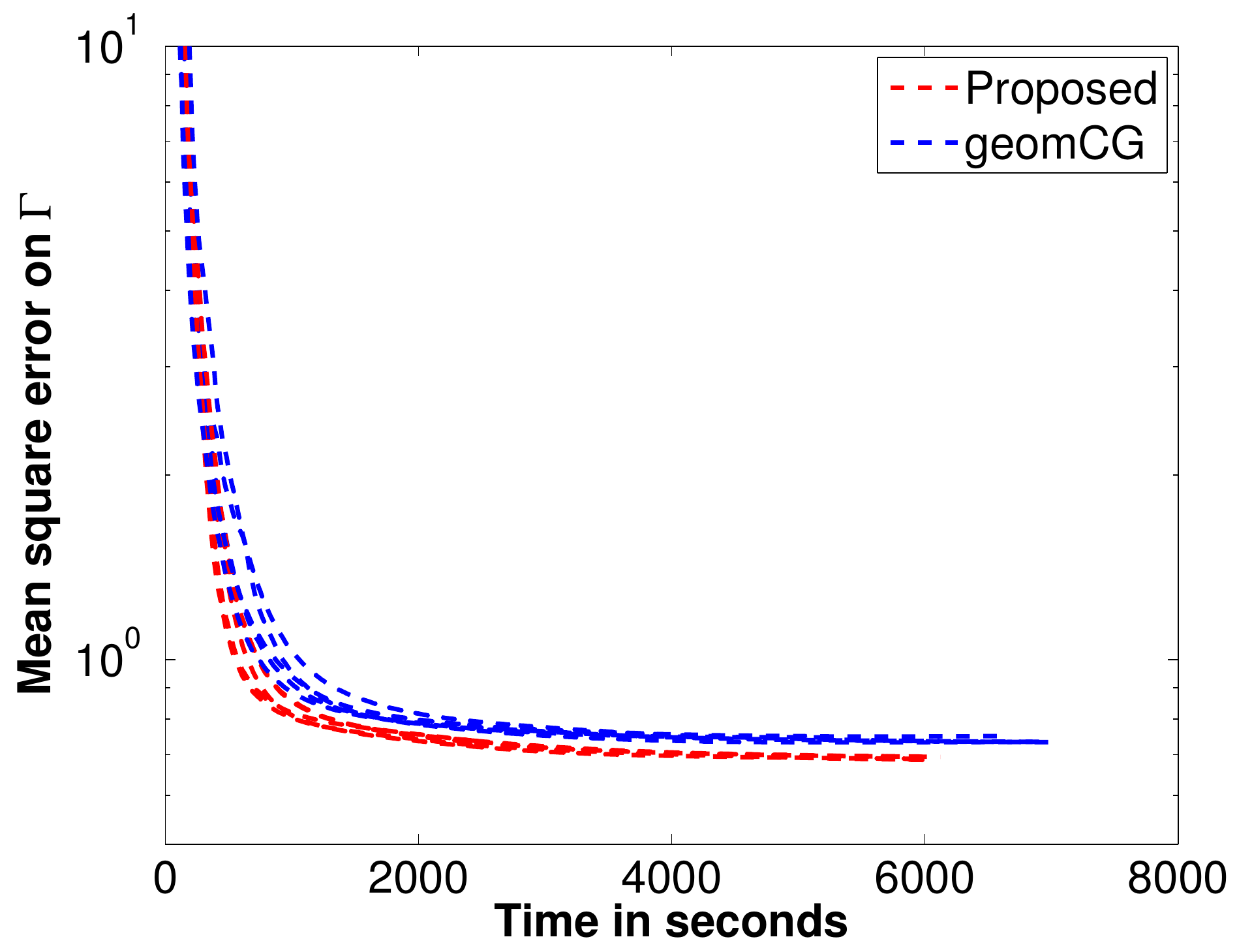}\\
{\scriptsize(b) {\bf r} = ($6\times 6\times 6$).} 
\end{center}
\end{minipage}\\

\begin{minipage}{0.48\hsize}
\begin{center}
\includegraphics[width=\hsize]{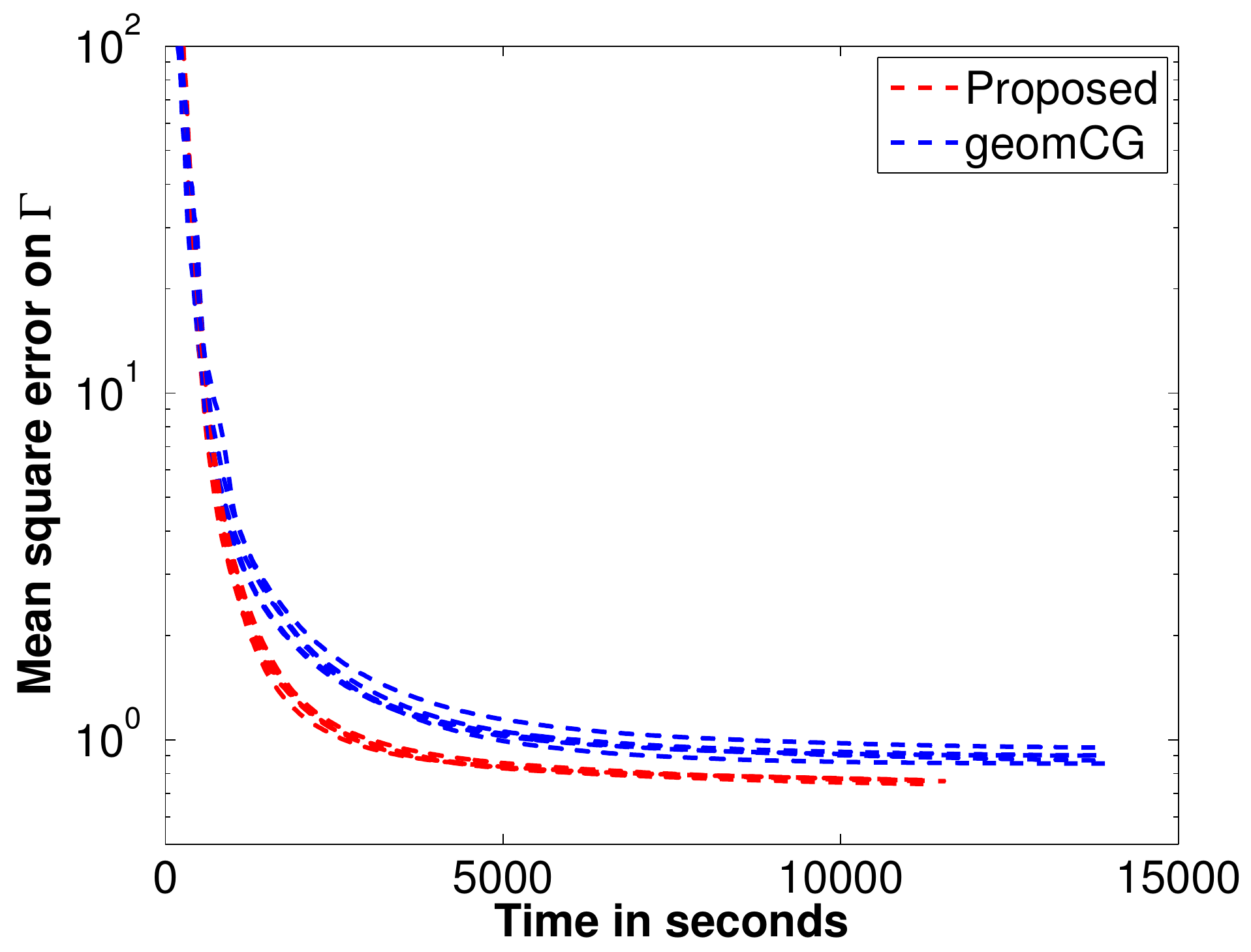}\\
{\scriptsize(c) {\bf r} = ($8\times 8\times 8$).} 
\end{center}
\end{minipage}
\begin{minipage}{0.48\hsize}
\begin{center}
\includegraphics[width=\hsize]{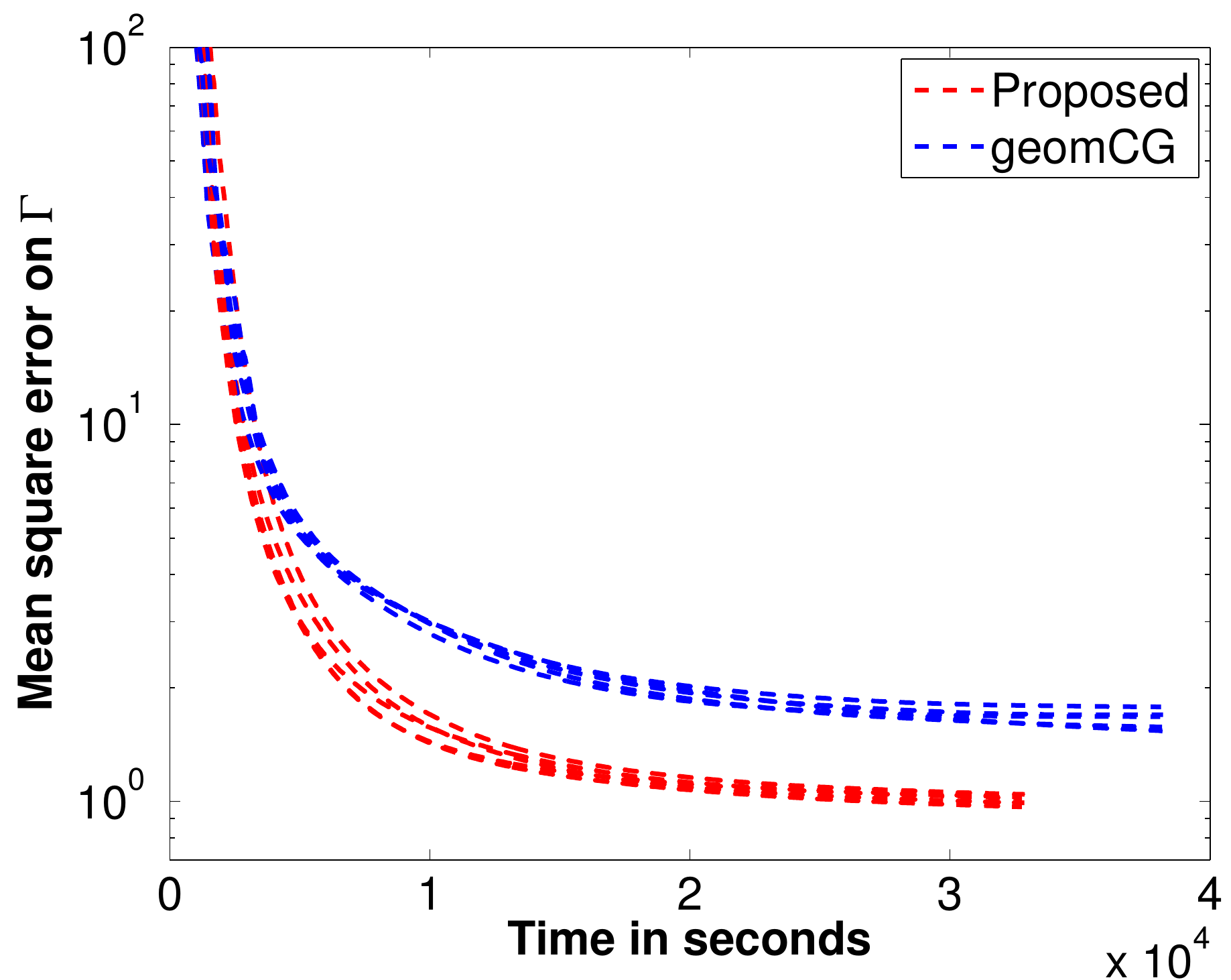}\\
{\scriptsize(d) {\bf r} = ($10\times 10\times 10$).} 
\end{center}
\end{minipage}
\end{tabular}
\caption{{\bf Case R2:} mean square error on $\Gamma$ (test error).}
\label{appnfig:R2-test}
\end{figure*}

\clearpage 
\begin{figure*}[t]
\begin{tabular}{cc}
\begin{minipage}{0.49\hsize}
\begin{center}
\includegraphics[width=\hsize]{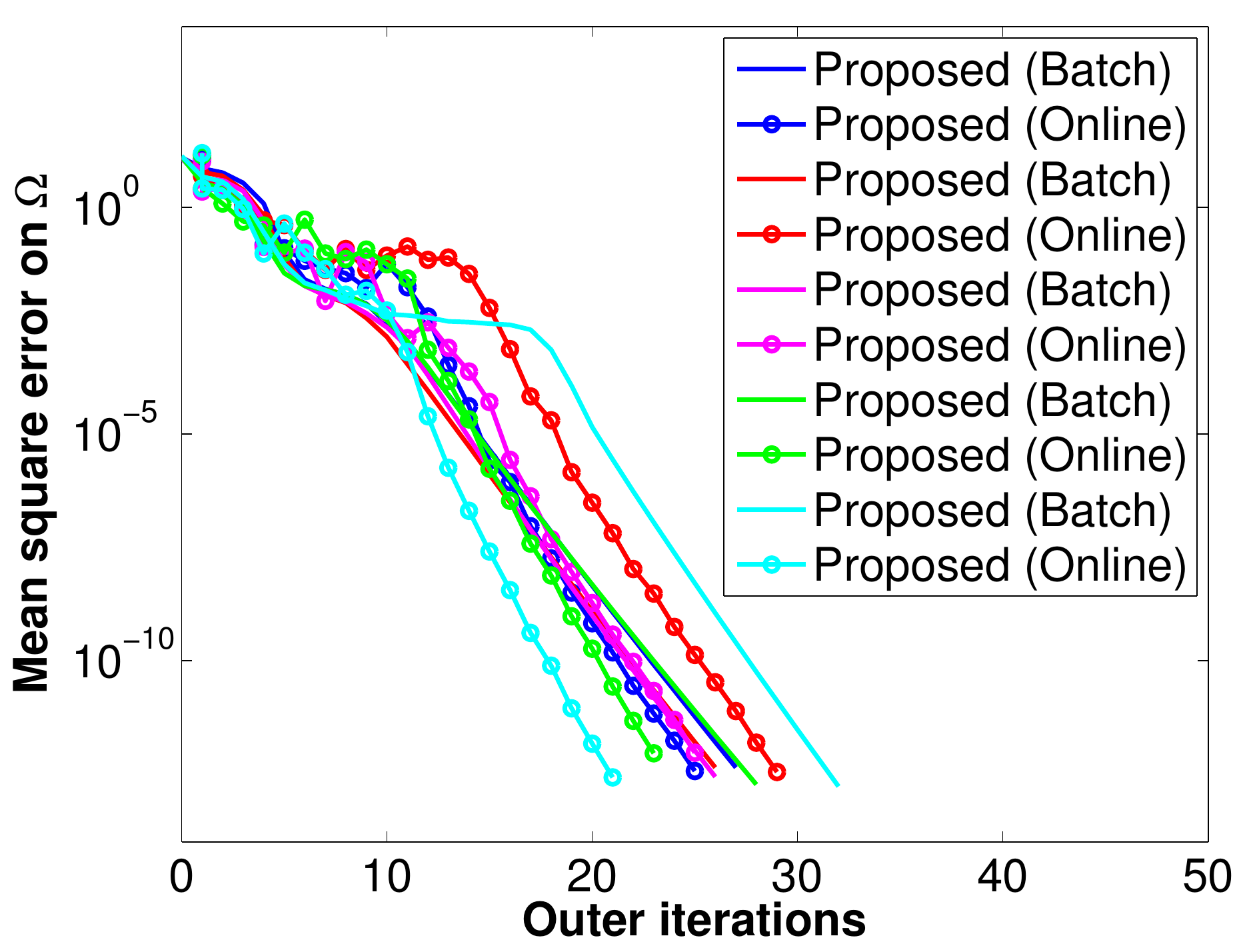}\\
\changeHK{{\scriptsize(a) Mean square error on $\Omega$ (train error).}}
\end{center}
\end{minipage}
\begin{minipage}{0.49\hsize}
\begin{center}
\includegraphics[width=\hsize]{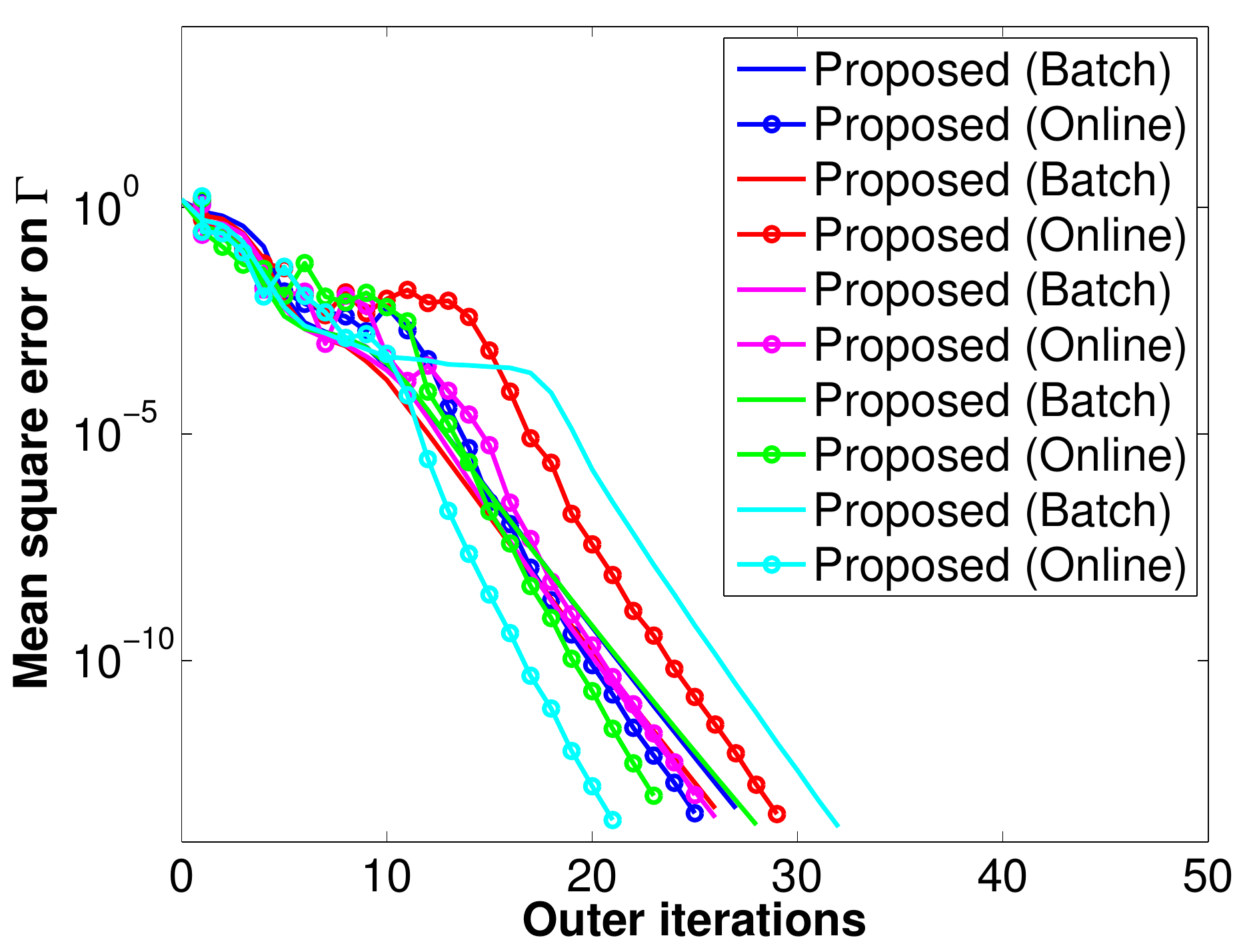}\\
\changeHK{{\scriptsize(b) Mean square error on $\Gamma$ (test error).}}
\end{center}
\end{minipage}
\end{tabular}
\caption{\changeHK{{\bf Case O}: mean square error on synthetic instance of size $100 \times 100 \times 5000$.}}
\label{appnfig:O1-5000}
%
\begin{tabular}{cc}
\begin{minipage}{0.49\hsize}
\begin{center}
\includegraphics[width=\hsize]{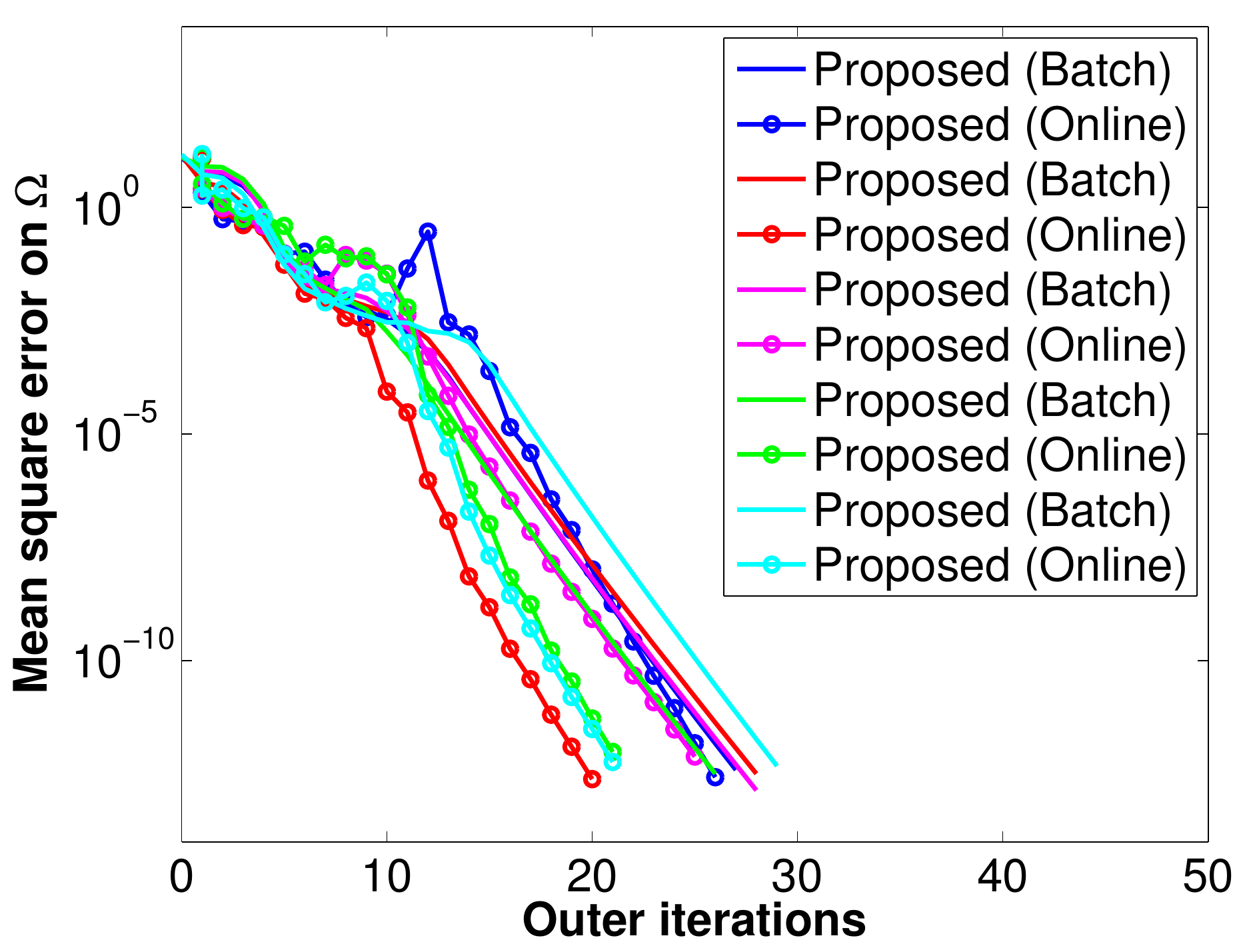}\\
\changeHK{{\scriptsize(a) Mean square error on $\Omega$ (train error).}}
\end{center}
\end{minipage}
\begin{minipage}{0.49\hsize}
\begin{center}
\includegraphics[width=\hsize]{new_figures/Testmse-synt-100x100x10000-5x5x5-10-5-eps-converted-to.pdf}\\
\changeHK{{\scriptsize(b) Mean square error on $\Gamma$ (test error).}}
\end{center}
\end{minipage}
\end{tabular}
\caption{\changeHK{{\bf Case O}: mean square error on synthetic instance of size $100 \times 100 \times 10000$.}}
\label{appnfig:O1-10000}
\end{figure*}
\begin{table*}[b]
\label{appntbl:O2}
\begin{center}
\caption{\changeHK{{\bf Case O}: mean square error (5 runs) on Airport Hall dataset.}}
\begin{tabular}{c|l|r|r|r|r|r}
\hline\hline
Error type & \multicolumn{1}{|c|}{Algorithm}	& \multicolumn{1}{|c|}{run 1}	& \multicolumn{1}{|c|}{run 2}& \multicolumn{1}{|c|}{run 3}	& \multicolumn{1}{|c|}{run 4}	& \multicolumn{1}{|c}{run 5} \\
\hline\hline
Training error &Proposed (Online) 	& 7.210000	& 7.211718	& 7.205027	& 7.255203	& 7.230000 \\
\cline{2-7}
on $\Omega$ &Proposed (Batch)	 & 7.215763	& 7.211496	& 7.208463	& 7.282901	& 7.218042 \\
\cline{2-7}
&TeCPSGD	& 7.335320		& 7.389269	& 7.364065	& 7.393318	& 7.390530 \\
\cline{2-7}
&OLSTEC	& 7.922385	& 7.653096	& 8.150799	& 8.248936	& 7.753596 \\
\hline	
\hline				
Test error&Proposed (Online) 	& 7.462097	& 7.440332	& 7.452799	& 7.443505	& 7.450065 \\
\cline{2-7}
on $\Gamma$ &Proposed (Batch)	& 7.471942	& 7.440508	& 7.446072	& 7.492786	& 7.218042 \\
\cline{2-7}
&TeCPSGD	& 7.592109	& 7.601955	& 7.600740	& 7.579759	& 7.600621 \\
\cline{2-7}
&OLSTEC	& 8.205765	& 7.840107	& 8.599819	& 8.625715	& 7.965405 \\
\hline
\end{tabular}
\end{center}
\end{table*}

\clearpage
\begin{figure*}[htbp]
\begin{tabular}{cc}
\begin{minipage}{0.32\hsize}
\begin{center}
\includegraphics[width=\hsize]{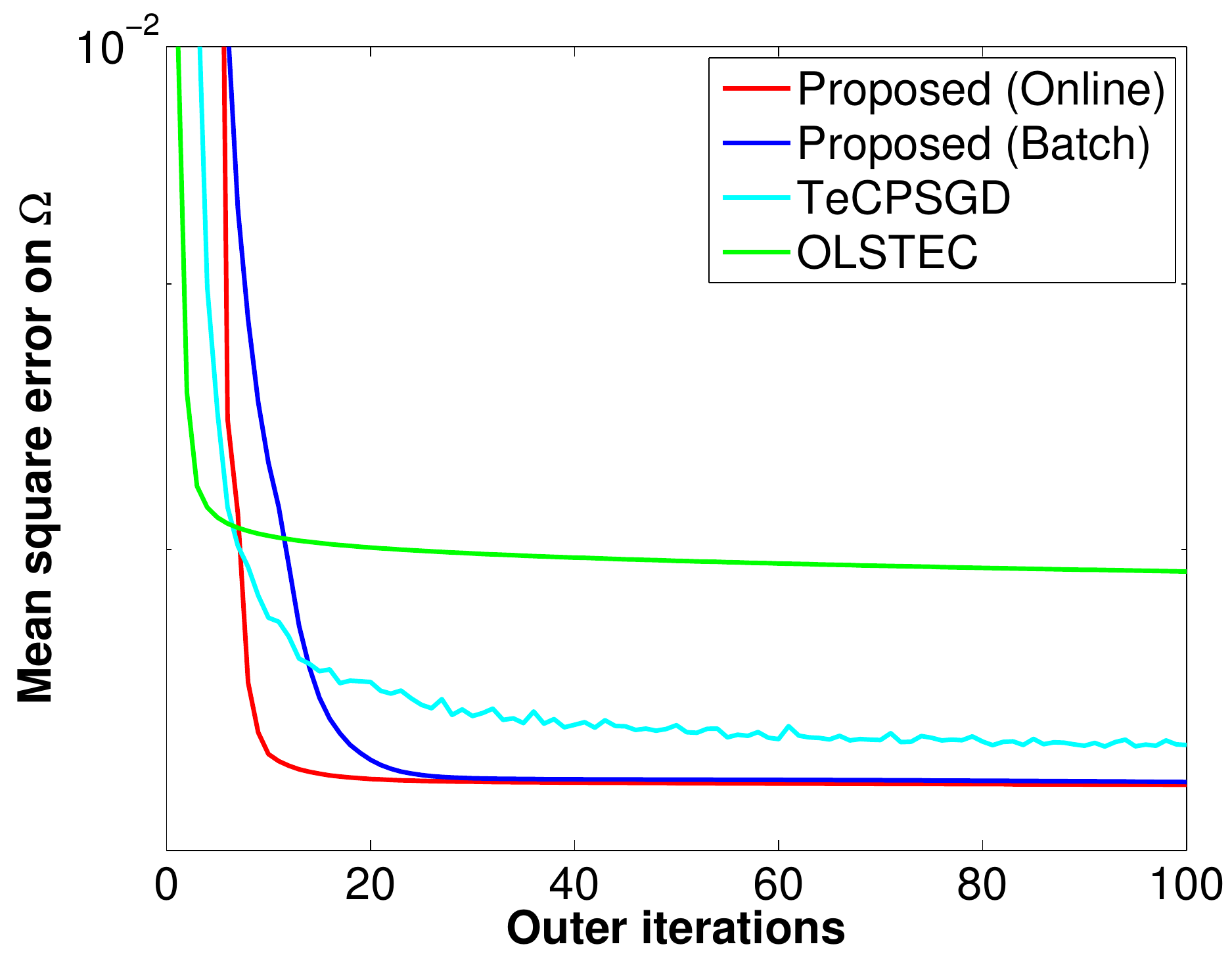}\\
{\scriptsize(a) Run 1.}
\end{center}
\end{minipage}
\begin{minipage}{0.32\hsize}
\begin{center}
\includegraphics[width=\hsize]{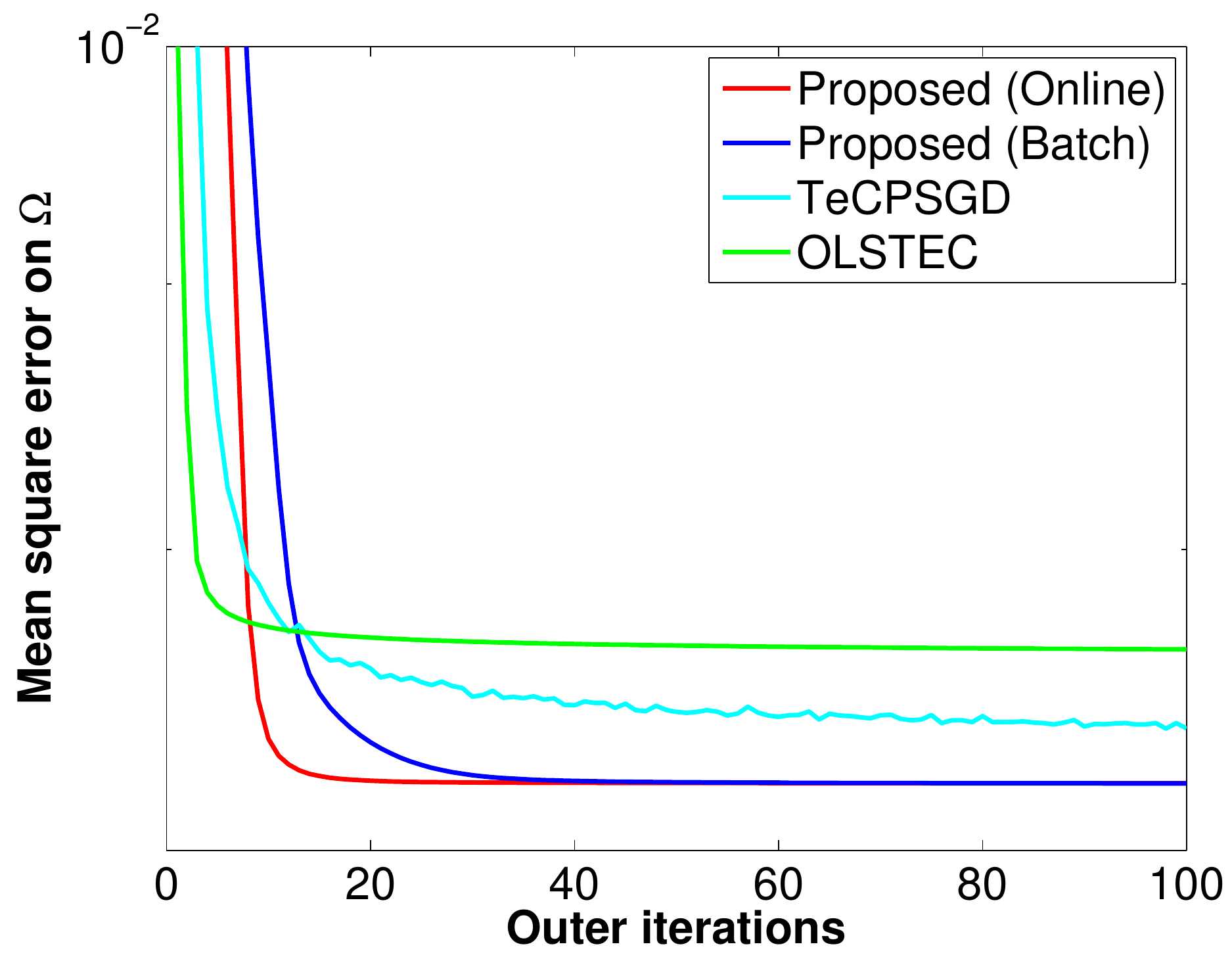}\\
{\scriptsize(b) Run 2.}
\end{center}
\end{minipage}
\begin{minipage}{0.32\hsize}
\begin{center}
\includegraphics[width=\hsize]{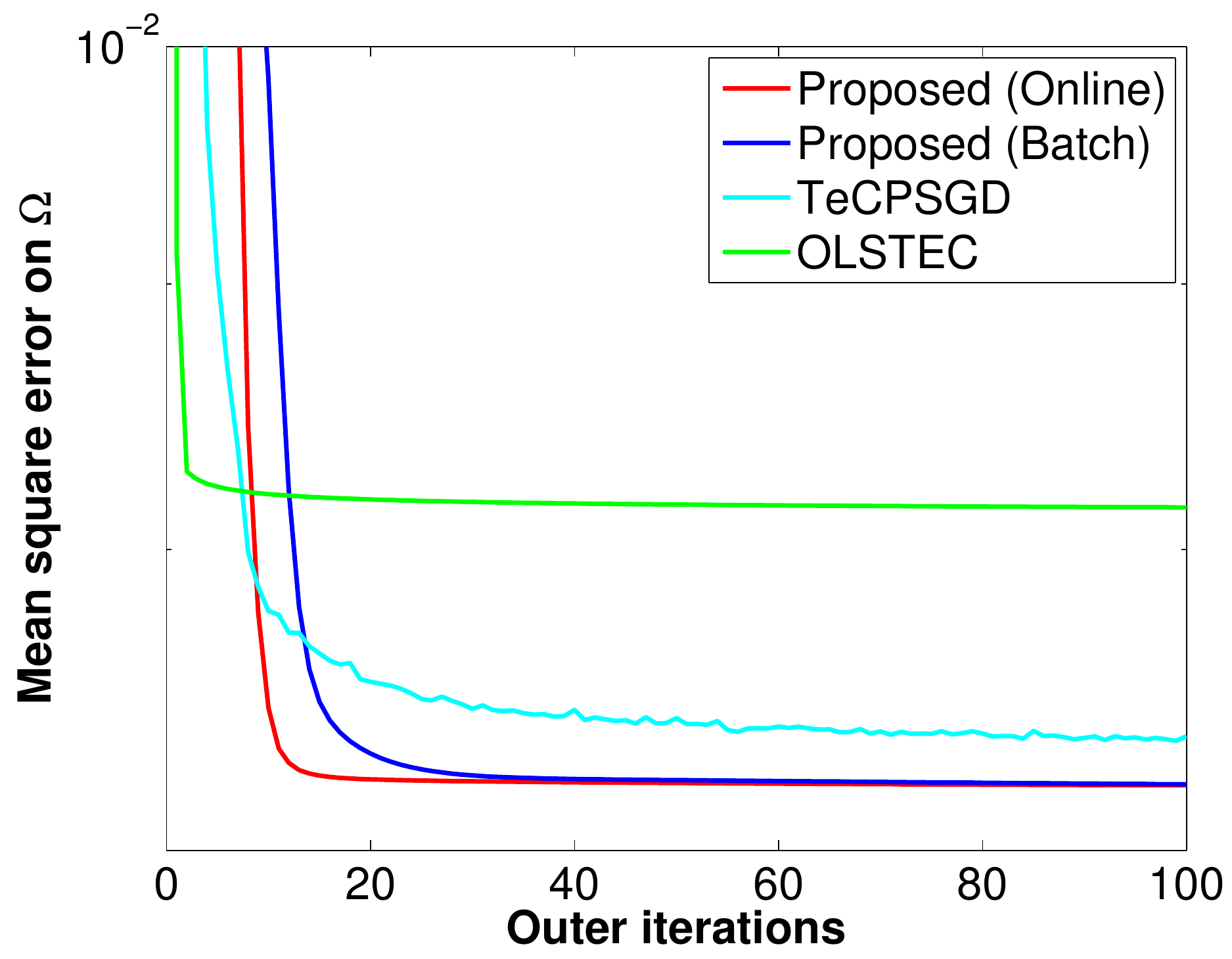}\\
{\scriptsize(c) Run 3.}
\end{center}
\end{minipage}
\vspace*{1cm}\\
\begin{minipage}{0.32\hsize}
\begin{center}
\includegraphics[width=\hsize]{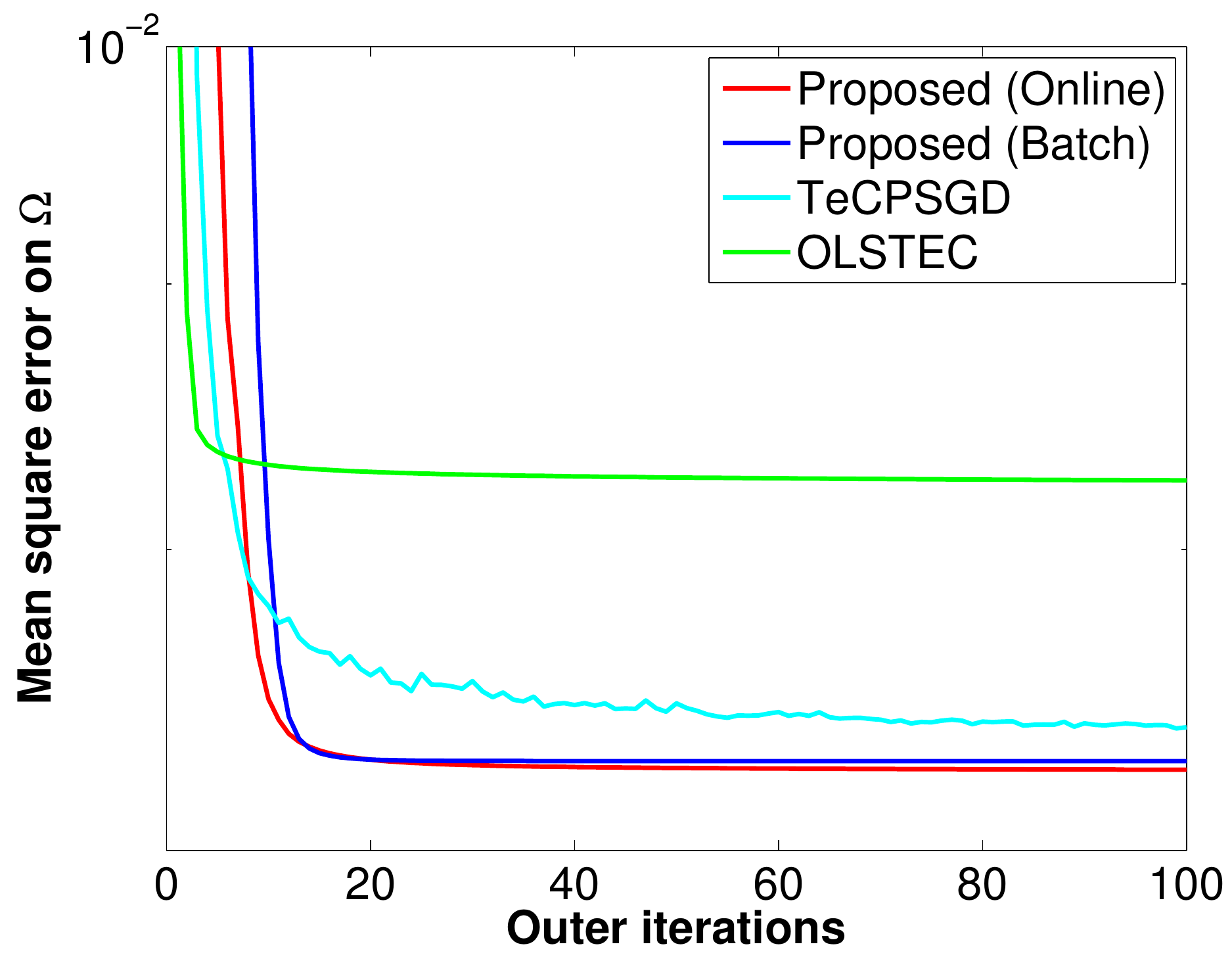}\\
{\scriptsize(d) Run 4.}
\end{center}
\end{minipage}
\begin{minipage}{0.32\hsize}
\begin{center}
\includegraphics[width=\hsize]{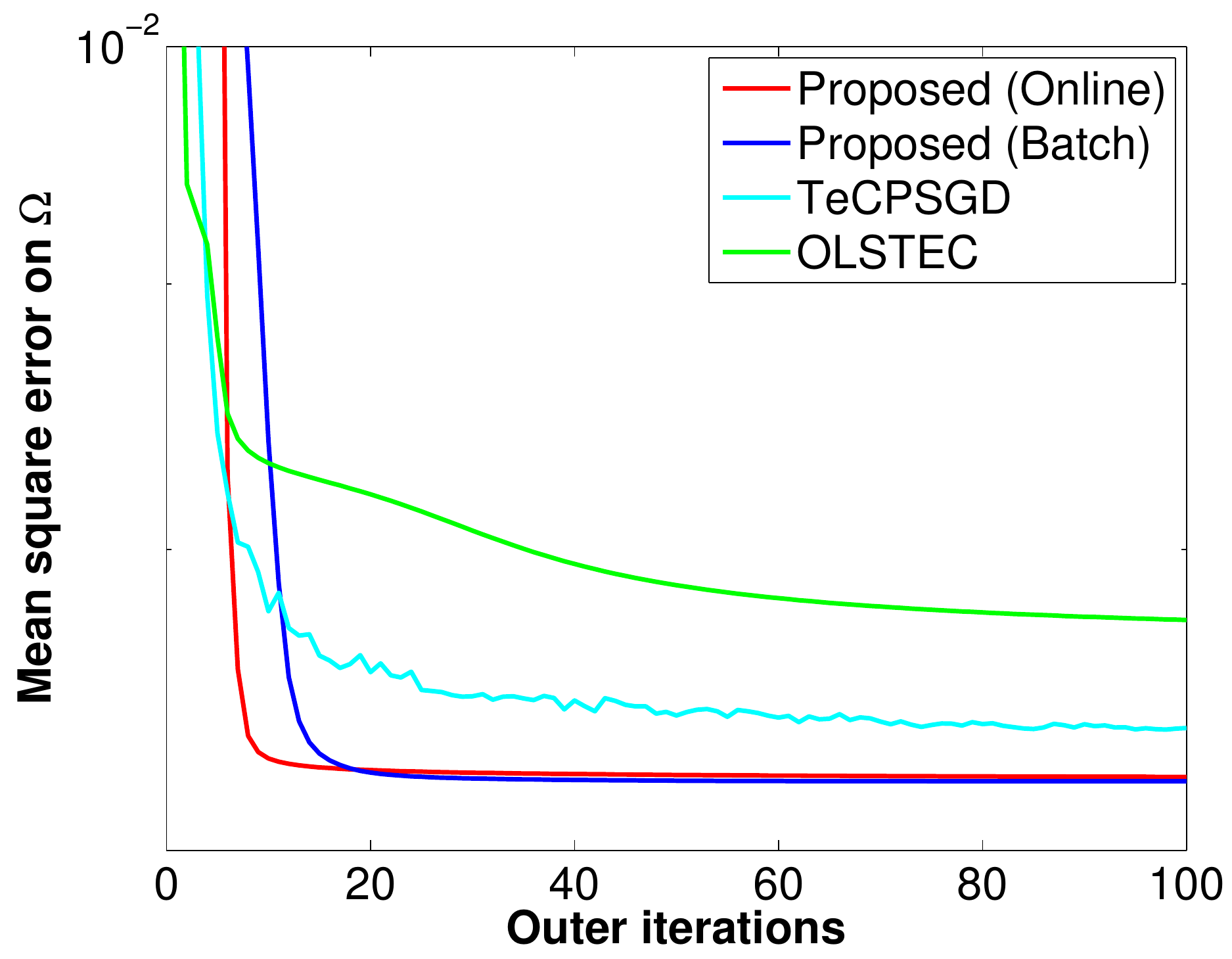}\\
{\scriptsize(e) Run 5.}
\end{center}
\end{minipage}
\end{tabular}
\caption{\changeHK{{\bf Case O}: mean square error on the training set $\Omega$ of the Airport Hall dataset.}}
\label{appnfig:O2-train}
%
\vspace*{1.5cm}
\begin{tabular}{cc}
\begin{minipage}{0.32\hsize}
\begin{center}
\includegraphics[width=\hsize]{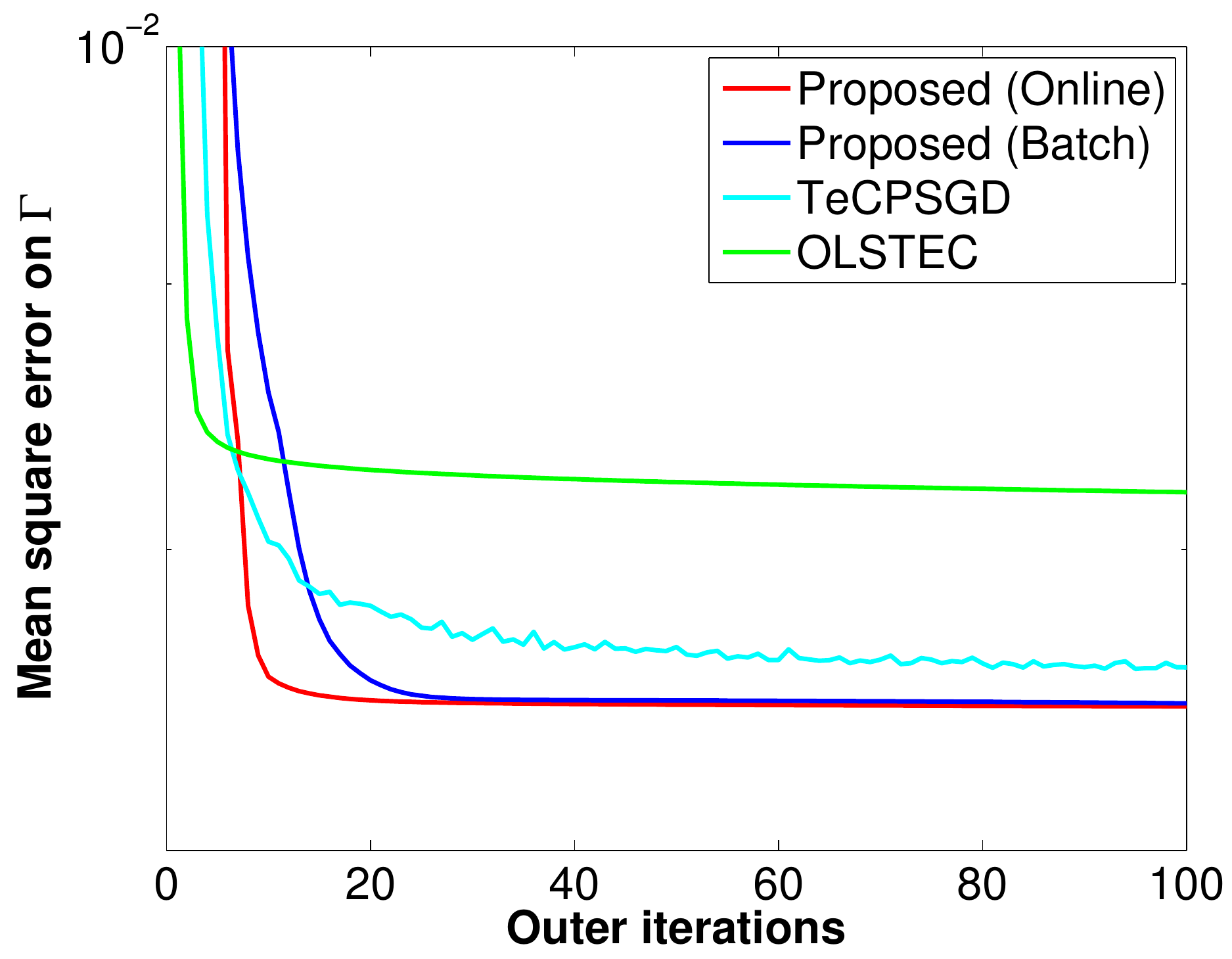}\\
{\scriptsize(a) Run 1.}
\end{center}
\end{minipage}
\begin{minipage}{0.32\hsize}
\begin{center}
\includegraphics[width=\hsize]{new_figures/Testmse-real-144x176x1000-5x5x5-2-2-eps-converted-to.pdf}\\
{\scriptsize(b) Run 2.}
\end{center}
\end{minipage}
\begin{minipage}{0.32\hsize}
\begin{center}
\includegraphics[width=\hsize]{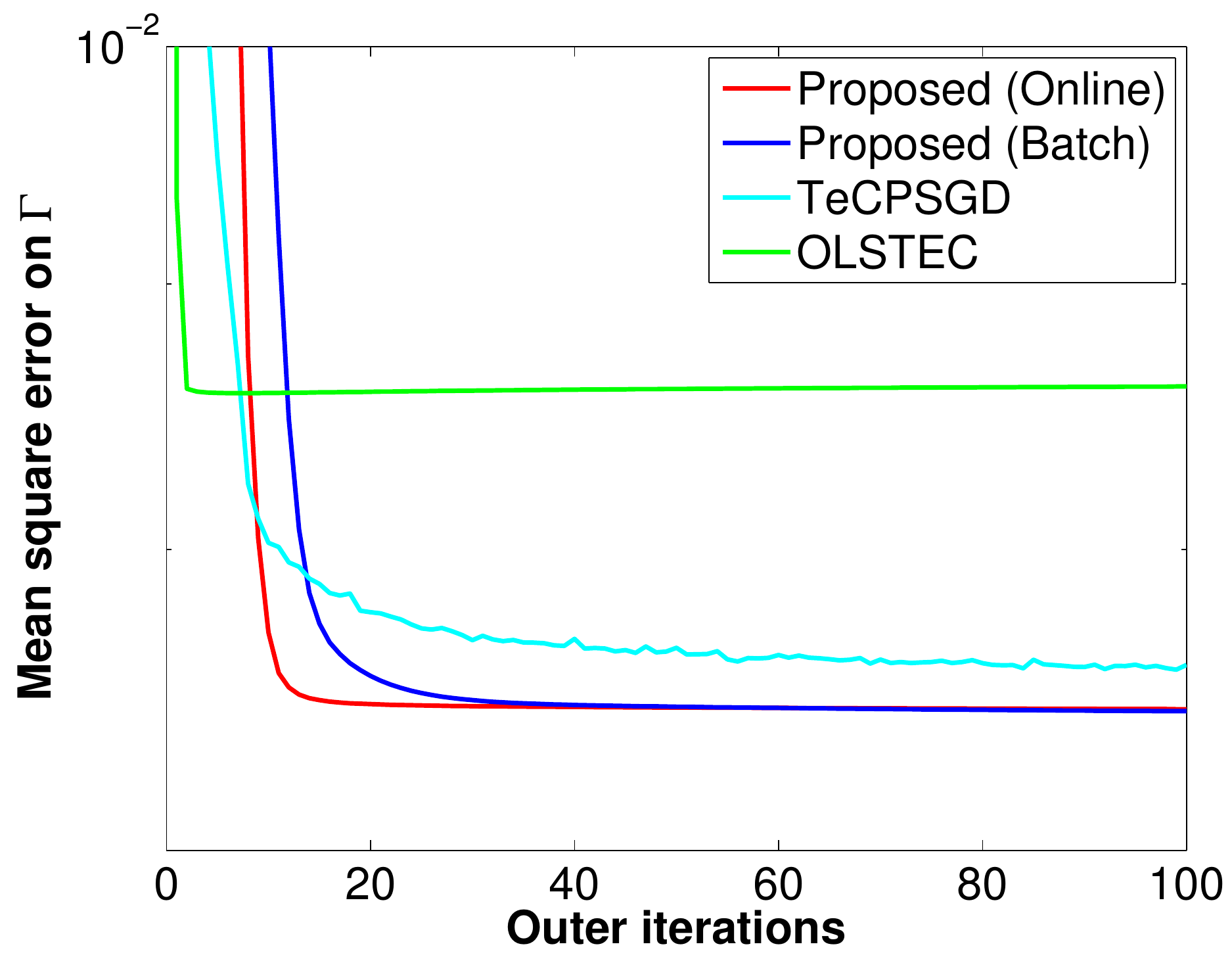}\\
{\scriptsize(c) Run 3.}
\end{center}
\end{minipage}
\vspace*{1cm}\\
\begin{minipage}{0.32\hsize}
\begin{center}
\includegraphics[width=\hsize]{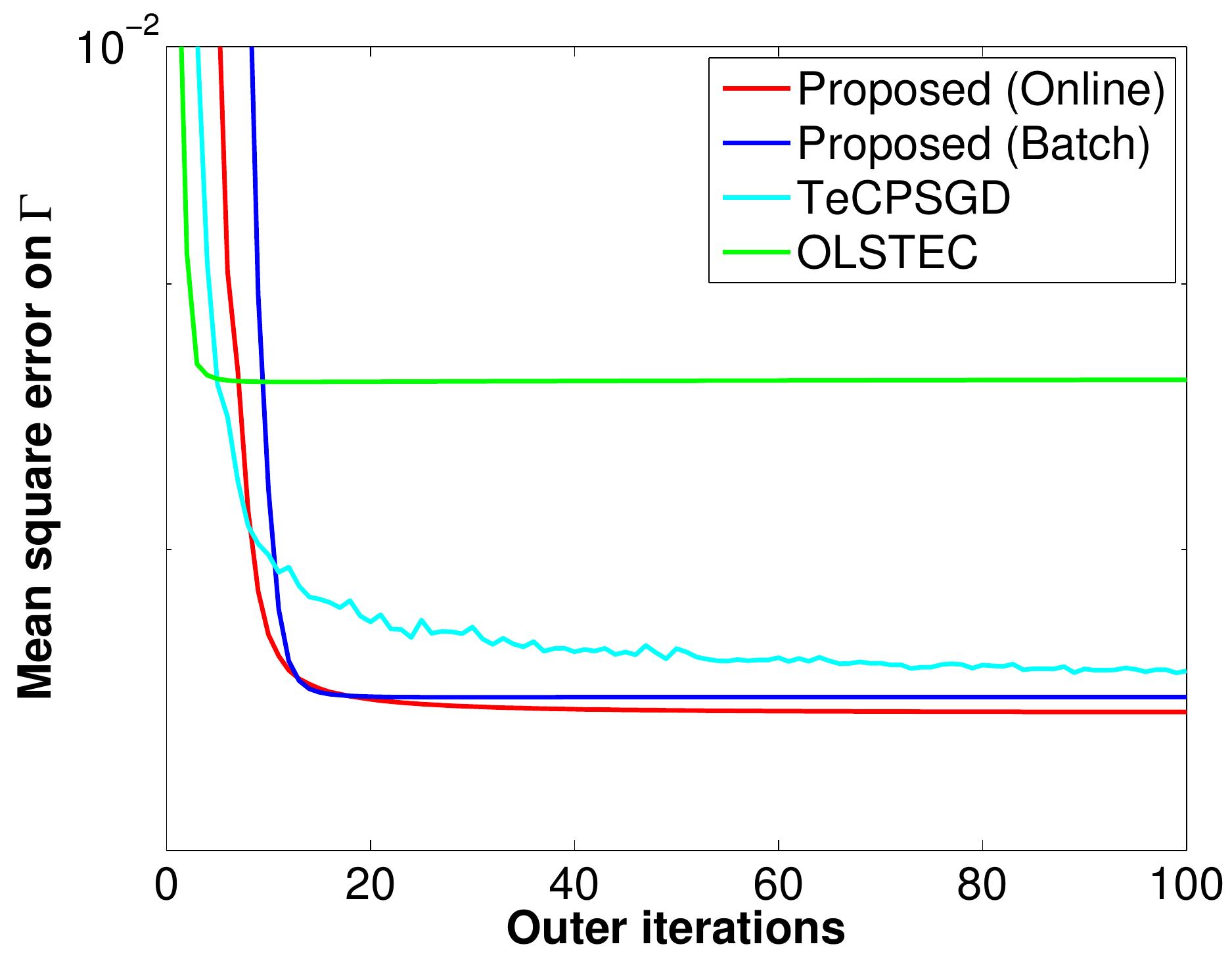}\\
{\scriptsize(d) Run 4.}
\end{center}
\end{minipage}
\begin{minipage}{0.32\hsize}
\begin{center}
\includegraphics[width=\hsize]{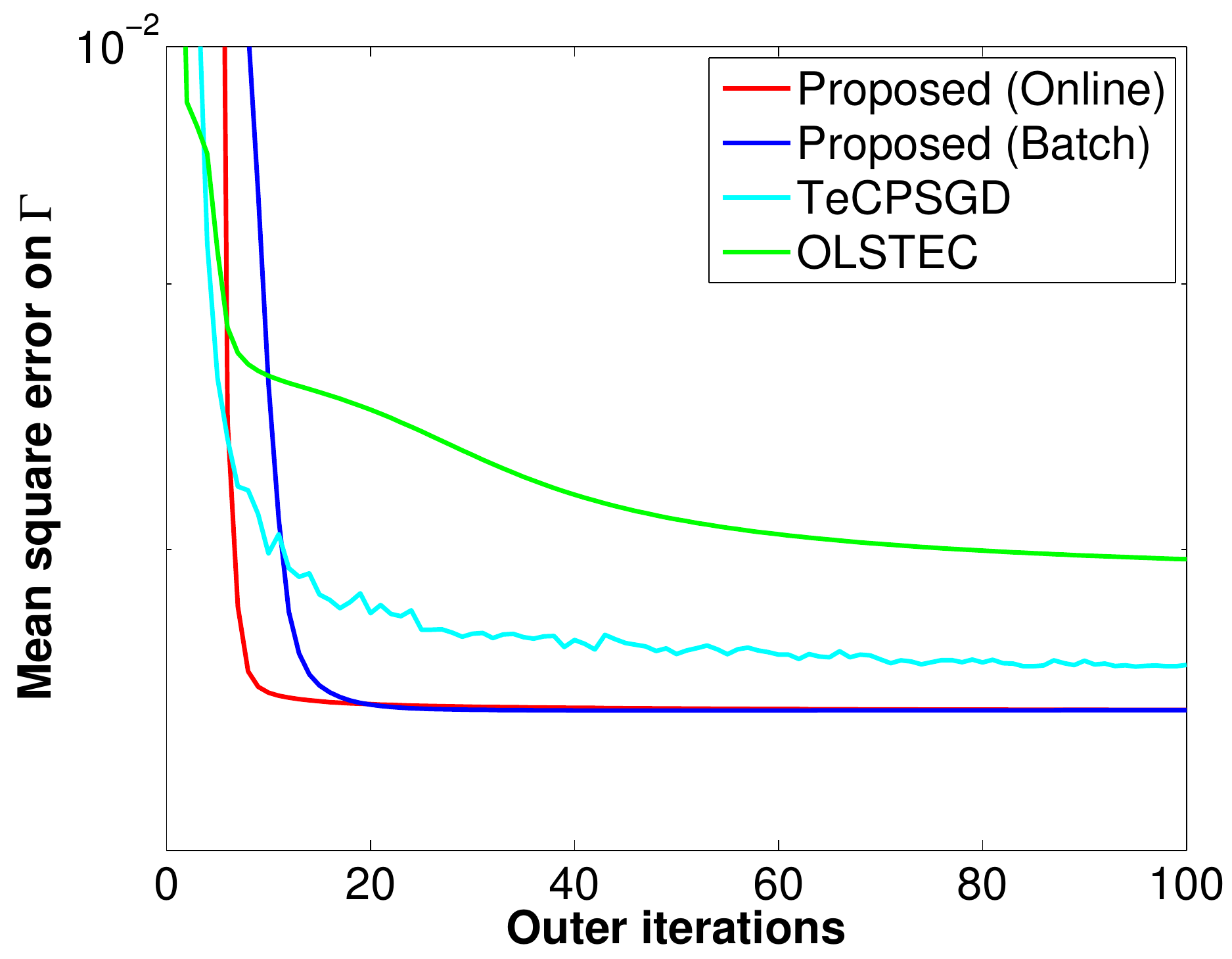}\\
{\scriptsize(e) Run 5.}
\end{center}
\end{minipage}
\end{tabular}
\caption{\changeHK{{\bf Case O}: mean square error on $\Gamma$ (test error) for the Aiport Hall dataset. }}
\label{appnfig:O2-test}
\end{figure*}


\end{document}